\documentclass[11pt]{article}
\usepackage[T1]{fontenc}
\usepackage[utf8]{inputenc}   
\usepackage{lmodern}          

\usepackage[letterpaper,margin=0.75in]{geometry}
\usepackage{microtype}

\usepackage{graphicx}
\usepackage{subcaption}
\usepackage{booktabs}
\usepackage{tikz}
\usepackage{float}
\usepackage{wrapfig}
\usepackage{placeins}
\usepackage{mathrsfs}
\usepackage{amsmath,amssymb,amsthm,mathtools}

\PassOptionsToPackage{numbers,compress}{natbib}
\usepackage{natbib}

\usepackage{hyperref}
\hypersetup{colorlinks=true, linktocpage}

\usepackage[capitalize,noabbrev]{cleveref}

\hypersetup{
    colorlinks,
    linkcolor={red!50!black},
    citecolor={blue!80!black},
    urlcolor={blue!80!black}
}

\usepackage{xcolor}
\definecolor{PythonGreen}{HTML}{2ca02c}
\usepackage{url}
\usepackage{nicefrac}
\usepackage{mdframed}
\usepackage[most]{tcolorbox}

\newtcolorbox{summarybox}{
  colback=blue!5,          
  colframe=blue!60!black,  
  boxrule=2pt,           
  arc=4pt,                 
  left=2pt,
  right=2pt,
  top=3pt,
  bottom=3pt,
  halign=center,
}

\usepackage{enumitem}
\usepackage{pifont}           

\theoremstyle{plain}
\newtheorem{theorem}{Theorem}
\newtheorem{proposition}{Proposition}
\newtheorem{lemma}{Lemma}
\newtheorem{corollary}{Corollary}
\theoremstyle{definition}
\newtheorem{definition}{Definition}
\newtheorem{assumption}{Assumption}
\theoremstyle{remark}
\newtheorem{remark}{Remark}

\newlist{todolist}{itemize}{2}
\setlist[todolist]{label=$\square$}

\newcommand{\red}[1]{\textcolor{red}{#1}}

\newcommand{\E}{\mathbb{E}}
\newcommand{\R}{\mathbb{R}}

\newcommand{\N}{\mathbb{N}}
\newcommand{\x}{\mathbf{x}}

\newcommand{\lmax}{\lambda_{\max}}
\newcommand{\mednorm}[1]{\lVert #1\rVert}

\newcommand{\eos}{\textsc{EoS}}

\DeclareMathOperator{\tr}{tr}

\DeclareMathOperator{\vect}{vec}

\title{
\vspace{-1cm}
\textbf{Edge of Stochastic Stability:\\Revisiting the Edge of Stability for SGD}}

\author{%
  Arseniy~Andreyev\thanks{Equal contribution. \textit{Affiliation:} Princeton University. \textit{Emails:} \href{mailto:andreyev@princeton.edu}{\texttt{andreyev@princeton.edu}} and \href{mailto:pierb@mit.edu}{\texttt{pierb@mit.edu}}.
  \\
  Code available at
  \href{https://github.com/arseniqum/edge-of-stochastic-stability}{github.com/arseniqum/edge-of-stochastic-stability.}
  Difference between versions discussed in Sec.\ \ref{section:conclusions}.}
  \and
  Pierfrancesco~Beneventano\footnotemark[1]\,
}

\date{}

\begin{document}
\maketitle

\vspace{-0.5cm}
\begin{abstract}
    Recent findings by \citet{cohen_gradient_2021} demonstrate that when training neural networks using full-batch gradient descent with step size $\eta$, the largest eigenvalue~$\lambda_{\max}$ of the full-batch Hessian consistently stabilizes 
    around $2/\eta$.
    These results have significant implications for convergence and generalization. 
    This, however, is not the case for mini-batch optimization algorithms, limiting the broader applicability of the consequences of these findings.
    We show that mini-batch Stochastic Gradient Descent (SGD) trains in a different regime, which we term Edge of Stochastic Stability (\textsc{EoSS}).
    In this regime, what stabilizes at $2/\eta$ is \emph{Batch Sharpness}: the expected directional curvature of mini-batch Hessians along their corresponding stochastic gradients.
    As a consequence, $\lambda_{\max}$---which is generally smaller than \emph{Batch Sharpness}---is suppressed, aligning with the long-standing empirical observation that smaller batches and larger step sizes favor flatter minima.
    We further discuss implications for mathematical modeling of SGD trajectories.
\end{abstract}


\vspace{-0.55cm}
\section{Introduction}
\label{section:introduction}
\vspace{-0.4cm}
The choice of training algorithm is a key ingredient in the deep learning recipe. 
Extensive evidence, e.g.\ \citet{keskar_large-batch_2016}, indeed shows that performance consistently depends on the optimizer and hyperparameters. 
What machinery induces this optimizer‑dependence is a central question
in the
theory of deep learning.

\citet{cohen_gradient_2021,cohen_understanding_2024} answered this question for Gradient Descent (GD): it optimizes neural networks in a regime of instability---termed the Edge of Stability (\textsc{EoS}) \citep{cohen_gradient_2021}---following the Central Flow \citep{cohen_understanding_2024}. With a constant step size $\eta$, the highest eigenvalue of the Hessian of the full-batch loss---denoted here as $\lambda_{\max}$---increases towards $2/\eta$ and hovers just above it, subject to small oscillations
\citep{cohen_gradient_2021,cohen_adaptive_2022,jastrzebski_relation_2019,jastrzebski_break-even_2020,xing_walk_2018}.
Although classical convex optimization theory would deem this step size ``too large’’, the loss continues to decrease.
These works established a number of surprising facts:
\textbf{(1)} we require an optimization theory which works in more general scenarios than the classical\footnote{$L$ as in $L$-smooth function.} $\eta < 2/L$;
\textbf{(2)} that NN training happens in a special regime of instability, establishing what the source of it is; 
\textbf{(3)} how the solution found depends on the choice of hyperparameters.

\vspace{-0.45cm}
\paragraph{Contribution: \textsc{EoSS}.}
Real-world training is almost always \emph{mini-batch}---given the large amounts of data. However, existing \textsc{EoS} analyses \textbf{explicitly state} they do not apply to this case. It remains unclear whether mini-batch optimizers exhibits an analogous stability-saturated regime and which curvature statistic would govern it. We bridge this gap by establishing that for common vision classification tasks with MSE:
\begin{summarybox}
    \textbf{Mini-batch SGD trains in a regime of instability akin to \textsc{EoS}} which we term the Edge of Stochastic Stability (\textsc{EoSS}). Precisely, \textit{Batch Sharpness}, our notion of curvature,
\vspace{-0.25cm}
\[
\textit{Batch Sharpness}\,(\theta)
\ :=\
\E_{B \sim \mathcal{P}_b}\bigg[
\frac{
\nabla L_B(\theta)^\top\,\mathcal{H}(L_B)\,\nabla L_B(\theta)
}{
\|\nabla L_B(\theta)\|^2
}
\bigg], \qquad \substack{\text{with } L_B \text{being loss on the}\\ \text{batch }B\text{ sampled from }\mathcal{P}_b.
\\ \mathcal{H}(\cdot)\text{ Hessian matrix.}}
\vspace{-0.2cm}
\]
hovers around $2/\eta$ and implicitly functions as sharpness for SGD. This implies that:\\ \textbf{stability for SGD is stability on the mini-batch landscape.} 
\end{summarybox}

\begin{figure}[ht!]
\vspace{-0.4cm}
  \centering
  \includegraphics[width=\linewidth]{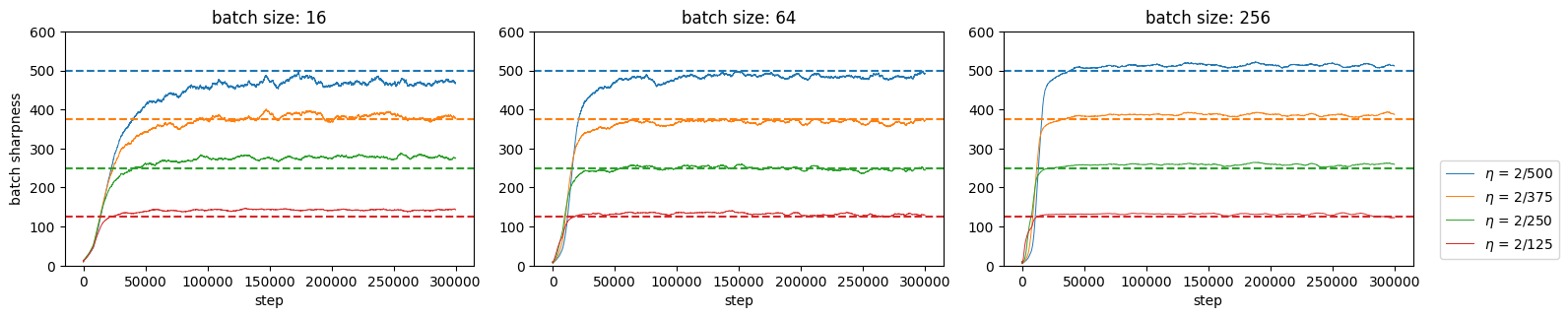}
  \vspace{-0.3cm}
  \caption{
  \small
    \textbf{SGD at \textsc{EoSS} under different step sizes and batch sizes.} 
    MLP on an 8k subset of CIFAR-10 with step size $\eta$. 
    \emph{Batch Sharpness} stabilizes at the $2/\eta$ threshold across varying batch sizes and step sizes.
  }
  \label{fig:1b_which_is_a_self-explanatory_name}
  \vspace{-0.3cm}
\end{figure}

\paragraph{From \textsc{EoS} to SGD: why oscillations are not the story.}
In full-batch GD, the \textsc{EoS} regime is easy to recognize: the loss enters a visibly oscillatory phase at the same time that the full-batch curvature approaches the quadratic stability boundary $\lambda_{\max}\approx 2/\eta$.
Mini-batch SGD fundamentally breaks this diagnostic.
With a non-vanishing step size, SGD typically exhibits loss oscillations essentially from the start---even in regimes that are perfectly stable in the classical stochastic-approximation sense.
Thus, for SGD, \emph{the presence of oscillations is not evidence of operating near an instability boundary}.

This motivates a key reframing.
Rather than defining an ``edge'' via loss plots, we define it via \emph{instability}: would the optimizer \emph{diverge} if we froze the local geometry and ran the algorithm on the resulting quadratic approximation?
In Section~\ref{section:two_types_of_oscillations} we formalize this distinction by separating
(i) \emph{noise-driven} oscillations (stable wandering induced by gradient noise) from
(ii) \emph{curvature-driven} oscillations (the hallmark of \textsc{Eo(S)S}-type behavior).
The purpose of this paper is to identify the latter regime for SGD and to isolate the curvature quantity that governs it---both mathematically (Sections~\ref{section:framework} and \ref{section:math}) and empirically (Section~\ref{section:eoss}).

\begin{summarybox}
\textbf{Why this extension is genuinely difficult.}
For stochastic dynamics, ``instability'' is not simply the negation of a standard one-step stability bound:
there are stable oscillatory regimes, multiple inequivalent Lyapunov viewpoints, and sharp quadratic stability conditions that are typically high-dimensional and not empirically trackable in modern networks.
We therefore need a \emph{computable scalar} quantity that (i) serves as a \emph{valid instability criterion} on the local quadratic model and (ii) is \emph{actually driven to saturation} along real SGD training trajectories.
\end{summarybox}

\paragraph{An instability-centric framework and a practical test.}
Section~\ref{section:framework} introduces a minimal framework built around \emph{certificates of instability}:
scalar tests whose violation is sufficient to force divergence on the local quadratic approximation.
This is deliberately one-sided.
A certificate of instability is \emph{not} the same object as a certificate of monotone descent, and this asymmetry matters for SGD, where stable non-monotone regimes are common.

A key advantage of phrasing \textsc{Eo(S)S} in terms of instability is that it yields an empirical diagnostic.
On the quadratic approximation, three phenomena coincide (Section~\ref{section:equivalence}):
breaking a valid instability criterion,
observing a \emph{catapult} (a rapid excursion out of the trusted quadratic neighborhood),
and observing a sufficiently large loss spike.
This equivalence lets us test whether training sits at an ``edge'' via \emph{hyperparameter interventions}:
if a quantity is truly saturated at its stability boundary, then a small destabilizing perturbation
(increasing $\eta$ or decreasing $b$) must trigger a catapult on the quadratic model.
In Section~\ref{section:math}, we instantiate this program on \textit{Batch Sharpness}---which we empirically identified in Section \ref{section:eoss}.

\begin{summarybox}
\textbf{Operational signature of an edge.}
At \textsc{Eo(S)S}, the dynamics is \emph{locally precarious}:\\
small destabilizing changes to $(\eta,b)$ induce catapults (and loss spikes) on the quadratic approximation, whereas the same changes are harmless earlier in training.
\end{summarybox}

\paragraph{\textsc{EoSS}: the edge quantity for SGD.}
Guided by this framework, we identify the curvature statistic that plays for SGD the role that $\lambda_{\max}$ plays for full-batch GD.
The key observation is that SGD does not update on a single loss landscape: each step sees its own mini-batch loss $L_B$, with its own gradient and curvature.
Accordingly, the relevant edge quantity must be \emph{mini-batch} and \emph{direction-aware}.
This leads to \emph{Batch Sharpness} (Definition~\ref{def:minibs}), the expected Rayleigh quotient of the mini-batch Hessian along the mini-batch gradient direction.

Empirically, across the vision-classification settings we study (with MSE), Batch Sharpness undergoes progressive sharpening and then hovers near the quadratic threshold $2/\eta$
(Figure~\ref{fig:1b_which_is_a_self-explanatory_name} and Section~\ref{section:eoss}).
Moreover, when we intervene on $(\eta,b)$ mid-training, catapults and renewed sharpening occur exactly when predicted by Batch Sharpness (Section~\ref{section:eoss_natural_generalization}).
Crucially, this is not only an empirical correlation: Section~\ref{section:math} proves that on the quadratic approximation, crossing $2/\eta$ (by any fixed margin) makes Batch Sharpness a \emph{genuine instability criterion} for SGD and forces catapults (Theorem~\ref{theo:1}).
Section~\ref{section:math} also clarifies why this criterion is:
\begin{itemize}[itemsep=0.15em,topsep=0.15em,parsep=0pt]
\vspace{-0.1cm}
\item \textbf{based on mini-batch gradient norms} (Section~\ref{section:stab_grad}), 
\item \textbf{about higher cumulant information of the distribution of the Hessians} (Section~\ref{section:stab_minibs}), and 
\item \textbf{explicitly directional}, encoding curvature--alignment interactions (Section~\ref{section:alignment}).
\end{itemize}
\vspace{-0.2cm}

\paragraph{What becomes of $\lambda_{\max}$.}
A natural question is what survives from the original \textsc{EoS} story built around the full-batch Hessian.
Section~\ref{section:lambda_max} shows that once \textsc{EoSS} is reached, the subsequent evolution of $\lambda_{\max}$ is largely a \emph{by-product} of Batch Sharpness saturation:
$\lambda_{\max}$ stops increasing when Batch Sharpness plateaus, and its eventual level depends strongly on batch size and on training history.
In particular, we provide a concrete mechanism behind the long-observed link between small-batch training and ``flatter'' solutions.
\vspace{-0.2cm}

\paragraph{Implications.}
\textsc{EoSS} makes the informal question “is the learning rate large?” operational for mini-batch training: with constant $\eta$, the relevant stability comparator is not the full-batch $\lambda_{\max}$, but the mini-batch instability threshold $\textit{Batch Sharpness}(\theta)\approx 2/\eta$ (which is a genuine one-sided certificate of divergence on the quadratic approximation; Section~\ref{section:math}).
This has two consequences.
\emph{First (optimization):} classical descent-lemma proof templates based on uniform smoothness (full-batch, or worst-case over mini-batches) do not describe the regime in which SGD operates under progressive sharpening:
training can make progress while locally saturating a mini-batch instability boundary
(Section~\ref{section:noise_vs_sgd}).
\emph{Second (location/modeling):} stabilization is governed by batch-dependent curvature and alignment, so “GD + noise” surrogates (noise injection and diffusion/SDE limits) that discard this mini-batch geometry can follow qualitatively different trajectories and stabilize in different regions of parameter space (Section~\ref{sec:location_distributional}).
\vspace{-0.2cm}

\paragraph{Contributions (informal).}
\vspace{-0.15cm}
\begin{itemize}[leftmargin=1.6em,itemsep=0.15em,topsep=0.15em,parsep=0pt]
    \item We show that, for constant-step-size SGD, loss oscillations are essentially ubiquitous and therefore not a diagnostic of operating at an instability boundary; we formalize the distinction between noise-driven and curvature-driven oscillations (Section~\ref{section:two_types_of_oscillations}).
    \item We introduce an instability-centric framework based on one-sided \emph{certificates of instability}, and derive an intervention-based diagnostic via the equivalence between criterion violation, catapults, and sufficiently large loss spikes on the quadratic approximation (Sections~\ref{section:framework}--\ref{section:equivalence}).
    \item We identify \emph{Batch Sharpness} as a direction-aware mini-batch curvature statistic and show empirically that it progressively sharpens and then saturates near $2/\eta$ across the vision-classification settings we study, governing the onset of \textsc{EoSS} (Section~\ref{section:eoss}).
    \item We characterize why smaller batch sizes lead to flatter solutions as a by-product of \textsc{EoSS} (Section~\ref{section:lambda_max}).
    \item \textsc{EoSS} places practical SGD trajectories outside the scope of classical descent-type lemmas (Section~\ref{section:noise_vs_sgd}).
    \item We show that in the progressive-sharpening regime, “where SGD stabilizes” becomes a distributional stability question governed by batch-dependent curvature and alignment, exposing concrete limitations of “GD + noise” and diffusion/SDE surrogates for predicting SGD’s implicit bias (Section~\ref{sec:location_distributional}).
\end{itemize}

\clearpage


\section{Related Work}
\label{section:related_work}

\paragraph{Progressive sharpening.} 
Early studies observed that the local shape of the loss landscape changes rapidly at the beginning of the training \citep{keskar_large-batch_2016,jastrzebski_relation_2019,lecun_efficient_2012,achille_critical_2017,jastrzebski_three_2018,fort_emergent_2019}.
Multiple papers noticed growth of different estimators of $\lambda_{\max}$ in the early training \citep{keskar_large-batch_2016,jastrzebski_relation_2019,sagun_eigenvalues_2016,fort_goldilocks_2019}.
Subsequently, \citet{jastrzebski_relation_2019,jastrzebski_break-even_2020} and \citet{cohen_gradient_2021} precisely characterized this behavior, demonstrating a steady rise in $\lambda_{\max}$ along GD and SGD trajectories, typically following a brief initial decline.
This phenomenon was termed \textit{progressive sharpening} by \citep{cohen_gradient_2021}.

\paragraph{A phase transition.}
Early works \citep{goodfellow_deep_2016,li_towards_2019,jiang_fantastic_2019,lewkowycz_large_2020} revealed that large initial step sizes often enhance generalization at the cost of initial loss reduction.
\citet{jastrzebski_break-even_2020} attributed this to a sudden phase transition, termed the break-even point, marking the end of progressive sharpening. This transition slows the convergence by confining the dynamics to a region of lower sharpness.
Unlike progressive sharpening, this phenomenon is considered to be induced by the gradient-based algorithm becoming unstable, not by the landscape.
\citet{jastrzebski_relation_2019,jastrzebski_break-even_2020,cohen_gradient_2021,cohen_adaptive_2022} indeed showed that the phase transition comes at different points for different algorithms on the same landscapes. 
\citet{jastrzebski_relation_2019,jastrzebski_break-even_2020} showed that in the case of mini-batch SGD this phase comes earlier for bigger step sizes and smaller batch sizes, without quantifying it.
\citet{cohen_gradient_2021,cohen_adaptive_2022} later showed indeed that it comes at the instability thresholds for the full-batch optimization algorithms. 
We manage here to quantify the value of the instability threshold for mini-batch SGD, thus characterizing when the phase transition happens for SGD.

\paragraph{Full-batch edge of stability.}
After the phase transition above, GD and full-batch Adam train in the \textsc{EoS} oscillatory regime \citep{cohen_gradient_2021,cohen_adaptive_2022,xing_walk_2018}, where the $\lambda_{\max}$ stabilizes and oscillates around a characteristic value.
The name is due to the fact that, in the case of full-batch GD, the $\lambda_{\max}$
hovers at $2/\eta$ which is the stability threshold for optimizing quadratics.
Observations from \citet{cohen_gradient_2021,cohen_adaptive_2022} indicate that, under mean square error (MSE), most training dynamics occur within this regime, effectively determining $\lambda_{\max}$ of the final solution.
\citet{chen_beyond_2023,lee_new_2023} explained why in this regime $\lambda_{\max}$ often slightly exceeds $2/\eta$: this deviation arises primarily from
nonlinearity of the loss gradient, which shifts the required value depending on higher-order derivatives.

\paragraph{Existing \textsc{EoS} work is limited to full-batch methods.} 
A growing body of research analyzes the surprising mechanism underlying \textsc{EoS} dynamics observed during training with GD. 
Classically, when gradients depend linearly on parameters, divergence occurs locally if $\eta > \frac{2}{\lambda_{\max}}$, as illustrated by one-dimensional quadratic models \citep{cohen_gradient_2021}. In contrast, neural networks often converge despite violating this classical stability condition, presumably due to the problem's non-standard geometry.
\citet{damian_self-stabilization_2023} propose an explanation under some, empirically tested, assumptions of alignment of third derivatives and gradients. 
For further insights into convergence and implicit regularization phenomena in the \textsc{EoS} regime, we direct the reader to
\citep{arora_understanding_2022, wang_large_2022,ahn_understanding_2022,ahn_learning_2023,zhu_understanding_2023,lyu_understanding_2023,beneventano_gradient_2025,ghosh_learning_2025}.

\paragraph{The work on \textsc{EoS} is about full-batch methods.}
While the empirical behavior of \textsc{EoS} for full-batch algorithms is relatively well-understood, neural networks are predominantly trained using mini-batch methods. As explicitly noted by \citet[Section 6, Appendices G and H]{cohen_gradient_2021}, their observations and analysis do not directly apply to mini-batch training. In particular, they emphasize:
\vspace{-0.3cm}
\begin{quote}
    \setlength{\leftskip}{-0.4cm}   
    \setlength{\rightskip}{-0.4cm}  
    \itshape
    [...] while the sharpness does not flatline at any value during SGD (as it does during gradient descent), the trajectory of the sharpness is heavily influenced by the step size and batch size \citet{jastrzebski_relation_2019,jastrzebski_break-even_2020}, which cannot be explained by existing optimization theory. Indeed, there are indications that the “Edge of Stability” intuition might generalize somehow to SGD, just in a way that does not center around the (full-batch) sharpness.
    [...]  In extending these findings to SGD, the question arises of how to model “stability” of SGD.
    \vspace{-0.3cm}
\end{quote}
In this work, we show that the \textsc{EoS} phenomenon does indeed generalize to SGD, and we identify the central quantity governing this generalization (\textit{Batch Sharpness} in Definition~\ref{def:minibs}). We model stability of SGD on the neural network landscapes: 
our results show 
that empirically SGD is stable if \textit{on average} the step is stable on the mini-batch landscape---not on the full-batch landscape.

\paragraph{What was empirically known for SGD.}
\begin{wrapfigure}{r}{0.45\columnwidth}       
    \vspace{-0.4cm}                             
    \centering
    \includegraphics[width=\linewidth]{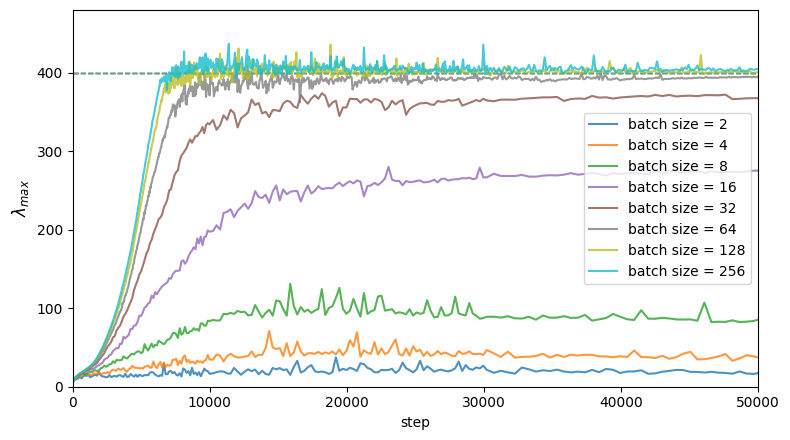}
    \vspace{-16pt}                             
    \caption{\small 
    \textbf{SGD on CIFAR-10:  $\eta = 1/400$.} 
    The full-batch Hessian’s $\lambda_{\max}$ plateaus \emph{below} $2/\eta$.
    Smaller batch sizes lead to lower plateau values.
    }
    \label{fig:1}
    \vspace{-0.2cm}                             
\end{wrapfigure}

In the context of mini-batch algorithms, \textbf{(i)} \citet{jastrzebski_relation_2019,jastrzebski_break-even_2020} noticed that for SGD the phase transition happens earlier for larger $\eta$ or smaller batch size $b$, but they did not quantify when.
\textbf{(ii)} \citet{cohen_gradient_2021,gilmer_curvature_2021} established that initialization and architecture choices affect stability of SGD, without providing a definitive condition.  \textbf{(iii)} When $\lambda_{\max}$ stabilizes, that always happens at a level they could not quantify which is below the $2/\eta$ threshold \citep{cohen_gradient_2021,keskar_large-batch_2016},
see Figure \ref{fig:1}, often without a proper progressive sharpening phase.
This leaves the most basic questions open: 
\textit{In what way does the location of convergence of SGD acclimate to the choice of hyperparameters? What are the key quantities involved?}
To be more specific, can we characterize the training phenomena in \textbf{(i), (ii), (iii)} above? What determines them? Does SGD train in an unstable regime?

\paragraph{Previous works on SGD stability.}
A series of works,
\citep{wu_how_2018,ma_linear_2021,granizol_learning_2021,wu_sgd_alignment_2022,mulayoff_exact_2024},
studies constant‑step‑size SGD on quadratic losses via \emph{linear stability} analysis of the second moment of the parameters. 
They obtain sharp conditions of the form
\(
\big\Vert\E_{B \sim \mathcal{P}_b}[(I-\eta H(L_B))^{\otimes 2}]\big\Vert \lessgtr 1
\),
which are well suited for analyzing behavior near minima.
From the perspective of Section~\ref{section:framework}, these are valid instability criteria for a particular Lyapunov function, but two limitations remain for our purposes:
(i) the resulting thresholds are expressed via $d^2$‑dimensional operators and are not computable for modern neural networks; and
(ii) they do not address whether, and in what sense, SGD on neural networks actually \emph{trains} in an \textsc{EoS}‑like regime of instability along its real trajectory (partially limited by the aforementioned incomputability). See further discussion in \cref{appendix:linear_stochastic_stability}.
A number of works \citep{wu_implicit_2023,agarwala2024high} showed that for SGD the regime of instability might be governed by the trace of loss Hessian/NTK, further discussed in \cref{appendix:other_quantities}.
Empirically, several works have documented oscillatory SGD dynamics in deep networks
\citep{cohen_gradient_2021,xing_walk_2018,lee_new_2023,ahn_understanding_2022}. However, these works do not establish whether any of those oscillations constitute a regime of instability; in particular, they do not distinguish between noise‑driven (Type‑1) and curvature‑driven (Type‑2) oscillations, see Section \ref{section:framework} and \ref{section:two_types_of_oscillations}.
Our work complements these efforts by (i) placing candidate criteria such as $\lambda_{\max}$, GNI, and \textit{Batch Sharpness} within a unified instability framework, and (ii) identifying \textit{Batch Sharpness} as a valid, empirically saturating, and computationally tractable instability criterion for SGD on neural networks.

\paragraph{Flatness and Generalization.}
SGD-trained networks consistently generalize better than GD-trained ones, with smaller batch sizes further enhancing generalization performance \citep{keskar_large-batch_2016,lecun_efficient_2012,jastrzebski_three_2018,goyal_accurate_2017,masters_revisiting_2018,beneventano_how_2024,smith_origin_2021}.
This advantage has been widely attributed to some notion of flatness of the minima~\citep{jiang_fantastic_2019,jastrzebski_catastrophic_2021,hochreiter_simplifying_1994,neyshabur_exploring_2017,wu_towards_2017,kleinberg_alternative_2018,xie_diffusion_2020}. Training algorithms explicitly designed to find flat minima have indeed demonstrated strong performance across various tasks \citep{izmailov_averaging_2019,foret_sharpness-aware_2021}.
Our result is inherently about mini-batch training improving flatness.
Specifically, we explain why:
    \textit{Training with smaller batches constrains the dynamics to areas with smaller eigenvalues of the full-batch Hessian}. 
This quantifies and characterizes prior observations that SGD tends to locate flat minima and that smaller batch sizes result in reduced Hessian sharpness \citep{keskar_large-batch_2016,jastrzebski_catastrophic_2021}.


\section{Theory Preliminaries: A Framework for Instability}
\label{section:framework}

In this section we set up a minimal formal framework for how we reason about \emph{(in)stability} in the rest of the paper.
We argue here that our goal is to isolate \textbf{(i)} \emph{instability criteria}---possible sufficient certificates for divergence of a (stochastic) dynamical system from a trusted region (multiple are possible here unlike in the deterministic case)---and \textbf{(ii)} practical ways to recognize the regime where such criteria \emph{empirically saturate}, which we call the Edge of (Stochastic) Stability (\textsc{Eo(S)S}).

On the quadratic approximation of the loss, we show that the following three viewpoints are equivalent:
\vspace{-0.3cm}
\begin{itemize}[itemsep=0.1em,parsep=0pt]
    \item breaking a valid instability criterion;
    \item observing a \textit{catapult}, i.e., a large excursion outside the region where the quadratic approximation is trusted;
    \item observing a spike of appropriate size in the loss.
    \vspace{-0.3cm}
\end{itemize}
Importantly, thanks to this equivalence we know how to empirically diagnose \textsc{Eo(S)S} and what we consider catapults in Section \ref{section:eoss}.

\subsection{Defining Instability}
\label{section:in_stab}
\paragraph{On the notion of stability.}
Every proof of convergence for a discrete-time optimization method imposes bounds on the hyperparameters of the algorithm (most notably, the step size). 
When we view the iterates $(\theta_t)_t$ of such an algorithm as a discrete-time dynamical system, these bounds can be interpreted as \emph{stability conditions}: they ensure that some Lyapunov function $V(\theta_t)$ decreases step-after-step along the trajectory, typically in expectation.
Several notions of stability can be considered for optimization algorithms.
For quadratic loss, where for every data point $z_i$ there exist $\mathcal{H}_i \in \R^{d\times d}$ and
$x_i \in \R^d$ such that the loss on $z_i$ is
$L_i(\theta) = \tfrac{1}{2}(\theta-x_i)^\top \mathcal{H}_i (\theta-x_i)$,
one can track any Lyapunov function $V(\theta_t)$, for instance, the loss or the squared distance $V(\theta_t)=\|\theta_t - \theta^*\|^2$ to a solution
$\theta^*$.
Typical stability requirements have the form
\begin{equation}
\label{eq:Lyap}
\E[V(\theta_{t+1})|\theta_t] \leq V(\theta_{t}),
\quad
\E\left[ \frac{V(\theta_{t+1})}{V(\theta_t)} \, \Big| \, \theta_t \right]
\leq 1\footnote{This is the one we find in our empirical work, with $V(\theta) = \mednorm{\nabla L_B(\theta)}^2$.}
, \quad \text{or} \quad 
\lim_{t \to \infty} \frac{1}{t} \log \big( V(\theta_t) \big) \leq 0,  
\quad \text{etc.}
\end{equation}
If one expands the inequalities above in a Taylor around $\theta_t$, one obtains an equivalent inequality that can be expressed in terms of a \emph{scalar Hessians-based quantity}---generally termed \emph{notion of curvature} (for instance, an eigenvalue or a combination of low-order cumulants of the distribution of $\{ \mathcal{H}_i \}_i$).
The example $V(\theta_t)=\|\theta_t - \theta^*\|^2$ above---termed as \emph{linear stability}---was studied by
\citet{wu_how_2018,ma_linear_2021,mulayoff_exact_2024}: we have one-step non-expansion if we can establish a bound on an operator of the form
\vspace{-0.1cm}
\begin{equation}
\vspace{-0.1cm}
\label{eq:ma-ying}
\E[V(\theta_{t+1})|\theta_t] \leq V(\theta_t) 
\quad \iff \quad 
    \big\Vert\E_{B \sim \mathcal{P}_b} \big[ (I-\eta H(L_B))^{\otimes 2} \big]\big\Vert \leq 1.
\end{equation}
We refer to Appendix~\ref{appendix:linear_stochastic_stability} for a detailed comparison of our results to linear stochastic stability.

\paragraph{Certifying instability.}
Classical Lyapunov-style conditions characterize regions of hyperparameters for which the iterates are non-expansive; in the quadratic case these regions can even be made necessary and sufficient, as in Equation~\eqref{eq:ma-ying}. In this work we take the complementary viewpoint and focus on \emph{certificates of instability}. Our instability test is formulated as a purely \emph{sufficient} condition for divergence: it is deliberately one-sided and may be conservative. Our goal is not to describe the full stability region, but to obtain a tractable scalar test whose violation yields a robust certificate of instability.

Importantly, a certificate of instability is not simply the converse of a certificate of stability. For most deterministic linear systems, the dynamics either monotonically diverges ($\lambda_{\max}>2/\eta$) or monotonically converges ($\lambda_{\max}<2/\eta$); this clean dichotomy is special and can be misleading. For non-linear or stochastic dynamical systems, there is no tight boundary separating instability from monotone convergence. Typical regimes include \textit{(i) monotonic convergence}, \textit{(ii) non-monotonic convergence}, \textit{(iii) periodic or chaotic orbits and random stable oscillations}, \textit{(iv) non-monotonic divergence}, and \textit{(v) monotonic divergence}. We define here certificates for regimes \textit{(iv)} and \textit{(v)}, which is not the same as simply flipping a bound that guarantees regime \textit{(i)}. In particular, SGD with fixed step size on quadratics often enters regime \textit{(iii)} \red{and (ii)?} of random stable oscillations, as described in Section~\ref{section:two_types_of_oscillations} and certified by \textit{GNI} (Definition~\ref{def:GNI}), which is not itself an instability criterion.

\begin{definition}[Instability criterion]
\label{def:stab_not}
Consider a training algorithm (a discrete‑time dynamical system)
$(\theta_t)_{t \ge 0}$ on a parameter space $\Theta \subseteq \R^d$ with fixed hyperparameters $h$ (e.g.\ learning rate, batch size).
Let $U \subseteq \Theta$ be an open set (typically, a region where a given local approximation of the loss is trusted).
Let $f : U \to \R$ and $c \in \R$.
We say that $f$ is a \emph{valid instability criterion with threshold $c$}
for the algorithm on $U$ if the following holds:

\medskip
\centerline{
$\displaystyle
f(\theta_0) > c
\quad\Longrightarrow\quad
(\theta_t)_{t \ge 0} \text{ leaves every compact subset of $U$.}
$
}

\noindent
That is, for any compact $K \subset U$ containing $\theta_0$, there exists a finite time $T$ such that $\E[\theta_T|\theta_0] \notin K$.
We say that the instability criterion $f$ is \emph{saturated} at $\theta$ if $f(\theta)$ is (approximately) equal to $c$; in practice, up to an $O(\eta \cdot poly(\log(\eta)))$ tolerance.
\end{definition}

In words: an instability criterion is a scalar quantity $f$ together with a threshold $c$ such that crossing $f>c$ is \emph{sufficient} to force divergence from the region $U$ where we trust a given local model.
For a specific $f$ we generally want the \emph{smallest} such $c$, which depends both on $f$ and on the underlying dynamical system (algorithm, data, loss, architecture).
A canonical example is GD on a quadratic, where $U = \R^d$, $f(\theta)=\lambda_{\max}(\nabla^2 L(\theta))$ and $c=2/\eta$: crossing $\lambda_{\max}>2/\eta$ makes GD linearly unstable on any compact set.

\begin{summarybox}
    The question of whether an optimizer operates at the \textit{Edge of (Stochastic) Stability} thus becomes:
    \vspace{-0.25cm}
        \begin{center}
            \textit{Are there locally-valid criteria of instability that empirically saturate during training?}
        \end{center}
\end{summarybox}

\subsection{Guiding Questions for \textsc{EoSS}}

\paragraph{Quadratic and deterministic case.}
To connect the discussion above with \citet{cohen_gradient_2021}, consider a full‑batch method (deterministic) on a quadratic loss $L$.
In this setting, the instability criteria obtained from Lyapunov functions of the form in Equation\ \eqref{eq:Lyap} simplify, up to higher order terms in $\eta$, as follows:
\begin{itemize}[itemsep=0.2em,parsep=0pt]
\vspace{-0.3cm}
    \item[\textbf{(i)}] only full‑batch quantities appear in the Taylor expansion, no higher cumulants of the Hessians $\{ \mathcal{H}_i \}_i$;
    \item[\textbf{(ii)}] shortly after reaching instability, the gradient aligns with the eigenvector corresponding to the largest eigenvalue $\lambda_{\max}$ of the Hessian (this is \textit{not} the case for mini-batch optimizers, see Appendix \ref{appendix:alignment}).
\end{itemize}
\vspace{-0.3cm}
In particular, GD is an asymptotically (and Lyapunov) stable linear dynamical system if and only if $\lambda_{\max} \leq 2/\eta$.
\citet{cohen_gradient_2021} empirically observed that GD trains neural networks in the regime where the inequality $\lambda_{\max} \leq 2/\eta$ is saturated, and \citet{damian_self-stabilization_2023} established this mathematically under an assumption of progressive sharpening, making precise that saturation means $\lambda_{\max} - 2/\eta \in [-c\eta|\log(\eta)|, +c\eta|\log(\eta)|]$. 

\paragraph{Stochastic and non-quadratic case.}
In the stochastic (mini‑batch) setting, different valid Lyapunov functions (for mini‑batch or full‑batch dynamics) typically give rise to different scalar quantities built from the Hessians—for example, different moments or cumulants of $\{\mathcal{H}_i\}_i$—and there is no a priori reason for these quantities to coincide or to saturate at the same time.
For instance, the operator in~\eqref{eq:ma-ying} depends only on the first and second cumulants of $\{\mathcal{H}_i\}_i$, whereas \textit{Batch Sharpness} (even for a quadratic loss) also involves third‑order terms, and $\lambda_{\max}$ depends only on the first.
This discrepancy leads to our first guiding question in the box above---that having a valid instability criteria is not enough, as there could be many such (not equivalent) criteria; we are looking for the one that actually saturates in the neural network setting, which can only be established empirically\footnote{Hypothetically, having a coherent theory of progressive sharpening, we could establish theoretically which criteria would saturate. Lacking such theory, we have to resort to establishing it empirically.}. Establishing if a given valid instability criterion creates an \textsc{EoSS} regime of instability therefore requires this criterion to be ``computable'' in the high-dimensional setting of modern neural networks.
The reason why this might prove to be an issue is that some natural instability criteria cannot be efficiently estimated in high dimensions.
For example, conditions based directly on operators such as~\eqref{eq:ma-ying} or on the exact instability criterion of \citet{mulayoff_exact_2024} have computational complexity $\Omega(d^2)$, which makes them infeasible to compute for modern neural networks. 
This motivates a refinement of the question above:
\begin{center}
    \emph{Are there valid instability criteria for SGD on neural networks that are both empirically saturating and computationally tractable in high dimensions?}
\end{center}
Operationally, we seek criteria that can be estimated in roughly $\mathcal{O}(d \cdot \mathrm{poly}(\log d))$ time.
Which of these candidate quantities is actually driven to saturation by the dynamics depends on how \emph{Progressive Sharpening} and \emph{self-stabilization} 
(in the sense of [\citealp{damian_self-stabilization_2023}]) act on the corresponding statistics:
\vspace{-0.3cm}
\begin{itemize}[itemsep=0.2em,parsep=0pt]
    \item progressive sharpening pushes up certain Hessian statistics during training;
    \item self-stabilization constrains them by preventing blow-up.
\vspace{-0.3cm}
\end{itemize}
When higher‑order terms in the Taylor expansion of the loss are non‑negligible, an instability condition derived from the second‑order approximation can be violated without a catapult event in the loss.
This happens for some quartic and cubic objectives and, as we observe, also for neural networks.\footnote{We give explicit examples of instability criteria are violated without affecting the dynamics in Appendix~\ref{appendix:other_quantities}.}.
In such cases, self‑stabilization effectively selects which instability criteria are actually relevant at \textsc{EoSS}.
\begin{center}
\emph{Which of these quantities are actually selected by the SGD dynamics?}
\end{center}
Understanding \emph{if and how} SGD trains at \textsc{EoSS}, and whether this notion is operationally meaningful, is therefore inherently linked to the questions above.
The rest of the paper answers these questions, at least partially:
we show that \textit{Batch Sharpness} is a valid and numerically tractable instability criterion for SGD that empirically saturates during neural network training.
However, the precise mechanism underlying self‑stabilization in SGD remains an open question, which we leave to future work.

\subsection{How to Identify a Regime of Instability}
\label{section:equivalence}
How can we detect whether such an instability criterion is saturated in practice?
In Section~\ref{section:two_types_of_oscillations} we show that, for mini‑batch algorithms, loss oscillations alone are not the right signal.
Instead, we diagnose instability through the observation of \emph{catapults} and sizeable \emph{loss spikes} under small hyperparameter perturbations.
In this subsection we formalize this connection on the local quadratic approximation.
On the second‑order Taylor approximation of the loss, the underlying property of interest is \emph{divergence of the dynamics on $U_t$}, and the three diagnostics above---which we use throughout the paper---are just different ways to detect it. 
\begin{summarybox}
The three phenomena below are \emph{equivalent manifestations of divergence on the quadratic approximation}, rather than three logically independent assumptions:
\vspace{-0.1cm}
\[
\text{breaking a valid instability criterion}
\quad\Longleftrightarrow\quad
\text{catapult (Def.~\ref{def:insta})}
\quad\Longleftrightarrow\quad
\text{loss spike of sufficient size.}
\]
\end{summarybox}

\paragraph{Catapults on the quadratic approximation.}
Fix a time $t$ and a point $\theta_t$.
Assume that for each $i$ the per‑sample loss $L_i$ is twice differentiable in a neighborhood of $\theta_t$, and let
$\widetilde L_i(\theta) := \tfrac12(\theta-x_i)^\top \mathcal{H}_i(\theta-x_i)$
denote the second‑order Taylor approximation of $L_i$ at $\theta_t$.
Let $U_t$ be an open neighborhood of $\theta_t$ such that
$\big\|\mathcal{H}(L_i(\theta)) - \mathcal{H}_i\big\| \le C\eta$ for all $\theta \in U_t$ and all $i$, for some constant $C>0$.
Consider SGD (or GD) run on the quadratic model
$\widetilde L(\theta) := \tfrac1N \sum_i \widetilde L_i(\theta)$
with the same hyperparameters as the original dynamics.

\begin{definition}[Catapults on the quadratic model]
\label{def:insta}
We say that the algorithm \emph{experiences a catapult at time $t$} if, when run on $\widetilde L$ from initialization $\theta_t$, the resulting trajectory $(\theta_s)_{s\ge t}$ leaves every compact subset of $U_t$ in finite time.
\end{definition}

This definition is deliberately phrased in the same language as Definition~\ref{def:stab_not}:
a catapult is precisely a divergence event for the quadratic dynamics on the region where the approximation is trusted. 

\paragraph{Hyperparameter perturbations and saturation.}
Let $f$ be a valid instability criterion as in Definition~\ref{def:stab_not}, with threshold $c$.
In practice, both $f$ and $c$ depend on hyperparameters $h$ (e.g.\ $f(\theta;\eta,b)$ and $c=c(\eta,b)$).
We are interested in settings where:
\begin{itemize}[itemsep=0.15em,parsep=0pt]
    \item[(a)] $f(\cdot;h)$ is a valid instability criterion for the quadratic model on $U_t$;
    \item[(b)] $f(\theta_t;h)$ is \emph{monotone} in some destabilizing direction in $h$ (e.g.\ increasing $\eta$ or decreasing $b$ increases $f(\theta_t;h)-c(h)$);
    \item[(c)] along the observed trajectory at hyperparameters $h_0$, $f(\theta_t;h_0)$ \emph{saturates}, i.e.\ $f(\theta_t;h_0)\approx c(h_0)$.
\end{itemize}

\begin{lemma}[Tight instability criterion $\Rightarrow$ catapult under perturbation]
\label{lem:perturb-catapult}
Assume $h_0$ saturates $f$ at $\theta_t$.
Under assumptions (a)–(c) above, any sufficiently small destabilizing perturbation of $h_0$ (e.g.\ $\eta\uparrow$ or $b\downarrow$) produces hyperparameters $h$ such that $f(\theta_t;h)>c(h)$.
By validity of the criterion, the quadratic‑model trajectory from $\theta_t$ then leaves every compact subset of $U_t$ in finite time, i.e., the algorithm experiences a catapult at time $t$ in the sense of Definition~\ref{def:insta}.
\end{lemma}
\vspace{-0.3cm}
See Appendix \ref{appendix:equivalence} for a proof and discussion.
Thus an instability criterion that is monotone in a hyperparameter gives a concrete way to test whether we are at the \textsc{Eo(S)S}: if it saturates along training, then a small destabilizing perturbation must trigger a catapult on the quadratic model.
Conversely, if a quantity empirically saturates but small destabilizing perturbations do \emph{not} lead to catapults, this quantity cannot be a valid instability criterion for the dynamics.


\section{SGD Typically Occurs at the \textsc{EoSS}}
\label{section:eoss}

Full-batch \textsc{EoS} says that, with constant step size $\eta$, training self-organizes near the quadratic stability boundary $\lambda_{\max}(\nabla^2 L)\approx 2/\eta$ \citep{cohen_gradient_2021}.
Mini-batch SGD, however, does \emph{not} move on a single loss landscape: each update sees its own mini-batch loss $L_{\mathbf{B}}$ and its own curvature.
We show here that the relevant “edge” quantity is therefore not $\lambda_{\max}$ of the full-batch Hessian, but a mini-batch, direction-aware curvature statistic.
\begin{itemize}[itemsep=0.15em,topsep=0.15em,parsep=0pt]
\vspace{-0.1cm}
    \item We introduce \emph{Batch Sharpness}, the expected Rayleigh quotient of the \emph{mini-batch} Hessian along the \emph{mini-batch} gradient direction.
    \item Empirically, we show that \textit{Batch Sharpness} undergoes progressive sharpening and then hovers near $2/\eta$ across datasets, architectures, and hyperparameters: SGD typically operates at the \emph{Edge of Stochastic Stability} (\textsc{EoSS}).
    \item By perturbing $(\eta,b)$ mid-training, we show that catapults and renewed sharpening are predicted by \textit{Batch Sharpness}, establishing it as the operational stability comparator for SGD.
\end{itemize}

\subsection{\textit{Batch Sharpness}}  
\label{section:batch_sharpness}
\vspace{-0.2cm}
If the \textit{local} curvature saturates in the directions of the steps, with respect to the hyperparameters, the updates become unstable in a manner analogous to the classic \textsc{EoS} \cite{cohen_gradient_2021}.
Empirically, \textit{Batch Sharpness} is the quantity that governs this saturation for neural networks. We will establish this empirically in the next part of this section, and we will show that it is a valid instability criterion mathematically in \ref{section:math}.
\begin{definition}[Batch Sharpness]
\label{def:minibs}
We define \emph{Batch Sharpness} as the ratio\footnote{\text{We use bold for the "\textbf{B}" to highlight the difference with Gradient-Noise Interaction in Definition \ref{def:GNI}.}}
\vspace{-0.2cm}
\begin{equation}
\textit{Batch Sharpness}\,(\theta)
\quad:=\quad
\E_{B \sim \mathcal{P}_b}\bigg[
\frac{
\nabla L_B(\theta)^\top\,\mathcal{H}(L_\mathbf{B})\,\nabla L_B(\theta)
}{
\|\nabla L_\mathbf{B}(\theta)\|^2
}
\bigg]
\,.
\end{equation}
\end{definition}

\textit{Batch Sharpness} is thus the expectation of the Rayleigh quotient between the mini-batch gradient $\nabla L_B$ (the step direction) and the mini-batch Hessian $\mathcal{H}(L_B)$.
It therefore measures the \textit{expected directional curvature of the mini-batch loss surface along the direction of the step}. 
We will discuss \textit{Batch Sharpness} and its meaning in Section \ref{section:math}.

\subsection{\textsc{EoSS} Empirically}
\vspace{-0.2cm}
Here, we empirically characterize the phenomenon of the Edge of Stochastic Stability.
We verify the emergence of \textsc{EoSS} across a range of step sizes, batch sizes, architectures (Figure \ref{fig:eoss} and Appendix \ref{appendix:further_bs}), and datasets (CIFAR-10 and SVHN, Appendix \ref{appendix:further_bs_svhn}).
\vspace{-0.2cm}

\paragraph{Stabilization of \emph{Batch Sharpness}.}
SGD typically trains in an \textsc{EoS}-like regime:
\begin{summarybox}
    \textit{SGD tends to train in a regime we call Edge of Stochastic Stability. Precisely, after a phase of progressive sharpening, \emph{Batch Sharpness} reaches a stability level of $2/\eta$, and hovers there.}
\end{summarybox}

In particular, the plateau level of \textit{Batch Sharpness} is $2/\eta$ and is independent of the batch size (Figure~\ref{fig:1b_which_is_a_self-explanatory_name}).  
Importantly, stable oscillations occur throughout most of the training, as highlighted by the quantity of Proposition \ref{prop:lee-jang} (see Figure \ref{fig:lee-jang-comparison_body}) but they do not impact progressive sharpening which leads to the second phase of \textsc{EoSS}-type oscillations. Importantly, analogously to \textsc{EoS}, training proceeds and the loss continues to decrease while \textit{Batch Sharpness} is constrained by the step size magnitude.
\vspace{-0.2cm}

\paragraph{Stabilization of $\lambda_{\max}$ and \textit{GNI}.}
Crucially, stabilization of \textit{Batch Sharpness} induces a stabilization of $\lambda_{\max}$. However, $\lambda_{\max}$ consistently settles at a lower level due to a batch-size-dependent gap between the two (Figure \ref{fig:eoss}). This is also influenced by the specific optimization trajectory, as shown in Figure \ref{fig:lr_bs_changes}; for further discussion see Section~\ref{section:lambda_max}. Moreover, stabilization of \textit{Batch Sharpness} around $2/\eta$ happens when \textit{GNI} has already been above $2/\eta$, as it independently stabilized at or above $2/\eta$ at the beginning of the training, see Figure \ref{fig:lee-jang-comparison_body}.
\vspace{-0.2cm}

\paragraph{Catapults.}
\begin{wrapfigure}{r}{0.48\columnwidth}   
    \vspace{-0.3cm}                         
    \centering
    \includegraphics[width=\linewidth]{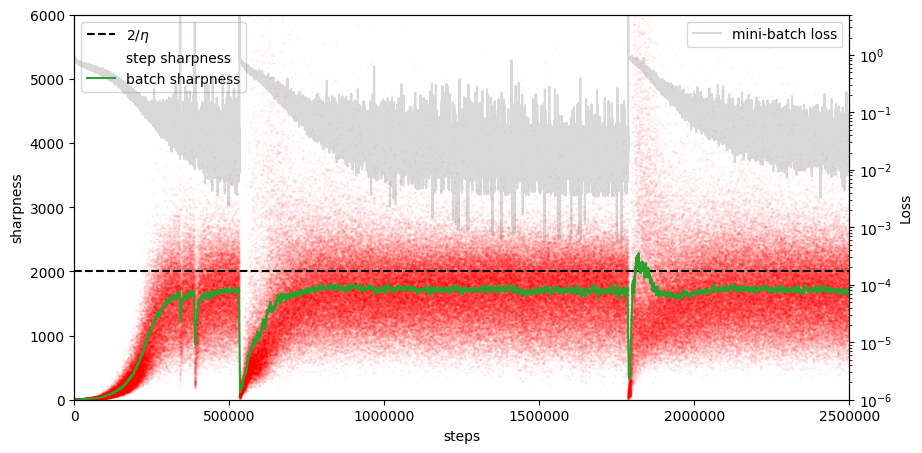}
    \vspace{-8pt}
    \caption{\small
    \textbf{Catapults at \textsc{EoSS}.}
    During \textsc{EoSS}, randomness in batch sampling might cause catapults, leading to renewed PS, and \textsc{EoSS} again. Notation follows Fig.\ \ref{fig:eoss}.
  }
  \label{fig:catapult}
    \vspace{-0.5cm}                         
\end{wrapfigure}
Unlike in \textsc{EoS}, in the \textsc{EoSS} regime what stabilizes is the \textit{expectation} of a quantity which the algorithm observes one step at a time. Occasionally, a sequence of sampled batches exhibits anomalously high sharpness---too high for the stable regime---and steps overshoot, triggering a catapult (Figure~\ref{fig:catapult}).
This causes a spike in the loss (Section \ref{section:equivalence}), after which the trajectory either diverges or re-enters a stable region. If in this region \textit{Batch Sharpness} is strictly less than $2/\eta$, this leads to a renewed phase of progressive sharpening, eventually returning to the \textsc{EoSS} regime. This aligns with, and may help explain, previous observations about catapult behavior, e.g., \citep{lewkowycz_large_2020,zhu_quadratic_2024}.

\begin{figure}[ht]
  \centering
\vspace{-0.3cm}
  \begin{tabular}{ccc}
    \includegraphics[width=0.31\textwidth]{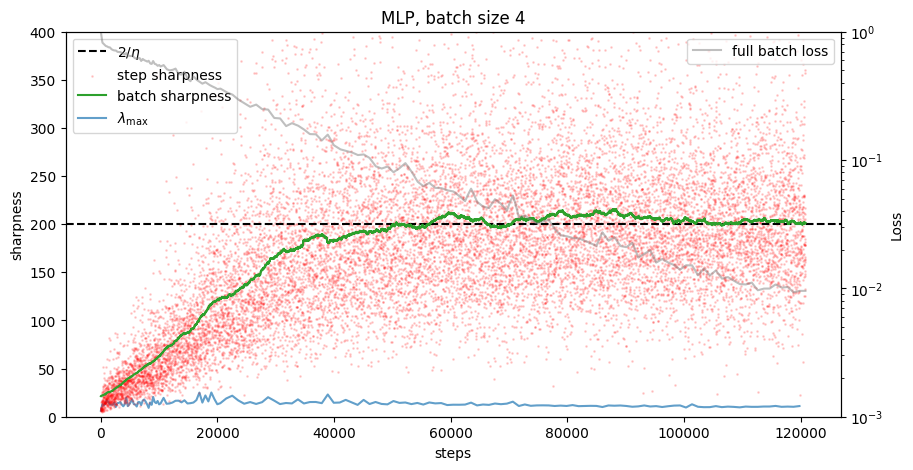} &
    \includegraphics[width=0.31\textwidth]{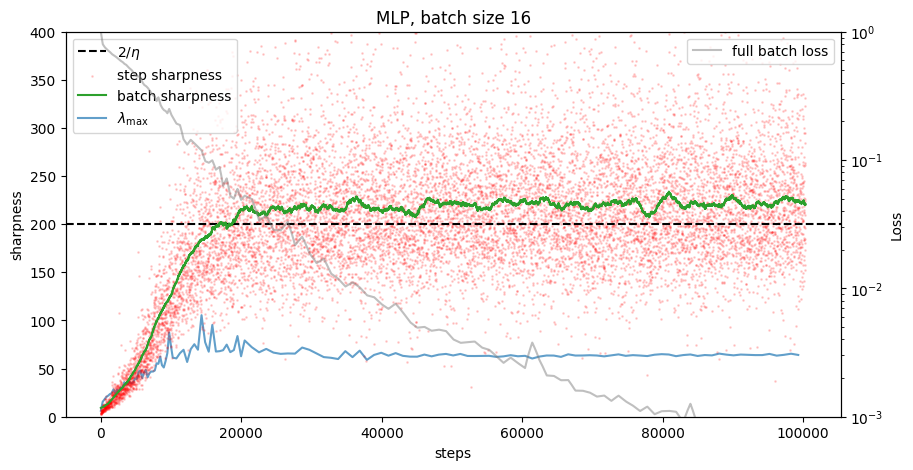} &
    \includegraphics[width=0.31\textwidth]{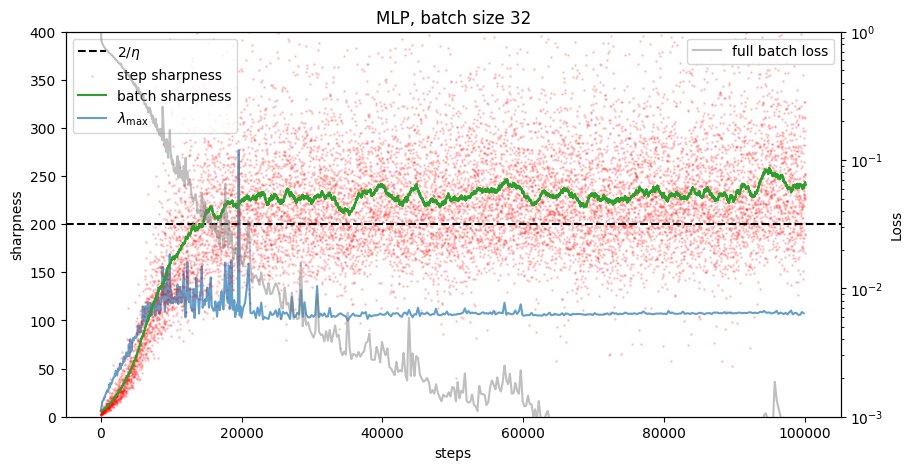} \\
    \includegraphics[width=0.31\textwidth]{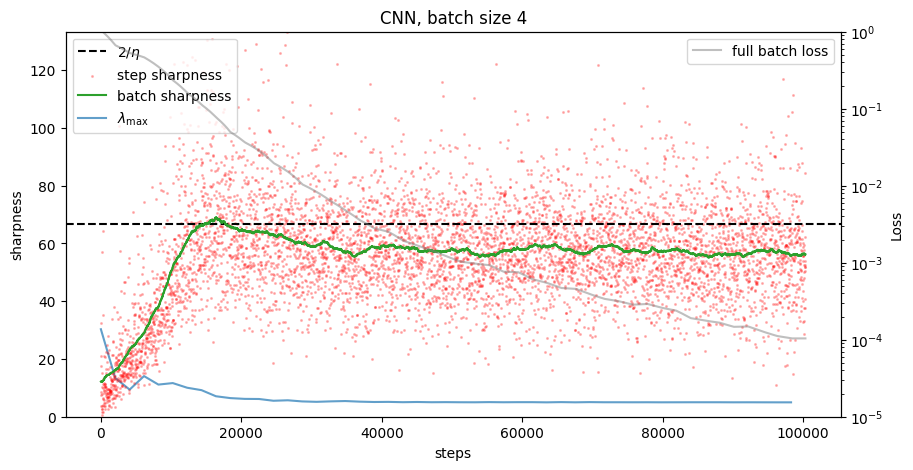} &
    \includegraphics[width=0.31\textwidth]{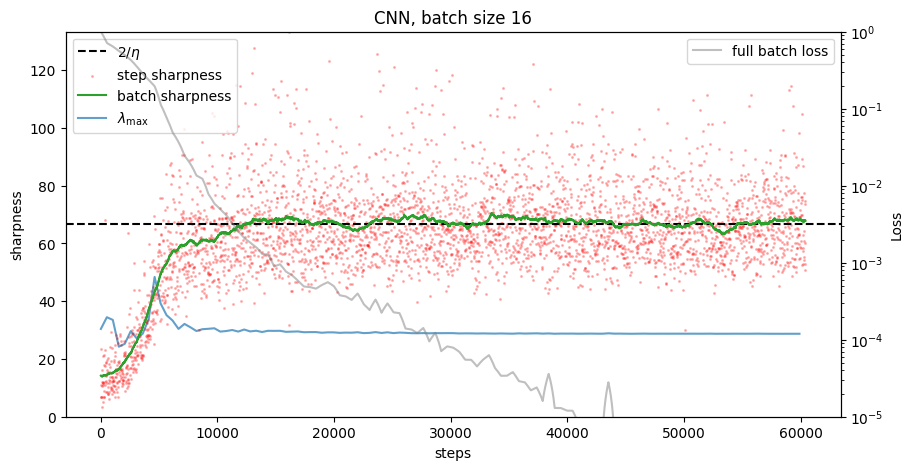} &
    \includegraphics[width=0.31\textwidth]{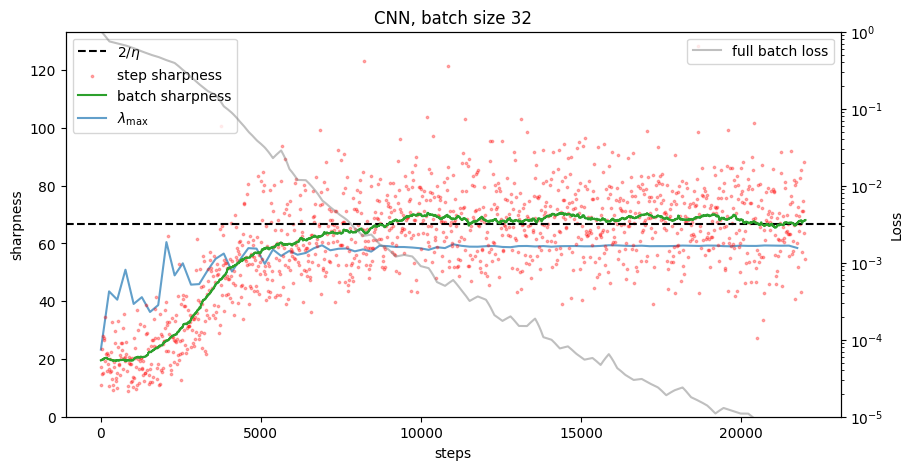} \\
    \includegraphics[width=0.31\textwidth]{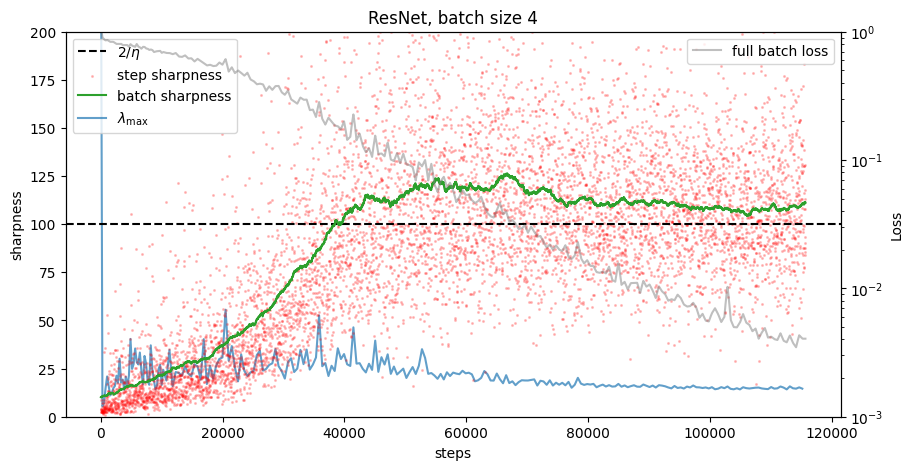} &
    \includegraphics[width=0.31\textwidth]{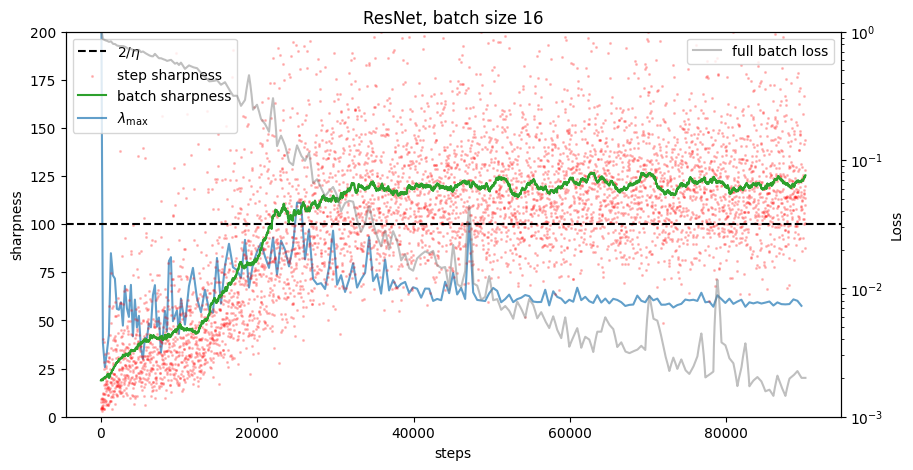} &
    \includegraphics[width=0.31\textwidth]{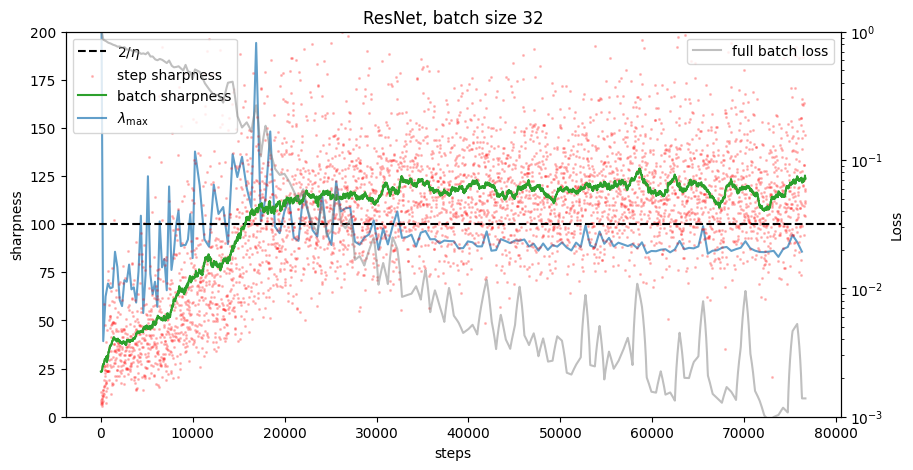}
  \end{tabular}
  \vspace{-0.5cm}
  \caption{\small
    \textbf{Comparing different sharpness measures.}
    \textcolor{red}{Red}: \textit{step sharpness}, observed sharpness on the current step's mini-batch---essentially \textit{Batch Sharpness} without the expectation; 
    \textcolor{PythonGreen}{Green}: \textit{Batch Sharpness} (Definition~\ref{def:minibs}); 
    \textcolor{blue}{Blue}: full-batch $\lambda_{\max}$. 
    Top row: MLP (2 hidden layers of width 512); 
    middle: 5-layer CNN; 
    bottom: ResNet-14; 
    all trained on an 8k subset of CIFAR-10.
  }
  \label{fig:eoss}
\end{figure}

\begin{figure}[ht!]
    \centering
    \begin{subfigure}[b]{0.23\linewidth}
        \centering
        \includegraphics[width=\linewidth]{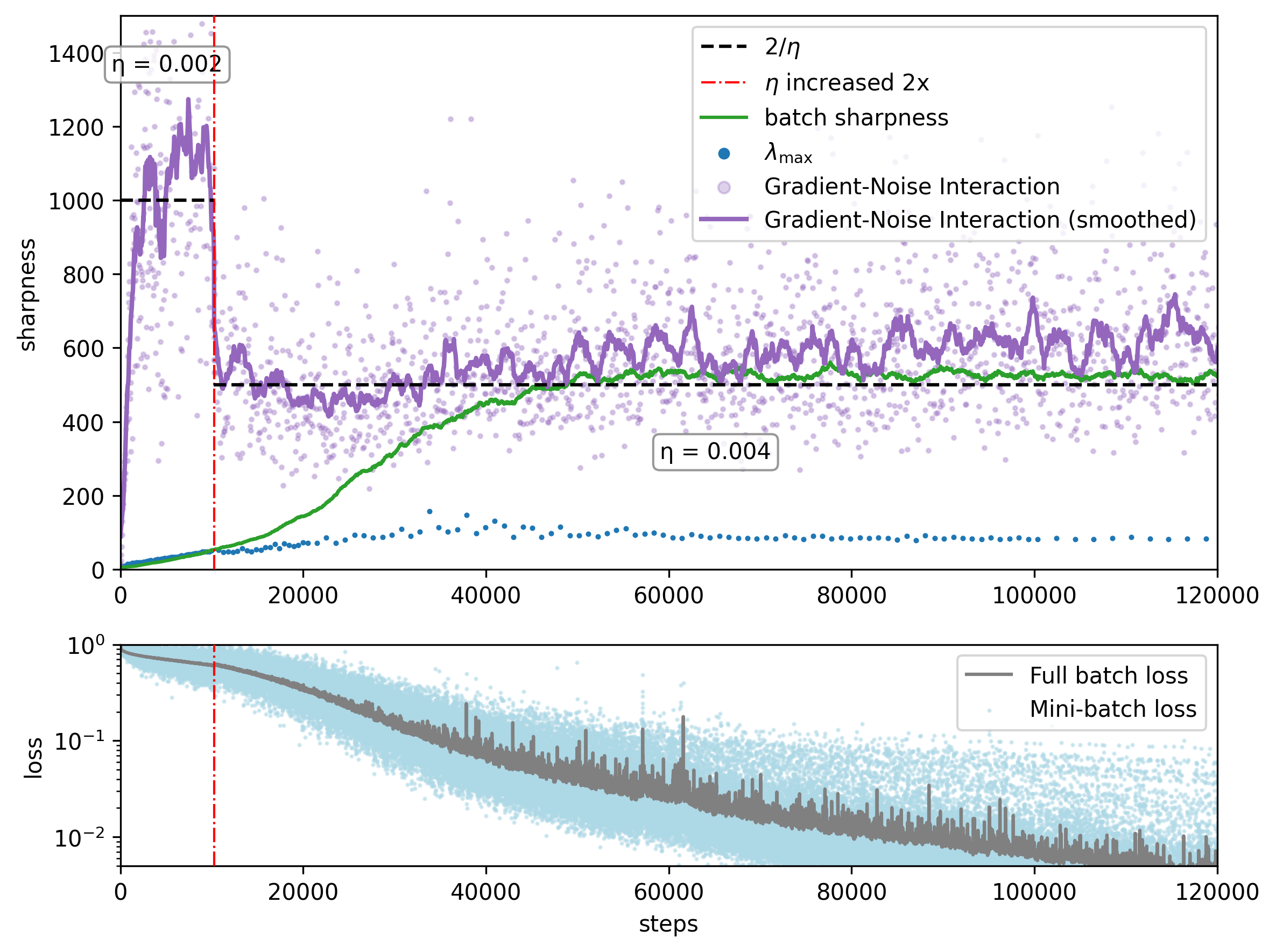}
        \caption{LR increased early}
        \label{fig:lr-early_body}
    \end{subfigure}
    \hfill
    \begin{subfigure}[b]{0.23\linewidth}
        \centering
        \includegraphics[width=\linewidth]{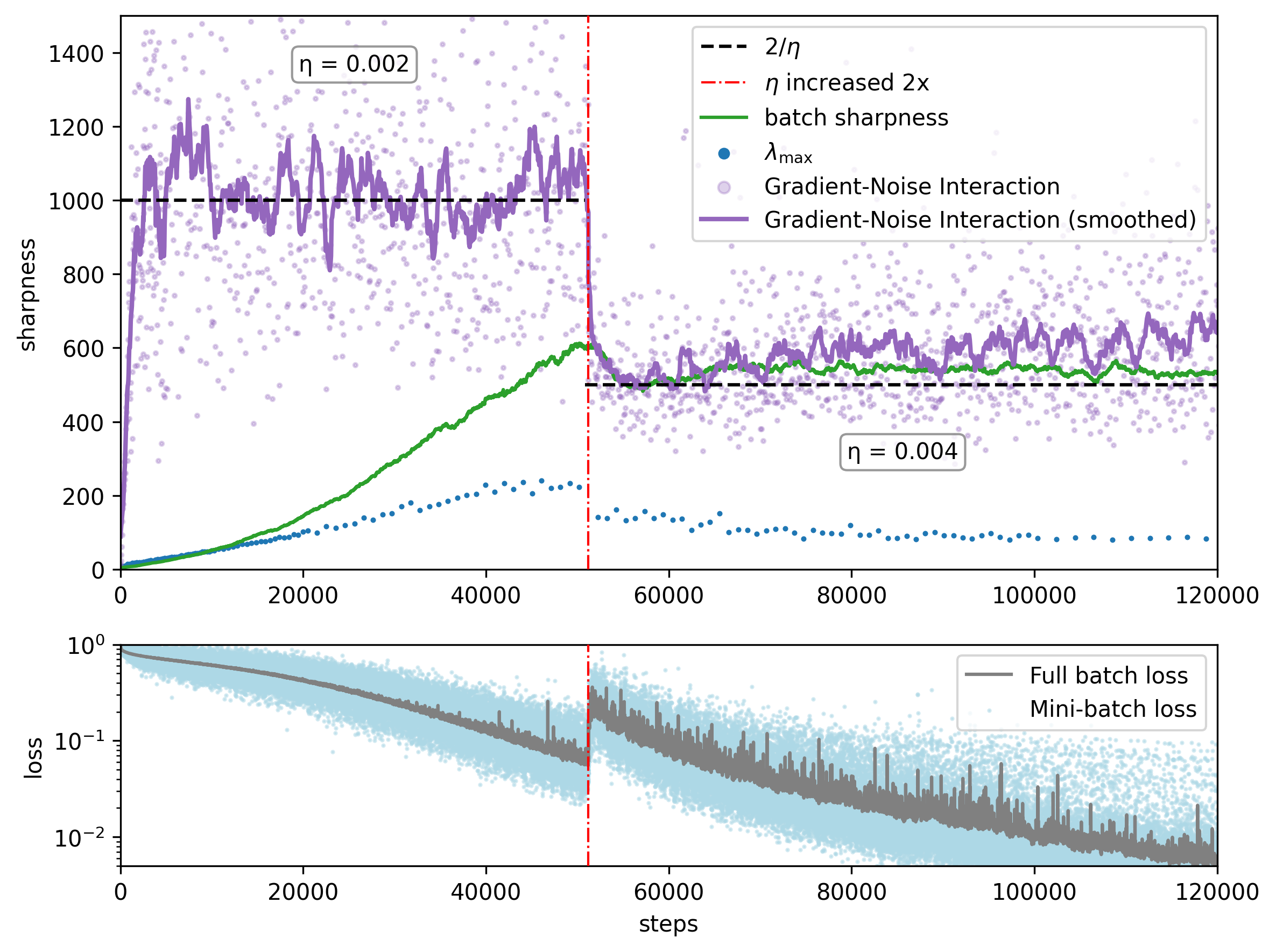}
        \caption{LR increased late}
        \label{fig:lr-late_body}
    \end{subfigure}
    \hfill
    \begin{subfigure}[b]{0.23\linewidth}
        \centering
        \includegraphics[width=\linewidth]{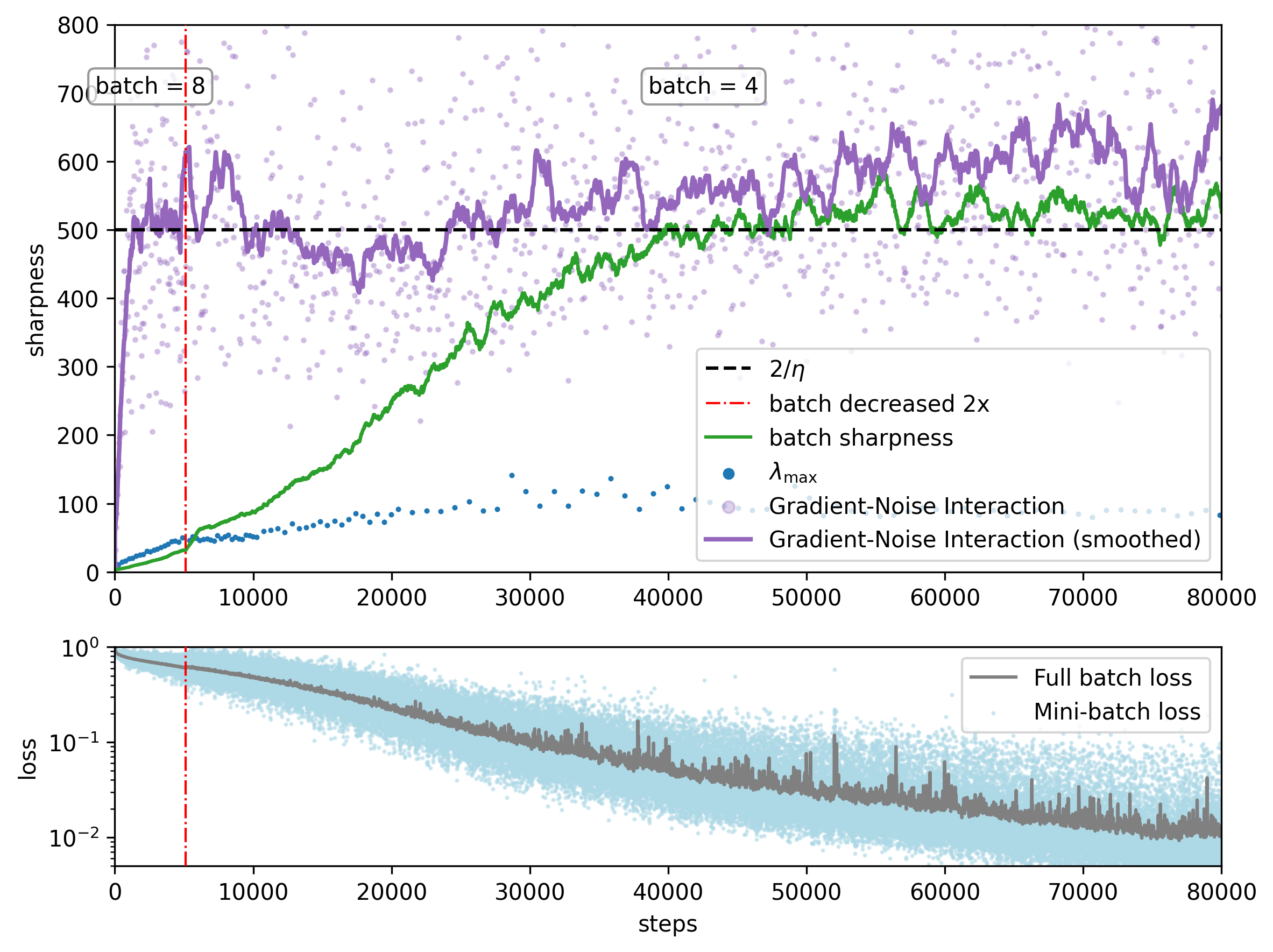}
        \caption{Batch decreased early}
        \label{fig:lee-jang-early_body}
    \end{subfigure}
    \hfill
    \begin{subfigure}[b]{0.23\linewidth}
        \centering
        \includegraphics[width=\linewidth]{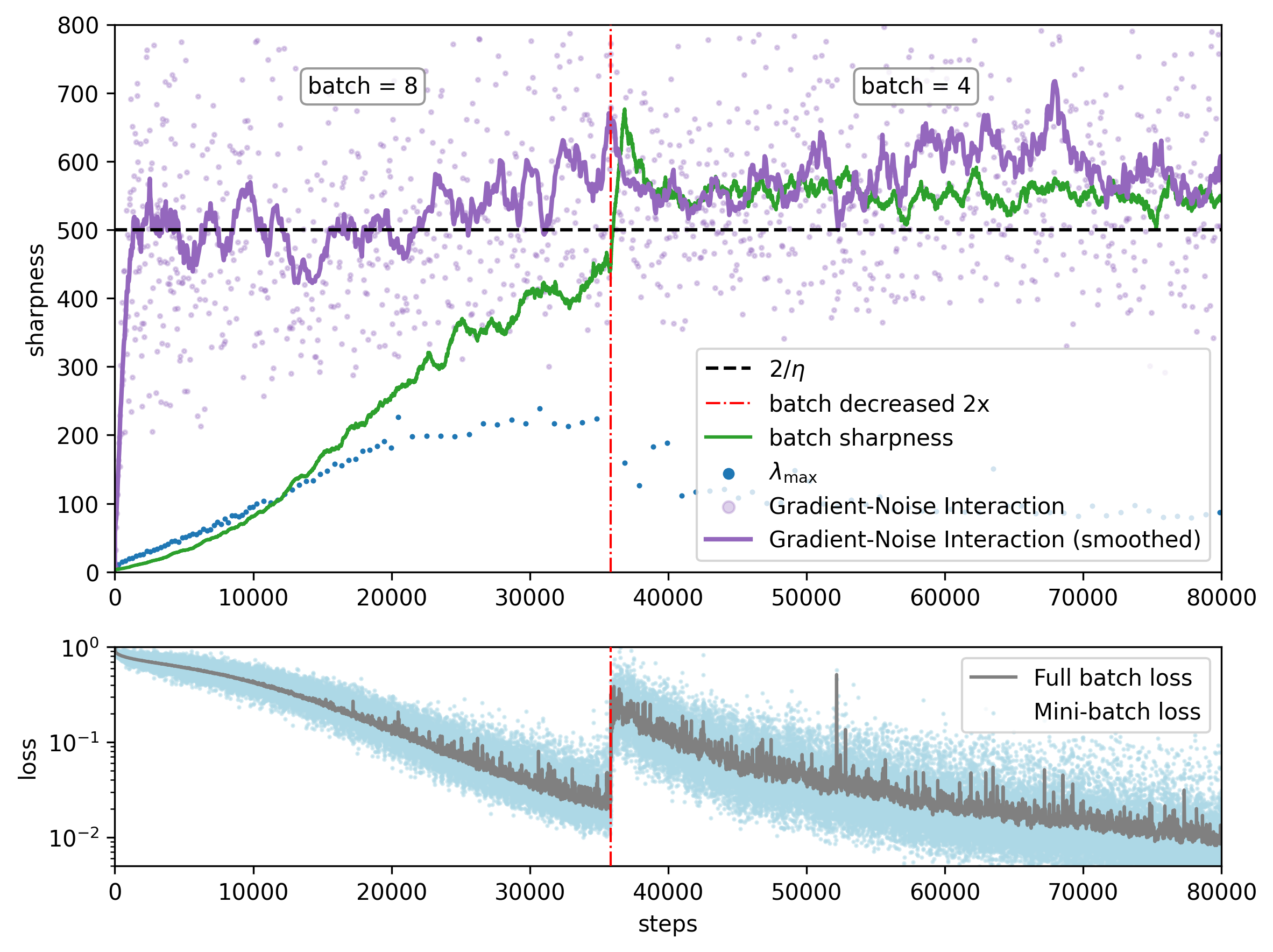}
        \caption{Batch decreased late}
        \label{fig:lee-jang-late_body}
    \end{subfigure}
    \vspace{-0.2cm}
    \caption{\small
    \textbf{(1)} The whole training happens with \textit{Type-1} oscillations (see Proposition \ref{prop:lee-jang}, \textit{GNI}$\approx 2/\eta$), however, \textbf{(2)} \textit{GNI} being $2/\eta$ does not govern \textit{Type-2} oscillations---in particular, highlighting the difference between the two types of oscillations.
    \textbf{(3)} \textit{Batch Sharpness} is instead an indicator of \textit{Type-2} oscillations, as illustrated by the fact that catapults happen only when the shift in hyperparameters occurs \textbf{after} \textit{Batch Sharpness} reaches $2/\eta$.}
    \label{fig:lee-jang-comparison_body}
\end{figure}

\subsection{\textit{Batch Sharpness} Governs \textsc{EoSS}}
\label{section:eoss_natural_generalization}

Following \citet{cohen_gradient_2021} and the discussion in Section~\ref{section:framework}, we track how the training dynamics change when perturbing the hyperparameters mid-training. 
Overall, we find that \emph{Batch Sharpness} governs \textsc{EoSS} behavior---mirroring how \(\lambda_{\max}\) operates in the full-batch \textsc{EoS}---while the full-batch \(\lambda_{\max}\) lags behind or settles inconsistently, underlining the mini-batch nature of SGD stability.
\begin{figure}[ht!]
  \centering
  \begin{subfigure}[b]{0.24\linewidth}
    \centering
    \includegraphics[width=\linewidth]{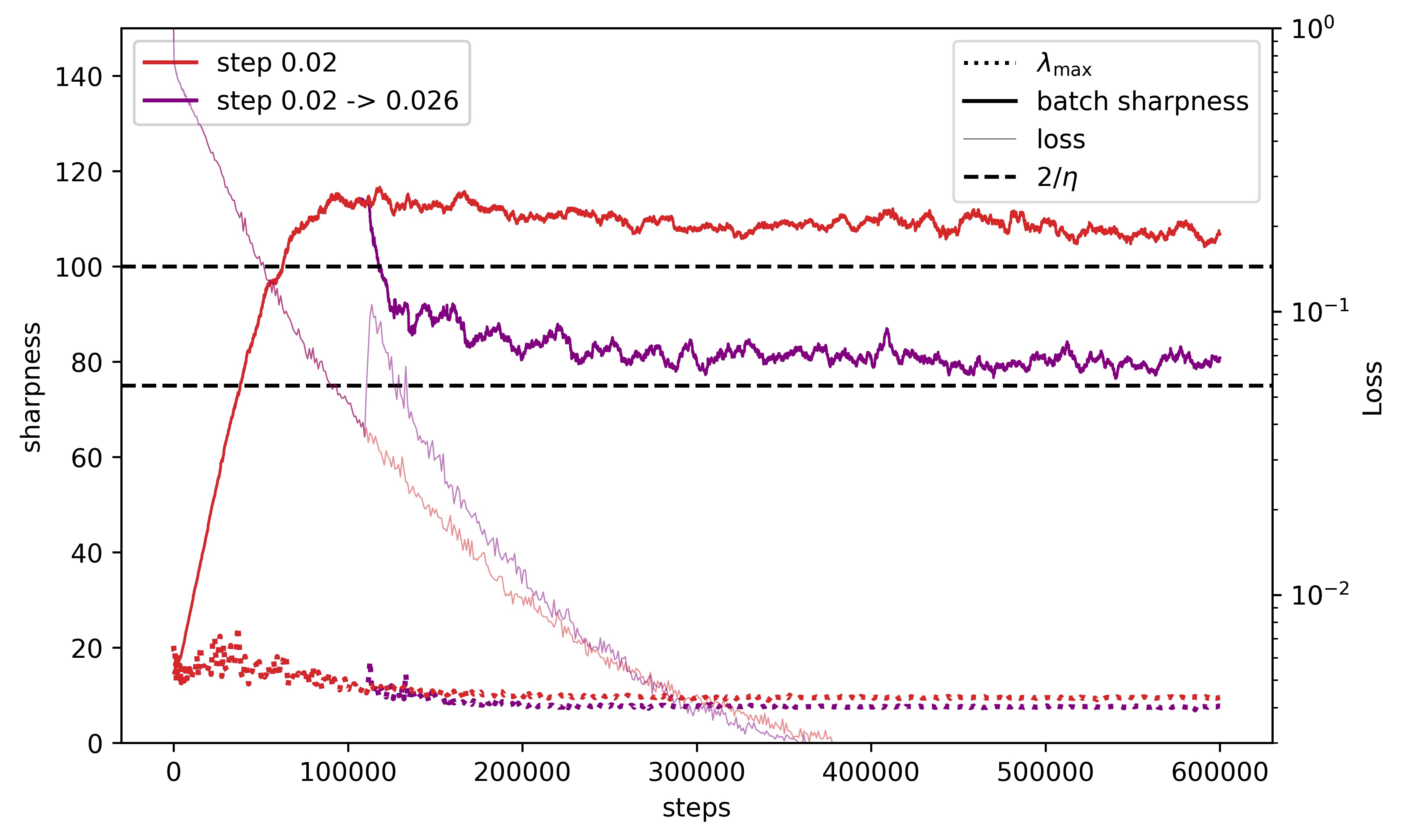}
    \caption{Increasing step}
    \label{fig:lr_bs_changes:a}
  \end{subfigure}
  \begin{subfigure}[b]{0.24\linewidth}
    \centering
    \includegraphics[width=\linewidth]{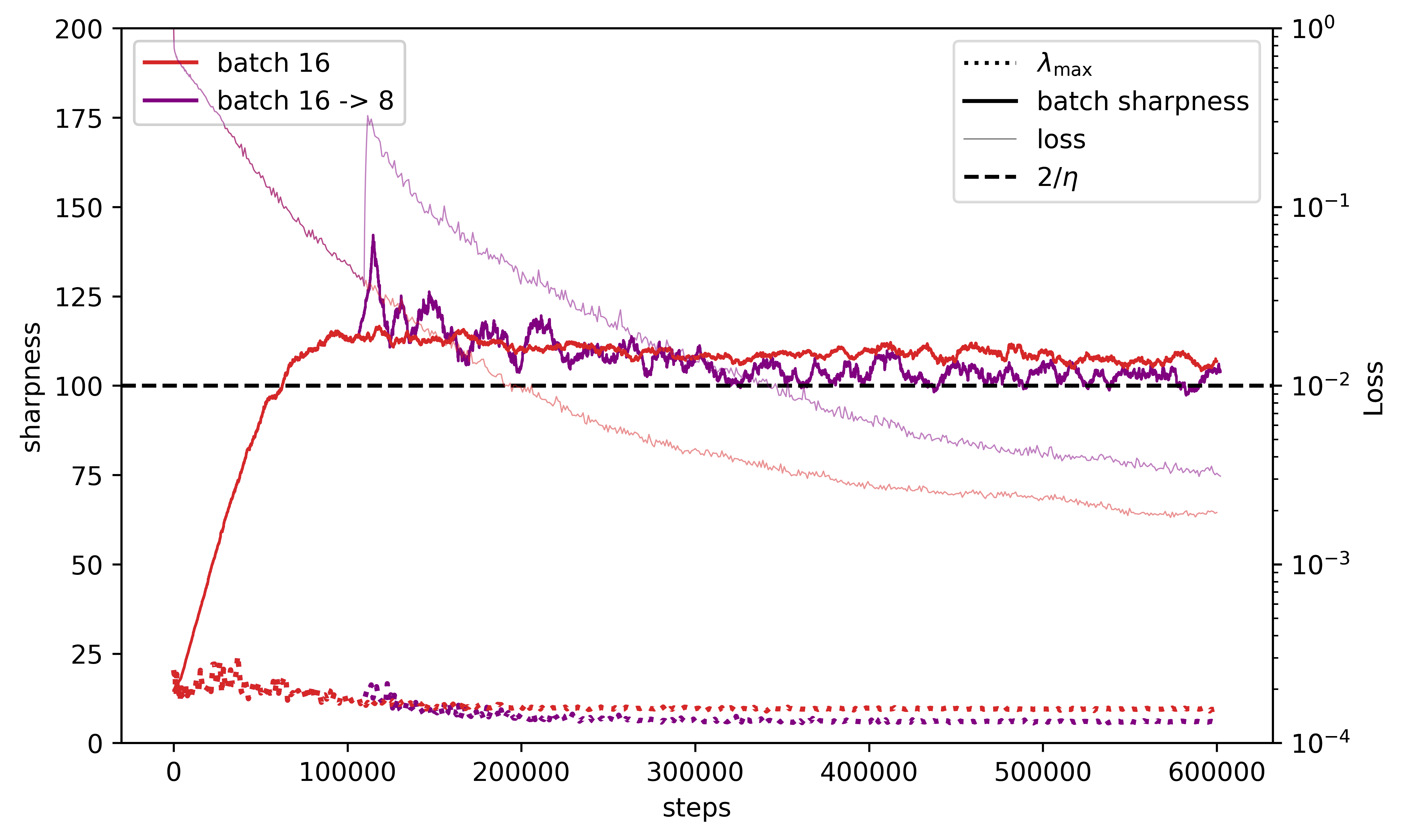}
    \caption{Decreasing batch}
    \label{fig:lr_bs_changes:b}
  \end{subfigure}
  \begin{subfigure}[b]{0.24\linewidth}
    \centering
    \includegraphics[width=\linewidth]{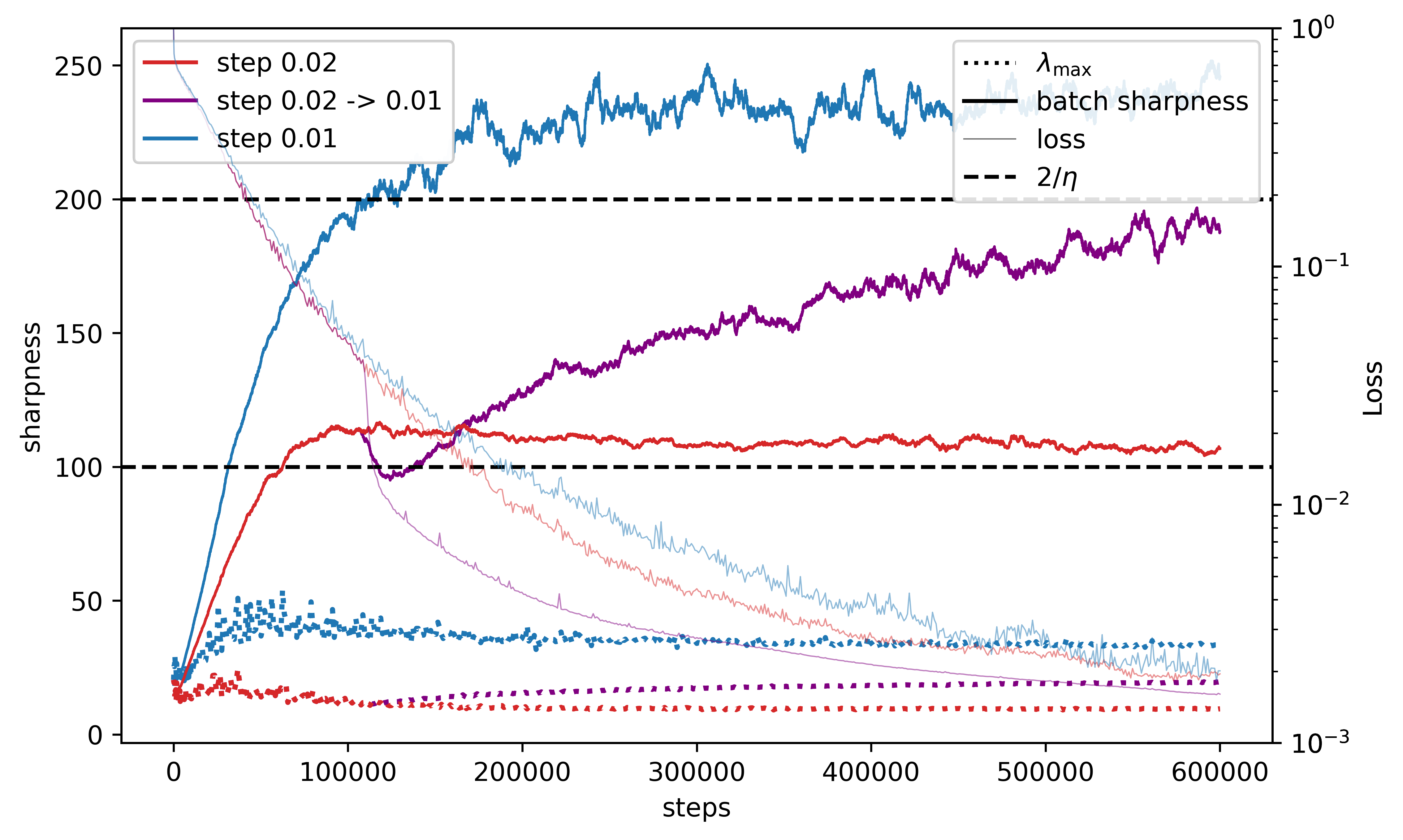}
    \caption{Decreasing step}
    \label{fig:lr_bs_changes:c}
  \end{subfigure}
  \begin{subfigure}[b]{0.24\linewidth}
    \centering
    \includegraphics[width=\linewidth]{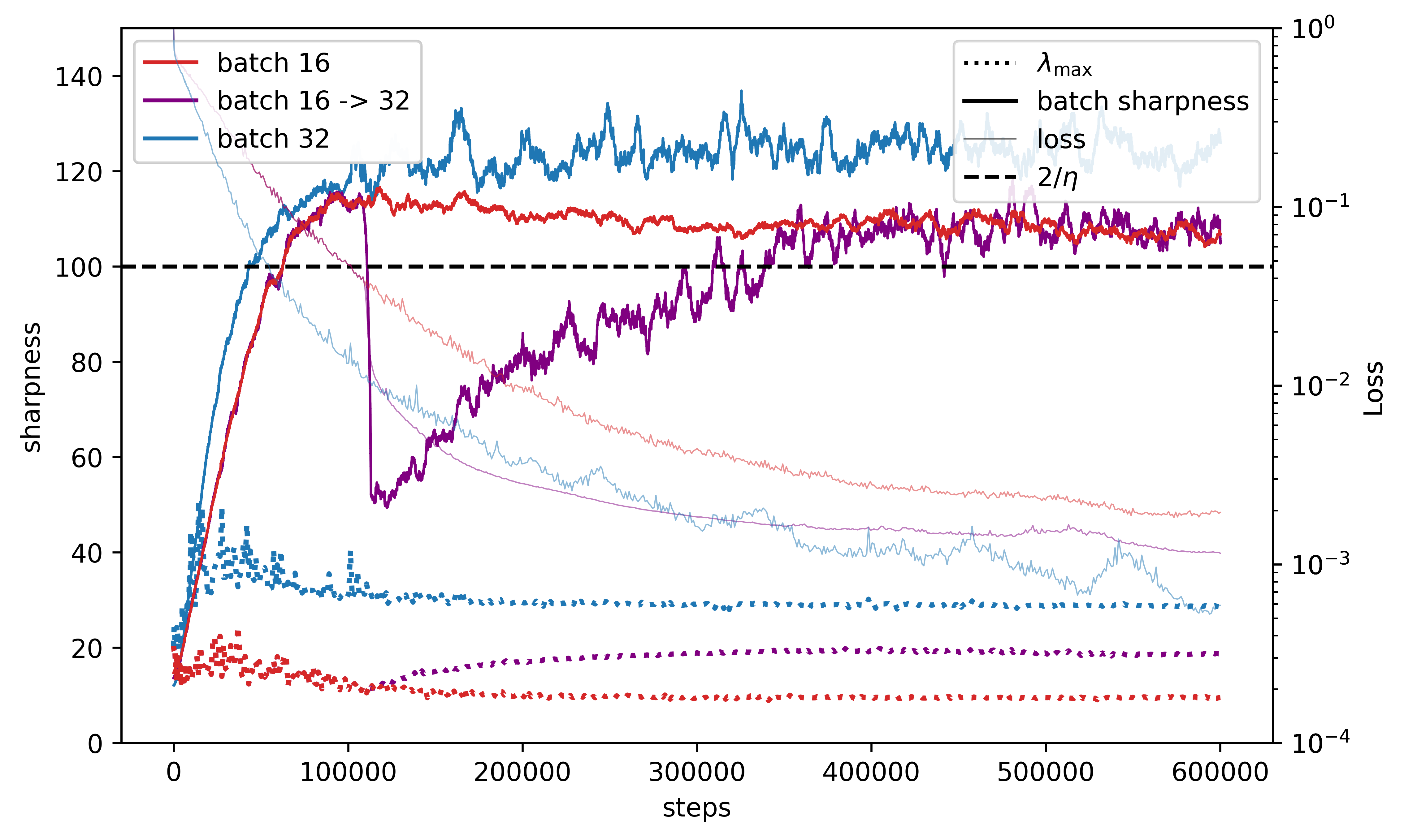}
    \caption{Increasing batch}
    \label{fig:lr_bs_changes:d}
  \end{subfigure}
  
  \vspace{-0.2cm}
  \caption{\small
    \textbf{Effects of changing step size or batch size in \textsc{EoSS}.}
    \textit{Catapults:} (a) Increasing the step size $\eta$ causes a catapult spike before \textit{Batch Sharpness} re-settles at the new $2/\eta$. 
    (b) Decreasing the batch size $b$ increases \textit{Batch Sharpness} and causes a catapult.
    \textit{Restarting PS:} (c) Decreasing $\eta$ prompts renewed progressive sharpening. 
    (d) Increasing $b$ lowers \textit{Batch Sharpness} and re-starts progressive sharpening. 
    The experiments are conducted on a 32k subset of CIFAR-10 to ensure sufficient training-set complexity, which is necessary for observing renewed progressive sharpening, consistent with observations by \citet{cohen_gradient_2021}. 
    }
  \label{fig:lr_bs_changes}
  \vspace{-0.3cm}
\end{figure}
Increasing the step size \(\eta\) or decreasing the batch size $b$ triggers a \emph{catapult} spike in all the quantities of interest and in the training loss, before \emph{Batch Sharpness} re-stabilizes near the updated threshold \(2/\eta\), see Figures \ref{fig:lr_bs_changes:a} and \ref{fig:lr_bs_changes:b}. This therefore pushes \(\lambda_{\max}\) lower.
Conversely, reducing \(\eta\) raises the $2/\eta$ threshold. Analogously, increasing the batch size leaves $\lambda_{\max}$ unchanged but reduces \textit{Batch Sharpness}. These changes prompt a new phase of progressive sharpening, see Figures \ref{fig:lr_bs_changes:c} and \ref{fig:lr_bs_changes:d}. Notice that, \textit{instantaneously}, the change in batch size does not change the full-batch loss landscape, but only changes the mini-batch landscapes---the fact that this causes a catapult/restarts PS is an indicator that it is indeed the mini-batch landscape (and therefore \emph{Batch Sharpness}) that governs the stability/instability of SGD. 
In such cases, \(\lambda_{\max}\) also rises, but ultimately stabilizes at a lower value than if the entire training had run with the smaller step size/larger batch size.
Again, if stability was governed by \(\lambda_{\max}\), this step-size adjustment would have had the same effect as starting from scratch with the new step size. \


{
\captionsetup[subfigure]{font=footnotesize}  
\begin{figure}[ht]
  \centering
  \begin{subfigure}[b]{0.19\linewidth}
    \centering
    \includegraphics[width=\linewidth]{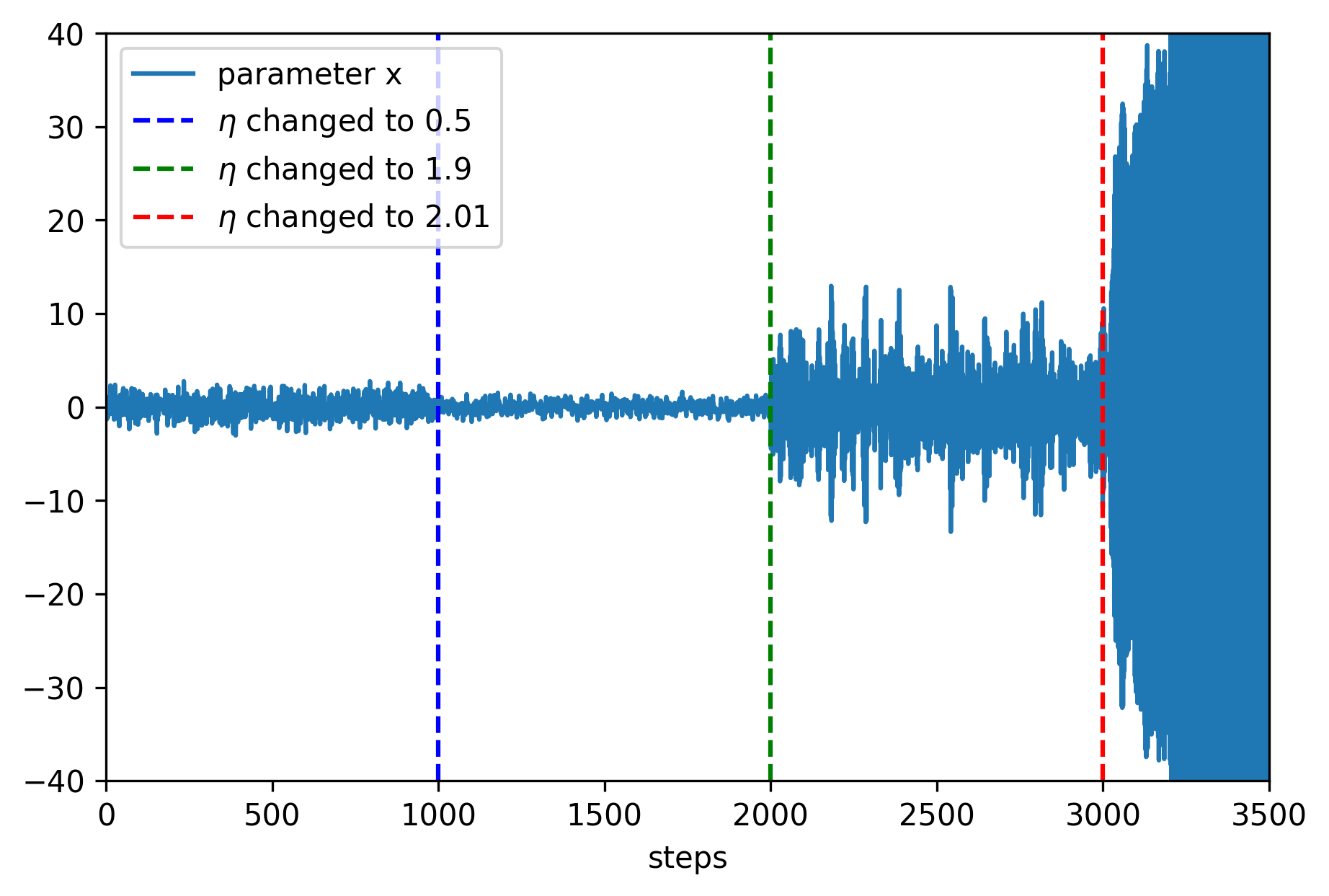}
    \caption{Trajectory}
    \label{fig:quadratic:a}
  \end{subfigure}
  \begin{subfigure}[b]{0.19\linewidth}
    \centering
    \includegraphics[width=\linewidth]{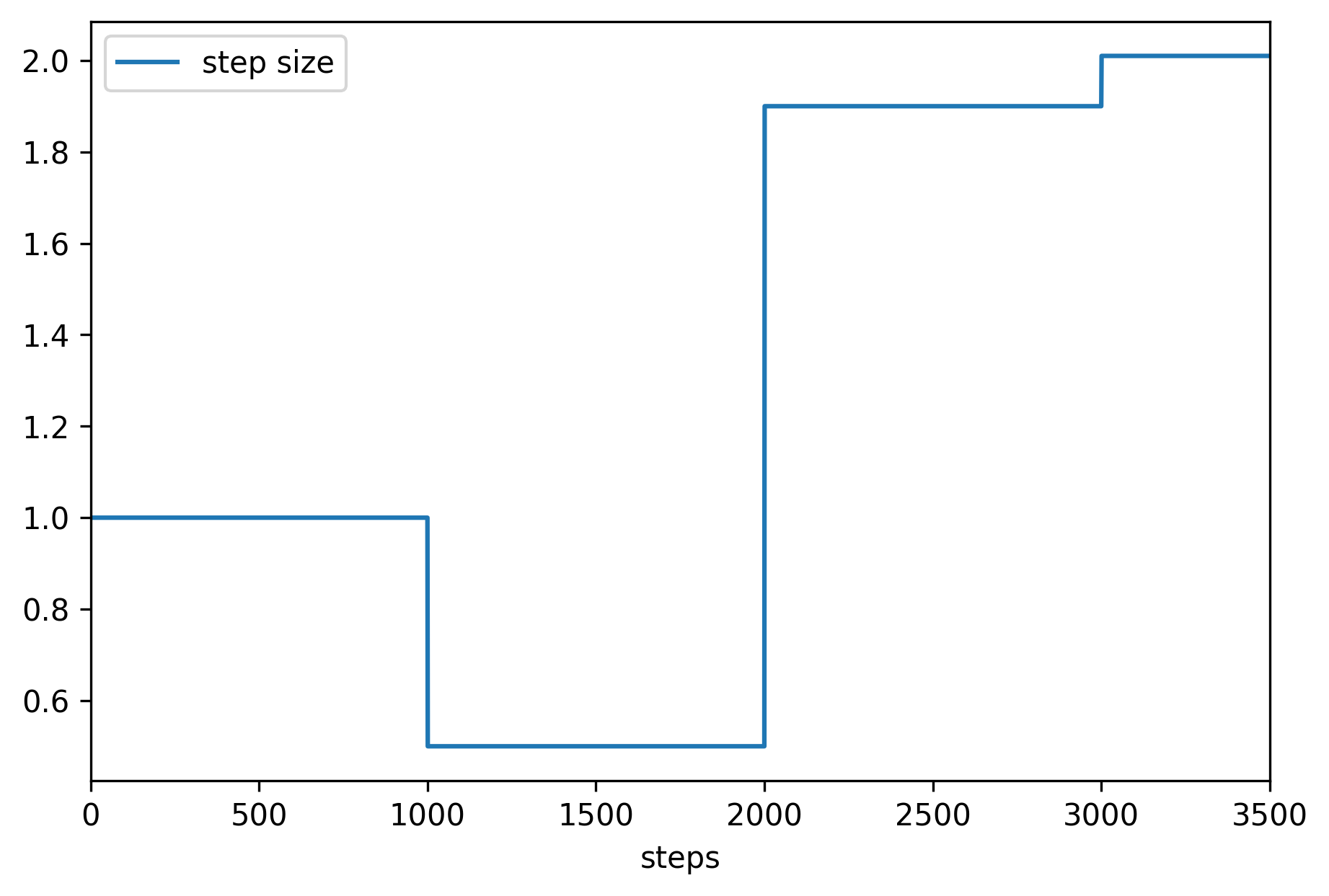}
    \caption{ Step size schedule}
    \label{fig:quadratic:c}
  \end{subfigure}
  \begin{subfigure}[b]{0.19\linewidth}
    \centering
    \includegraphics[width=\linewidth]{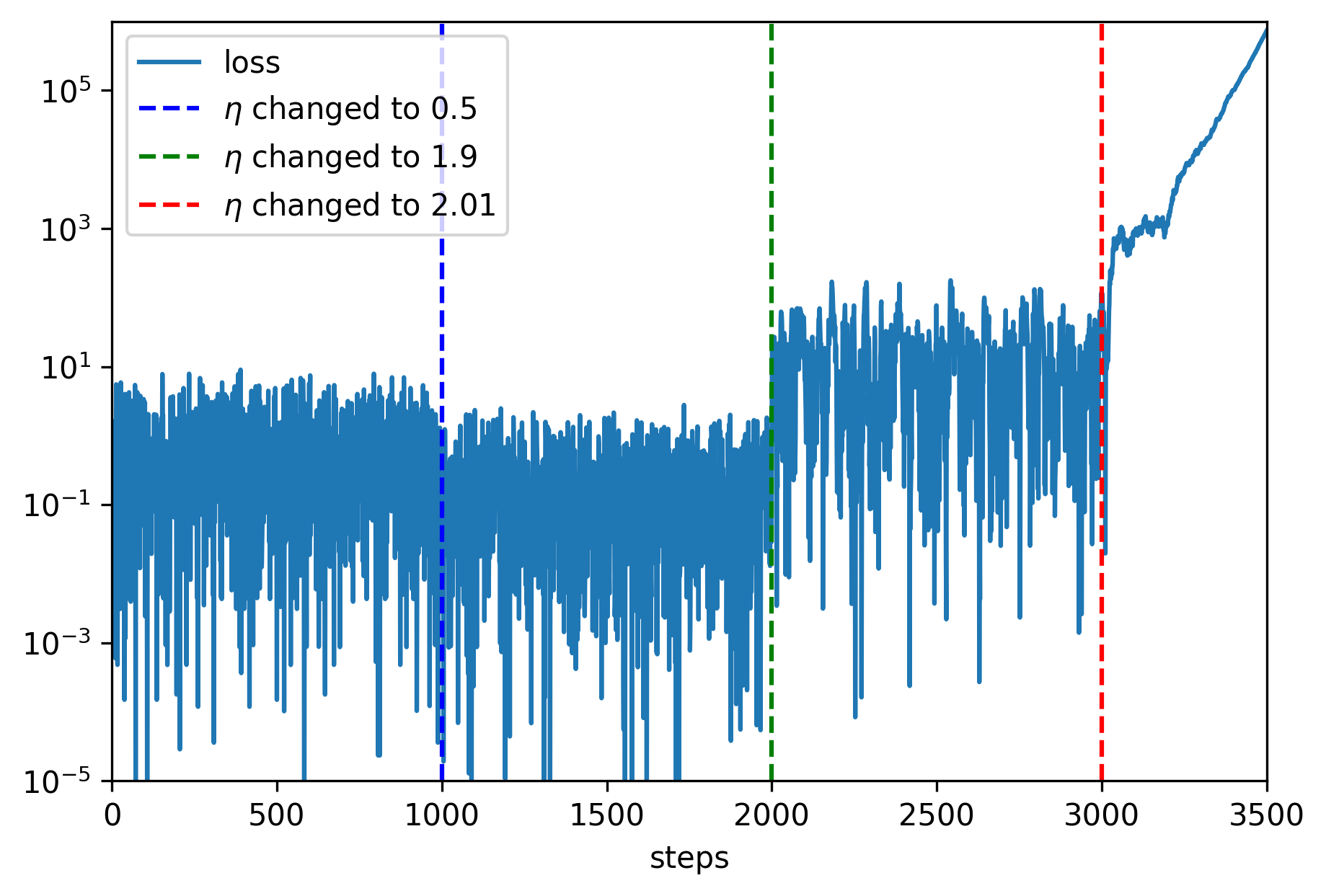}
    \caption{ Loss}
    \label{fig:quadratic:loss}
  \end{subfigure}
  \begin{subfigure}[b]{0.19\linewidth}
    \centering
    \includegraphics[width=\linewidth]{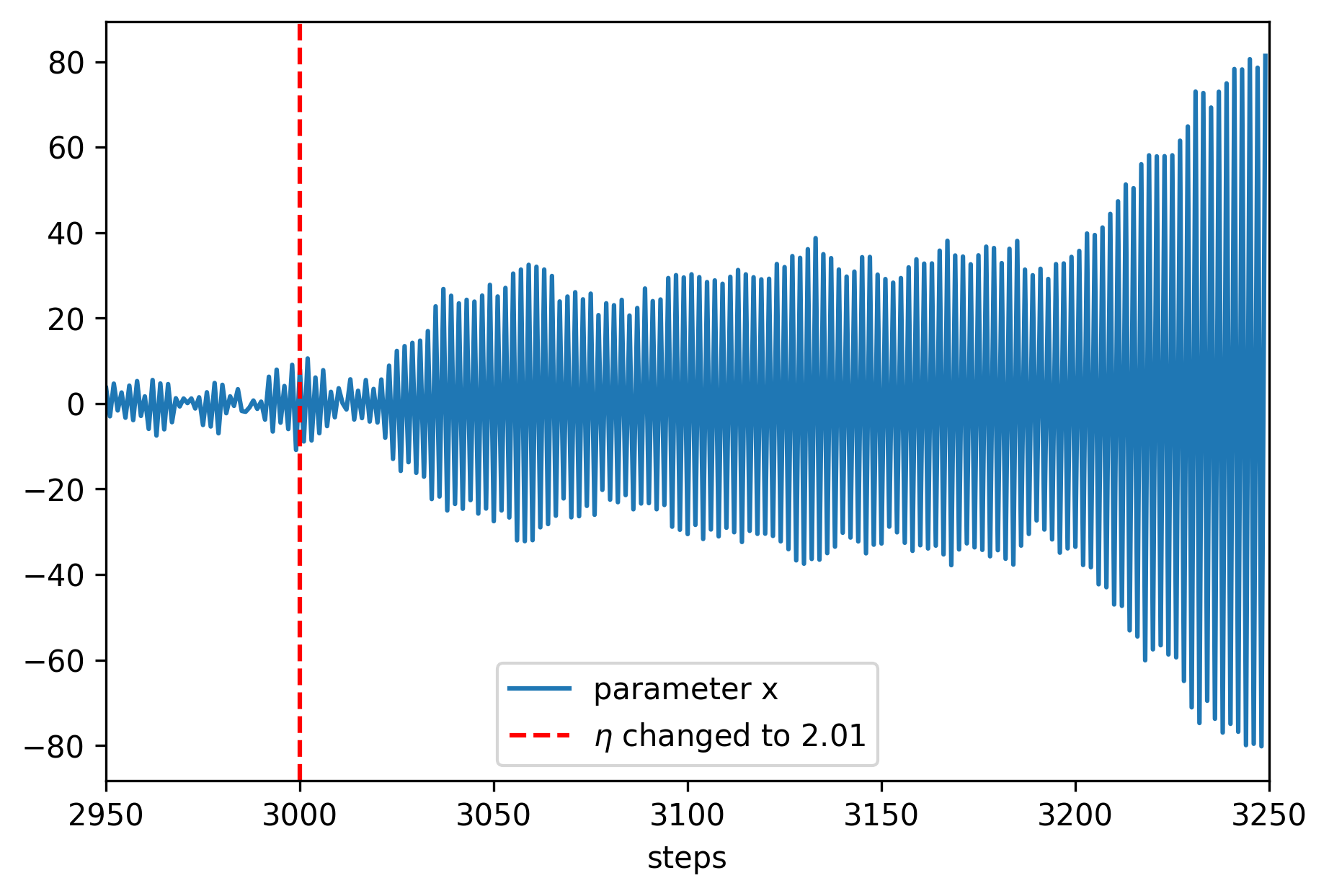}
    \caption{Zoomed-in}
    \label{fig:quadratic:zoom}
  \end{subfigure}
  \begin{subfigure}[b]{0.19\linewidth}
    \centering
    \includegraphics[width=\linewidth]{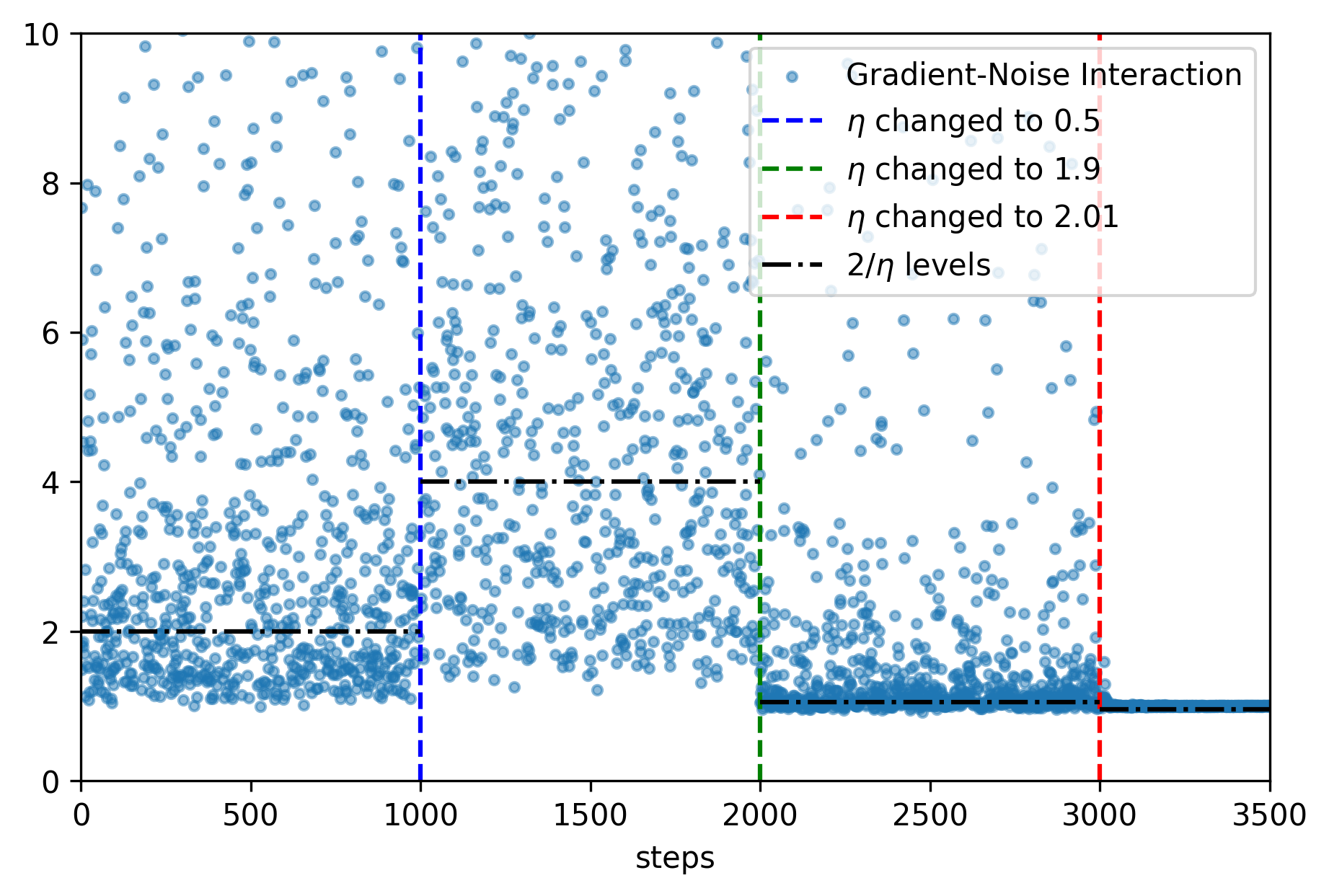}
    \caption{GNI of Eq.\ \eqref{eq:GNI}}
    \label{fig:quadratic:e}
  \end{subfigure}
  \caption{\small
    \textbf{Quadratics:}
    Dynamics of SGD on a 1-D quadratic with $N$ datapoints, $L(x) = \frac{1}{2N} \sum_i (x-a_i)^2$, where $a_i \sim \mathcal{N}(0, 1)$.
    At step size, oscillations are present. Yet, only when the step size becomes larger than 2/sharpness = 2 (last step size change, \red{red} line), the oscillations become unstable (d), the loss diverges (c). Meanwhile, the quantity of \eqref{eq:GNI} consistently stays at $2/\eta$, indicating the presence of \textit{Type-1} oscillations.
    }
  \label{fig:quadratic}
\end{figure}
}


\section{Why Oscillations do not Signal \textsc{EoSS}}
\label{section:two_types_of_oscillations}

In the case of full-batch algorithms, \textsc{EoS} manifests through the emergence of an \textbf{oscillatory regime}.
Practically speaking, one recognizes that full-batch GD trains at \textsc{EoS} when the loss starts oscillating.
Mini-batch SGD, however, always oscillates because its gradient is noisy and the step size does not vanish. 
In this section, we disentangle \emph{which} oscillations signal curvature-limited (or stability-saturated) dynamics as \textsc{EoS} and \textit{which} are just due to noise in the gradients.

\subsection{Two notions of oscillation}
\label{subsec:two-osc-defs}

We first isolate the purely kinematic notion of ``loss wobbling'' and then refine it using the instability framework of
Section~\ref{section:framework}.

\begin{definition}[Oscillatory time window]
\label{def:osc-window}
Let $(\theta_t)_{t\ge0}$ be a training trajectory and write
$\Delta L(\theta_t) := L(\theta_{t+1}) - L(\theta_t)$.
Fix an interval of integer times $I=[t_0,t_1]$ and a constant
$\alpha\in(0,\tfrac12)$.
We say that the loss \emph{oscillates on $I$ at level $\alpha$} if
the loss changes direction with non–negligible frequency:
\[
  \frac{1}{|I|-1}\,\bigl|\{\,t\in I:\Delta L_t\cdot\Delta L_{t+1}<0\,\}\bigr|
  \;\ge\;\alpha.
\]
\end{definition}

This definition is still purely kinematic: it is agnostic to the
underlying mechanism and only records the empirical fact that, on the time window $I$, the loss repeatedly goes up and down and these upward and downward moves are interleaved with non--negligible frequency.
The definition is also independent of the long--term behavior of the
dynamics: the running minimum or running average of the loss could either go down, up, or stay roughly constant.
In what follows we connect oscillations with (in)stability by defining
\textit{Type--1} (stable, noise--dominated wobbling) and \textit{Type--2}
(unstable, curvature--driven) oscillations and characterizing them.

When the trajectory exhibits oscillations (Definition \ref{def:osc-window}) but a valid instability criterion on the local approximation is saturated---or analogously when a perturbation of hyperparameters induces cataputs---we say the oscillations are of \textit{Type-2}.
Unfortunately, \textit{Type-2} oscillations generally look like \textit{Type-1} when just inspecting the loss, thus the most reliable way to understand if the oscillations are stable or not is through the method developed in Section \ref{section:equivalence}.

\subsection{A Certificate of Oscillation: \textit{GNI}}  
We now define the Gradient--Noise Interaction (GNI), which emerges from the second--order Taylor expansion of the loss. This will serve as a certificate of oscillations.

\begin{definition}
\label{def:GNI}
    We denote \underline{G}radient-\underline{N}oise \underline{I}nteraction \textit{(GNI)}:\footnote{\text{Note that both the Hessian $\mathcal H$ and the gradient at the denominator are on the full-batch loss.}}
\vspace{-0.2cm}
\begin{equation}
    \textit{GNI}\,(\theta) \quad := \quad \frac{\E_{B\sim\mathcal{P}_b} \big[ \nabla L_B(\theta)^\top \mathcal{H}(L) \, \nabla L_B(\theta) \big]}{\mednorm{\nabla L(\theta)}^2}.
\end{equation}
\end{definition}
Note that as a direct implication of the descent lemmas, the trajectory oscillates---as for definition above 
$\Delta L_t\cdot\Delta L_{t+1}<0$---if and only if \textit{GNI} is bigger than $2/\eta$ at step $t$, smaller at $t+1$, or vice versa:
\begin{lemma}[\textit{GNI} is a certificate of oscillations]
\label{lem:GNI}
The trajectory oscillates at level $\alpha$ if and only if
\[
\frac{\nabla L_B(\theta)^\top \mathcal{H} \, \nabla L_B(\theta)}{\mednorm{\nabla L(\theta)}^2} + \mathcal{O}(\eta^2) \ - \ \frac{2}{\eta} \qquad
\text{changes sign at consecutive steps $\alpha$ times.}
\]
\end{lemma}
The formal version of Lemma \ref{lem:GNI} is Proposition \ref{prop:GNI_formal} in Appendix \ref{appendix:proof_lem_GNI}. The oscillatory regime of SGD has been extensively documented empirically, e.g, by measuring the expected total loss decrease \citet[Appendix H]{cohen_gradient_2021}, \citet{ahn_understanding_2022}. Moreover, \citet{lee_new_2023} explicitly tracked GNI for neural networks\footnote{In their notations $\tr(HS_b)/\tr(S_n)$. They, however, did not established that any form of oscillation depends on it, as for Lemma \ref{lem:GNI}. See Appendix \ref{appendix:comparison} for further detailed comparison with previous work.}. 
Importantly, we defined \textit{GNI} as an expectation over sampling the batches of the quantity in Lemma \ref{lem:GNI}. \textit{GNI} thus oscillates whenever the expected behavior of the loss is oscillatory---no matter its stability, see Figure \ref{fig:lee-jang-comparison_body}. 
Note also, that oscillations as in Lemma \ref{lem:GNI} may appear even though the time-average or the running minimum of the loss decreases, as we see in the experiments.

\subsection{\textit{GNI} is not an Instability Criterion: Oscillations in SGD Are Expected}
\label{subsec:gni-not-instability}

Mini-batch SGD with a fixed, linearly stable step size typically exhibits noisy oscillations around a minimum, even in flat regions where the Hessian is small. This is classical in stochastic approximation \citep{robbins_stochastic_1951,mandt_variational_2016,bottou_optimization_2018,mishchenko_random_2020} and is precisely why step sizes are traditionally annealed like $t^{-1}$. In the language of Section~\ref{section:framework}, such behavior is \emph{stable}: the iterates wiggle but do not diverge, so oscillations alone cannot be taken as evidence of instability.

Our framework of Section \ref{section:framework} makes explicit what instability is and what an instability criterion must achieve:
\begin{itemize}[leftmargin=1.5em,itemsep=0.2em,parsep=0pt]
\vspace{-0.4cm}
    \item\textbf{Not simply:} once its threshold is crossed, a single SGD step increases a quantity linked to a Lyapunov function of the dynamics.
    \item\textbf{But, importantly:} this increase persists over many steps, forcing the trajectory into a catapult regime.
\vspace{-0.5cm}
\end{itemize}
Thus a valid instability criterion for the local dynamics must rule out absorption into a stationary oscillatory state. \textit{GNI} fails precisely at this requirement. In convex settings with a fixed, linearly stable, step size $\eta$, SGD converges to a stationary distribution $\pi$ supported near a local minimum, even when initialized at that minimum \citep{robbins_stochastic_1951,mandt_variational_2016,bottou_optimization_2018,mishchenko_random_2020}.
In Proposition~\ref{prop:lee-jang} and Appendix~\ref{appendix:oscillations_nn} we go further showing that, if $\theta\sim\pi$, then:
\begin{itemize}[leftmargin=1.2em,itemsep=0.2em,parsep=0pt]
\vspace{-0.3cm}
    \item $\textit{GNI}(\theta)$ is centered at (or above) $2/\eta$ with large variance;
    \item The training loss time-averages at $\E[L(\theta_t)] \approx \eta^2/2$.
    \item If one perturbs slightly $\eta \curvearrowleft \tilde \eta$, \textit{GNI} and the dynamics locally diverge to restabilize again to a new $\tilde \pi$.
\vspace{-0.3cm}
\end{itemize}
These are exactly the hallmark properties of stable Type--1 (noise-driven) oscillations.

\begin{proposition}
\label{prop:lee-jang}
Assume each $L_i$ is convex.
Around a local minimum $\theta^*$, fix a linearly stable\footnote{In the sense either of \citet{ma_linear_2021} or Definition \ref{def:insta}. }  $\eta>0$
Then the trajectory of SGD settles in a stationary distribution $\theta \sim \pi$ characterized by \textit{Type-1} oscillations but not \textit{Type-2} and satisfying
\begin{equation}
\label{eq:GNI}
    \frac{\E_{\theta \sim \pi} \big[\E_{B\sim\mathcal{P}_b} \big[ \nabla L_B(\theta)^\top \mathcal{H} \, \nabla L_B(\theta) \big]\big]}{\E_{\theta \sim \pi}\big[\mednorm{\nabla L(\theta)}^2\big]}
    \quad = \quad \frac{2}{\eta} \Bigl[\,1+\mathcal{O}(\eta)\Bigr],
\end{equation}
independently of the size of $\mathcal{H}$ the moments of the distribution of $\mathcal{H}(L_B)$, $B \sim \mathcal{P}_b$.
\end{proposition}

This has two \textit{consequences}:
\begin{itemize}[leftmargin=1.5em,itemsep=0.2em,parsep=0pt]
\vspace{-0.3cm}
    \item[\textbf{1)}] The \emph{\textbf{typical} long-run behavior} of SGD is to have $\textit{GNI}(\theta)$ centered near $2/\eta$---and to be in an oscillatory regime. 
    \item[\textbf{2)}]  Observing values around this level or oscillations is not, by itself, a sign of being close to instability. Increasing slightly the step size, increases the variance of the stationary distribution and the time-averaged loss; does \emph{not} induce a catapult or explosion.
\vspace{-0.3cm}
\end{itemize}
\textit{GNI} is therefore \emph{only a certificate of oscillation}, not an instability criterion in the sense of Definition~\ref{def:stab_not} and we expect oscillations to happen even in the strongly convex setting.


\section{On Stability}
\label{section:math}

Section~\ref{section:eoss} showed \textit{empirically} that \textbf{(i)} mini-batch SGD typically enters an \textsc{EoSS} regime where \textit{Batch Sharpness} stabilizes near \(2/\eta\), and that \textbf{(ii)} small destabilizing perturbations of the hyperparameters (raising \(\eta\), decreasing batch size) then trigger catapults, whereas the same perturbations are harmless before saturation. This latter observation is enough to show that SGD trains at \textsc{EoSS} and that the phase transition happen at the same time as \textit{Batch Sharpness} stabilizing, see Section \ref{section:equivalence}. However, this is not yet a proof of the \textbf{dependence} on \textit{Batch Sharpness} of this phenomenon.

Section~\ref{section:framework} introduced \emph{instability criteria}:
scalar tests \(f(\theta)\) with thresholds \(c\) such that
\(f(\theta_0)>c\) is sufficient to ensure divergence of the local dynamics on a trusted region~\(U\)---Definition \ref{def:stab_not}. 
In this section we instantiate that framework with \textit{Batch Sharpness} and show that the empirical picture above is not accidental: on the quadratic approximation, \textit{Batch Sharpness} plays for stochastic optimization exactly the role that \(\lambda_{\max}\) plays for full-batch gradient descent. 

\emph{Batch Sharpness at \(2/\eta\) corresponds with the saturation of a genuine instability criterion}.

Beyond Theorem \ref{theo:1} itself, this section introduces three key conceptual consequences of our experiments:
\begin{itemize}[leftmargin=1.2em,itemsep=0.2em,parsep=0pt]
  \item \textbf{Stability of gradients rather than parameters}
        (Section~\ref{section:stab_grad}):
        for neural networks the relevant mini-batch Lyapunov function is $\|\nabla L_B\|^2$, not the global \(\|\theta\|^2\) as in classical linear stability analyses.
  \item \textbf{Average stability on mini--batch landscapes}
        (Section~\ref{section:stab_minibs}):
        \textit{Batch Sharpness} emerges as an expected Rayleigh quotient on the \emph{mini-batch} landscapes; from a stability perspective, SGD is better viewed as a dynamics on a random ensemble of landscapes than as a noisy
        trajectory on a single averaged one.
  \item \textbf{Generalizing \(\lambda_{\max}\) via alignment}
        (Section~\ref{section:alignment}):
        in full-batch GD alignment with the top-eigenvector comes for free, so \(\lambda_{\max}\) suffices; in SGD the alignment between \(\nabla L_B\) and \(\mathcal{H}(L_B)\) is itself random and must be built into the curvature statistic, which is exactly what Batch Sharpness does.
\end{itemize}

\subsection{Batch Sharpness as an instability criterion}
\label{section:bs-instability}

We first make precise in what sense Batch Sharpness governs instability on the quadratic approximation.

Recall the mini-batch Rayleigh quotient
\[
S_B(\theta)
\;:=\;
\frac{\nabla L_B(\theta)^\top H_B(\theta)\,\nabla L_B(\theta)}
     {\|\nabla L_B(\theta)\|^2},
\]
whenever $\nabla L_B(\theta)\neq 0$ and otherwise $S_B(\theta) = \lambda_{\max}(H_B)$.
\textit{Batch Sharpness} is the expected Rayleigh quotient
\begin{equation}
\label{eq:bs-def-again}
Batch\ Sharpness(\theta)
\;:=\;
\E_{B\sim\mathcal P_b}\big[ S_B(\theta)\big].
\end{equation}

over the distribution of batch-sampling $\mathcal{P}_b$ used by the optimizer. On the quadratic model of Definition~\ref{def:insta}, on the mini-batch $B$ the update
$\theta_{t+1}=\theta_t-\eta \nabla L_B(\theta_t)$ is linear in~$\theta_t$, and a natural Lyapunov function is the squared norm of the mini-batch gradients,
\[
V(\theta_t) \;:=\; \|\nabla L_B(\theta_t)\|^2.
\]
The following result, proved in Appendix~\ref{appendix:proof_theo}, shows that \textit{Batch Sharpness} crossing the level $2/\eta$ by a fixed margin induces a catapult by forcing $\E_{B \sim \mathcal{P}_b}[ V(\theta_{t+1})/V(\theta_t) | \theta_t ] > 1$.

\begin{theorem}[Batch Sharpness is an instability criterion]
\label{theo:1}
Work on the quadratic approximation of Definition~\ref{def:insta}, and fix a stepsize $\eta>0$.
If at some time $t$ there exists $\epsilon>0$ such that
\[
Batch\ Sharpness(\theta_t) \;\ge\; \frac{2+\epsilon}{\eta},
\]
then along the quadratic approximation, the expectation over batch sampling of the parameter vector obtained by SGD started at $\theta_t$ experiences a catapult. 
In particular, $Batch\ Sharpness(\theta)$ with threshold $c=2/\eta$ is a valid instability criterion in the sense of Definition~\ref{def:stab_not}.
\end{theorem}

The proof uses only the second-order Taylor expansion of the mini-batch losses, together with Jensen and Cauchy–Schwarz inequalities, to show that $Batch\ Sharpness(\theta_t)>(2+\epsilon)/\eta$ forces a multiplicative drift on $V(\theta_t)$.
Theorem~\ref{theo:1} is intentionally one-sided: it does not claim that $Batch\ Sharpness(\theta_t)\le 2/\eta$ is \emph{necessary} for stability.
This matches our framework from Section~\ref{section:framework}: for diagnosing \textsc{EoSS} we need a scalar quantity whose threshold is \emph{sufficient} for instability, not a complete characterization of all stable regimes.

Combining Theorem~\ref{theo:1} with Lemma~\ref{lem:perturb-catapult} yields the link to the empirical picture of Section~\ref{section:eoss}: once \textit{Batch Sharpness} empirically saturates near $2/\eta$, any small destabilizing perturbation of $\eta$ or $b$ pushes it above $(2+\epsilon)/\eta$ and must trigger a catapult on the quadratic model. 
This is precisely the operational notion of \textsc{EoSS} used in our experiments.

\subsection{Key Implications: Gradients and Mini Batches}
\label{section:stab_grad}

Classical linear stability analyses for SGD on quadratics typically track the second moment of the parameters $\|\theta_t\|^2$ as a Lyapunov function for \textbf{the full trajectory}, and derive conditions under which this quantity stays bounded, see e.g.\ \citet{ma_linear_2021,mulayoff_exact_2024}.
Theorem~\ref{theo:1} instead identifies the squared norm of the mini-batch gradients $V_B:=\|\nabla L_B \|^2$ as the relevant Lyapunov function for every batch $B$ independently as the Lyapunov functions $\{V_B\}$ characterizing the \textsc{EoSS} regime.

\begin{summarybox}
    \textit{Theorem \ref{theo:1} thus has two key consequences:}
\vspace{-0.2cm}
\begin{itemize}[leftmargin=1.8em,itemsep=0.2em,parsep=0pt]
    \item SGD is stable when the steps are, on average, stable on their own mini-batch bowls, not when they are stable with respect to a single averaged loss surface.
    \item \textsc{EoSS} is about mini-batch gradients norm $\E_{B}\|\nabla L_B(\theta_t)\|^2$, rather than, e.g., $\|\theta_t - \theta^*\|^2$ or $L(\theta_t)$.
\end{itemize}
\end{summarybox}

\paragraph{Stability of mini-batch gradients, not parameters.}
On the quadratic model, the two key steps in the proof of Theorem~\ref{theo:1} can be informally summarized as
\begin{equation}
\label{eq:grad-explosion}
Batch\ Sharpness(\theta_t)  > \frac{2+\epsilon}{\eta}
\quad \Longrightarrow\quad
\E_B \! \Bigg[
\frac{\|\nabla L_B(\theta_{t+1})\|^2}{\|\nabla L_B(\theta_t)\|^2}
\,\Bigg|\,
\theta_t
\Bigg]
 > (1+\epsilon)^2
\end{equation}
and
\begin{equation}
\E_B \! \Bigg[
\frac{\|\nabla L_B(\theta_{t+1})\|^2}{\|\nabla L_B(\theta_t)\|^2}
\,\Bigg|\,
\theta_t
\Bigg]
 > (1+\epsilon)^2
\quad \Longrightarrow\quad
\E_B \! \Bigg[
\frac{\|\nabla L_B(\theta_{\red{\mathcal T}})\|^2}{\|\nabla L_B(\theta_0)\|^2}
\,\Bigg|\,
\theta_0
\Bigg] 
 > (1 + \epsilon)^{2 \red{\mathcal T}}
\end{equation}
for all sufficiently small stepsizes.
Iterating \eqref{eq:grad-explosion} shows that
$t\mapsto \E\|\nabla L_B(\theta_t)\|^2$ grows geometrically as long as $Batch\ Sharpness(\theta_t)$ stays above $2/\eta$.
From a stability viewpoint, \textsc{EoSS} is therefore best understood as a transition in the behavior of $\E_{B}\|\nabla L_B(\theta_t)\|^2$, rather than in $\|\theta_t-\theta^*\|^2$ or the loss $L$ itself.
In particular, one chooses a Lyapunov function to have a practical measure to control the dynamics. We showed in Section \ref{section:two_types_of_oscillations} that the loss $L(\theta)$ is not a good such local measure, as descent lemma is not informative of convergence and oscillations are present. 
The functions $\E[ \theta \theta^\top]$ or $\|\theta_t-\theta^*\|^2$ are not informative of the behavior as on this dynamical system subject to progressive sharpening, and often moving towards solutions at infinity, its divergence pattern may not be indicative of stability of the trajectory.
We showed empirically that a combination of the measures $\| \nabla L_B (\theta)\|^2$, which are Lyapunov functions for the quadratic approximation of the landscape on the single batches $B$, are the right point of view to take when discussing instability or convergence of SGD for neural network.

This perspective matches the phenomenology of the \textit{Type-2} trajectories in Section~\ref{section:two_types_of_oscillations}: parameters may remain confined to a bounded region of parameter space for a long time, while the mini-batch gradients become increasingly erratic and large, eventually generating loss spikes once the local quadratic approximation breaks down.
This perspective aligns with some observations from \citet{zhang_neural_2022}.
Last, note that also full-batch \textsc{EoS} is about gradients' instability, just in the full-batch setting any Lyapunov function behaves in the same way, see Section \ref{section:equivalence}.

\paragraph{Average stability on mini-batch landscapes.}
\label{section:stab_minibs}

We now clarify why Batch Sharpness appears as an \emph{expectation over batches}, and why this means that instability in argument of SGD is naturally measured on the \emph{mini-batch landscapes} rather than on the averaged loss.

Consider a single mini-batch $B$ and the SGD step
$\theta_{t+1}=\theta_t-\eta \nabla L_B(\theta_t)$.
On the quadratic approximation and for small~$\eta$, the second-order Taylor expansion of $L_B$ gives the usual descent lemma:
\begin{equation}
\label{eq:mini-desc-body}
L_B(\theta_{t+1}) - L_B(\theta_t)
\;=\;
-\eta\,\|\nabla L_B(\theta_t)\|^2
\;+\;
\frac{\eta^2}{2}\,\nabla L_B(\theta_t)^\top H_B(\theta_t) \nabla L_B(\theta_t)
\;+\;\mathcal O(\eta^3).
\end{equation}
Up to $\mathcal O(\eta^2)$ corrections, the sign of the right-hand side is controlled by the Rayleigh quotient $S_B(\theta_t)$:
\[
S_B(\theta_t) \;\lessgtr\; \frac{2}{\eta}
\quad\Longleftrightarrow\quad
L_B(\theta_{t+1}) - L_B(\theta_t) \;\lessgtr\; 0.
\]
In other words, a single step on batch $B$ is locally stable on its own mini-batch quadratic approximation if and only if its directional curvature satisfies $R_B(\theta_t)\le 2/\eta$.
Taking expectations of $S_B(\theta)$ in \eqref{eq:mini-desc-body} over the mini-batch sampling distribution $\mathcal P_b$ yields the expected Rayleigh quotient $Batch\ Sharpness(\theta_t)$ in~\eqref{eq:bs-def-again}:
\begin{center}
    \textit{Stability (on average) on the mini-batch landscape, not stability on the averaged landscape},
\end{center}

A “max over batches” criterion, such as $\max_B S_B(\theta_t)\le 2/\eta$, would enforce stability of \emph{every} mini-batch step separately, monotonic convergence with probability 1.
A “min over batches” $\max_B S_B(\theta_t)\ge 2/\eta$ would induce divergence with probability 1.
The expectation of the Rayleigh quotient $Batch\ Sharpness(\theta_t) \sim 2/\eta$ implies that the multiplier to $\E_B \big[ \| \nabla L_B (\theta)\|^2 \big]$ is approximately 1.

Thus the expectation in Batch Sharpness is precisely the scalar that:
\vspace{-0.2cm }
\begin{itemize}[itemsep=0.2em,parsep=0pt]
  \item arises from the Taylor expansion of natural Lyapunov functions;
  \item is monotone in destabilizing hyperparameters such as $\eta$ and $1/b$; and
  \item is estimable in high dimension by sampling mini-batches.
\end{itemize}

\subsection{Alignment and Relation to \texorpdfstring{$\lambda_{\max}$}{lambda\_max}}
\label{section:alignment}

We showed in the previous part of this section in what sense \textit{Batch Sharpness} generalizes $\lambda_{\max}$ in the mini-batch setting: it is the instability criterion that empirically saturates. The question though, is \textit{why}. Meaning, why is the case that more cumulants play a role here, unlike in the classical linear stability? 
\begin{center}
\emph{
Why is Batch Sharpness about mini batch gradients, why is Batch Sharpness directional,\\ while $\lambda_{\max}$ is just static curvature?}
\end{center}
Both $\lambda_{\max}$ and the quantities governing classical linear stochastic stability \citep{wu_how_2018,ma_linear_2021,mulayoff_exact_2024} do not change along the quadratic approximation.
This is a major issue in trying to develop an self-stabilization argument as in \citep{damian_self-stabilization_2023}. Indeed, to develop an analogous argument one needs a well-defined value of curvature along the central flow \citep{cohen_understanding_2024} (minimum of the quadratic approximation, while \textit{Batch Sharpness} changes.

\paragraph{In SGD alignment no longer comes for free.}
For mini-batch SGD, different batches see different Hessians $H_B(\theta_t)$ and gradients $\nabla L_B(\theta_t)$. In neural networks, the gradients do not align neither under instability and the mini-batch Hessians do not commute---do not share a single eigenbasis.
The deterministic quantity $\lambda_{\max}$ governs \textit{the} direction of divergence in GD---in SGD there is no such direction.

\textit{Batch Sharpness}, on the other hand, precisely captures this interaction between curvature and alignment:
\[
S_B(\theta_t)
\;=\;
\sum_j \lambda_{B,j}(\theta_t)\,
\frac{|g_{B,j}(\theta_t)|^2}{\|\nabla L_B(\theta_t)\|^2},
\]
where $(\lambda_{B,j},u_{B,j})$ are the eigenpairs of $H_B(\theta_t)$ and
$g_{B,j}$ are the gradient components in that basis.
Each eigenvalue is weighted by how much of the step actually lies in its direction, and $Batch\ Sharpness(\theta_t)$ averages this directional curvature over batches.
Alignment, which was an automatic byproduct of the dynamics in full-batch GD, becomes an explicit ingredient of the curvature statistic in the stochastic setting.
See Appendix \ref{appendix:alignment} for empirical plots of the evolution of the alignments of minibatch gradients.

\newpage

\section{On the Fate of $\lmax$}
\label{section:lambda_max}
\begin{figure}[htbp]
  \centering
  \begin{subfigure}[t]{0.32\textwidth}
    \includegraphics[width=\linewidth]{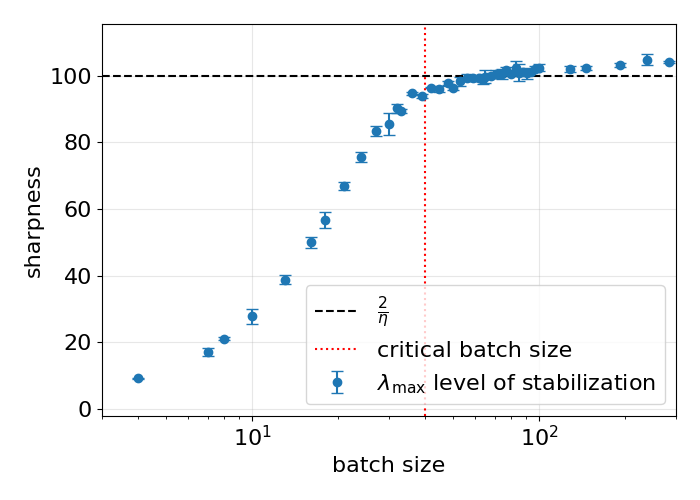}
    \caption{$\lmax$ level, MLP}
    \label{fig:a}
  \end{subfigure}\hfill
  \begin{subfigure}[t]{0.32\textwidth}
    \includegraphics[width=\linewidth]{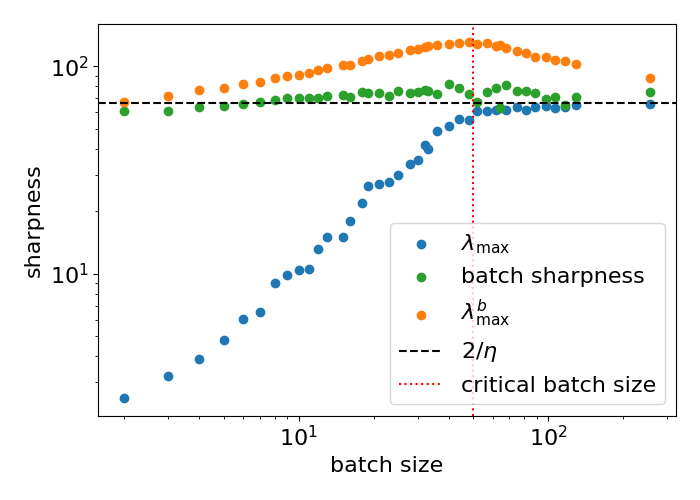}
    \caption{$\lmax$ and others, CNN}
    \label{fig:b}
  \end{subfigure}\hfill
  \begin{subfigure}[t]{0.32\textwidth}
    \includegraphics[width=\linewidth]{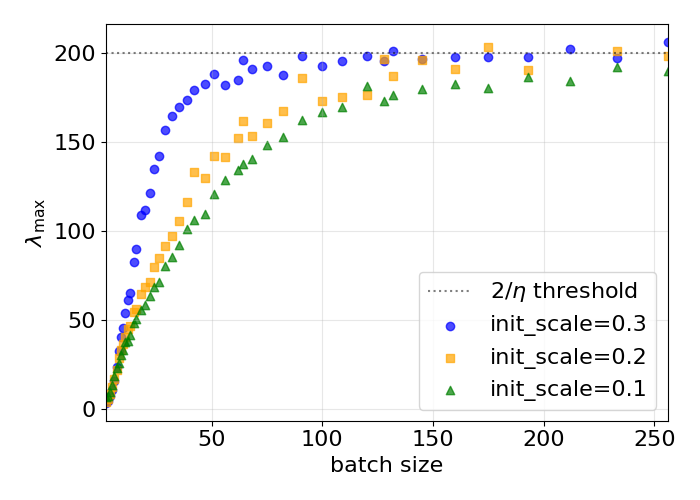}
    \caption{$\lmax$ level, varying init-scale}
    \label{fig:c}
  \end{subfigure}
  \caption{\textbf{Level of Stabilization of $\lmax$.} We track the final level of $\lmax$ at convergence vs batch size. (a) Stabilization level of $\lmax$ for MLP, $\eta=0.02$, error bars show spread over initialization seeds. (b) Stabilization levels of $\lmax$, $\lmax^b$ and \textit{Batch Sharpness} for CNN, $\eta=0.03$. Note how only \textit{Batch Sharpness} has a consistent level of stabilization, see also Appendix \ref{appendix:other_quantities}. (c) Variance in the level of stabilization of $\lmax$ depending on the scale of weights at \textit{initialization}.}
  \label{fig:lmax_levels}
\end{figure}
In previous sections, we have established that in the case of mini-batch training, it is the \textit{Batch Sharpness} that stabilizes at $2/\eta$ and governs the \textsc{EoSS} dynamics. Here, we examine the behavior of $\lambda_{\max}$ once
\textsc{EoSS} is reached. A key aspect of the original \textsc{EoS} analysis is, indeed, that the controlling
quantity---$\lambda_{\max}$---has an immediate geometric interpretation, is widely studied in neural network literature, and plays a fundamental role in convergence analyses.  
In contrast, the \textsc{EoSS} framework replaces $\lambda_{\max}$ with \emph{Batch Sharpness}, a statistic whose connection to generalization and role
in optimization theory is largely unexplored.

Below, we describe the phenomena we observed in vision classification tasks trained with MSE, ablating on batch sizes, step sizes, architectures and datasets:
\begin{itemize}[leftmargin=0.8em,itemsep=0.2em,parsep=0pt]
    \item \textbf{Progressive Sharpening. }
    $\lambda_{\max}$ also undergoes progressive sharpening, 
    increasing at most as long as \textit{Batch Sharpness} does.
    \item \textbf{Phase Transition.}
    Once \textit{Batch Sharpness} plateaus at $2/\eta$, $\lambda_{\max}$ stops increasing. If it changes subsequently, it only decreases from this time on.
    \item \textbf{By-product of \textsc{EoSS}.} Stabilization of $\lambda_{\max}$ arises as a consequence of \textit{Batch Sharpness} stabilizing, as demonstrated in Section \ref{section:eoss_natural_generalization}.
    \item \textbf{Smaller Batches $\Rightarrow$ Flatter minima.} Reducing the batch size monotonically decreases the plateau level of $\lambda_{\max}$. This occurs due to a gap between it and \textit{Batch Sharpness}, which becomes bigger with smaller batch size---due to the misalignment of per-sample Hessians, see Appendix \ref{appendix:alignment}. This behavior aligns with the long‑standing empirical observation that smaller batches locate flatter minima \citep{keskar_large-batch_2016,jastrzebski_catastrophic_2021}). Notably, the gap itself grows dynamically during progressive sharpening, sometimes even forcing $\lambda_{\max}$ downward for small batch sizes (see, e.g., Figure \ref{fig:full-cifar}).
    \item \textbf{A Critical Batch Size Marks the SGD\,$\to$\,GD Crossover. }
    If we analyze the dependence of stabilization level of $\lambda_{\max}$ vs 
    batch size, a distinct transition emerges at what we term the \emph{critical batch size}, $b_c$, see Figure \ref{fig:lmax_levels}. For $b < b_c$, the plateau levels of $\lambda_{\max}$ decrease sharply as $b$ becomes smaller; for $b > b_c$, the curve flattens, gradually converging toward the full‑batch value.     This transition marks the point at which the mini‑batch loss landscapes begin to approximate the full‑batch landscape closely enough, restoring GD-like dynamics.
    \item \textbf{Path-dependence.}
    The trajectory of $\lambda_{\max}$ 
    is not determined solely by final hyperparameter values; rather, the level of $\lambda_{\max}$ is \emph{path‑dependent}: it inherits the history of progressive sharpening up to the moment \textsc{EoSS} is reached (see Figure \ref{fig:lr_bs_changes}). This underscores the complex role of step size and batch size scheduling.
    \item \textbf{Dependence on Other Hyperparameters.} Another manifestation of 
    path-dependence is shown through dependence of level of stabilization on seemingly inconsequential hyperparameters---e.g. the scaling of the random weights \textit{at initialization} (Figure  \ref{fig:lmax_levels}c). 
    Although one would expect the initialization to be "forgotten" throughout the training, it instead affects the dynamics of progressive sharpening, highlighting the inconsistency of stabilization level of $\lmax$.
\end{itemize}
These observations clarify that fully understanding SGD dynamics will necessitate a detailed dynamical theory encompassing both the progressive sharpening phase and subsequent stabilization. Additional insights regarding the dynamics of other quantities relevant to SGD dynamics---trace of the loss Hessian and $\lmax^b := \E_B[\lmax(\mathcal H_B)]$---are provided in Appendix \ref{appendix:other_quantities}.


\section{Implications}
\label{section:noise_vs_sgd}

\subsection{Failure of Classical Optimization Arguments}

Classical optimization theory comes with a stability 
constraints, e.g.,
for GD pick a step size $\eta \le 2/\lambda_{\max}$ along the trajectory, see Section \ref{section:in_stab}.
These are required to prove a descent-type lemma telling the loss goes down at \emph{every step}.
\textsc{EoS} result was so disruptive precisely because it makes the violation impossible to ignore: full-batch GD routinely operates with curvature at (or beyond) the textbook stability boundary, the loss is \emph{not} monotone, and yet training keeps making progress.
In fact, this is not merely “nonconvexity is hard”---it is a regime where the standard proof templates do not even get off the ground for convex functions, and where trajectories need not converge to stationary points in the usual sense \citep{ahn_understanding_2022,zhang_neural_2022}.

One could hope that SGD was “safe”. Indeed, (1) along the mini-batch trajectories it was well known that $\lambda_{\max}$ is smaller \citep{keskar_large-batch_2016,jastrzebski_three_2018}; and (2) a clean, tight, stability threshold to use for mini-batch proof was not previously known.
In the mini-batch case, most analyses defensively assume a \emph{uniform} smoothness bound over all mini-batch objectives (see e.g.\ \citet{gurbuzbalaban_why_2021,mishchenko_random_2020}).
But in practice that worst-case Lipschitz constant is so large\footnote{It is infinite on the whole parameter space and anyways practically too large in a neighborhood of the trajectory.} that it turns the question “is $\eta$ too large?” into a non-answer: \emph{everything} would be “too large,” including the settings that train perfectly well.
In this article we finally unify all these issues by characterizing the dynamics and explaining a number of previous observations:
\begin{summarybox}
\begin{enumerate}[leftmargin=1.8em,itemsep=0.2em,parsep=0pt]
    \item SGD is not safe either, descent-type lemmas are violated too.
    \item We found a mechanism for which smaller batches force smaller Hessian \citep{keskar_large-batch_2016}.
    \item We have a much better, tighter, guess for what a too large step size is, one such that
    \vspace{-0.2cm}
    \[\textit{Batch Sharpness} \geq 2/\eta.\]
\end{enumerate}
\end{summarybox}

Our empirical answer is simple and clean:
SGD, as GD, “shouldn’t work, but does” in neural networks.
Precisely, SGD has its own edge, which is not set by $\lambda_{\max}$.
What saturates at the $2/\eta$ threshold is \emph{Batch Sharpness}---the expected directional curvature of the \emph{mini-batch} loss along its own stochastic gradient.
Equivalently, the meaningful instability comparator for SGD is the mini-batch landscape
seen by the step, not the full-batch landscape seen by the theorist.
In short, \textsc{EoSS} says SGD breaks the monotone-descent intuition---but with the “edge” measured on the mini-batch geometry that SGD actually experiences.
Even better, this condition, unlike some proposed previously, is not a $d\times d\times d\times d$ object \citep{ma_linear_2021,mulayoff_exact_2024}---which is theoretically sharp but operationally opaque.

\subsection{Failure of Arguments about Minima}
\label{sec:location_distributional}

Following claimes by \citet{hochreiter_simplifying_1994,hochreiter_flat_1997} first, and \citet{keskar_large-batch_2016,jastrzebski_three_2018} later, a large body of theory asked a location question: \emph{where does SGD end up, and why?}
There, intuition that the noise enables escaping local minima and traveling towards flatter areas has been established in a number of settings---all of them in the stable regime.
In most of the implicit-bias literature, indeed, one assumes (or proves) that after some time the iterates behave like descent on a fixed landscape and converge to stationary points, so it is mathematically consistent to attribute ``selection'' to \emph{gradient noise} on top of a deterministic drift
(e.g., \citep{beneventano_how_2024,smith_origin_2021,blanc_implicit_2019,haochen_shape_2020,damian_label_2021,li_what_2022,beneventano_trajectories_2023,shalova_singular-limit_2024}.).

\textsc{EoSS} forces an uncomfortable reframing. 
In the progressive-sharpening regime, not only SGD is not merely “GD + noise” exploring a fixed landscape, but location of convergence needs to be changed to area of stabilization.
Most importantly, the search for \textit{the} stationary point is ill-defined, the area of stabilization is often not stationary \cite{ahn_understanding_2022,damian_self-stabilization_2023}. We show here that this is the case also for mini-batch algorithms. 
Moreover, under \emph{Progressive Sharpening}, SGD does not interact with a single curvature object: each step sees a different mini-batch Hessian, and stability is controlled by the \emph{mini-batch geometry that SGD actually experiences}.
Consequently, location is not a property of the full-batch Hessian alone, nor of a covariance of gradient noise on a static landscape:
it can depend qualitatively on the \emph{distribution} of mini-batch Hessians---including higher cumulants.
\begin{summarybox}
\begin{itemize}[leftmargin=1.25em,itemsep=0.15em,topsep=0.15em,parsep=0pt]
    \item \textbf{It is not about minima:} studying the properties of stationary points and minima is misleading. Mini-batch algorithms may stabilize in regions without any of them, as for GD \citep{damian_self-stabilization_2023}.
    \item \textbf{Location becomes distributional:} in the progressive-sharpening regime, stabilization can depend on higher-order statistics of $\{\mathcal{H}(L_{\mathbf{B}})\}_{\mathbf{B}}$, not just on $\mathbb{E}\,\mathcal{H}(L_{\mathbf{B}})=\mathcal{H}$.
    \item \textbf{``Flatness'' becomes stability-grounded:} what SGD must control in \textsc{EoSS} is \emph{Batch Sharpness} hovering near $2/\eta$, so any link between generalization and curvature proxies (trace/Fisher/$\lambda_{\max}$, etc.) must factor through how those proxies relate to \emph{Batch Sharpness} along the trajectory.
\end{itemize}
\end{summarybox}

\begin{wrapfigure}{r}{0.48\columnwidth}
\vspace{-0.2cm}
    \centering
    \includegraphics[width=\linewidth]{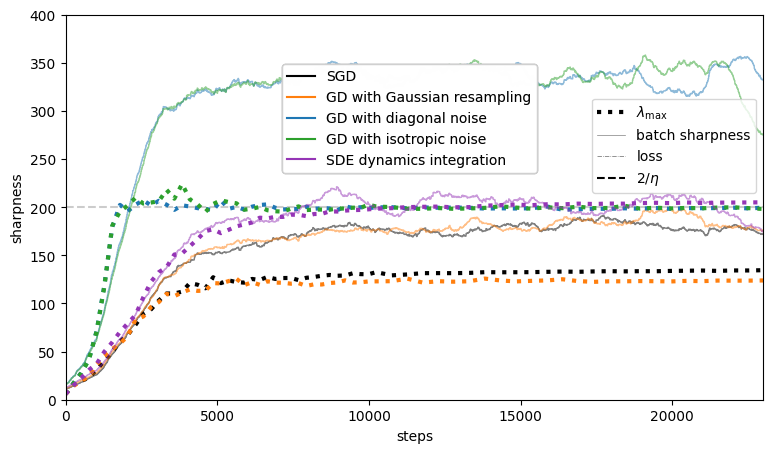}
    \caption{\small
    \textbf{SGD vs.\ Noisy GD vs.\ SDE.}
    Only Gaussian resampling (which preserves the mini-batch sampling structure) reproduces the \textsc{EoSS}-like suppression of $\lambda_{\max}$; diagonal/isotropic noise injection and SDE integration do not.
    }
    \label{fig:noise_gd}
    \vspace{-0.6cm}
\end{wrapfigure}

\paragraph{Noise injection as modeling.}
This distributional viewpoint also sharpens what it even means to ``model SGD by noise''.
Synthetic noise-injection schemes (and diffusion limits) are typically motivated as proxies for the \emph{gradient noise} of SGD.
But if the governing object is batch-dependent curvature, then a proxy that perturbs gradients while leaving the mini-batch curvature statistics essentially wrong is not a controlled approximation—it is a different dynamical system, and it may land in different regions.
Empirically (Figure~\ref{fig:noise_gd}; Appendix~\ref{appendix:noisy_gd_sde}), only noise that preserves the relevant mini-batch sampling structure (Gaussian resampling) reproduces the \textsc{EoSS}-like regime where $\lambda_{\max}$ plateaus well below $2/\eta$; more generic diagonal/isotropic injections and standard SDE integration fail to do so.

\paragraph{Implication for SDE modeling.}
\label{sec:challenges_sde}
Standard SDE-style accounts replace discrete mini-batch geometry by a diffusion with (full-batch) drift plus a noise term, thereby discarding higher-order information about $\{\mathcal{H}(L_{\mathbf{B}})\}_{\mathbf{B}}$.
Our results isolate a concrete obstruction in the progressive-sharpening regime: when batches are small, both eigenvalues \emph{and eigenspaces} of mini-batch Hessians can differ substantially from the full-batch ones, while \textsc{EoSS} is governed by this mini-batch geometry.
This is compatible with (and makes more operational) previously noted limitations of diffusion-based explanations of SGD implicit regularization
(e.g., ill-posedness \citep{yaida_fluctuation-dissipation_2018}, failure outside restrictive regimes \citep{li_validity_2021}, and qualitative mismatches in selected solutions \citep{haochen_shape_2020}).
Developing faithful continuous-time models for \textsc{EoSS} therefore seems to require explicitly modeling batch-dependent curvature beyond mean-field or covariance-only descriptions.


\section{Conclusions, Limitations, and Future Work} 
\label{section:conclusions}


\paragraph{Conclusions.}
We have addressed the longstanding question of \textit{if} and \textit{how} mini-batch SGD enters a regime reminiscent of the “Edge of Stability” previously observed in full-batch methods. Contrary to the usual focus on the global Hessian’s top eigenvalue, we uncovered that \emph{Batch Sharpness}---the expected directional curvature of the mini-batch landscape in the direction of its own gradient---consistently rises (progressive sharpening) and then hovers around $2/\eta$, independent of batch size. This behavior characterizes a new regime “Edge of Stochastic Stability”, which explains how mini-batch training can exhibit catapult-like surges and settle into flatter minima even when the full-batch Hessian remains below $2/\eta$. Our analysis clarifies why smaller batch sizes and larger step sizes both constrain the final curvature to a lower level, thereby linking these hyperparameters to flatter solutions and often improved generalization. Furthermore, we show that this phenomenon depends on the noise injected into the Hessians by mini-batch optimizers, highlighting important limitations of SDE-based approximations. Overall, the \textsc{EoSS} framework unifies several empirically observed effects---catapult phases, dependence on batch size, and progressive sharpening---under a single perspective focused on the \emph{mini-batch} landscape and its directional curvature.

\paragraph{Limitations.}
(\emph{i}) We have tested only image-classification tasks, leaving open whether similar phenomena arise in NLP, RL, or other domains.
(\emph{ii}) Our experiments mainly use fixed step sizes and standard architectures, so very large-scale or large-batch settings remain less explored.
(\emph{iii}) We have not analyzed momentum-based or adaptive methods (e.g.\ Adam), even though full-batch \textsc{EoS} has been seen there \citep{cohen_adaptive_2022}.
(\emph{iv}) Our experiments are run with MSE, a complete study of training with cross entropy is needed.

\paragraph{Future Work.}
Beyond addressing these limitations, several fundamental directions remain unexplored:
Understanding (\emph{i}) \textit{where} $\lambda_{\max}$ stabilizes; 
(\emph{ii}) how \textsc{EoSS} and \textsc{EoS} affect performances and the features learned by the neural network, e.g. \citep{lyu_understanding_2023,arora_understanding_2022,ahn_learning_2023,zhu_understanding_2023,wang_large_2022,beneventano_gradient_2025};
(\emph{iii}) consequently if it is benign effect or not;
(\emph{iv}) what the \textit{other} sources of instability are there in the (pre-)training; (\emph{v}) better describing the phenomenon of progressive sharpening and understanding its causes; (\emph{vi}) establishing the self-stabilization mechanism \citep{damian_self-stabilization_2023} for SGD.

\paragraph{Versions.}
In the version of June 2025, we added Theorem \ref{theo:1} and heavily reworked the text---not the experiments and the message.
In the version of December 2025, we added the discussion about the framework for stability in Section \ref{section:framework} and we reorganized the text. Note that the notion of \textit{Batch Sharpness} changed by pulling the expectation outside in the definition: importantly, in the experiments it was already outside in the previous versions. The disagreement between the quantity checked in the experiments and the one defined in the Definition \ref{def:minibs} of the previous versions was a mistake of the authors.

\paragraph{Acknowledgement.}
A special thanks to Stanis\l{}aw Jastrz\k{e}bski, Alex Damian, Jeremy Cohen, Afonso Bandeira, Boris Hanin, Renée Carmona, and Ziyin Liu for invaluable discussions, which were crucial to the development of this project. We also want to thank Mark Lowell, Yee Whye Teh, and the referees of NeurIPS 2025, for comments and feedbacks on the first version of the preprint.
We acknowledge the use of ChatGPT, DeepSeek and Claude for providing code assistance, debugging support, and editing suggestions.

\newpage



\bibliography{reference}
\bibliographystyle{unsrtnat}

\newpage
\appendix

\tableofcontents

\bigskip

\section{Comparison with Previous Empirical Work}
\label{appendix:comparison}

\label{appendix:comparison_with_empirical}
\begin{figure}[ht!]
    \centering
    \begin{subfigure}[b]{0.32\linewidth}
        \centering
        \includegraphics[width=\linewidth]{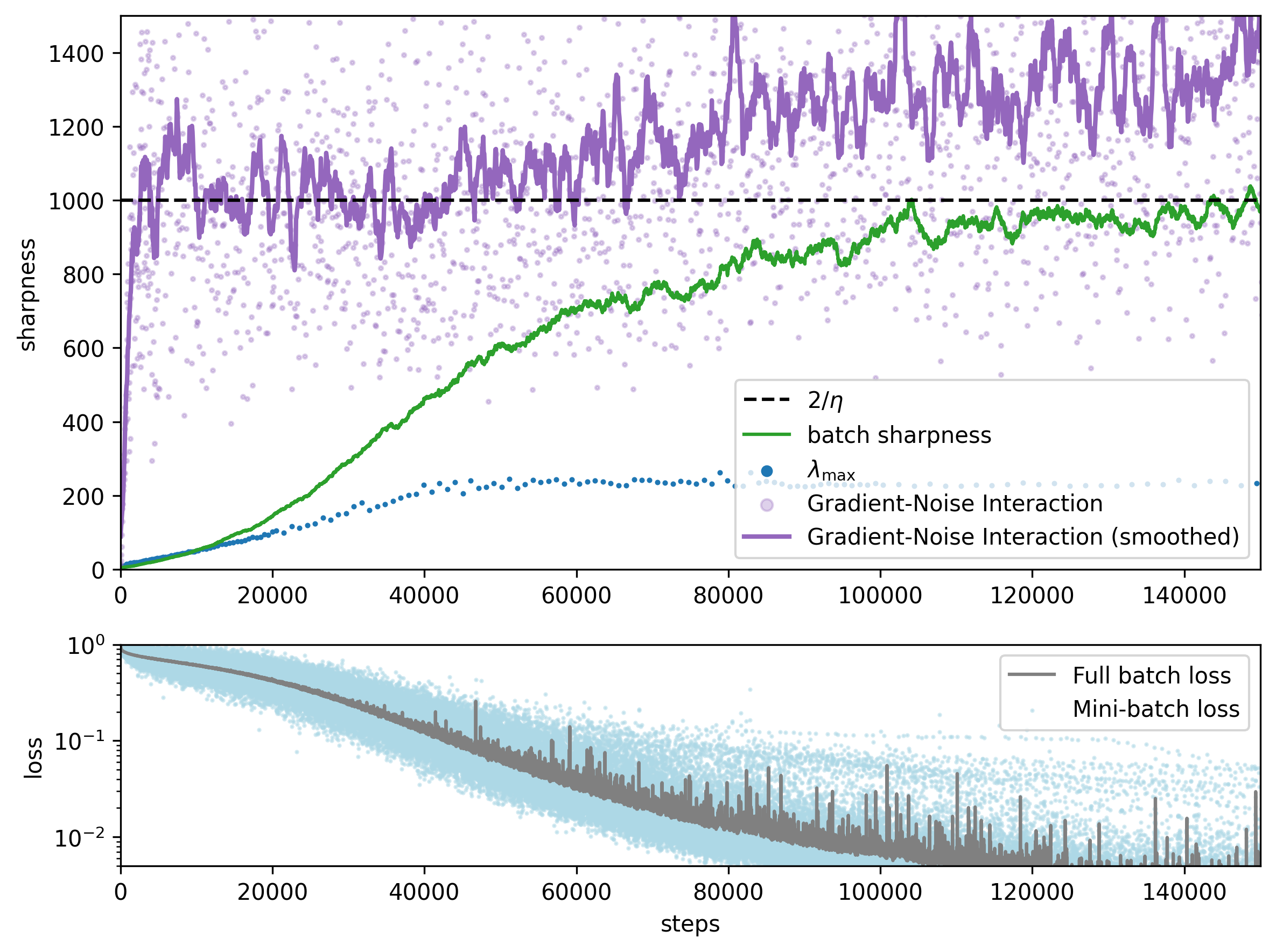}
        \caption{Constant step size}
        \label{fig:lr-const}
    \end{subfigure}
    \hfill
    \begin{subfigure}[b]{0.32\linewidth}
        \centering
        \includegraphics[width=\linewidth]{img/lee_jang_v_us/lr_early.png}
        \caption{Step size increased early}
        \label{fig:lr-early}
    \end{subfigure}
    \hfill
    \begin{subfigure}[b]{0.32\linewidth}
        \centering
        \includegraphics[width=\linewidth]{img/lee_jang_v_us/lr_late.png}
        \caption{Step size increased late}
        \label{fig:lr-late}
    \end{subfigure}

    
    \caption{We demonstrate that the saturation of \textit{GNI} does not govern a sharpness-related regime of instability typical of Type-2 oscillations - and in particular, highlighting the difference between the two types of oscillations. When we double the step size after \emph{batch sharpness} is at least half of $2/\eta$ threshold (so that it is beyond the new $2/\eta$ level), training exhibits a catapult surge in the loss (c). But if we make the same change \emph{before} \textit{Batch Sharpness} crosses that level—despite \textit{GNI} already saturating—no catapult occurs. (b)}
    \label{fig:lee-jang-comparison-lr}
\end{figure}

\begin{figure}[ht!]
    \centering
    \begin{subfigure}[b]{0.32\linewidth}
        \centering
        \includegraphics[width=\linewidth]{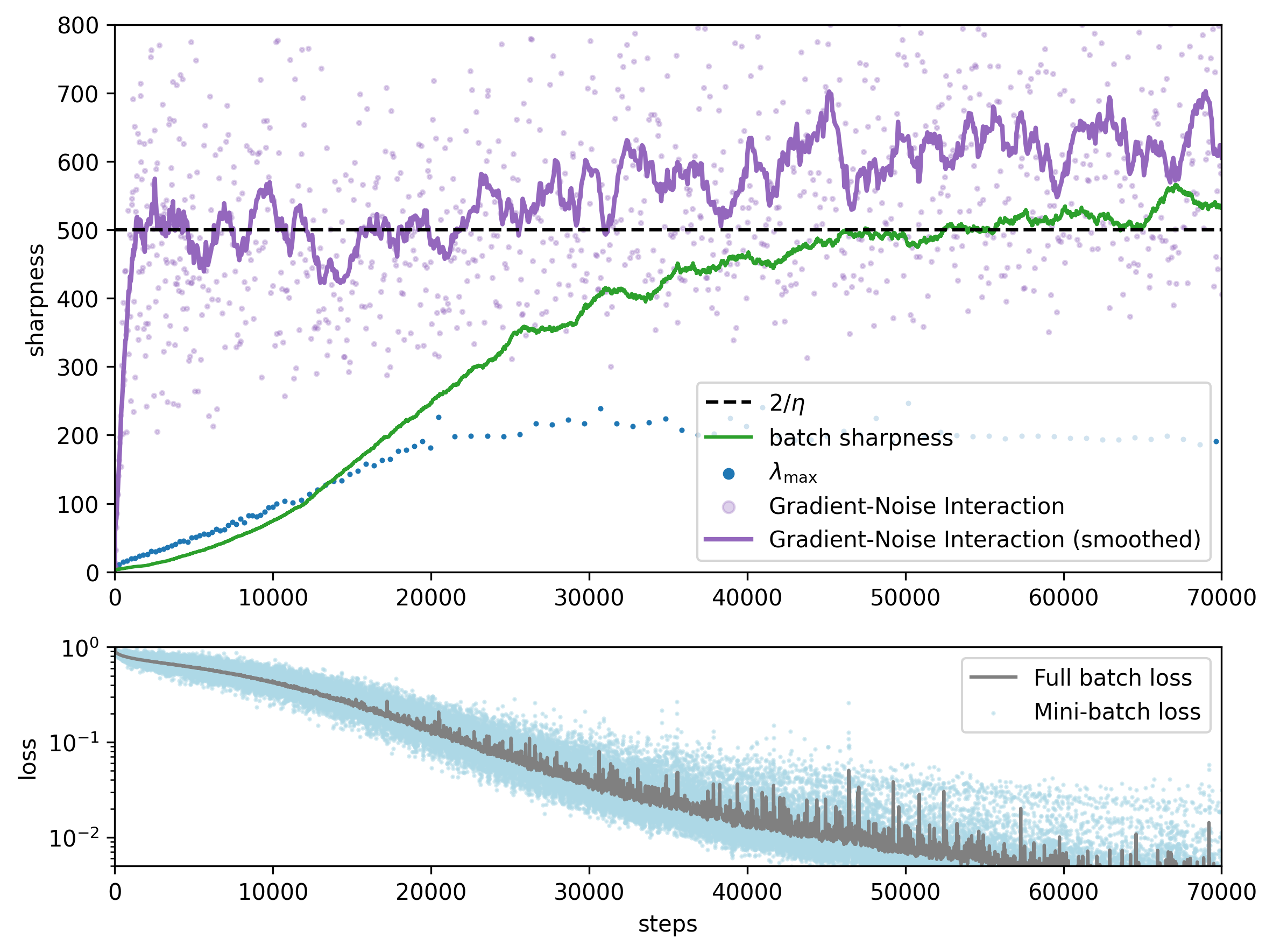}
        \caption{Constant batch size}
        \label{fig:lee-jang-const}
    \end{subfigure}
    \hfill
    \begin{subfigure}[b]{0.32\linewidth}
        \centering
        \includegraphics[width=\linewidth]{img/lee_jang_v_us/bs_early.png}
        \caption{Batch size decreased early}
        \label{fig:lee-jang-early}
    \end{subfigure}
    \hfill
    \begin{subfigure}[b]{0.32\linewidth}
        \centering
        \includegraphics[width=\linewidth]{img/lee_jang_v_us/bs_late.png}
        \caption{Batch size decreased late}
        \label{fig:lee-jang-late}
    \end{subfigure}

    
    \caption{Similarly, reducing the batch size only triggers catapults if batch sharpness, not \textit{GNI}, exceeds the threshold.}
    \label{fig:lee-jang-comparison-bs}
\end{figure}

As mentioned in Section \ref{section:two_types_of_oscillations}, the oscillations in SGD were documented throughout multiple works, starting with \citet{xing_walk_2018}. In particular, there, the authors demonstrate the presence of oscillations in SGD from very early phases of the training. Yet, they do not explain the origin of those oscillations. \citet{cohen_gradient_2021}

\citet{lee_new_2023} introduce several quantities crucial for understanding neural network training dynamics. Below, we discuss the relationships among $\lambda_{\max}$, \emph{Batch Sharpness}, and Interaction-Aware Sharpness (IAS, \citet{lee_new_2023}), emphasizing that a comprehensive theory of mini-batch dynamics should explain their distinct plateau timings and interconnected behaviors. We conjecture that a complete theory of stochastic gradient descent (SGD) dynamics would elucidate these metrics' precise interrelations and their different plateau timings.

\paragraph{Interaction-Aware Sharpness.}
\citet{lee_new_2023} introduce Interaction-Aware Sharpness (IAS), denoted $\|\mathcal{H}\|_{S_b}$:
\[
\|\mathcal{H}\|_{S_b} \quad := \quad \frac{\E_{B\sim\mathcal{P}_b} \big[ \nabla L_B(\x)^\top \mathcal{H} \, \nabla L_B(\x) \big]}{\E_{B \sim \mathcal{P}_b}\bigl[\|\nabla L_B\|^2\bigr]}.
\]
This quantity shares structural similarities with both \textit{Batch Sharpness} (Definition \ref{def:minibs}) and the \textit{Gradient-Noise Interaction} (Proposition \ref{prop:lee-jang}), differing from the latter only in the denominator. The key distinction from \textit{Batch Sharpness} lies in which Hessian is evaluated: IAS measures the directional curvature of the \textbf{full-batch} loss landscape $L$ along mini-batch gradient directions, while \textit{Batch Sharpness} measures the directional curvature of \textbf{mini-batch} loss landscape $L_B$ along their corresponding gradients. This distinction is crucial, as mini-batch Hessians vary with batch selection while the full-batch Hessian remains fixed.

Notably, with full-batch GD, IAS serves as a directional alternative to the maximal Hessian eigenvalue, $\lambda_{\max}$, introduced by \citet{cohen_gradient_2021}. 
IAS aligns closely with the $2/\eta$ threshold, unlike $\lambda_{\max}$, which often remains slightly above this threshold during \eos, especially at the beginning of it.
Since IAS measures \textit{directional} curvature, we have $\|\mathcal{H}\|_{S_n} \leq \lambda_{\max}$. Consequently, in the mini-batch setting, IAS stabilizes below $2/\eta$, consistent with empirical observations from \citet{jastrzebski_relation_2019,jastrzebski_break-even_2020, cohen_gradient_2021} and our Figure \ref{fig:1}. Notably, when $B=n$, our \textit{Batch Sharpness} coincides with IAS rather than $\lambda{\max}$, reinforcing the interpretation of \textit{Batch Sharpness} as the relevant metric stabilizing at $2/\eta$ even under full-batch conditions.

\paragraph{Relation to Gradient-Noise Interaction.}
Another metric from \citet{lee_new_2023} is defined as:
\[
\frac{\tr(H S_b)}{\tr(S_n)} = \frac{\E_{B\sim\mathcal{P}_b} \big[ \nabla L_B(\x)^\top \mathcal{H} \, \nabla L_B(\x) \big]}{\mednorm{\nabla L}^2} 
\]
which coincides exactly with our definition of GNI (Proposition \ref{prop:lee-jang}). 
As detailed in Section \ref{section:two_types_of_oscillations} and Appendix \ref{appendix:oscillations_nn}, the stabilization of GNI around $2/\eta$ signals the presence of oscillations, at least Type-1 oscillations.
\citet{lee_new_2023} provide extensive empirical evidence demonstrating that neural networks spend much of their training within this oscillatory regime (see also Figures \ref{fig:lee-jang-comparison-lr}a and \ref{fig:lee-jang-comparison-bs}a). This contrasts traditional theoretical analyses (\citet{bottou_optimization_2018, mandt_variational_2016}), which consider oscillations only near the manifold of minima.

\paragraph{Distinguishing oscillation types.} 
It is crucial to note that GNI around $2/\eta$ does not inherently indicate instability.
As clarified in  Sections \ref{section:two_types_of_oscillations}, \ref{section:eoss} and Appendix \ref{appendix:oscillations_nn}, not all oscillations are inherently unstable. 
Figures \ref{fig:lee-jang-comparison_body}, \ref{fig:lee-jang-comparison-lr}b, \ref{fig:lee-jang-comparison-bs}b illustrate that altering hyperparameters when GNI is around $2/\eta$ typically does not trigger instability (catapult-like divergence), contrary to expectations if the system was in an \eos-like regime of instability.
Instead, as shown in Figures \ref{fig:lee-jang-comparison_body}, \ref{fig:lee-jang-comparison-lr}c, \ref{fig:lee-jang-comparison-bs}c, \textit{Batch Sharpness} more reliably predicts a regime of instability.
Additionally, Figure \ref{fig:gni_cifar_ez} highlights GNI's independence from progressive sharpening, a necessary precursor to Type-2 (curvature-driven) oscillations and \eos-like instabilities, as detailed in Appendix \ref{appendix:oscillations_nn}.

\paragraph{Missing Progressive Sharpening.}
Extensively, both in our experiments and in the ones of \citet{lee_new_2023}, GNI grows to $2/\eta$ in a few initial steps (and sometimes from the very beginning if the initialization size is large) without ever being in subject to a phase of progressive sharpening unlike \textit{Batch Sharpness} and $\lambda_{\max}$. The phase of growth of GNI is generally short and independent of the size, the behavior, and the phase in which \textit{Batch Sharpness} and $\lambda_{\max}$ are.

\begin{figure}[ht!]
    \centering
    \includegraphics[width=0.5\linewidth]{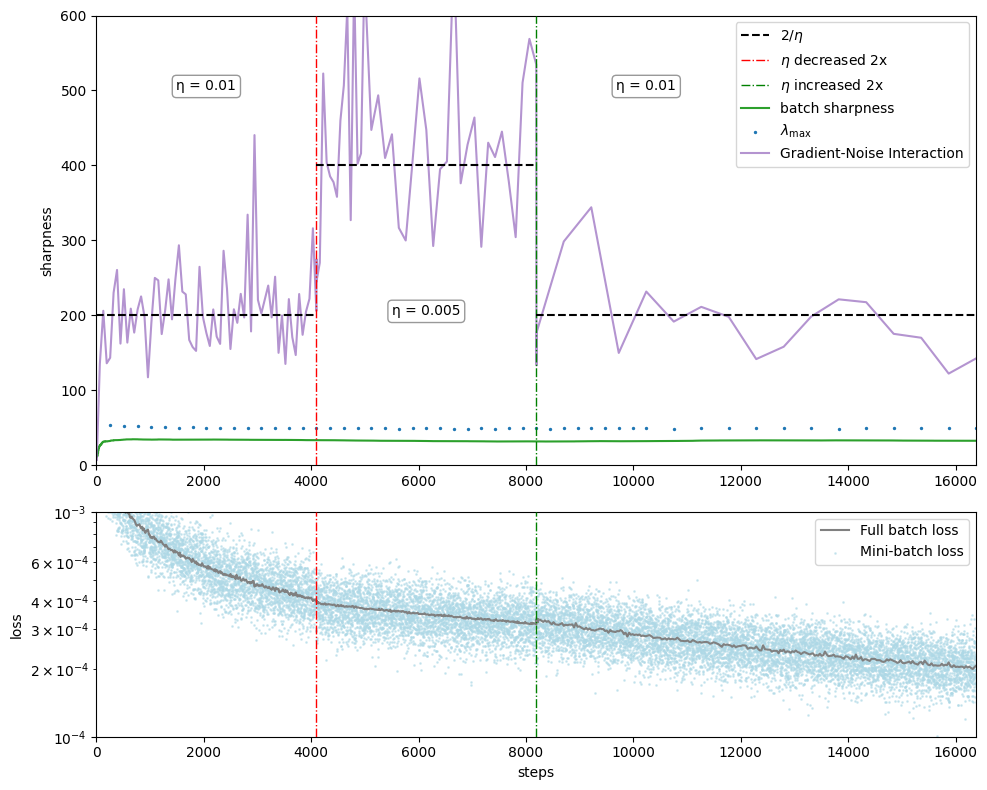}
    \caption{We construct a 32k-point "easy" CIFAR-10, where we "pull apart" all the 10 classes, so the classes become linearly separable. In this case, there is virtually no "learning" to be done, and therefore, there is barely any progressive sharpening happening (as established \citet{cohen_gradient_2021}, progressive sharpening does not happen if the dataset "is not complex enough"). Yet, \textit{GNI} still stabilizes at the initial level of $2/\eta$. More importantly, when we decrease and then increase the step size, the \textit{GNI} measure restabilizes to the corresponding new thresholds, while $\lambda_{\max}$ does not change. That means that \textit{GNI} is \textit{independent of the curvature of the loss landscape} and is unrelated to progressive sharpening, and thus Type-2 oscillations and \eos-like instability regimes.
    }
    \label{fig:gni_cifar_ez}
\end{figure}

\clearpage

\section{On the Two Types of Oscillations in SGD Dynamics}
\label{appendix:oscillations_nn}
A fundamental challenge in analyzing SGD compared to GD stems from the inherent oscillations induced by mini-batch gradient noise. This appendix, together with Appendix \ref{appendix:comparison} (also see proofs in \ref{app:proof_gni} and \ref{appendix:proof_theo}), extends the discussion in Section \ref{section:two_types_of_oscillations} by formally distinguishing between two distinct types of oscillations: noise-driven (Type-1) and curvature-driven (Type-2). This distinction is crucial because Type-1 oscillations occur independently of the loss landscape's curvature and thus do not exert a regularizing effect on the sharpness of the final solution. In contrast, Type-2 oscillations are directly caused by landscape curvature and induce an implicit regularization effect by discouraging convergence towards sharp minima.

We begin with a minimalistic example to illustrate the nature of Type-1
\subsection{A Minimalistic Quadratic Example.} 
\label{appendix:type_1_one_dim}
Here we show mathematically what we see empirically in Figure \ref{fig:quadratic} (the simplified version---only two data points).
Consider a regression problem with two datapoints, $1$ and $-1$, and a linear model $f(x)=x$ under the quadratic loss. The (scaled) full-batch loss is given by:
\[
L(x)\;=\;\frac{1}{4}(x-1)^2\;+\;\frac{1}{4}(x+1)^2.
\] 
Batch-1 SGD updates with step-size $0<\eta<2$ result in oscillatory behavior around the optimum $x=0$ due entirely to gradient noise, with amplitude approximately $\sqrt{\frac{\eta}{2-\eta}}$. Crucially, the Hessian in this example is small ($\frac{d^2L}{dx^2}=1$), demonstrating that these persistent oscillations are entirely noise-driven (Type-1). 

Formally, the SGD update is:
\[
x_{t+1} = x_t - \eta\nabla \ell_{i_t} = (1-\eta)x_t + \eta\xi_{t}
\]
where $l_{i_t}$s are the individual datapoint losses, and $\xi_t$s are i.i.d Rademacher random variables. Thus, we obtain the first two moments explicitly:
\[
\E[x_{t}] = (1-\eta)\E[x_{t-1}] = (1-\eta)^t x_0 
\]
\[
\E[x_t^2] = (1-\eta)^2 \E[x_{t-1}^2] + \eta^2 = (1-\eta)^{2t}x_0 + \frac{\eta^{2}}{1-(1-\eta)^{2}}\Bigl(1-(1-\eta)^{2t}\Bigr)
\]
This implies convergence in expectation for $0<\eta<2$, with a limiting variance given by:
\[
\lim_{t\to\infty}\mathbb{E}[x_t^{2}] = \frac{\eta}{2-\eta}
\]
and divergence for $\eta > 2$.

A key observation is that increasing $\eta$ to any value $\eta_1 < 2$ merely changes the amplitude of oscillations to $\sqrt{\frac{\eta_1}{2-\eta_1}}$ without triggering any catapult-like behavior. The only step size for which we observe Type-2 (curvature-driven) oscillations and an \eos-like\footnote{The key difference between these oscillations and genuine \textsc{EoS} behavior in neural networks is that, in the quadratic case, the full-batch loss does not decrease, making this scenario inherently less informative. In contrast, neural networks exhibit a surprising, albeit non-monotonic, decrease in loss within this instability regime, an effect arising from the multidimensional nature of their optimization landscape \citep{damian_self-stabilization_2023}} instability is precisely $\eta=2$, where the dynamics effectively become a random walk, and any larger step size leads to divergence.



\begin{mdframed}
\centering
    \textit{Crucially, when $\eta<2$ oscillations occur persistently on the full-batch loss, despite the individual steps on the mini-batch loss remaining \textit{stable}. }
\end{mdframed}

The oscillation is due to the fact that the mini-batch loss landscape shifts from step to step, not to the fact that the steps are unstable.

\subsection{Proof of Lemma \ref{lem:GNI}}
\label{appendix:proof_lem_GNI}
We propose here the formal version of Lemma \ref{lem:GNI}.

\begin{proposition}[Loss increment and GNI]
\label{prop:GNI_formal}
Assume $L$ is three times continuously differentiable and its Hessian is
$L_2$–Lipschitz in a neighborhood of $\theta$.
Then for any mini–batch $B$ and step size $\eta>0$ small enough we have
\begin{equation}
\label{eq:descent-lemma-GNI}
\E_{B}\bigl[L(\theta - \eta \nabla L_B(\theta)) - L(\theta)\,\big|\,\theta\bigr]
= -\eta \|\nabla L(\theta)\|^2
+ \frac{\eta^2}{2}\,\E_B\bigl[\nabla L_B(\theta)^\top \mathcal{H}(\theta)\,\nabla L_B(\theta)\bigr]
+ \mathcal{O}(\eta^3),
\end{equation}
where the $\mathcal{O}(\eta^3)$ constant depends only on $L_2$ and an upper
bound on $\|\nabla L_B(\theta)\|$.
Equivalently,
\begin{equation}
\label{eq:GNI-sign}
\E_{B}\bigl[L(\theta_{t+1}) - L(\theta_t)\,\big|\,\theta_t\bigr]
= -\eta \|\nabla L(\theta_t)\|^2
\Bigl(1 - \frac{\eta}{2}\,\mathrm{GNI}(\theta_t) + \mathcal{O}(\eta^2)\Bigr),
\end{equation}
with
\[
\mathrm{GNI}(\theta)
:= \frac{\E_{B}[\nabla L_B(\theta)^\top \mathcal{H}(\theta)\,\nabla L_B(\theta)]}
        {\|\nabla L(\theta)\|^2}.
\]
In particular, there exists $c>0$ such that for all sufficiently small $\eta$:
if
\[
\bigl|\mathrm{GNI}(\theta_t) - 2/\eta \bigr| \ge c\,\eta,
\]
then the sign of the expected loss increment satisfies
\[
\operatorname{sign}\,\E[L(\theta_{t+1}) - L(\theta_t)\mid\theta_t]
= \operatorname{sign}\,\bigl(2/\eta - \mathrm{GNI}(\theta_t)\bigr).
\]    
\end{proposition}

\begin{proof}[Proof of Proposition \ref{prop:GNI_formal}]
Fix $\theta$ and a mini–batch $B$.
Write a third–order Taylor expansion of $L$ around $\theta$:
\[
L(\theta - \eta \nabla L_B(\theta))
= L(\theta) - \eta \nabla L(\theta)^\top \nabla L_B(\theta)
+ \frac{\eta^2}{2}\,\nabla L_B(\theta)^\top \mathcal{H}(\theta)\,\nabla L_B(\theta)
+ R(\eta,\nabla L_B(\theta)),
\]
where $\nabla L_B(\theta) := \nabla L_B(\theta)$ and
$\mathcal{H}(\theta)$ is the full–batch Hessian.
The third–order remainder satisfies the standard bound
$\vert R(\eta,\nabla L_B(\theta))\vert \le C L_2 \eta^3 \|\nabla L_B(\theta)\|^3$ for some numerical $C$,
because the Hessian is $L_2$–Lipschitz.

Now take the expectation over $B\sim\mathcal{P}_b$.
Using $\E_B[\nabla L_B(\theta)] = \nabla L(\theta)$ we obtain
\[
\E_B\bigl[L(\theta - \eta \nabla L_B(\theta)) - L(\theta) \mid \theta\bigr]
=
-\eta \|\nabla L(\theta)\|^2
+ \frac{\eta^2}{2}\,\E_B\bigl[\nabla L_B(\theta)^\top \mathcal{H}(\theta)\,\nabla L_B(\theta)\bigr]
+ \mathcal{O}(\eta^3),
\]
which is \eqref{eq:descent-lemma-GNI}.

Divide the right–hand side by $-\eta \|\nabla L(\theta)\|^2$ to get
\[
\E_B[L(\theta_{t+1}) - L(\theta_t) \mid \theta_t]
= -\eta \|\nabla L(\theta_t)\|^2
\left( 1 - \frac{\eta}{2}\,\mathrm{GNI}(\theta_t)
+ \mathcal{O}(\eta^2)\right),
\]
since
\[
\mathrm{GNI}(\theta_t)
=
\frac{\E_B[\nabla L_B(\theta)^\top \mathcal{H}(\theta_t)\,\nabla L_B(\theta)]}
{\|\nabla L(\theta_t)\|^2}.
\]

The last claim (sign agreement) follows immediately:
for sufficiently small~$\eta$, the $\mathcal{O}(\eta^2)$ term is dominated
whenever $|\mathrm{GNI}-2/\eta|\ge c\eta$ for a fixed constant $c>0$.
\end{proof}

\subsection{General Case: From One-Dimensional Toy to Multidimensional Lyapunov Analysis}
\label{subsec:general_case_bridge}

The simple one-dimensional regression in \S\ref{appendix:oscillations_nn}.1 already
demonstrates how \emph{noise-driven (Type-1) oscillations} can persist
indefinitely and yield a stable “two-cycle” around the optimum, independent
of the actual Hessian magnitude. In higher dimensions, the story is similar:
when the step size \(\eta\) is fixed, the \emph{randomness} in mini-batch
gradients still injects a continual “kick” at each iteration, causing the
iterates to hover in a noisy neighborhood of the minimum. The main difference
is that now there can be many directions---some with higher curvature than
others, or even flat (\(\lambda=0\)) directions. Nonetheless, the essential
mechanism remains:
\[
  \Delta_{t+1}
  \;=\;
  \bigl(I-\eta\,\mathcal{H}\bigr)\;\Delta_t
  \;-\;\eta\,\xi_t,
\]
where \(\Delta_t=x_t-x^\star\) is the displacement from the optimum,
\(\mathcal{H}\) is the Hessian at \(x^\star\), and \(\xi_t\) encodes the
random fluctuation (gradient noise). Once \(\Delta_t\) settles into a
\emph{stationary distribution}, the covariance \(\Sigma_x\) can be found by
solving a discrete Lyapunov equation similar to the one-dimensional case.

\paragraph{Key References.}
A number of works formalize this “SGD noise equilibrium” by viewing the
updates as a linear Markov chain in a neighborhood of \(x^\star\).  
Classical references include \citet{mandt_variational_2016} for the
stochastic differential analogy (Ornstein–Uhlenbeck process), and
\citet{bottou_optimization_2018} for a thorough discussion of how the
constant stepsize prevents exact convergence. 
Intuitively, the argument for Proposition \ref{prop:lee-jang} goes thus as follows:
\begin{enumerate}[leftmargin=2em,itemsep=0.2em,parsep=0.1em]
    \item The iterates oscillate with a stationary covariance $\Sigma_x$ around $x^*$.
    \item Full-batch (expected) gradient is zero at $x^*$ and grows roughly linearly with distance for small deviations (by Taylor expansion $\nabla L(x) \approx \mathcal{H}(x - x^*)$). So on average \textit{over the iterations} we have
    \[ 
    \E_k[\mednorm{\nabla L(x_k)}^2] = \mathrm{Tr}\big(\mathcal{H} \Sigma_x \mathcal{H}).
    \]
    \item The stationary covariance of the gradients is $\Sigma_g$ satisfies: 
    \[
    \Sigma_x \approx \frac{\eta}{2} (\mathcal{H}^{-1} \Sigma_g).
    \]
\end{enumerate}
Putting all together, this implies that $\Sigma_g$ of the gradients is about $\frac{2}{\eta}\mathcal{H}^{-1}$ bigger than the full-batch $\nabla L \nabla L^\top$. Precisely the following quantity (where $L_i$ is the loss on the $i-th$ data point):
\begin{equation}
    \frac{\E_i \big[ \nabla L_i(\x)^\top \mathcal{H} \nabla L_i(\x) \big]}{\mednorm{\nabla L}^2}
    \,=\, \frac{\mathrm{Tr}\big( \mathcal{H} \Sigma_g \big)}{\mathrm{Tr}\big(\mathcal{H} \Sigma_x \mathcal{H})},
\end{equation}
and this can be rewritten as
\begin{equation}
\label{eq:type-1}
    \textit{Gradient-Noise Interaction (GNI)}
    \, = \,
    \frac{\E_i \big[ \nabla L_i(\x)^\top \mathcal{H} \nabla L_i(\x) \big]}{\mednorm{\nabla L}^2}
    \, = \,
    \frac{2}{\eta} \cdot \underbrace{\frac{\mathrm{Tr}\big( \mathcal{H} \Sigma_x \mathcal{H} \big)}{\mathrm{Tr}\big(\mathcal{H} \Sigma_x \mathcal{H})}}_{=1}
    \,=\,
    \frac{2}{\eta}.
\end{equation}
This $2/\eta$ thus comes out of the only fact of oscillating and it is \textit{unrelated} to the Hessian value. Moreover, \textsc{EoS} happens only in the eigenspace of the highest eigenvalue, \textit{Type-1} noise on the whole subspace spanned by the eigenspaces of the \textit{non-negative} eigenvalues.

\paragraph{On Stability}
For this reason, linear stability analyses of stochastic gradient descent and noise-injected gradient descent on quadratic objectives--originally explored by \citet{wu_how_2018} and further developed by \citet{ma_linear_2021, wu_sgd_alignment_2022, mulayoff_exact_2024}---explicitly exclude \textit{Type-1} oscillations by categorizing them as \textit{stable}. Specifically, \citet{ma_linear_2021} restrict their analysis to interpolating minima, where all individual gradients vanish, thus effectively eliminating noise-driven oscillations and isolating the curvature-driven (\textit{Type-2}) scenario. \citet{mulayoff_exact_2024} extends this to a more general class of minima by restricting the analysis to the orthogonal complement of the null space of the Hessian, and demonstrating that the noise-driven oscillations do not affect stability. 
\citet{lee_new_2023} empirically established that Gradient-Noise Interaction (\textit{GNI}) consistently remains around $2/\eta$ throughout training. From the above, this implies that most training occurs in an oscillatory regime (at least \textit{Type-1})---see Figure \ref{fig:lee-jang-comparison-bs} and Appendix \ref{appendix:comparison_with_empirical}. In contrast, our study specifically investigates the emergence and implications of \textit{Type-2} oscillations, given their significant role in implicitly regularizing the loss landscape.

\paragraph{From Informal to Formal.}
In the next Appendix \ref{app:proof_gni}, we present a general
\emph{discrete Lyapunov} proof of Lemma~\ref{prop:lee-jang}, allowing
\(\mathcal{H}\succeq0\) to be possibly degenerate and not necessarily
commuting with the noise covariance.  The result is summarized in
Proposition~\ref{prop:general_GNI}, showing rigorously that
\(\mathrm{GNI}\approx2/\eta\) arises under \emph{any} stable constant-stepsize
mini-batch SGD orbit.  This “\(\tfrac{2}{\eta}\)-law” is precisely the
high-dimensional extension of the toy one-dimensional phenomenon above.
\begin{mdframed}
\textit{
    In particular, we show here that the appearance of some quantity being $2/\eta$ means that the system is oscillating but does not mean in principle that the landscape or the curvature \textbf{adapted} to the hyper parameters.}
\end{mdframed}

\subsection{On the Importance of Type-2 Oscillations Compared to Type-1}
Noise-induced (Type-1) oscillations are not unstable when introducing slight perturbations (increase step size or decrease batch size), as showcased in Figure \ref{fig:lee-jang-comparison-lr} and \ref{fig:lee-jang-comparison-bs}. Therefore, they do not constitute an \textsc{EoS}-type phenomena, where slight perturbations do cause divergence (”complete” divergence as long as we consider just the quadratic terms and can ignore higher terms — the fact that it does not fully diverge is exactly the higher-terms effect). Instead, after a perturbation, noise-induced oscillations quickly re-stabilize at a higher level. 

Crucially, a lack of such divergence means that noise-induced oscillations would not exhibit the self-stabilization mechanism of \citet{damian_self-stabilization_2023} characteristic of \textsc{EoS} (differing it from classical convex optimization). Moreover, as shown in the quadratic example and in the proofs, noise-induced oscillations happen for any quadratic, for a wide range of step sizes, making them inherently “unsurprising”, while \textsc{EoS} is a beyond-quadratic phenomena (and, as far as we know, a deep-learning-specific phenomena), as it relies on both progressive sharpening and the aforementioned self-stabilization, both being an effects of higher order terms. And the reason why we care specifically about effects of beyond-quadratic terms is specifically the adaptation of the landscape to the hyper-parameters, which is, by definition, an effect of higher order terms. That is the reason we specifically care about curvature-driven oscillations. 

Now, with all of the above, GNI, being an indicator of those noise-induced oscillations, is therefore not an indicator of \textsc{EoS}-like regime. This is despite the fact that GNI in SGD comes from the same place as $\lambda_{\max}$/Rayleigh quotient in GD — i.e. from the descent lemma; yet, it does not mean that the two quantities serve the same role. Instead, it is the presence of the natural noise in SGD that makes the analysis much more complex. Instead, GNI has its usefulness as a measure of the level of noise coming from SGD. That is, noise-induced oscillations are influenced by the Hessian, but are also strongly influenced by the ratio between the noise covariance and the norm of full batch gradient, with the latter being the leading cause of change. In particular, GNI is decoupled from the Hessian, and can change drastically without any change of landscape sharpness, as showcased in our experiments. Lastly, another important consequence of \textsc{EoS} is that the landscape adapts to the hyper-parameters (rather than the other way around in classical optimization). With GNI being decoupled from the Hessian, GNI being at 2/eta is not an indication of landscape adopting to the hyper-parameters, as is the case with $\lambda_{\max}$ being at $2/\eta$ during GD.

\clearpage

\section{Proof of Proposition \ref{prop:lee-jang}}


\label{app:proof_gni}

\subsection{Setup and notation for Proposition~\ref{prop:lee-jang}}

Let
\[
  L(\theta)\;=\;\frac{1}{n}\sum_{i=1}^n \ell_i(\theta)
\]
be three-times continuously differentiable, and let $\theta^\star$ be a (possibly non‑isolated) local minimiser.
Denote the full‑batch Hessian at $\theta^\star$ by
\[
  \mathcal{H} \;:=\; \nabla^2 L(\theta^\star)\succeq 0.
\]

For each sample \(i\), define its (local) Hessian at $\theta^\star$,
\[
  \mathcal{H}_i \;:=\; \nabla^2 \ell_i(\theta^\star),
\]
so that $\mathcal{H} = \tfrac{1}{n}\sum_{i=1}^n \mathcal{H}_i$.
We also define the (single‑sample) gradient noise covariance at $\theta^\star$ as
\[
  \Sigma_g \;:=\; \E_i\bigl[\nabla \ell_i(\theta^\star)\,\nabla \ell_i(\theta^\star)^\top\bigr].
\]

We decompose the parameter space as
\[
  \R^d \;=\; E_+ \oplus E_0,
  \qquad
  E_+ := \operatorname{Im}(\mathcal{H}),\quad
  E_0 := \ker(\mathcal{H}),
\]
with associated orthogonal projectors $P_+$ and $P_0$.
We will only require control of the dynamics in $E_+$ and assume that
gradient noise in the flat subspace $E_0$ is not too large.

For each iteration $t$, a mini‑batch $B_t$ of size $b$ is drawn
(with or without replacement) and the SGD update is
\[
  \theta_{t+1}
  \;=\;
  \theta_t \;-\;\eta\,\nabla L_{B_t}(\theta_t),
  \qquad
  \nabla L_{B_t}(\theta)
  := \frac{1}{b}\sum_{i\in B_t} \nabla \ell_i(\theta).
\]

Finally, define the Kronecker–sum operator
\[
  \mathcal{K} : \R^{d\times d} \to \R^{d\times d},
  \qquad
  \mathcal{K}(X) := \mathcal{H} X + X\mathcal{H}.
\]
On $E_+\otimes E_+$, $\mathcal{K}$ is positive definite and has a Moore–Penrose
pseudoinverse $\mathcal{K}^\dagger$.

We work under the following assumptions in a neighbourhood of $\theta^\star$.

\begin{enumerate}[label=\textbf{(A\arabic*)}, leftmargin=3em]
\item \textbf{Local quadratic approximation.}\;
  Each $\ell_i$ is twice differentiable with $L_2$–Lipschitz Hessian near $\theta^\star$,
  and admits the Taylor expansion
  \[
    \nabla \ell_i(\theta)
    \;=\;
    \nabla \ell_i(\theta^\star)
    \;+\;
    \mathcal{H}_i\,(\theta - \theta^\star)
    \;+\;
    R_i(\theta),
  \]
  where $\|R_i(\theta)\| = \mathcal{O}(\|\theta-\theta^\star\|^2)$ uniformly in $i$.

\item \textbf{Compatible noise in flat directions.}\;
  The gradient noise covariance in the flat subspace is small:
  \[
    \|P_0\,\Sigma_g\,P_0\| \;\lesssim\; \eta,
  \]
  so that the iterates do not perform an unbounded random walk along $\ker(\mathcal{H})$.

\item \textbf{Linear stability of the linear dynamics.}\;
  The SGD on the quadratic approximation is linearly stable on $E_+$, i.e.
  \[
    \rho\E\bigl[(I - \eta\,\mathcal{H}(L_B))^{\otimes 2}\bigr]_{|E_+} \;<\; 1.
  \]
\end{enumerate}

\begin{remark}[Remarks on the assumptions]\quad
\begin{description}[leftmargin=1.6em,font=\normalfont\itshape]

\item[Exact vs.\ Lipschitz Hessian (on (A1))]\;
When each $\ell_i$ is strictly quadratic, the local linearity
\[
  \nabla \ell_i(x)
  =
  \nabla \ell_i(x^\star)
  + \mathcal{H}_i (x-x^\star),
  \qquad
  \mathcal{H}_i := \nabla^2 \ell_i(x^\star),
\]
holds exactly, and $\mathcal{H} = \frac1n\sum_i \mathcal{H}_i$.
In the general case, if $\nabla^2 \ell_i$ is $L_2$–Lipschitz in a
neighborhood of $x^\star$, a second–order Taylor expansion gives a remainder
$\mathcal{O}(\|x-x^\star\|^2)$. For sufficiently small~$\eta$, the SGD
iterates typically remain in an $\mathcal{O}(\sqrt{\eta})$–neighborhood of
$x^\star$, so these higher–order terms contribute only $\mathcal{O}(\eta^2)$
corrections in the discrete Lyapunov equation, which are dominated by the
main $\mathcal{O}(\eta)$ term in Proposition~\ref{prop:general_GNI}.

\item[Small drift in flat directions (on (A2))]\;
The requirement $P_0 \Sigma_g P_0 = 0$ can be relaxed to
$\|P_0 \Sigma_g P_0\| \le \delta$.
A standard discrete–Lyapunov analysis shows that the stationary covariance
$\Sigma_x$ remains finite provided $\delta = \mathcal{O}(\eta)$: roughly,
if $\|P_0 \Sigma_g P_0\|$ is at most a constant multiple of
$\eta$ times the curvature scale on $E_+$, then the null–space covariance
$\Sigma_x^{00}$ grows at most on the same $\mathcal{O}(\eta)$ scale as the
covariance on $\operatorname{Im}(\mathcal{H})$.
If instead $P_0 \Sigma_g P_0$ is large, the dynamics executes an
(uncontrolled) random walk along $E_0$, and no finite stationary covariance
exists in those directions.

\item[Linear stability and spectral gap (on (A3))]\;
The condition
\[
  \rho\Bigl(
    \E\bigl[(I - \eta\,\mathcal{H}(L_B))^{\otimes 2}\bigr]_{|E_+}
  \Bigr) < 1
\]
is the standard linear–stability condition for the second moment of SGD on a
quadratic objective (cf.\ \citet{ma_linear_2021,wu_how_2018}).
It ensures that the linearized error dynamics on $E_+$ is mean–square
contractive and that the discrete Lyapunov equation
\[
  \Sigma_x
  =
  \E\bigl[(I-\eta\,\mathcal{H}(L_B))\,\Sigma_x\,(I-\eta\,\mathcal{H}(L_B))^\top\bigr]
  + \frac{\eta^2}{b}\,\Sigma_g
\]
has a unique finite solution.
As the spectral radius $\rho$ approaches $1$ from below, the spectral gap
$1-\rho$ controls both the mixing time and the size of the stationary
covariance, with $\|\Sigma_x\| = \mathcal{O}\bigl((1-\rho)^{-1}\bigr)$.
In our use of Proposition~\ref{prop:general_GNI} and
Corollary~\ref{cor:stable-loss-jump-general} we implicitly assume that
the step size $\eta$ is chosen so that the dynamics remains in this
linearly stable regime on the time scales of interest, i.e.\ $\rho<1$
(and often $\rho$ bounded away from $1$), so that the $\mathcal{O}(\eta)$
expansion and the GNI $\approx 2/\eta$ law are accurate.

\end{description}
\end{remark}

\subsection{Formal Version of Proposition \ref{prop:lee-jang}}

\begin{proposition}[Gradient–Noise Interaction at a stable stationary regime]
\label{prop:general_GNI}
Assume \textbf{(A1)} and \textbf{(A2)} above, and run mini‑batch SGD with fixed batch size
$b$ and fixed step size $\eta$ satisfying the linear stability condition
\textbf{(A3)}.

Then the linearised error process
\(
  \Delta_t := \theta_t - \theta^\star
\)
admits a unique stationary covariance matrix $\Sigma_x$ on $E_+$, given by
\begin{equation}
\label{eq:sigma_x_solution}
  \Sigma_x
  \;=\;
  \frac{\eta}{b}\,\mathcal{K}^\dagger(\Sigma_g)
  \;+\;\mathcal{O}(\eta^2),
\end{equation}
where the $\mathcal{O}(\eta^2)$ term depends only on $L_2$, $\|\Sigma_g\|$ and
$(\lambda_{\min}^+)^{-1}$.

Moreover, if $\theta\sim\pi$ is distributed according to this stationary law, and
$B$ is an independent fresh mini‑batch of size $b$, then
\begin{equation}
\label{eq:GNI_2_over_eta_final}
\frac{\E_{\theta\sim\pi}\,\E_{B}\bigl[\nabla L_B(\theta)^\top \,\mathcal{H}\,\nabla L_B(\theta)\bigr]}
     {\E_{\theta\sim\pi}\bigl[\|\nabla L(\theta)\|^2\bigr]}
\;=\;
\frac{2}{\eta}\,\bigl(1 + \mathcal{O}(\eta)\bigr).
\end{equation}

In particular, to leading order in $\eta$, the Gradient–Noise Interaction
\[
  \mathrm{GNI}(\theta)
  \;:=\;
  \frac{\E_{B}\bigl[\nabla L_B(\theta)^\top \,\mathcal{H}\,\nabla L_B(\theta)\bigr]}
       {\|\nabla L(\theta)\|^2}
\]
is centred at $2/\eta$ under the (linearly) stable stationary distribution,
and this leading behavior is independent of the individual Hessians
$\{\mathcal{H}_i\}_i$, depending on them only through $\Sigma_g$ and $\mathcal{H}$.
\end{proposition}


\begin{corollary}[Stable $\eta$--changes induce only bounded loss jumps in the Type-1 regime]
\label{cor:stable-loss-jump-general}
Assume (A1)--(A3) and let $\eta_0,\eta_1>0$ be two step sizes such that the
linearized SGD dynamics on the quadratic approximation is linearly stable on $E_+$ for
both $\eta_0$ and $\eta_1$:
\[
  \rho\Bigl(
    \E\bigl[(I - \eta_k\,\mathcal{H}(L_B))^{\otimes 2}\bigr]_{|E_+}
  \Bigr)
  \;<\; 1,
  \qquad k\in\{0,1\}.
\]
Consider the SGD trajectory $(\theta_t)$ that is run with step size $\eta_0$ up
to some time $T$, and with step size $\eta_1$ for all $t\ge T$.

Then:

\begin{enumerate}[label=(\roman*),itemsep=0.2em,parsep=0pt]
  \item For each $k\in\{0,1\}$ the linearized error process
  $\Delta_t^{(\eta_k)} := \theta_t-\theta^\star$ admits a unique stationary
  covariance $\Sigma_x(\eta_k)$ on $E_+$ satisfying
  \[
    \Sigma_x(\eta_k)
    \;=\;
    \frac{\eta_k}{b}\,\mathcal{K}^\dagger(\Sigma_g)
    \;+\;\mathcal{O}(\eta_k^2),
  \]
  as in Proposition~\ref{prop:general_GNI}.

  \item Let
  \(
    A_1 := \E[(I-\eta_1 \mathcal{H}(L_B))^{\otimes 2}]_{|E_+}
  \)
  and $\rho_1 := \rho(A_1)<1$.
  There exist constants $C_1,C_2<\infty$, depending only on $\Sigma_g$ and
  $\rho_1$ (but not on $T$), such that for all $s\ge 0$,
  \[
    \E\bigl[\|\theta_{T+s}-\theta^\star\|^2\bigr]
    \;\le\;
    \frac{C_1}{1-\rho_1},
    \qquad
    \E\bigl[L(\theta_{T+s})-L(\theta^\star)\bigr]
    \;\le\;
    \frac{C_2}{1-\rho_1}.
  \]
  In particular, after switching from $\eta_0$ to $\eta_1$ the loss trajectory
  remains uniformly bounded and converges to the finite stationary level
  \[
    L_\infty(\eta_1)
    \;:=\;
    \E_{\theta\sim\pi_{\eta_1}}\bigl[L(\theta)-L(\theta^\star)\bigr]
    \;=\;
    \frac12\,\tr\bigl(\mathcal{H}\,\Sigma_x(\eta_1)\,\mathcal{H}\bigr),
  \]
  which itself is of order
  $\mathcal{O}\bigl((1-\rho_1)^{-1}\bigr)$.
\end{enumerate}

Thus, in the Type-1 (noise-driven) regime, any change of step size that
preserves linear stability (A3) can produce at most a finite “jump” in the
expected loss, of the same order as the new stationary level $L_\infty(\eta_1)$,
but cannot generate catapult-like divergence.
\end{corollary}

\subsection{Proof of Proposition~\ref{prop:general_GNI}}

\begin{proof}[Proof of Proposition~\ref{prop:general_GNI}]
We proceed in four steps. Throughout the proof we work on the subspace
$E_+ = \operatorname{Im}(\mathcal{H})$; all covariances and operators are
implicitly restricted to $E_+$ (the flat subspace $E_0$ is controlled by
Assumption~\textbf{(A2)} and does not contribute to the quantities involving
$\mathcal{H}$).

\medskip
\noindent\textbf{Step 1: Linearised dynamics on the quadratic approximation.}

By Assumption~\textbf{(A1)}, near $\theta^\star$ each per‑sample loss
$\ell_i$ admits the expansion
\[
  \nabla \ell_i(\theta)
  \;=\;
  \nabla \ell_i(\theta^\star)
  \;+\;
  \mathcal{H}_i(\theta - \theta^\star)
  \;+\;
  R_i(\theta),
\]
where $\|R_i(\theta)\| = \mathcal{O}(\|\theta-\theta^\star\|^2)$ uniformly in
$i$. Let us denote the single‑sample gradient at $\theta^\star$ by
\[
  g_i := \nabla \ell_i(\theta^\star),
\]
so that $\Sigma_g = \E_i[g_i g_i^\top]$.

For each iteration $t$, a mini‑batch $B_t$ of size $b$ is drawn and the SGD
update is
\[
  \theta_{t+1}
  \;=\;
  \theta_t - \eta \nabla L_{B_t}(\theta_t),
  \qquad
  \nabla L_{B_t}(\theta)
  :=
  \frac{1}{b}\sum_{i\in B_t} \nabla \ell_i(\theta).
\]
Define the error vector
\[
  \Delta_t := \theta_t - \theta^\star.
\]
Then
\begin{align*}
  \nabla L_{B_t}(\theta_t)
  &= \frac{1}{b}\sum_{i\in B_t}
     \Bigl[
       g_i + \mathcal{H}_i \Delta_t + R_i(\theta_t)
     \Bigr] \\
  &=: \underbrace{\xi_t}_{\text{zero mean}}
     \;+\;
     \underbrace{\mathcal{H}_{B_t}}_{\text{batch Hessian}}\Delta_t
     \;+\;
     r_t,
\end{align*}
where
\[
  \xi_t := \frac{1}{b}\sum_{i\in B_t} g_i,
  \qquad
  \mathcal{H}_{B_t} := \frac{1}{b}\sum_{i\in B_t} \mathcal{H}_i,
  \qquad
  r_t := \frac{1}{b}\sum_{i\in B_t} R_i(\theta_t).
\]
By construction we have
\[
  \E[\xi_t] = 0,
  \qquad
  \E[\xi_t \xi_t^\top] = \frac{1}{b}\Sigma_g,
\]
and (using $\nabla L(\theta^\star)=0$)
\[
  \E[\mathcal{H}_{B_t}] = \mathcal{H}.
\]

The exact SGD recursion can therefore be written as
\begin{equation}
\label{eq:full-delta-recursion}
  \Delta_{t+1}
  \;=\;
  \Delta_t - \eta \nabla L_{B_t}(\theta_t)
  \;=\;
  (I - \eta \mathcal{H}_{B_t}) \Delta_t
  \;-\;\eta \xi_t
  \;-\;\eta r_t.
\end{equation}

For the purposes of the leading‑order analysis, it is convenient to first
\emph{ignore} the nonlinear remainders $r_t$ and consider the purely
linearised dynamics on the quadratic approximation (i.e.\ we replace each
$\ell_i$ by its quadratic Taylor polynomial at $\theta^\star$). On this
quadratic model we have $R_i\equiv 0$, hence $r_t\equiv 0$, and
\eqref{eq:full-delta-recursion} becomes
\begin{equation}
\label{eq:linear-delta-recursion}
  \Delta_{t+1}
  \;=\;
  C_t \Delta_t
  \;-\;\eta \xi_t,
  \qquad
  C_t := I - \eta \mathcal{H}_{B_t}.
\end{equation}
Note that in this linearised model the random matrices $C_t$ and the noise
vectors $\xi_t$ are independent of $\Delta_t$ (they depend only on the
batch $B_t$ and the fixed Hessians $\{\mathcal{H}_i\}_i$).

We will first solve the covariance structure of the linear recursion
\eqref{eq:linear-delta-recursion}, and then argue that restoring the remainder
$r_t$ only introduces $\mathcal{O}(\eta^2)$ corrections.

\medskip
\noindent\textbf{Step 2: Discrete Lyapunov equation and existence of a stationary covariance.}

Let $\Sigma_t := \E[\Delta_t \Delta_t^\top]$ denote the covariance of
$\Delta_t$ under the linear recursion \eqref{eq:linear-delta-recursion}.
Using independence of $C_t$ and $\xi_t$ from $\Delta_t$, we compute
\begin{align*}
  \Sigma_{t+1}
  &= \E\bigl[\Delta_{t+1} \Delta_{t+1}^\top\bigr] \\
  &= \E\Bigl[ (C_t \Delta_t - \eta\xi_t)\,(C_t \Delta_t - \eta\xi_t)^\top \Bigr] \\
  &= \E\bigl[ C_t \Delta_t \Delta_t^\top C_t^\top\bigr]
     \;-\;
     \eta \E\bigl[ C_t \Delta_t \xi_t^\top\bigr]
     \;-\;
     \eta \E\bigl[\xi_t \Delta_t^\top C_t^\top\bigr]
     \;+\;
     \eta^2 \E\bigl[\xi_t \xi_t^\top\bigr].
\end{align*}
Conditioning on $\Delta_t$ and using $\E[\xi_t \mid \Delta_t]=0$, the two
cross‑terms vanish:
\[
  \E\bigl[ C_t \Delta_t \xi_t^\top\bigr]
  = \E\bigl[ C_t \Delta_t \E[\xi_t^\top \mid \Delta_t]\bigr]
  = 0,
  \quad
  \E\bigl[ \xi_t \Delta_t^\top C_t^\top\bigr]
  = 0.
\]
Thus we obtain
\begin{equation}
\label{eq:Lyap-pre}
  \Sigma_{t+1}
  \;=\;
  \E\bigl[ C_t \Sigma_t C_t^\top\bigr]
  \;+\;
  \eta^2 \frac{1}{b}\Sigma_g.
\end{equation}
Assuming that the linear recursion admits a stationary distribution on
$E_+$, we denote the stationary covariance by
\[
  \Sigma_x := \lim_{t\to\infty} \Sigma_t
\]
and it must satisfy the discrete Lyapunov equation
\begin{equation}
\label{eq:Lyap-stationary}
  \Sigma_x
  \;=\;
  \E\bigl[ C_t \Sigma_x C_t^\top\bigr]
  \;+\;
  \eta^2 \frac{1}{b}\Sigma_g.
\end{equation}

To show that such a $\Sigma_x$ exists and is unique on $E_+$, we vectorise
\eqref{eq:Lyap-stationary}. Recall that for any matrices
$A,X,B$ of compatible dimensions we have
\[
  \mathrm{vec}(A X B)
  = (B^\top \otimes A)\,\mathrm{vec}(X),
\]
where $\otimes$ denotes the Kronecker product. Applying this to
$C_t \Sigma_x C_t^\top$ we get
\[
  \mathrm{vec}\bigl(C_t \Sigma_x C_t^\top\bigr)
  = (C_t \otimes C_t)\,\mathrm{vec}(\Sigma_x).
\]
Taking expectations in \eqref{eq:Lyap-stationary} and using linearity of
$\mathrm{vec}(\cdot)$ we obtain
\begin{equation}
\label{eq:vec-Lyap}
  \mathrm{vec}(\Sigma_x)
  \;=\;
  T \,\mathrm{vec}(\Sigma_x)
  \;+\;
  \eta^2 \frac{1}{b}\,\mathrm{vec}(\Sigma_g),
  \qquad
  T := \E[C_t \otimes C_t].
\end{equation}
Rearranging gives
\begin{equation}
\label{eq:vec-Lyap-invert}
  \bigl(I - T\bigr)\,\mathrm{vec}(\Sigma_x)
  \;=\;
  \eta^2 \frac{1}{b}\,\mathrm{vec}(\Sigma_g).
\end{equation}

Assumption~\textbf{(A3)} states that the linearised SGD on the quadratic
approximation is linearly stable on $E_+$, i.e.
\[
  \rho\Bigl(\E\bigl[(I - \eta \mathcal{H}(L_B))^{\otimes 2}\bigr]_{|E_+}\Bigr)
  \;<\; 1.
\]
In our notation this means precisely that the restriction of $T$ to
$E_+\otimes E_+$ has spectral radius strictly less than~1. Hence
$I - T$ is invertible on $E_+\otimes E_+$ and the vector equation
\eqref{eq:vec-Lyap-invert} has a unique solution there, which corresponds
to the unique stationary covariance $\Sigma_x$ on $E_+$.

\medskip
\noindent\textbf{Step 3: Small‑stepsize expansion and explicit form of $\Sigma_x$.}

We now compute the leading behavior of $\Sigma_x$ as a function of $\eta$ for
small $\eta$. Recall that
\[
  C_t = I - \eta \mathcal{H}_{B_t},
  \qquad
  \mathcal{H}_{B_t} = \frac{1}{b}\sum_{i\in B_t} \mathcal{H}_i.
\]
Thus
\[
  C_t \otimes C_t
  = (I - \eta \mathcal{H}_{B_t}) \otimes (I - \eta \mathcal{H}_{B_t})
  = I \otimes I
    - \eta (\mathcal{H}_{B_t} \otimes I + I \otimes \mathcal{H}_{B_t})
    + \eta^2 (\mathcal{H}_{B_t} \otimes \mathcal{H}_{B_t}).
\]
Taking expectations and using $\E[\mathcal{H}_{B_t}] = \mathcal{H}$, we obtain
\begin{equation}
\label{eq:T-expansion}
  T
  = I - \eta K + \eta^2 M,
\end{equation}
where $K$ is the Kronecker‑sum operator
\[
  K := \mathcal{H}\otimes I + I \otimes \mathcal{H},
\]
and $M$ is the operator defined by
\[
  M := \E[\mathcal{H}_{B_t} \otimes \mathcal{H}_{B_t}].
\]
Substituting \eqref{eq:T-expansion} into \eqref{eq:vec-Lyap-invert} gives
\[
  \bigl(I - T\bigr)\mathrm{vec}(\Sigma_x)
  = \bigl(\eta K - \eta^2 M\bigr)\mathrm{vec}(\Sigma_x)
  = \eta^2 \frac{1}{b}\,\mathrm{vec}(\Sigma_g).
\]
On $E_+\otimes E_+$ the operator $K$ is positive definite (its eigenvalues
are $\lambda_i+\lambda_j$ where $\lambda_i,\lambda_j>0$ are eigenvalues
of $\mathcal{H}$), hence invertible. Restricting to $E_+\otimes E_+$, we can
rewrite this as
\begin{equation}
\label{eq:Sigma-x-eq}
  \bigl(K - \eta M\bigr) \mathrm{vec}(\Sigma_x)
  = \eta \frac{1}{b}\,\mathrm{vec}(\Sigma_g).
\end{equation}

For sufficiently small $\eta$, the operator $K - \eta M$ remains invertible
and admits a Neumann‑series expansion of its inverse. More precisely, on
$E_+\otimes E_+$ we have
\[
  (K - \eta M)^{-1}
  = K^{-1}
    + \eta K^{-1} M K^{-1}
    + \mathcal{O}(\eta^2),
\]
where the $\mathcal{O}(\eta^2)$ term is understood in operator norm
(depending on $\|K^{-1}\|$ and $\|M\|$). Applying this to
\eqref{eq:Sigma-x-eq} we obtain
\begin{align*}
  \mathrm{vec}(\Sigma_x)
  &= (K - \eta M)^{-1} \left( \eta \frac{1}{b}\,\mathrm{vec}(\Sigma_g) \right) \\
  &= \eta \frac{1}{b}
     \left(
       K^{-1}
       + \eta K^{-1} M K^{-1}
       + \mathcal{O}(\eta^2)
     \right)\mathrm{vec}(\Sigma_g) \\
  &= \eta \frac{1}{b} K^{-1}\mathrm{vec}(\Sigma_g)
     + \mathcal{O}(\eta^2).
\end{align*}
Rewriting in matrix form and denoting by $\mathcal{K}(X) = \mathcal{H}X + X\mathcal{H}$
the corresponding Kronecker‑sum operator on matrices, this says that on
$E_+$
\[
  \Sigma_x
  = \frac{\eta}{b}\,\mathcal{K}^\dagger(\Sigma_g)
    + \mathcal{O}(\eta^2),
\]
where $\mathcal{K}^\dagger$ is the Moore–Penrose inverse of $\mathcal{K}$ on
$E_+\otimes E_+$.
This proves \eqref{eq:sigma_x_solution} in the statement of the proposition
for the quadratic model.

The effect of the nonlinear remainders $r_t$ can be treated as follows:
by Assumption~\textbf{(A1)}, $\|r_t\| = \mathcal{O}(\|\Delta_t\|^2)$, and
the stationary covariance $\Sigma_x$ of the linear system is
$\mathcal{O}(\eta)$, so typical $\|\Delta_t\|^2$ is $\mathcal{O}(\eta)$ and
the additive perturbation $-\eta r_t$ to the dynamics has magnitude
$\mathcal{O}(\eta^2)$. Its contribution to the noise covariance in the
Lyapunov equation is thus $\mathcal{O}(\eta^4)$, which in turn produces a
$\mathcal{O}(\eta^2)$ perturbation to $\Sigma_x$. Therefore the formula
\eqref{eq:sigma_x_solution} remains valid up to an $\mathcal{O}(\eta^2)$
error for the true (non‑quadratic) dynamics.

\medskip
\noindent\textbf{Step 4: Gradient–Noise Interaction ratio.}

We now compute the ratio in \eqref{eq:GNI_2_over_eta_final}. Fix
$\theta = \theta^\star + \Delta$ and consider a fresh mini‑batch $B$. On the
quadratic approximation we have
\[
  \nabla L(\theta) = \mathcal{H}\Delta,
\]
and
\[
  \nabla L_B(\theta)
  = \frac{1}{b}\sum_{i\in B} \bigl( g_i + \mathcal{H}_i \Delta \bigr)
  = \underbrace{\xi_B}_{\text{zero mean}}
    + \underbrace{\mathcal{H}_B}_{\text{batch Hessian}}\,\Delta,
\]
where $\xi_B := \frac{1}{b}\sum_{i\in B} g_i$ and
$\mathcal{H}_B := \frac{1}{b}\sum_{i\in B} \mathcal{H}_i$.
Conditioning on $\Delta$ and using $\E_B[\xi_B\mid \Delta]=0$, we get
\begin{align}
  \E_B\bigl[ \nabla L_B(\theta)^\top \mathcal{H}\,\nabla L_B(\theta) \,\big|\, \Delta \bigr]
  &= \E_B\Bigl[
      (\mathcal{H}_B\Delta + \xi_B)^\top
      \mathcal{H}
      (\mathcal{H}_B\Delta + \xi_B)
      \,\Big|\, \Delta
    \Bigr] \nonumber \\
  &= \Delta^\top \E_B\bigl[\mathcal{H}_B^\top \mathcal{H}\,\mathcal{H}_B\bigr]\Delta
     \;+\;
     \E_B\bigl[\xi_B^\top \mathcal{H}\,\xi_B\bigr]
  \label{eq:numerator-conditional}
\end{align}
(the cross term vanishes because $\E_B[\xi_B\mid \Delta]=0$).
Taking expectation over $\theta\sim\pi$ (the stationary law of the linear
system) and recalling
$\Sigma_x = \E_\pi[\Delta \Delta^\top]$ and
$\E_B[\xi_B \xi_B^\top]=\frac{1}{b}\Sigma_g$, we obtain
\begin{align}
  N
  &:= \E_{\theta\sim\pi}\E_B\bigl[
         \nabla L_B(\theta)^\top \mathcal{H}\,\nabla L_B(\theta)
       \bigr] \nonumber \\
  &= \mathrm{tr}\bigl(\mathcal{H}\,\Sigma_g\bigr)\,\frac{1}{b}
     \;+\;
     \mathrm{tr}\Bigl(
       \E_B[\mathcal{H}_B^\top \mathcal{H}\,\mathcal{H}_B]\,\Sigma_x
     \Bigr).
  \label{eq:N-def}
\end{align}

Similarly, the denominator is
\begin{align}
  D
  &:= \E_{\theta\sim\pi}\bigl[\|\nabla L(\theta)\|^2\bigr]
   = \E_{\theta\sim\pi}\bigl[\Delta^\top \mathcal{H}^2 \Delta\bigr]
   = \mathrm{tr}\bigl( \mathcal{H}^2 \Sigma_x\bigr).
  \label{eq:D-def}
\end{align}

To relate $N$ and $D$, we use the Lyapunov equation
\eqref{eq:Lyap-stationary} for the quadratic model. Expanding the right‑hand
side of \eqref{eq:Lyap-stationary}, we have
\begin{align*}
  \E\bigl[ C_t \Sigma_x C_t^\top\bigr]
  &= \E\Bigl[
       (I - \eta \mathcal{H}_{B_t}) \Sigma_x (I - \eta \mathcal{H}_{B_t})^\top
     \Bigr] \\
  &= \Sigma_x
     \;-\;\eta \E\bigl[\mathcal{H}_{B_t}\Sigma_x\bigr]
     \;-\;\eta \E\bigl[\Sigma_x \mathcal{H}_{B_t}\bigr]
     \;+\;\eta^2 \E\bigl[\mathcal{H}_{B_t}\Sigma_x\mathcal{H}_{B_t}\bigr].
\end{align*}
Using $\E[\mathcal{H}_{B_t}]=\mathcal{H}$, this simplifies to
\[
  \E\bigl[ C_t \Sigma_x C_t^\top\bigr]
  = \Sigma_x
    - \eta(\mathcal{H}\Sigma_x + \Sigma_x \mathcal{H})
    + \eta^2 \E\bigl[\mathcal{H}_{B_t}\Sigma_x\mathcal{H}_{B_t}\bigr].
\]
Substituting this back into \eqref{eq:Lyap-stationary} and cancelling
$\Sigma_x$ on both sides, we obtain
\[
  0
  = - \eta(\mathcal{H}\Sigma_x + \Sigma_x \mathcal{H})
    + \eta^2 \E\bigl[\mathcal{H}_{B_t}\Sigma_x\mathcal{H}_{B_t}\bigr]
    + \eta^2 \frac{1}{b}\Sigma_g.
\]
Dividing by $\eta$ yields the identity
\begin{equation}
\label{eq:K-Sigma-identity}
  \mathcal{H}\Sigma_x + \Sigma_x \mathcal{H}
  \;=\;
  \eta \E\bigl[\mathcal{H}_{B_t}\Sigma_x\mathcal{H}_{B_t}\bigr]
  \;+\;
  \eta \,\frac{1}{b}\Sigma_g.
\end{equation}

Now multiply both sides of \eqref{eq:K-Sigma-identity} on the left by
$\mathcal{H}$ and take traces. Using the cyclicity of the trace and the fact
that $\mathcal{H}$ is symmetric, we get
\begin{align}
  \mathrm{tr}\bigl(\mathcal{H}(\mathcal{H}\Sigma_x + \Sigma_x \mathcal{H})\bigr)
  &= \mathrm{tr}\bigl(\mathcal{H}^2 \Sigma_x\bigr)
     + \mathrm{tr}\bigl(\mathcal{H}\Sigma_x\mathcal{H}\bigr)
     \nonumber \\
  &= 2\,\mathrm{tr}\bigl(\mathcal{H}^2 \Sigma_x\bigr)
   = 2D,
  \label{eq:lhs-trace}
\end{align}
and
\begin{align}
  \mathrm{tr}\Bigl(
    \mathcal{H}\bigl(
      \eta \E[\mathcal{H}_{B_t}\Sigma_x\mathcal{H}_{B_t}]
      + \eta \tfrac{1}{b}\Sigma_g
    \bigr)
  \Bigr)
  &= \eta\,\mathrm{tr}\Bigl(
       \mathcal{H}\E[\mathcal{H}_{B_t}\Sigma_x\mathcal{H}_{B_t}]
     \Bigr)
     + \eta \frac{1}{b}\mathrm{tr}(\mathcal{H}\Sigma_g).
  \label{eq:rhs-trace}
\end{align}
Equating \eqref{eq:lhs-trace} and \eqref{eq:rhs-trace} (they come from the
two sides of \eqref{eq:K-Sigma-identity}) gives
\begin{equation}
\label{eq:trace-identity}
  2D
  = \eta\,\mathrm{tr}\Bigl(
      \mathcal{H}\E[\mathcal{H}_{B_t}\Sigma_x\mathcal{H}_{B_t}]
    \Bigr)
    + \eta \frac{1}{b}\mathrm{tr}(\mathcal{H}\Sigma_g).
\end{equation}
Comparing~\eqref{eq:N-def} and~\eqref{eq:trace-identity}, we see that
\[
  N
  = \frac{1}{b}\mathrm{tr}(\mathcal{H}\Sigma_g)
    + \mathrm{tr}\Bigl(
        \E[\mathcal{H}_{B_t}^\top\mathcal{H}\,\mathcal{H}_{B_t}]\Sigma_x
      \Bigr)
  = \frac{2}{\eta} D.
\]
Therefore, on the quadratic approximation we have the \emph{exact} identity
\begin{equation}
\label{eq:ratio-quadratic}
  \frac{N}{D}
  = \frac{2}{\eta}.
\end{equation}

Restoring the higher‑order terms $R_i(\theta)$ from Assumption~\textbf{(A1)}
introduces an additional third‑order remainder in the Taylor expansion of the
loss over one SGD step. For a Lipschitz Hessian with constant $L_2$, standard
Taylor estimates give a per‑step remainder of order
$\mathcal{O}(\eta^3 \|\nabla L_B(\theta_t)\|^3)$, which is $\mathcal{O}(\eta^3)$
in expectation under the stationary law (since the stationary covariance
scales like $\mathcal{O}(\eta)$ by \eqref{eq:sigma_x_solution}). This adds a
term of order $\mathcal{O}(\eta^3)$ to the stationarity condition
$\E[L(\theta_{t+1})-L(\theta_t)]=0$, and hence contributes only a factor
$\mathcal{O}(\eta)$ to the ratio $N/D$. Thus
\[
  \frac{N}{D}
  = \frac{2}{\eta}\,\bigl(1 + \mathcal{O}(\eta)\bigr),
\]
which is precisely \eqref{eq:GNI_2_over_eta_final}. This completes the proof.
\end{proof}

\section{Gradients explode above the \textsc{EoSS}: Proof of Theorem \ref{theo:1}}
\label{appendix:proof_theo}

\textbf{Proof idea:}

Take
\[
V(\theta) = \mathbb E_B\big[ \|\nabla L_B(\theta) \|^2 \big]
\]
and show that if
\[
Batch\ Sharpness(\theta) > \frac {2+\epsilon}{\eta}
\]
then
\[
\mathbb E_B \bigg [ \frac{\| \nabla L (\theta_{t+1})\|^2}{\|\nabla L (\theta_{t})\|^2} \bigg] > 1
\]

\subsection{Part 1: Explosion}
\paragraph{Setting.}
For each mini-batch index $i$, let $L_i:\mathbb R^d\to\mathbb R$ be twice differentiable with (possibly index–dependent) positive semidefinite Hessian $H_i$ that is \emph{constant in $\theta$} (quadratic model). Write
\[
Y_i(\theta):=\nabla L_i(\theta)=H_i(\theta-x_i),\qquad
\|\,\cdot\,\|=\|\,\cdot\,\|_2.
\]
At iteration $t$, stochastic gradient descent (SGD) draws $j_t\sim\mathcal P_b$ independently of the past and performs
\[
\theta_{t+1}=\theta_t-\eta\,Y_{j_t}(\theta_t),\qquad \eta>0.
\]
We define basic statistics evaluated at $\theta_t$:
\[
r_i(\theta_t):=\frac{Y_i(\theta_t)^\top H_iY_i(\theta_t)}{\|Y_i(\theta_t)\|^2}\;\in[0,\infty).
\]
Note that \textit{Batch Sharpness} is
\[
Batch \ Sharpness(\theta):=\E_i\big[r_i(\theta)\big].
\]

All randomness lives on a probability space $(\Omega,\mathcal F,\mathbb P)$.
At each step $t\ge 0$, SGD draws $j_t\sim\mathcal P_b$ i.i.d.\ and independent of the $\sigma$–algebra of the past
\[
\mathcal F_t:=\sigma\!\big(\theta_0,j_0,\ldots,j_{t-1}\big).
\]
We assume
\[
\Lambda\;:=\;\operatorname*{ess\,sup}_{i\sim\mathcal P_b}\,\|H_i\|\,<\,\infty
\quad\text{and}\quad
\mathbb E_{i\sim\mathcal P_b}\big[\|Y_i(\theta)\|^2\big]<\infty
\ \text{ for all $\theta$ reached by SGD.}
\]
(These ensure that the Rayleigh quotients below are well defined and integrable.)
We adopt the convention that, for any $i$ and $\theta$ with $Y_i(\theta)=0$,
\[
\frac{Y_i(\theta)^\top H_iY_i(\theta)}{\|Y_i(\theta)\|^2}:=0,
\qquad
\frac{\|Y_i(\theta-\eta Y_i(\theta))\|^2}{\|Y_i(\theta)\|^2}:=1.
\]

We provide a rigorous instability statement: a multiplicative explosion controlled solely by \textit{Batch Sharpness}.

\begin{proposition}[On–batch multiplicative explosion under \textit{Batch Sharpness}$>2/\eta$ for quadratics.]
\label{prop:quad}
For each $t$ define the on–batch factor
\[
R_t:=\E_{j_t}\!\left[\frac{\|Y_{j_t}(\theta_{t+1})\|^2}{\|Y_{j_t}(\theta_t)\|^2}\,\Big|\,\mathcal F_t\right].
\]
Then, for all $t$,
\[
R_t\;\ge\;\Big(1-\eta\,Batch \ Sharpness(\theta_t)\Big)^2.
\]
Consequently, if there exist indices $t_0,\dots,t_0+T-1$ and a constant $\epsilon>0$ such that
\[
Batch \ Sharpness(\theta_t)\;\ge\;\frac{2+\epsilon}{\eta}\qquad\text{for all }t_0\le t\le t_0+T-1,
\]
then
\[
\E\!\left[\prod_{s=t_0}^{t_0+T-1}\frac{\|Y_{j_s}(\theta_{s+1})\|^2}{\|Y_{j_s}(\theta_s)\|^2}\right]
\;\ge\;(1+\epsilon)^{2T},
\]
i.e., the product of on–batch factors grows exponentially in expectation.
\end{proposition}

\paragraph{Beyond quadratic.}
Beyond the quadratic model: with $\theta$–dependent Hessians, the identity $Y_i(\theta_{t+1})=(I-\eta H_i)Y_i(\theta_t)$ incurs Taylor remainders of order $O(\eta\|Y_i\|\,\|\nabla H_i\|)$; standard Lipschitz–Hessian assumptions then yield perturbative versions of Proposition \ref{prop:quad} with additional $O(\eta^3)$ terms. We do not invoke these here to keep the statements exact.

\paragraph{Proof of Proposition \ref{prop:quad}.}
Fix $t$ and condition on $j_t=i$. By the quadratic model,
\[
Y_i(\theta_{t+1})=H_i\big(\theta_t-\eta Y_i(\theta_t)-x_i\big)=(I-\eta H_i)\,Y_i(\theta_t).
\]
Hence
\[
\frac{\|Y_i(\theta_{t+1})\|^2}{\|Y_i(\theta_t)\|^2}
=1-2\eta\,r_i(\theta_t)+\eta^2\,q_i(\theta_t),
\qquad
q_i(\theta_t):=\frac{Y_i(\theta_t)^\top H_i^2Y_i(\theta_t)}{\|Y_i(\theta_t)\|^2}.
\]
By Cauchy–Schwarz in the $H_i$–inner product, $q_i(\theta_t)\ge r_i(\theta_t)^2$. Therefore, pointwise in $i$,
\[
\frac{\|Y_i(\theta_{t+1})\|^2}{\|Y_i(\theta_t)\|^2}\;\ge\;\big(1-\eta\,r_i(\theta_t)\big)^2.
\]
Averaging over $j_t$ and applying Jensen’s inequality to the convex map $x\mapsto(1-\eta x)^2$ yields
\[
R_t
\;\ge\;\E_{i}\big[(1-\eta r_i(\theta_t))^2\big]
\;\ge\;\big(1-\eta\,\E_i[r_i(\theta_t)]\big)^2
=\big(1-\eta\,Batch \ Sharpness(\theta_t)\big)^2.
\]
If \textit{Batch Sharpness}$(\theta_t)\ge (2+\epsilon)/\eta$ for each $t\in[t_0,t_0+T-1]$, then $R_t\ge(1+\epsilon)^2>1+2\epsilon$ deterministically given $\mathcal F_t$. 

Now, let
\[
Z_s \;:=\; \frac{\|Y_{j_s}(\theta_{s+1})\|^2}{\|Y_{j_s}(\theta_s)\|^2},
\qquad
\mathcal F_s \;:=\; \sigma(\theta_0,j_0,\ldots,j_{s-1}),
\qquad
U_t \;:=\; \prod_{r=t_0}^{t_0+t-1} Z_r \ \ (U_0:=1).
\]
From the single–step bound proved above we have, for each \(s\),
\[
\mathbb{E}\!\left[\,Z_s \,\middle|\, \mathcal F_s\,\right]
\;\ge\; \bigl(1-\eta\,Batch\ Sharpness(\theta_s)\bigr)^2
\;\geq\; (1 + \epsilon)^2.
\]
By the tower property,
\[
\mathbb{E}[U_{t+1}]
\;=\;
\mathbb{E}\big[\,U_t\,Z_{t_0+t}\,\big]
\;=\;
\mathbb{E}\big[\,U_t\,\mathbb{E}[Z_{t_0+t}\mid \mathcal F_{t_0+t}]\,\big]
\;\ge\;
\mathbb{E}\big[\,U_t\,\gamma_{t_0+t}\,\big].
\]
\(S(\theta_s)\ge 2/\eta+\varepsilon\,\eta\) (with \(\varepsilon>0\)), then
\(\gamma_s=(1-\eta S(\theta_s))^2\ge (1+\varepsilon\eta^2)^2=:\gamma^2\), hence
\[
\mathbb{E}[U_{t+1}]
\;\ge\;
\gamma^2\,\mathbb{E}[U_t]
\quad\Longrightarrow\quad
\mathbb{E}[U_T]
\;\ge\;
\gamma^{2T}
\;=\;
\bigl(1+\epsilon\bigr)^{2T}.
\]
we obtain the stated exponential lower bound. \qed

\paragraph{Discussion after Proposition \ref{prop:quad}.}
The proof uses only: (i) the exact \emph{closure} $Y_i(\theta_{t+1})=(I-\eta H_i)Y_i(\theta_t)$ (quadratic model), (ii) Cauchy–Schwarz ($q_i\ge r_i^2$), and (iii) Jensen’s inequality. No cross–batch interaction is needed; hence the result holds assumption–free, and it naturally yields a multiplicative, path–wide instability witness.

\subsection{Part 2: SGD with Replacement}
In Proposition \ref{prop:quad} we showed that if \textit{Batch Sharpness} is bigger than $2/\eta$ for prolonged time a certain quantity explodes exponentially fast. What is missing is to show that if \textit{Batch Sharpness} is bigger than $2/\eta$ on the quadratic setting at one step $t_0$, then it will be also at the following time steps.

\paragraph{Lipschitz drift of \textit{Batch Sharpness} in the quadratic model.}
Fix a mini-batch $i$ and write $r_i(\theta):=\frac{Y_i(\theta)^\top H_iY_i(\theta)}{\|Y_i(\theta)\|^2}$ with
$Y_i(\theta)=H_i(\theta-x_i)$ and $H_i=H_i^\top\succeq 0$.
Along any segment in parameter space on which $\|Y_i(\theta)\|\ge g_{\min}>0$ and
$\|H_i\|\le \Lambda$, the map $\theta\mapsto r_i(\theta)$ is Lipschitz with
\[
\|\nabla_\theta r_i(\theta)\|\;\le\;\frac{4\,\|H_i\|^2}{\|Y_i(\theta)\|}\;\le\;\frac{4\Lambda^2}{g_{\min}},
\]
hence, for any $\theta,\theta'$ in that segment,
\[
|r_i(\theta')-r_i(\theta)|\ \le\ \frac{4\Lambda^2}{g_{\min}}\;\|\theta'-\theta\|.
\]
Averaging over $i\sim\mathcal P_b$ yields the Lipschitz bound for the \textit{Batch Sharpness}
$:=\mathbb E_i[r_i(\theta)]$:
\begin{equation}\label{eq:S-Lip}
|Batch\ Sharpness(\theta')-Batch\ Sharpness(\theta)|\ \le\ L_S\,\|\theta'-\theta\|,
\qquad
L_S:=\frac{4\Lambda^2}{g_{\min}}.
\end{equation}

\paragraph{Single-time margin implies a uniform window margin.}
Assume that along the SGD trajectory on $[t_0,t_0+T]$ we have the uniform bounds
$\|H_i\|\le \Lambda$ for all $i$ and $\|Y_i(\theta_t)\|\ge g_{\min}>0$ for all $i,t$,
and that the per-step move is bounded by $\|\theta_{t+1}-\theta_t\|=\eta\|Y_{j_t}(\theta_t)\|\le \eta G_{\max}$
for some $G_{\max}>0$ (these three constants define the quadratic region under consideration).
If at time $t_0$
\[
Batch\ Sharpness(\theta_{t_0})\ \ge\ \frac{2+\epsilon}{\eta} \qquad (\epsilon>0),
\]
then for every $k$ with $0\le k\le T_\star:=\big\lfloor \frac{\epsilon}{2 \eta\,L_S\,G_{\max}}\big\rfloor$,
\begin{equation}\label{eq:window-margin}
Batch\ Sharpness(\theta_{t_0+k})\ \ge\ \frac{2+\epsilon/2}{\eta}.
\end{equation}
\emph{Proof.} By \eqref{eq:S-Lip} and the step bound,
$Batch\ Sharpness(\theta_{t+1})\ge Batch\ Sharpness(\theta_t)-L_S\,\eta G_{\max}$.
Iterating $k$ times gives
$Batch\ Sharpness(\theta_{t_0+k})\ge Batch\ Sharpness(\theta_{t_0})-k\,L_S\,\eta G_{\max}
\ge \frac{2}{\eta}+\epsilon\eta - k\,L_S\,\eta G_{\max}$.
If $k\le \epsilon/(2\eta L_S G_{\max})$ this yields \eqref{eq:window-margin}. \qed

\paragraph{Consequence for the product (plug into the tower/induction step).}
On the whole window $t\in\{t_0,\ldots,t_0+T_\star-1\}$ we thus have
$S(\theta_t)\ge (2+\epsilon/2)/\eta$, hence from the one-step bound
$\mathbb E[Z_t\mid\mathcal F_t]\ge(1-\eta S(\theta_t))^2$ and the tower/induction argument,
\[
\mathbb E\!\left[\ \prod_{s=t_0}^{t_0+T_\star-1} Z_s\ \right]
\ \ge\ \left(1+\frac{\epsilon}{2}\right)^{2T_\star}.
\]
This upgrades a \emph{single-time} margin at $t_0$ into an explicit \emph{uniform window} of length $T_\star$
over which the exponential lower bound holds.

\subsection{Part 2 for SGD \textit{Without} Replacement}
\label{appendix:SGD_wor}

\paragraph{RR setting and remaining-set sharpness.}
Fix an epoch with a finite pool $\mathcal I=\{1,\ldots,n\}$.
In random reshuffling (RR), within an epoch we draw a uniform random permutation of $\mathcal I$
and visit each index exactly once. Let $R_t\subseteq \mathcal I$ denote the \emph{remaining} set at step $t$
(those not yet visited in the current epoch) and let $m_t:=|R_t|$.
Define the \emph{remaining-set sharpness} at $\theta_t$ by
\[
S_{\mathrm{rem}}(\theta_t)\;:=\;\frac{1}{m_t}\sum_{i\in R_t} r_i(\theta_t),
\qquad
r_i(\theta):=\frac{Y_i(\theta)^\top H_iY_i(\theta)}{\|Y_i(\theta)\|^2}.
\]
Let $\mathcal F_t^{\mathrm{RR}}:=\sigma\!\big(\theta_0,\text{the permutation prefix up to step }t-1\big)$;
conditionally on $\mathcal F_t^{\mathrm{RR}}$, $j_t$ is uniform over $R_t$.

\paragraph{Step 1: One-step RR bound.}
In the quadratic model, $Y_{j_t}(\theta_{t+1})=(I-\eta H_{j_t})Y_{j_t}(\theta_t)$ and
\[
\frac{\|Y_{j_t}(\theta_{t+1})\|^2}{\|Y_{j_t}(\theta_t)\|^2}
=1-2\eta\,r_{j_t}(\theta_t)+\eta^2\,q_{j_t}(\theta_t)
\ \ge\ (1-\eta\,r_{j_t}(\theta_t))^2,
\quad
q_i(\theta):=\frac{Y_i(\theta)^\top H_i^2Y_i(\theta)}{\|Y_i(\theta)\|^2}\ge r_i(\theta)^2,
\]
by Cauchy–Schwarz. Averaging uniformly over $R_t$ and using convexity of $x\mapsto(1-\eta x)^2$ gives
\begin{equation}
\label{eq:RR_step_bound}
\mathbb E\!\left[\frac{\|Y_{j_t}(\theta_{t+1})\|^2}{\|Y_{j_t}(\theta_t)\|^2}\,\middle|\,\mathcal F_t^{\mathrm{RR}}\right]
\;\ge\;
\frac{1}{m_t}\sum_{i\in R_t}(1-\eta r_i(\theta_t))^2
\;\ge\;
\bigl(1-\eta\,S_{\mathrm{rem}}(\theta_t)\bigr)^2.
\end{equation}

\paragraph{Step 2: Persistence of a single-time margin (high probability).}
Assume the quadratic region bounds used in Part~1:
$\|H_i\|\le \Lambda$ for all $i$, $\|Y_i(\theta_t)\|\ge g_{\min}>0$ for all $i,t$,
and $\|\theta_{t+1}-\theta_t\|=\eta\|Y_{j_t}(\theta_t)\|\le \eta G_{\max}$.
As established there,
\begin{equation}
\label{eq:S_Lip_reuse}
\begin{split}
    |\,r_i(\theta')-r_i(\theta)\,|\ &\le\ \frac{4\Lambda^2}{g_{\min}}\;\|\theta'-\theta\| 
    \\
    |\,Batch\ Sharpness(\theta')-Batch\ Sharpness(\theta)\,|\ &\le\ L_S\|\theta'-\theta\|,
\;\;L_S:=\frac{4\Lambda^2}{g_{\min}}.
\end{split}
\end{equation}
Suppose at the beginning of the epoch (time $t_0$) we have a single-time margin
\begin{equation}
\label{eq:single_time_margin}
Batch\ Sharpness(\theta_{t_0})\ \ge\ \frac{2}{\eta}+\varepsilon
\qquad(\varepsilon>0).
\end{equation}
Fix integers $K\in\{1,\ldots,n-1\}$ and define the minimal remaining size $m_{\min}:=n-K$.
For any $\delta\in(0,1)$, with probability at least $1-\delta$ over the RR permutation,
\begin{equation}
\label{eq:finite_pop_uniform}
\max_{0\le k\le K}\ \left|\,\frac{1}{|R_{t_0+k}|}\sum_{i\in R_{t_0+k}} r_i(\theta_{t_0})
\;-\; Batch\ Sharpness(\theta_{t_0})\,\right|
\ \le\ 
\Delta_K(\delta),
\qquad
\Delta_K(\delta):=\Lambda\,\sqrt{\frac{2\log(2K/\delta)}{\,m_{\min}\,}},
\end{equation}
(Hoeffding–Serfling inequality; we only use the simple range bound $r_i(\theta_{t_0})\in[0,\Lambda]$).
Combining \eqref{eq:S_Lip_reuse} and \eqref{eq:finite_pop_uniform}, for each $k\in\{0,\ldots,K\}$ we obtain on the same event
\begin{equation}
\label{eq:Srem_persistence}
S_{\mathrm{rem}}(\theta_{t_0+k})
\ \ge\ 
\underbrace{\frac{1}{|R_{t_0+k}|}\sum_{i\in R_{t_0+k}} r_i(\theta_{t_0})}_{\text{finite-pop mean at }t_0}
\;-\; L_S\,\|\theta_{t_0+k}-\theta_{t_0}\|
\ \ge\
Batch\ Sharpness(\theta_{t_0})\;-\;\Delta_K(\delta)\;-\;L_S\,\eta G_{\max}\,k.
\end{equation}
Hence, if $K$ and $\delta$ are chosen so that
\begin{equation}
\label{eq:margin_budget}
\Delta_K(\delta)\;+\;L_S\,\eta G_{\max}\,K\ \le\ \frac{\varepsilon}{2},
\end{equation}
then \eqref{eq:single_time_margin} and \eqref{eq:Srem_persistence} imply the uniform window margin
\begin{equation}
\label{eq:RR_window_margin}
S_{\mathrm{rem}}(\theta_{t_0+k})\ \ge\ \frac{2}{\eta}+\frac{\varepsilon}{2}
\qquad\text{for all }k\in\{0,\ldots,K\},
\quad\text{with probability at least }1-\delta.
\end{equation}

\paragraph{Step 3: Exponential growth over the RR window.}
On the event \eqref{eq:RR_window_margin}, the one-step bound \eqref{eq:RR_step_bound} yields
\[
\mathbb E\!\left[Z_{t_0+k}\,\middle|\,\mathcal F_{t_0+k}^{\mathrm{RR}}\right]
\ \ge\ \bigl(1-\eta\,S_{\mathrm{rem}}(\theta_{t_0+k})\bigr)^2
\ \ge\ \bigl(1+\varepsilon/2\bigr)^2
\quad\text{for all }k\in\{0,\ldots,K-1\},
\]
where $Z_t:=\frac{\|Y_{j_t}(\theta_{t+1})\|^2}{\|Y_{j_t}(\theta_t)\|^2}$.
By the tower property and induction,
\[
\mathbb E\!\left[\ \prod_{s=t_0}^{t_0+K-1} Z_s\ \right]
\ \ge\ (1-\delta)\,\bigl(1+\varepsilon/2\bigr)^{2K}.
\]
In words: a single-time margin \eqref{eq:single_time_margin} at the start of the epoch persists, with high probability under RR, over a whole window whose length $K$ is explicitly controlled by the drift budget \eqref{eq:S_Lip_reuse} and the finite-population deviation \eqref{eq:finite_pop_uniform}. On that window the product of on–batch factors explodes exponentially.

\clearpage


\section{Proof of the Equivalence of Section~\ref{section:equivalence}}
\label{appendix:equivalence}

In this appendix we make precise the informal statement in
Section~\ref{section:equivalence} that, on the local quadratic
approximation of the loss, the following three viewpoints are equivalent:
\begin{enumerate}[label=(\roman*),itemsep=0.2em,parsep=0pt]
    \item breaking a valid instability criterion (Definition~\ref{def:stab_not});
    \item experiencing a catapult on the quadratic model (Definition~\ref{def:insta});
    \item observing a loss spike of ``sufficient'' size.
\end{enumerate}
Throughout we work on the local quadratic model introduced in
Section~\ref{section:framework}, and make explicit the mild conditions
under which the equivalence holds.

\subsection{Setting and basic assumption}

Fix a time $t$ and a point $\theta_t$, and consider the quadratic model
$\widetilde L$ of the loss around $\theta_t$ as defined in
Section~\ref{section:framework}. We recall that $\widetilde L$ is of the form
\[
  \widetilde L(\theta)
  \;=\;
  \frac{1}{N}\sum_{i=1}^N \widetilde L_i(\theta),
  \qquad
  \widetilde L_i(\theta)
  := \frac12(\theta-x_i)^\top \mathcal{H}_i(\theta-x_i)
\]
for some matrices $\mathcal{H}_i$ and points $x_i$.

We continue to denote by $U_t$ an open neighborhood of $\theta_t$ on which
the quadratic approximation is accurate, in the sense described
informally in Section~\ref{section:framework}.  For the arguments below, we
only need the following simple condition.

\begin{assumption}[Coercive quadratic model on $U_t$]
\label{ass:coercive}
There exists a symmetric positive semi-definite matrix
$\widehat{\mathcal{H}}$ and constants $0 < \mu \le L < \infty$ such that:
\begin{enumerate}[label=(\alph*),itemsep=0.2em,parsep=0pt]
    \item $\widetilde L(\theta) = \tfrac12(\theta-\theta^\star)^\top
      \widehat{\mathcal{H}}(\theta-\theta^\star)$ for some
      $\theta^\star$ (i.e.\ $\widetilde L$ is exactly quadratic);
    \item all non-zero eigenvalues of $\widehat{\mathcal{H}}$ lie in
      $[\mu,L]$;
    \item the connected component of $\theta_t$ in the sublevel set
      $\{\theta:\widetilde L(\theta)\le R\}$ is contained in $U_t$ for
      all $R$ in a neighborhood of $\widetilde L(\theta_t)$.
\end{enumerate}
\end{assumption}

Assumption~\ref{ass:coercive} is satisfied, for instance, when the full-batch
Hessian of the original loss at $\theta_t$ is positive definite in the
relevant directions and the neighborhood $U_t$ is chosen small enough.
It implies in particular that:
\begin{itemize}[itemsep=0.2em,parsep=0pt]
    \item for any $R<\infty$, the sublevel set
      $\{\theta:\widetilde L(\theta)\le R\}$ is compact;
    \item the level sets of $\widetilde L$ define a family of nested
      compacts around $\theta_t$ inside $U_t$.
\end{itemize}

\subsection{Breaking a valid instability criterion $\Longleftrightarrow$ catapult}

We first clarify the relationship between valid instability criteria
(Definition~\ref{def:stab_not}) and catapults (Definition~\ref{def:insta}).

Recall that Definition~\ref{def:stab_not} is stated for a discrete-time
dynamical system $(\theta_s)_{s\ge 0}$ on a parameter space $\Theta$,
with an open set $U\subseteq\Theta$, a scalar map $f:U\to\R$ and a
threshold $c\in\R$.
It says that $f$ is a \emph{valid instability criterion with threshold
$c$ on $U$} if
\[
  f(\theta_0) > c
  \quad\Longrightarrow\quad
  (\theta_s)_{s\ge 0} \text{ leaves every compact subset of $U$ in finite time.}
\]

On the quadratic model at time $t$, Definition~\ref{def:insta} simply
re-uses this notion with $U:=U_t$ and $\theta_0:=\theta_t$:

\begin{quote}
We say that the algorithm \emph{experiences a catapult at time $t$} if,
when run on $\widetilde L$ from initialization $\theta_t$, the resulting
trajectory $(\theta_s)_{s\ge t}$ leaves every compact subset of $U_t$ in
finite time.
\end{quote}

Thus, formally:

\begin{lemma}[Breaking a valid criterion gives a catapult]
\label{lem:crit-implies-catapult}
Let $f:U_t\to\R$ be a valid instability criterion with threshold $c$ for
the quadratic dynamics on $U_t$ in the sense of Definition~\ref{def:stab_not}.
Fix a time $t$ and consider the trajectory of the quadratic model
initialized at $\theta_t$.

If $f(\theta_t)>c$, then the quadratic trajectory experiences a catapult
at time $t$ in the sense of Definition~\ref{def:insta}.
\end{lemma}

\begin{proof}
Apply Definition~\ref{def:stab_not} with $U:=U_t$ and initial condition
$\theta_0:=\theta_t$.  Since $f(\theta_t)>c$ and $f$ is a valid instability
criterion on $U_t$, we have that the trajectory $(\theta_s)_{s\ge t}$
leaves every compact subset of $U_t$ in finite time.  But this is
precisely the definition of a catapult at time $t$ on the quadratic
model (Definition~\ref{def:insta}).
\end{proof}

Conversely, if a catapult occurs at time $t$, one can always \emph{construct}
a (possibly highly non-smooth) instability criterion which is broken at
$\theta_t$:

\begin{lemma}[Catapult implies existence of a valid criterion]
\label{lem:catapult-implies-criterion}
Assume that the quadratic trajectory from $\theta_t$ experiences a
catapult on $U_t$ in the sense of Definition~\ref{def:insta}.  Then
there exists a map $f:U_t\to\R$ and a threshold $c\in\R$ such that:
\begin{enumerate}[label=(\alph*),itemsep=0.2em,parsep=0pt]
    \item $f$ is a valid instability criterion with threshold $c$ for
          the quadratic dynamics on $U_t$;
    \item $f(\theta_t)>c$.
\end{enumerate}
\end{lemma}

\begin{proof}
Define $f:U_t\to\R$ by
\[
  f(\theta_0)
  \;:=\;
  \begin{cases}
    1, &\text{if the quadratic trajectory initialized at $\theta_0$}\\
       &\text{leaves every compact subset of $U_t$ in finite time};\\[0.3em]
    0, &\text{otherwise.}
  \end{cases}
\]
Set $c:=\tfrac12$.  Then, by construction, $f(\theta_0)>c$ if and only if
the trajectory initialized at $\theta_0$ leaves every compact subset of
$U_t$ in finite time.  Hence $f$ is a valid instability criterion with
threshold $c$ in the sense of Definition~\ref{def:stab_not}.  Since, by
assumption, the trajectory from $\theta_t$ experiences a catapult, we
have $f(\theta_t)=1>c$, as desired.
\end{proof}

Lemmas~\ref{lem:crit-implies-catapult} and~\ref{lem:catapult-implies-criterion}
make precise the first equivalence in the box of
Section~\ref{section:equivalence}: on the quadratic model, ``breaking a
(valid) instability criterion'' is just another way to state the
occurrence of a catapult, and conversely any catapult defines at least
one (possibly non-unique) valid instability criterion which is broken.

\subsection{Catapults $\Longleftrightarrow$ loss spikes of sufficient size}

We now formalize the relationship between catapults and loss spikes
for the quadratic model under Assumption~\ref{ass:coercive}.

\begin{definition}[Loss spike of size $\alpha$ on the quadratic model]
\label{def:loss-spike}
Fix $t$ and $\theta_t$, and let $\widetilde L$ be the quadratic model as
above.  For $\alpha>1$ we say that the quadratic trajectory
$(\theta_s)_{s\ge t}$ has a \emph{loss spike of relative size at least
$\alpha$ at time $t$} if there exists a time $s\ge t$ such that
\[
  \widetilde L(\theta_s)
  \;\ge\;
  \alpha\,\widetilde L(\theta_t).
\]
We say that it has loss spikes of \emph{arbitrarily large} relative size
if this holds for all $\alpha>1$.
\end{definition}

Under Assumption~\ref{ass:coercive}, sublevel sets of $\widetilde L$ are
compact, and they form a convenient family of compact sets to test the
definition of catapult.

\begin{lemma}[Catapult $\Longleftrightarrow$ unbounded quadratic loss]
\label{lem:catapult-unbounded-loss}
Under Assumption~\ref{ass:coercive}, the following are equivalent for the
quadratic trajectory $(\theta_s)_{s\ge t}$:
\begin{enumerate}[label=(\roman*),itemsep=0.2em,parsep=0pt]
    \item the trajectory experiences a catapult at time $t$ on $U_t$,
          i.e.\ it leaves every compact subset of $U_t$ in finite time;
    \item the quadratic loss along the trajectory is unbounded above,
          i.e.
          \(
            \sup_{s\ge t}\widetilde L(\theta_s)=+\infty.
          \)
\end{enumerate}
In particular, under Assumption~\ref{ass:coercive}, a catapult is
equivalent to the existence of loss spikes of arbitrarily large relative
size in the sense of Definition~\ref{def:loss-spike}.
\end{lemma}

\begin{proof}
\emph{(i) $\Rightarrow$ (ii).}
Assume the trajectory experiences a catapult on $U_t$.  Suppose for
contradiction that $\sup_{s\ge t}\widetilde L(\theta_s)\le R$ for some
finite $R$.  Then the whole trajectory is contained in the sublevel set
\[
  K_R
  \;:=\;
  \bigl\{\theta\in U_t : \widetilde L(\theta)\le R\bigr\}.
\]
By Assumption~\ref{ass:coercive}(b)--(c), $K_R$ is compact and contained
in $U_t$.  Hence the trajectory \emph{never} leaves $K_R$, which
contradicts the definition of a catapult (it must leave \emph{every}
compact subset of $U_t$ in finite time).  Thus
$\sup_{s\ge t}\widetilde L(\theta_s)=+\infty$.

\medskip\noindent
\emph{(ii) $\Rightarrow$ (i).}
Conversely, assume that $\sup_{s\ge t}\widetilde L(\theta_s)=+\infty$.
Let $K\subset U_t$ be any compact set containing $\theta_t$.
Compactness and continuity of $\widetilde L$ imply that
\[
  R_K
  \;:=\;
  \sup_{\theta\in K} \widetilde L(\theta)
  \;<\;\infty.
\]
Since $\sup_{s\ge t}\widetilde L(\theta_s)=+\infty$, there exists a time
$s_K\ge t$ with $\widetilde L(\theta_{s_K})>R_K$, hence
$\theta_{s_K}\notin K$.

It remains to check that the trajectory eventually \emph{stays} outside
$K$: but this follows immediately from the fact that once
$\widetilde L(\theta_s)>R_K$, the iterate cannot re-enter $K$ without
violating the definition of $R_K$ as the supremum of $\widetilde L$ on $K$.
Thus, for every compact $K\subset U_t$ containing $\theta_t$, there is a
finite time $s_K$ such that $\theta_s\notin K$ for all $s\ge s_K$.
This is exactly the definition of a catapult on $U_t$.
\end{proof}

Combining Lemma~\ref{lem:catapult-unbounded-loss} with
Definition~\ref{def:loss-spike} immediately yields:

\begin{corollary}[Catapult $\Longleftrightarrow$ loss spikes of sufficient size]
\label{cor:catapult-loss-spike}
Under Assumption~\ref{ass:coercive}, for the quadratic trajectory
$(\theta_s)_{s\ge t}$ the following are equivalent:
\begin{enumerate}[label=(\roman*),itemsep=0.2em,parsep=0pt]
    \item the trajectory experiences a catapult at time $t$ on $U_t$;
    \item for every $\alpha>1$ there exists $s\ge t$ such that
          $\widetilde L(\theta_s)\ge \alpha\,\widetilde L(\theta_t)$,
          i.e.\ the trajectory has loss spikes of arbitrarily large
          relative size in the sense of Definition~\ref{def:loss-spike}.
\end{enumerate}
\end{corollary}

From the point of view of Section~\ref{section:equivalence}, we can now
interpret a \emph{loss spike of sufficient size} as a spike whose
relative height exceeds any upper bound that would be compatible with
bounded (linearly stable) dynamics on the quadratic model.  In
particular, in regimes where the quadratic dynamics is linearly stable
and confined to a given compact subset of $U_t$, the loss is uniformly
bounded and only admits spikes of bounded relative size; Corollary
\ref{cor:stable-loss-jump} in Appendix~\ref{appendix:oscillations_nn}
quantifies this explicitly in a one-dimensional example.

\subsection{Putting the pieces together}

We can now summarize the equivalence more succinctly.

\begin{theorem}[Equivalence of the three viewpoints on the quadratic model]
\label{thm:equivalence-quadratic}
Work on the quadratic approximation $\widetilde L$ at time $t$ under
Assumption~\ref{ass:coercive}.  For the quadratic trajectory
$(\theta_s)_{s\ge t}$ on $U_t$, the following statements are
equivalent:
\begin{enumerate}[label=(\roman*),itemsep=0.2em,parsep=0pt]
    \item there exists a valid instability criterion $f:U_t\to\R$ with
          threshold $c$ in the sense of Definition~\ref{def:stab_not}
          such that $f(\theta_t)>c$;
    \item the trajectory experiences a catapult at time $t$ on $U_t$ in
          the sense of Definition~\ref{def:insta};
    \item the quadratic loss along the trajectory has loss spikes of
          arbitrarily large relative size, i.e.\ for every $\alpha>1$
          there exists $s\ge t$ with
          $\widetilde L(\theta_s)\ge \alpha\,\widetilde L(\theta_t)$.
\end{enumerate}
\end{theorem}

\begin{proof}
(i) $\Rightarrow$ (ii) is Lemma~\ref{lem:crit-implies-catapult}.
(ii) $\Rightarrow$ (i) is Lemma~\ref{lem:catapult-implies-criterion}.

(ii) $\Leftrightarrow$ (iii) is exactly Lemma~\ref{lem:catapult-unbounded-loss}
together with the definition of loss spikes
(Definition~\ref{def:loss-spike}) and Corollary~\ref{cor:catapult-loss-spike}.
\end{proof}

Theorem~\ref{thm:equivalence-quadratic} formalizes the slogan in
Section~\ref{section:equivalence}: on the quadratic approximation, the
underlying property is \emph{divergence of the dynamics on $U_t$}, and
the three diagnostics we use throughout the paper—breaking an instability
criterion, observing a catapult, and observing a large enough loss
spike—are just different ways of detecting the same phenomenon.

\section{Modeling of SGD: SGD vs Noisy GD vs SDE}
\label{appendix:noisy_gd_sde}
In this appendix we will discuss the common ways to approximate SGD -- with noisy GD and SDE.

We also present Figure \ref{fig:app_noise_gd} -- a version of the Figure \ref{fig:noise_gd}, but with the loss curves added to showcase the difference in the loss decrease dynamics, together with overall convergence with all the dynamics.
\begin{figure}
    \centering
    \includegraphics[width=.99\linewidth]{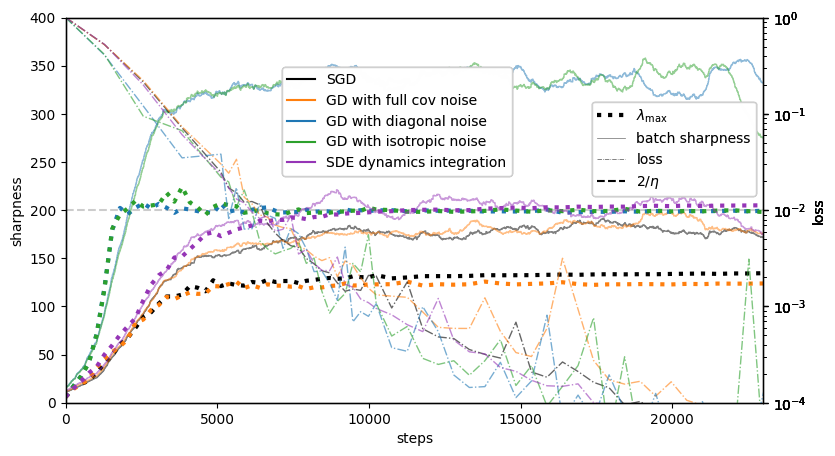}
    \caption{Version of Figure \ref{fig:noise_gd}, with the loss curves added.}
    \label{fig:app_noise_gd}
\end{figure}

\subsection{Longer Discussion on SGD vs.\ Noisy Gradient Descent}

If you are tempted to say “SGD is just GD plus some Gaussian noise,” here is the litmus test:
\emph{inject noise into GD and see if you recover the same curvature-limited regime.}
Under progressive sharpening, the answer is brutally simple: most noise models miss the phenomenon.


To test this, we compare mini-batch SGD (batch size 16) against three noisy GD variants and an SDE integration
(see Appendix~\ref{appendix:noisy_gd_sde} for details):
\textbf{(i)} \emph{GD with Gaussian resampling:}
Gaussian reweighting on samples \citep{wu_noisy_gd_2020}, which preserves the mini-batch structure (and injects noise into the Hessians).
\textbf{(ii)} \emph{Diagonal noise injection:}
Gaussian noise restricted to the diagonal of the SGD noise covariance \citep{zhu_anisotropic_2019}.
\textbf{(iii)} \emph{Isotropic noise injection:}
Gaussian noise with isotropic covariance \citep{zhu_anisotropic_2019}.
\textbf{(iv)} \emph{SDE dynamics integration:}
as in \citep{li_stochastic_2017}.

As shown in Figure~\ref{fig:noise_gd} and Appendix~\ref{appendix:noisy_gd_sde},
only noise that preserves the higher moments of the mini-batch Hessians (and thus the mini-batch landscape structure)
yields an \textsc{EoSS}-like regime with $\lambda_{\max}$ stabilizing well below $2/\eta$.
More generic noise (diagonal or isotropic) fails to reproduce this behavior and instead behaves in a GD-like way.
These experiments suggest a structural separation:
\emph{the stability threshold for mini-batch SGD is governed by \textit{Batch Sharpness}, while noise-injected GD/SDE proxies are governed by $\lambda_{\max}$.}
Notably, this is consistent with the message in \citet{zhu_anisotropic_2019} (though their focus is generalization).
Unsurprisingly, when the injected noise perturbs only gradients (and not the Hessians),
the relevant edge reverts to the GD one---$\lambda_{\max} \approx 2/\eta$---as in \citep{ma_linear_2021,mulayoff_exact_2024}.

\paragraph{Challenges for SDE Modeling.}

A standard response to “SGD is noisy” is to write down an SDE and move on.
Our results explain why this can be qualitatively misleading under progressive sharpening.
\begin{summarybox}
    \textit{Standard SDE (and related) approximations cannot capture the location of stabilization of SGD}\\
\textit{in neural networks under progressive sharpening, because they collapse mini-batch geometry to a mean-field picture.}
\end{summarybox}

Classical analyses typically ignore all Hessian statistics beyond the mean and work on a static landscape:
(\emph{i}) an \textbf{online} view where each step uses an unbiased noisy gradient estimator, or
(\emph{ii}) an \textbf{offline} view treating SGD as noisy GD on the empirical loss.
In both perspectives, the \emph{full-batch} Hessian is treated as the curvature object that matters.

Our results highlight a different mechanism:
each update experiences a Hessian $\mathcal{H}(L_\mathbf{B})$ that can differ substantially from $\mathcal{H}(L)$,
and it is the induced \textit{Batch Sharpness} that stabilizes at $2/\eta$ even when $\lambda_{\max}$ remains smaller.
Prior work already flags serious limitations of SDE-based approaches for implicit regularization:
they can be mathematically ill-posed \citep{yaida_fluctuation-dissipation_2018},
invalid except under restrictive conditions \citep{li_validity_2021},
converge to qualitatively different minima \citep{haochen_shape_2020},
or miss higher-order effects \citep{damian_label_2021,li_what_2022}.
Recent discrete analyses \citep{smith_origin_2021,beneventano_trajectories_2023,roberts_SGD_2021}
aim to address some of these gaps.
Nonetheless, \textsc{EoSS} points to a deeper obstacle:
when batches are small, the geometry of $\mathcal{H}(L_\mathbf{B})$ differs markedly from that of $\mathcal{H}(L)$,
altering both eigenvalues and eigenvector alignments---exactly the information that location and stabilization depend on.

\subsection{Noisy GD}
We are running a number of different versions of noisy GD.
\subsubsection{Noisy GD with Anisotropic Noise (Gaussian Resampling)}
This is version is us running noisy GD with a random sampling vector.This version of noisy GD essentially preserves the second-moment structure of mini-batch landscapes. Specifically, it takes a Gaussian sampling vector with the same first and second moments as the true sampling vector of SGD. This trivially makes the expectation of the mini-batch Hessians to be the same between SGD and Gaussian resampling (and equal to the full-batch Hessian). Importantly, though, this also makes the covariance of the mini-batch Hessians to be the same between SGD and GD with Gaussian resampling noise. Precisely, we have a sampling vector $w \in \R^{N}$ (where $N$ is the number of samples in the dataset), with mean $\mu$ and covariance $\Sigma_w$, and $\mathcal H_i \in \R^{d \times d}$ are the per-sample Hessians of the loss. Now, let's introduce a matrix stacking vectorized these per-sample Hessians, $\mathscr{H} \in \R^{d^2 \times N}$:
\[
\mathscr{H} := \bigg [ 
\vect\left(\mathcal H_1\right), \dots, \vect \left(\mathcal H_N \right)
\bigg]
\]
Then the mini-batch Hessian is
\[
\vect (\mathcal H_B) = \mathscr H w
\]
Therefore, the (uncentered) second moment of the mini-batch Hessian is 
\begin{multline*}
    M_2(\mathcal H_B) = \E_B[\vect(H_B)\vect(H_B)^\top] = \E_B [\mathscr H w w^\top \mathscr H^\top] =
    \mathscr H \E [w w^\top] \mathscr H = \mathscr H (\Sigma_w + \mu\mu^\top)\mathscr H
\end{multline*}
Therefore, the second moment of mini-batch Hessians is fully determined by the first two moments of the sampling vector. Matching those two to the ones of the true SGD sampling vector allows one to preserve the second moment of the mini-batch Hessians. Together with the fact that GD with the Gaussian resampling behaves in the same manner as SGD from the point of view of stability, \emph{Batch Sharpness}, and suppression of $\lmax$ (as opposed to the version of noisy GD that do not preserve the second moment of the Hessians)---it is an indicator of the fact that it is the higher moments of the mini-batch Hessians that determine the dynamics of SGD (at least up to the second moment); and it is indeed the noise in the Hessians that creates the instability regime of \textsc{EoSS} and its consequences.

For implementation details refer to \citet{wu_noisy_gd_2020}. In summary, we re-draw the sampling vector at each step with the corresponding covariance.

\subsubsection{Noisy GD with Diagonal Noise}
This implementation follows \citet{zhu_anisotropic_2019} --- it recreates what they refer to as "GLD diagonal". This is essentially noisy GD with the noise covariance being equal to the diagonal of the covariance of the noise produced by SGD. This preserves each parameter’s marginal variance while ignoring off‑diagonal correlations.
Conceptually, we are approximating SGD’s noise by $\mathcal{N}\left(0, \frac{1}{b}\mathrm{diag}(\Sigma(\theta)\right )$ and add it to the full‑batch gradient before the optimizer step.
Essentially, this is one step further from a true SGD then the aforementioned GD with anisotropic noise. In particular, it does not preserve the mini-batch landscape structure. As a result, the behavior of GD with diagonal noise differs from SGD from the point of view of $\lmax$ stabilizing below $2/\eta$, and instead stabilizing at $2/\eta$. We refer the reader to \citep{zhu_anisotropic_2019} for the details of implementation. In our implementation, we compute the diagonal of the covariance every 30 steps and reuse it on those 30 steps (as it is too computationally expensive to compute it at every step).

\subsubsection{Noisy GD with Isotropic Noise}
This implementation follows \citet{zhu_anisotropic_2019} --- it recreates what they refer to as "GLD dynamic". This is essentially noisy GD with the noise covariance being identity (hence the "isotropic"), scaled such that the magnitude of the noise conincides with that of SGD. That is, this is isotropic gradient noise that matches the average variance of SGD noise but ignores both parameter-wise variability and correlations. 
Conceptually, we are approximating SGD’s noise by
$\mathcal{N}\!\left(0, \tfrac{\sigma^{2}}{b} I\right)$ add it to the full‑batch gradient before the optimizer step, where \(\sigma^{2} = \tfrac{\operatorname{tr}(\Sigma)}{d}\) 
is the mean per-parameter variance from the per-sample gradient covariance \(\Sigma\), 
\(b\) is the target batch size, and \(d\) is the number of parameters.
This is one step "further" from SGD then the noisy GD with diagonal noise. Consequently, this sort of noisy GD does not preserve the regularization effect of SGD on $\lmax$ either.

\begin{figure}
    \centering
    \includegraphics[width=0.99\linewidth]{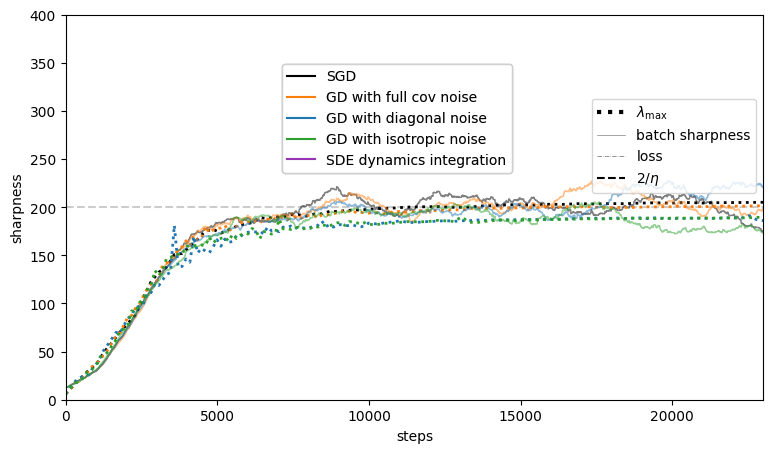}
    \caption{\textbf{SDE sample paths} Multiple realizations of SDE trajectory to showcase the similarity of the solutions found by SDE dynamics}
    \label{fig:sde_trajectories}
\end{figure}

\subsection{SDE}
We are taking the standard SDE approximation of SGD:  (see e.g. \citet{li_sme_2018}) 
\[
d\theta_t \;=\; -\nabla f(\theta_t)\,dt \;+\; \sqrt{\eta}\,\,\Sigma^{1/2}(\theta_t)\,dW_t
\;
\]
where $dW_t$ is the standard d-dimensional Wiener process, and $\Sigma$ is the covariance matrix of mini-batch gradients.

To simulate its dynamics, we are using the Euler–Maruyama discretization with a step size of $0.0005$, chosen to be sufficiently small compared to  $\eta$ ($1/20$th of $\eta=0.01$ in this example). In Figure \ref{fig:sde_trajectories} we are showing a number of sample paths of the SDE trajectory illustrate the similarity in the properties of the solutions found by those dynamics -- in particular, that $\lmax$ stabilizes around $2/\eta$, rather than below as it does for SGD dynamics. In all the experiments, batch size is 16, and $\eta$ is $0.01$.

\clearpage


\section{Alignment}
\label{appendix:alignment}
The stability of mini-batch SGD is governed by the geometry of the mini-batch loss landscape, rather than solely the full-batch landscape. 
As further discussed in Sections \ref{section:framework}, \ref{section:math}, \ref{appendix:linear_stochastic_stability}, \ref{appendix:other_quantities}, the stability of SGD depends not only on the magnitude of mini-batch Hessians but also critically on their alignment (both pairwise and with the loss gradients).
This appendix offers a limited characterization of the alignment structure relevant to mini-batch Hessians.

We approximate both the full-batch and mini-batch Hessians by their Gauss–Newton matrices, an approximation commonly used in analyses of SGD (e.g., \citep{wu_how_2018, ma_linear_2021, mulayoff_exact_2024}), valid at convergence, and supported empirically (e.g., \citep{papyan_full_2019} and Appendix \ref{appendix:FIM}). Concretely,
\[
\nabla^2_\theta L_B(\theta) \approx \frac{1}{b}\sum_{i=1}^B J_i^\top H_{z, i} J_i
\]
where $J_i$ is the Jacobian of the model output with respect to $\theta$ and $H_{z,i}$ is the loss Hessian with respect to the output evaluated at the $i$-th sample. 
We work with MSE and, for simplicity, consider a two-class setting (i.e., a single model output), which yields:
\[
\nabla^2_\theta L_B(\theta) \approx \frac{1}{b}\sum_{i=1}^b \nabla_\theta f_i \nabla_\theta f_i^\top
\]
where $\nabla f_i$ is the per-sample model gradient. 

Under this structure, properties of mini-batch Hessians---including their top eigenvalues and the cross-batch alignment/commutativity---are controlled (though not fully determined) by the pairwise alignment of the per-sample model gradients.
We report the empirical distributions of these pairwise alignments in Figure~\ref{fig:alignmnet_mlp} (MLP) and Figure~\ref{fig:alignmnet_cnn} (CNN), plotting pairwise dot products, plotting pairwise dot products, cosine similarity and individual norms. The evolution of the distributions throughout the training indicates the effects of progressive sharpening---for example, the growth of norm of model gradients corresponds to the increase of $\lmax$ and $\lmax^b$, see Appendix \ref{appendix:other_quantities}. A complete description of the dynamics would require a precise account of progressive sharpening and falls outside the scope of this work.

The important observation for our purposes is that per-sample model gradients are only \emph{weakly} aligned, with cosine similarities clustered around $0.1$---which is still much higher than random $d$-dimensional vectors would have. Two immediate implications follow. First, mini-batch Hessians are generically non-commuting (as the model gradients are not orthogonal or completely collinear)---an aspect that matters for linear stochastic stability (Appendix \ref{appendix:linear_stochastic_stability}).
\ref{appendix:linear_stochastic_stability}. Second, if we fix the parameters $\theta$ and vary the batch size b, the eigenspaces of the mini-batch Hessian mix gradually, which induces a gap between \textit{Batch Sharpness} and \(\lmax^b\) (see Appendix \ref{appendix:other_quantities}). We leave full characterization of the mini-batch Hessians, which would depend on the higher moments of model gradients for future work. Still, these observations underscore that a comprehensive, training-time characterization of the structure of mini-batch Hessians is an important future direction of research for understanding SGD dynamics.

\begin{figure}[htbp]
  \centering

\begin{subfigure}[t]{0.32\textwidth}
  \centering \includegraphics[height=3.2cm]{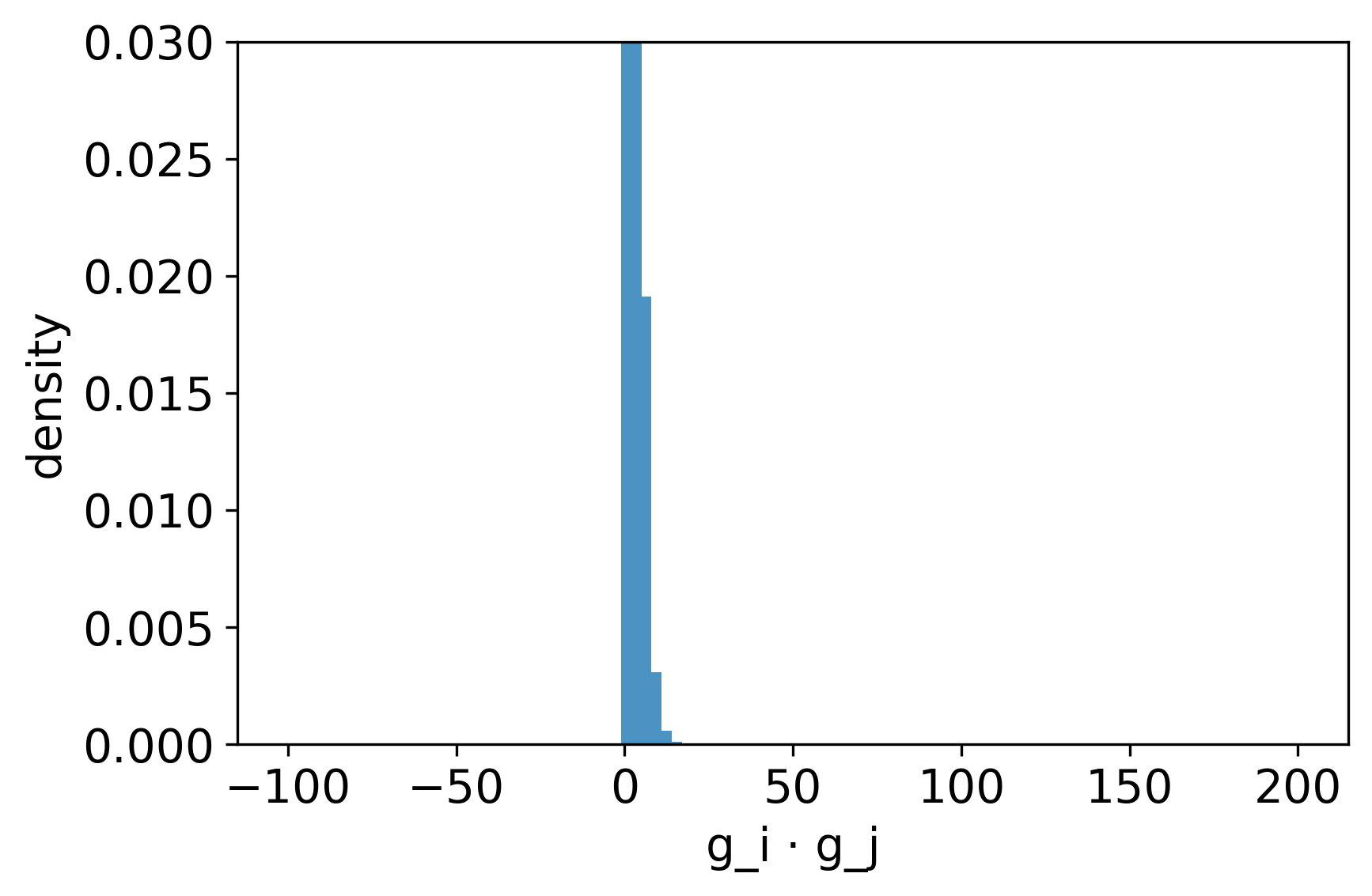}
  \caption{Dot product at init}
\end{subfigure}\hfill
\begin{subfigure}[t]{0.32\textwidth}
  \centering \includegraphics[height=3.2cm]{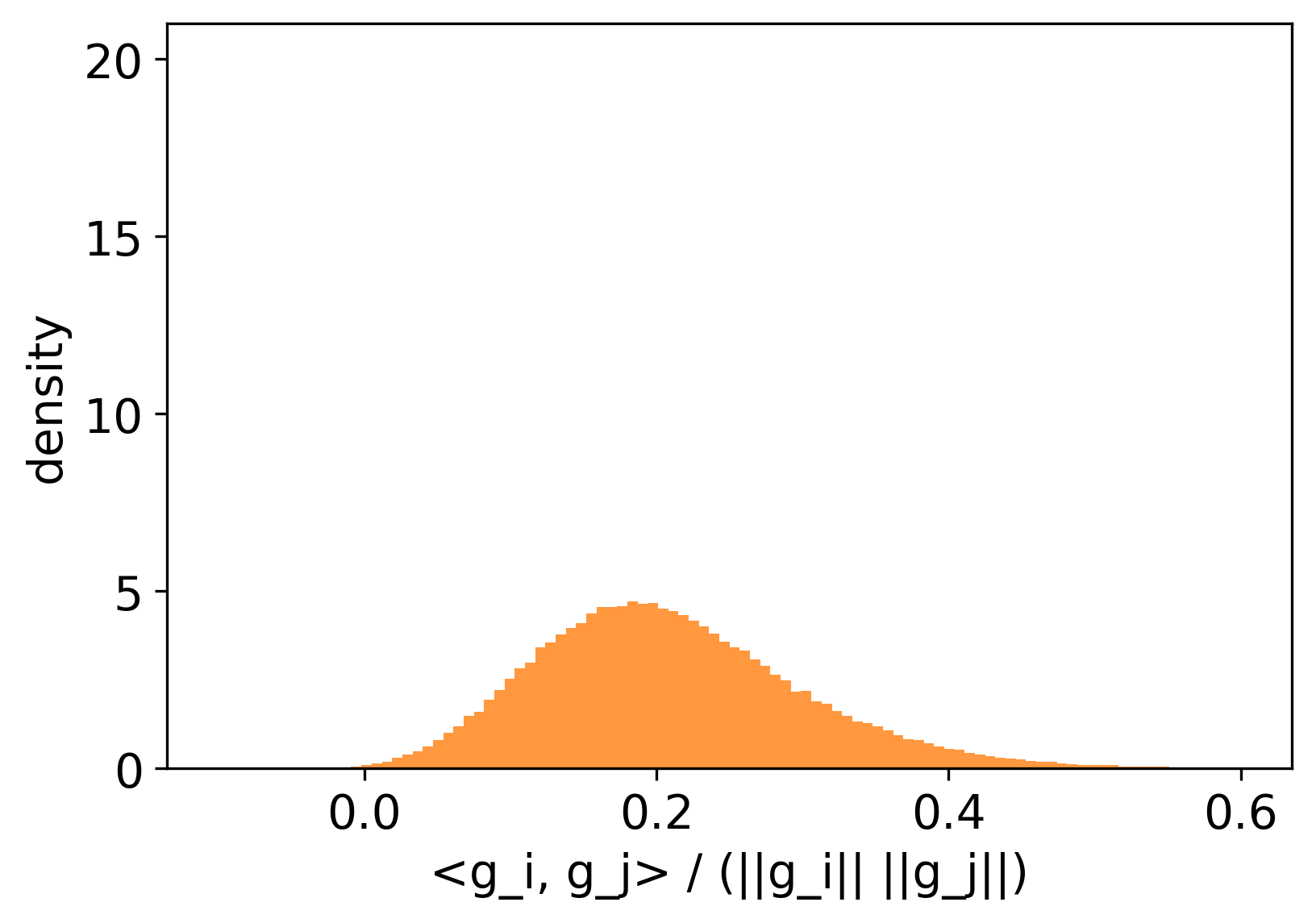}
  \caption{Cosine similarity at init}
\end{subfigure}\hfill
\begin{subfigure}[t]{0.32\textwidth}
  \centering \includegraphics[height=3.2cm]{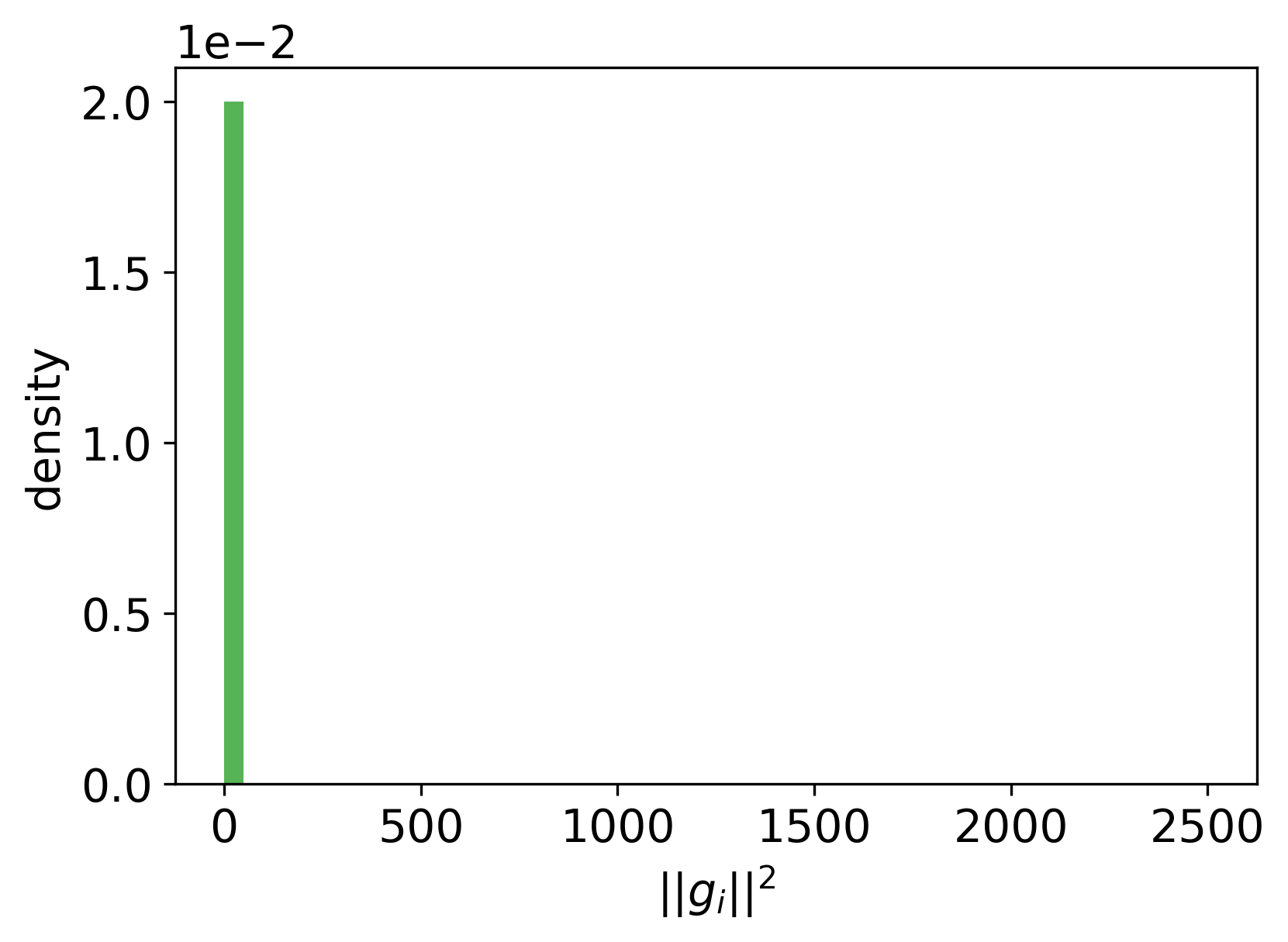}
  \caption{Norms at init}
\end{subfigure}

\medskip

\begin{subfigure}[t]{0.32\textwidth}
  \centering \includegraphics[height=3.2cm]{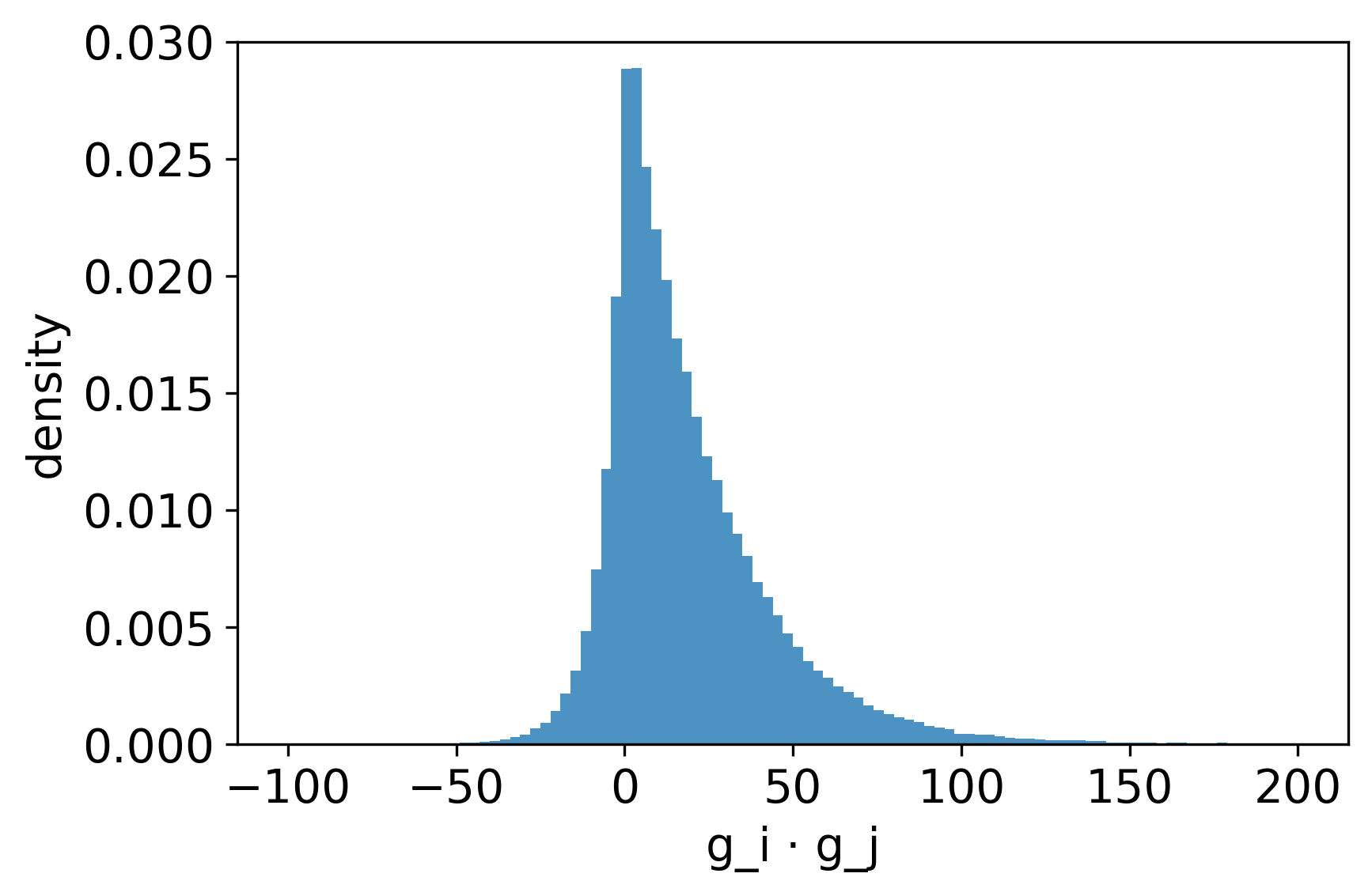}
  \caption{Dot product early-training}
\end{subfigure}\hfill
\begin{subfigure}[t]{0.32\textwidth}
  \centering \includegraphics[height=3.2cm]{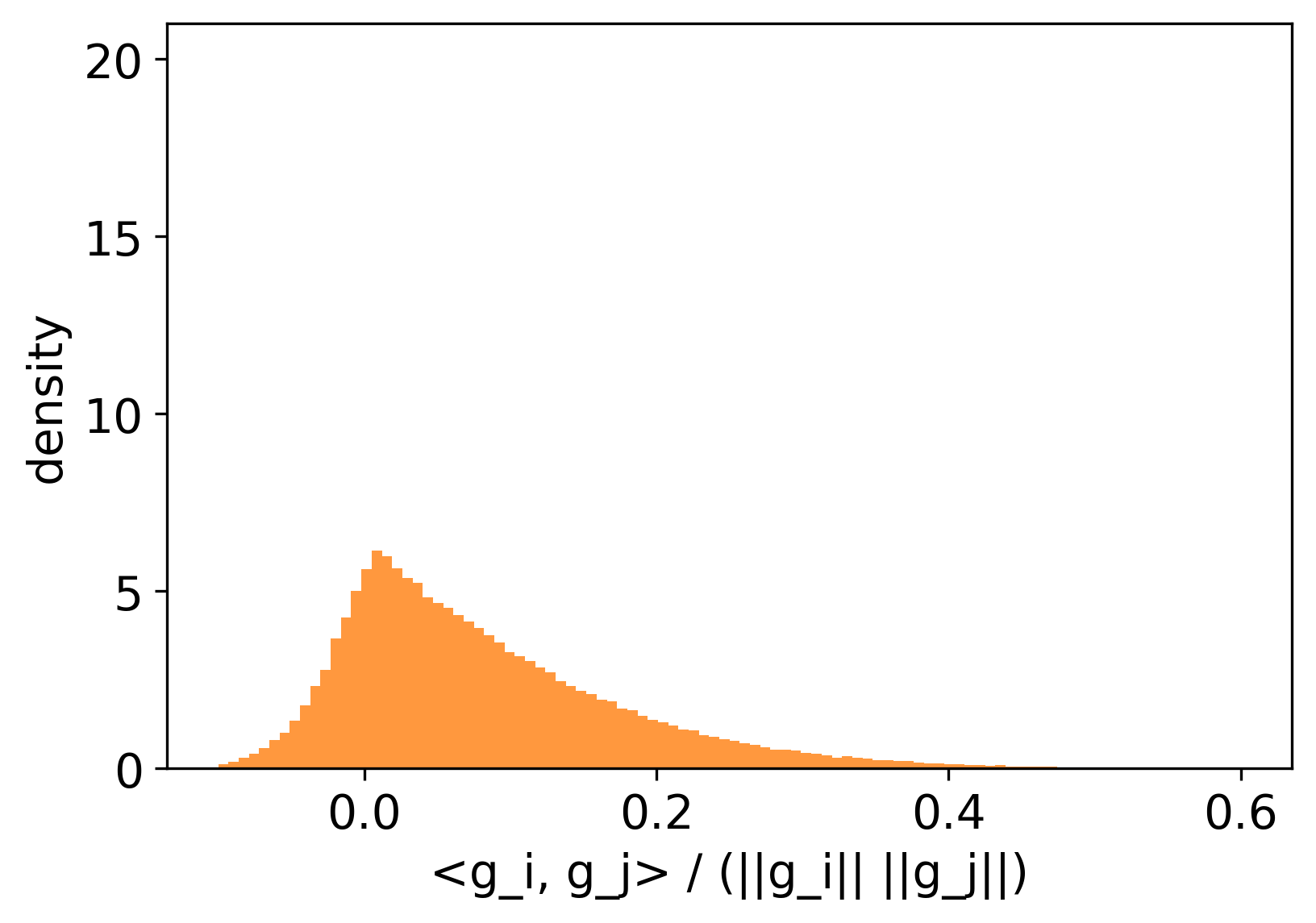}
  \caption{Cosine similarity early-training}
\end{subfigure}\hfill
\begin{subfigure}[t]{0.32\textwidth}
  \centering \includegraphics[height=3.2cm]{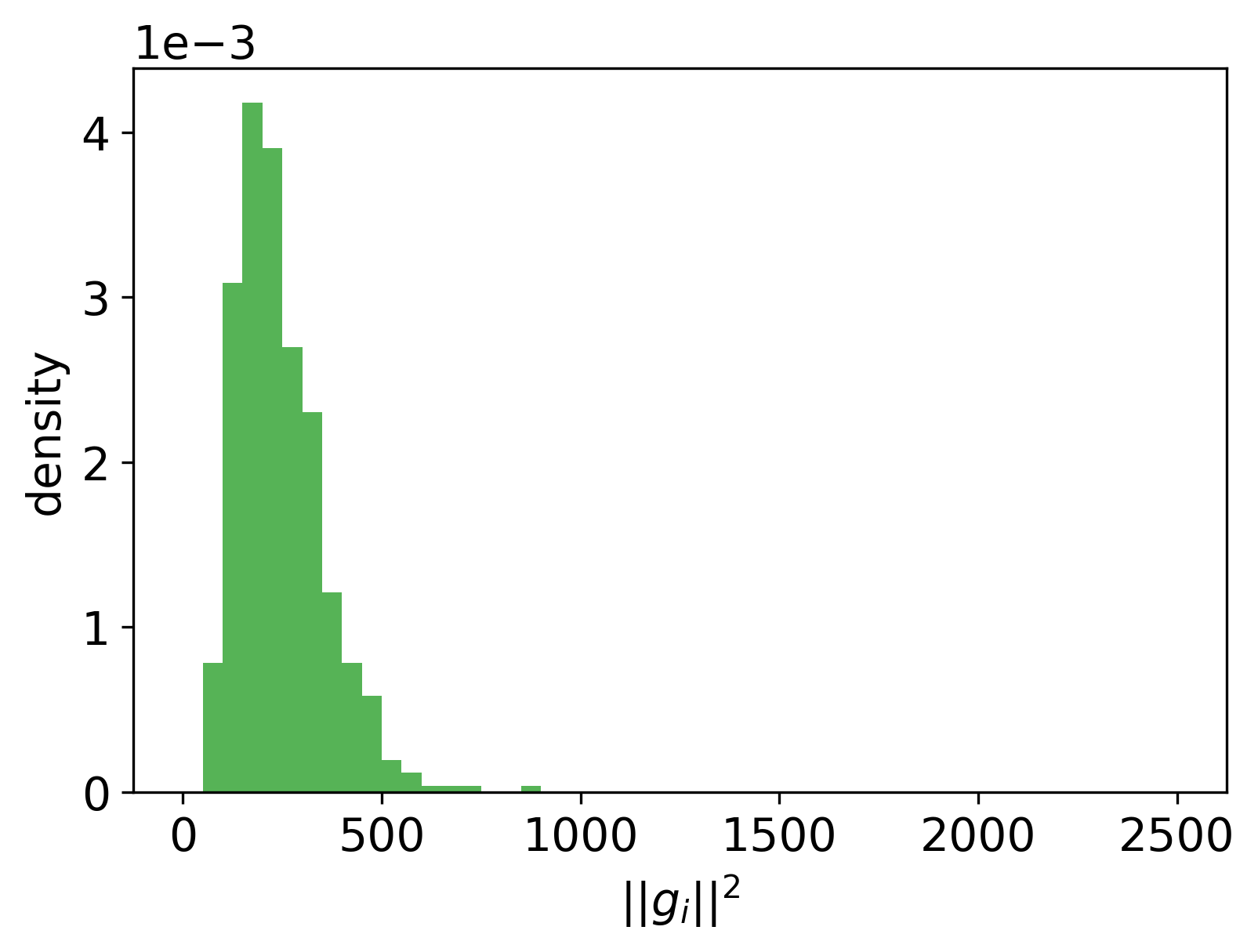}
  \caption{Norms early-training}
\end{subfigure}

\medskip

\begin{subfigure}[t]{0.32\textwidth}
  \centering \includegraphics[height=3.2cm]{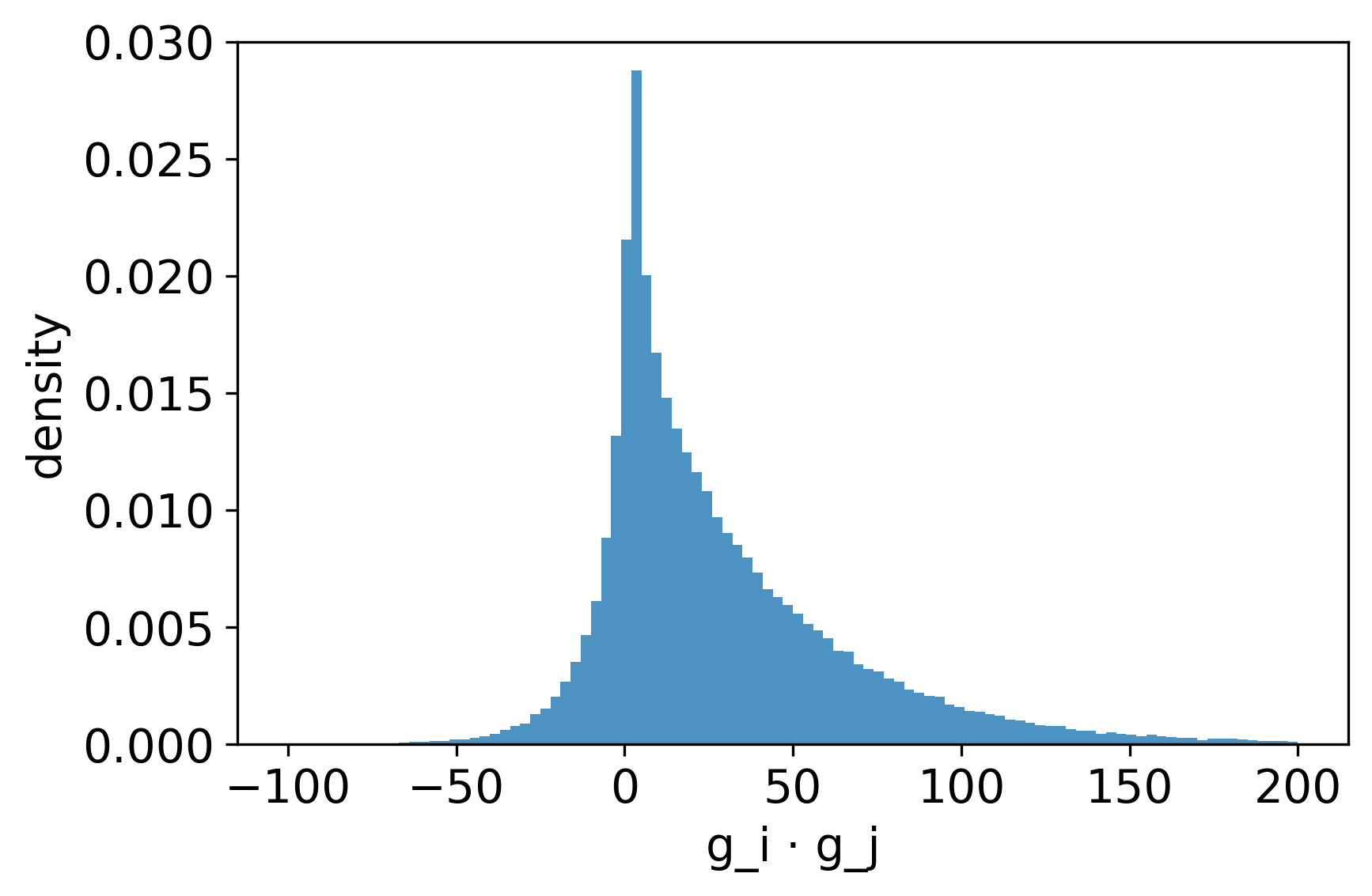}
  \caption{Dot product mid-training}
\end{subfigure}\hfill
\begin{subfigure}[t]{0.32\textwidth}
  \centering \includegraphics[height=3.2cm]{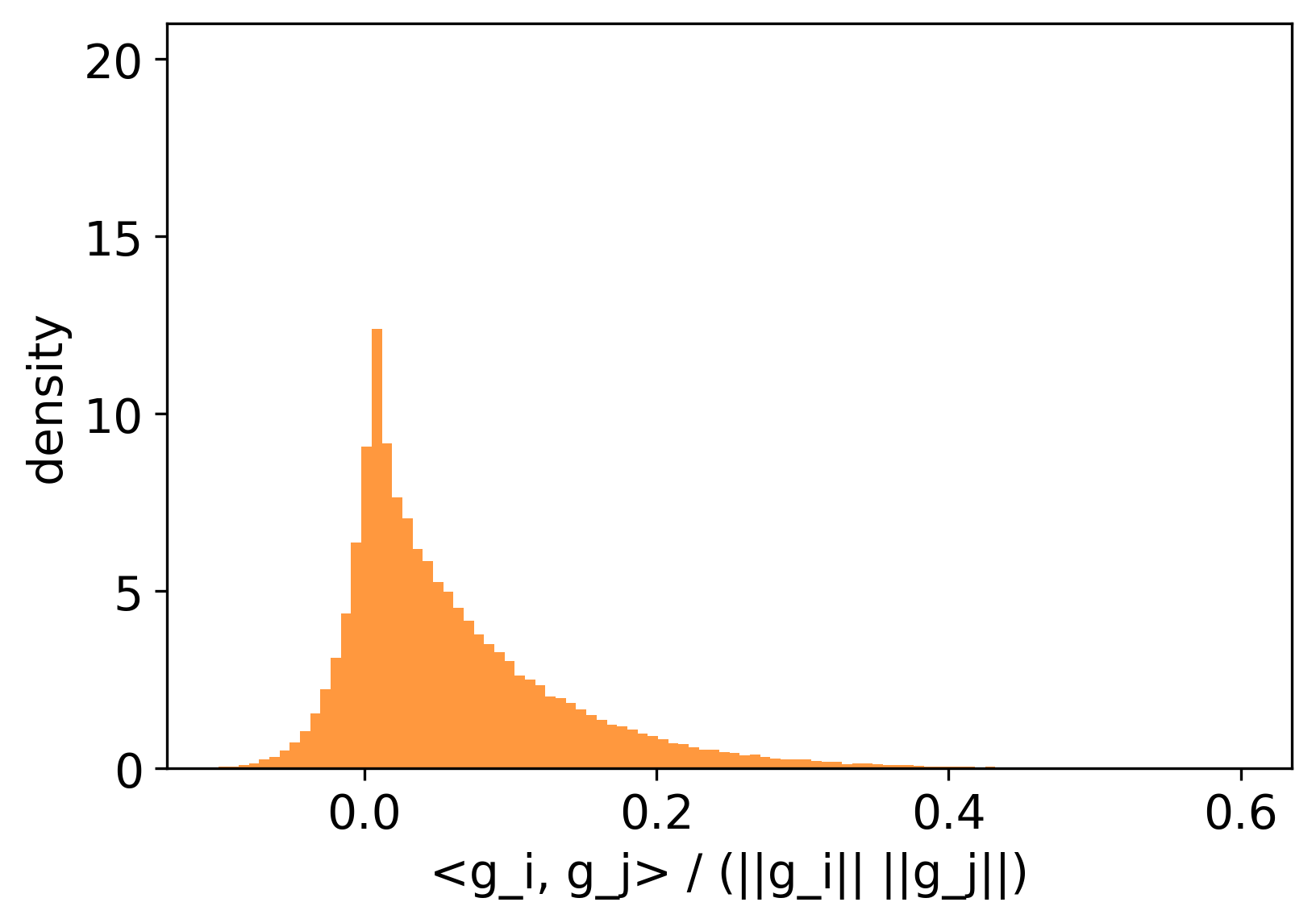}
  \caption{Cosine similarity mid-training}
\end{subfigure}\hfill
\begin{subfigure}[t]{0.32\textwidth}
  \centering \includegraphics[height=3.2cm]{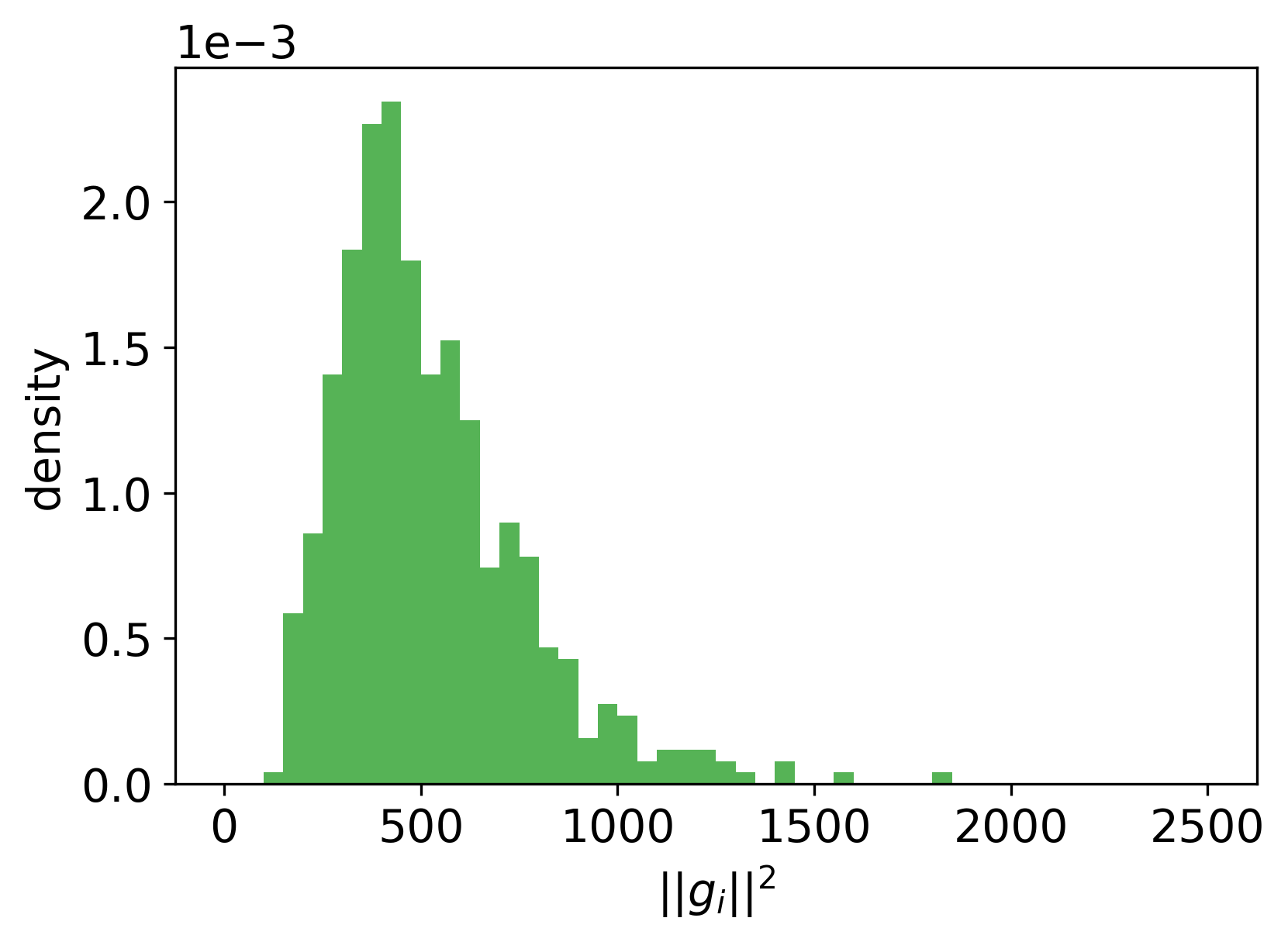}
  \caption{Norms mid-training}
\end{subfigure}

\medskip

\begin{subfigure}[t]{0.32\textwidth}
  \centering \includegraphics[height=3.2cm]{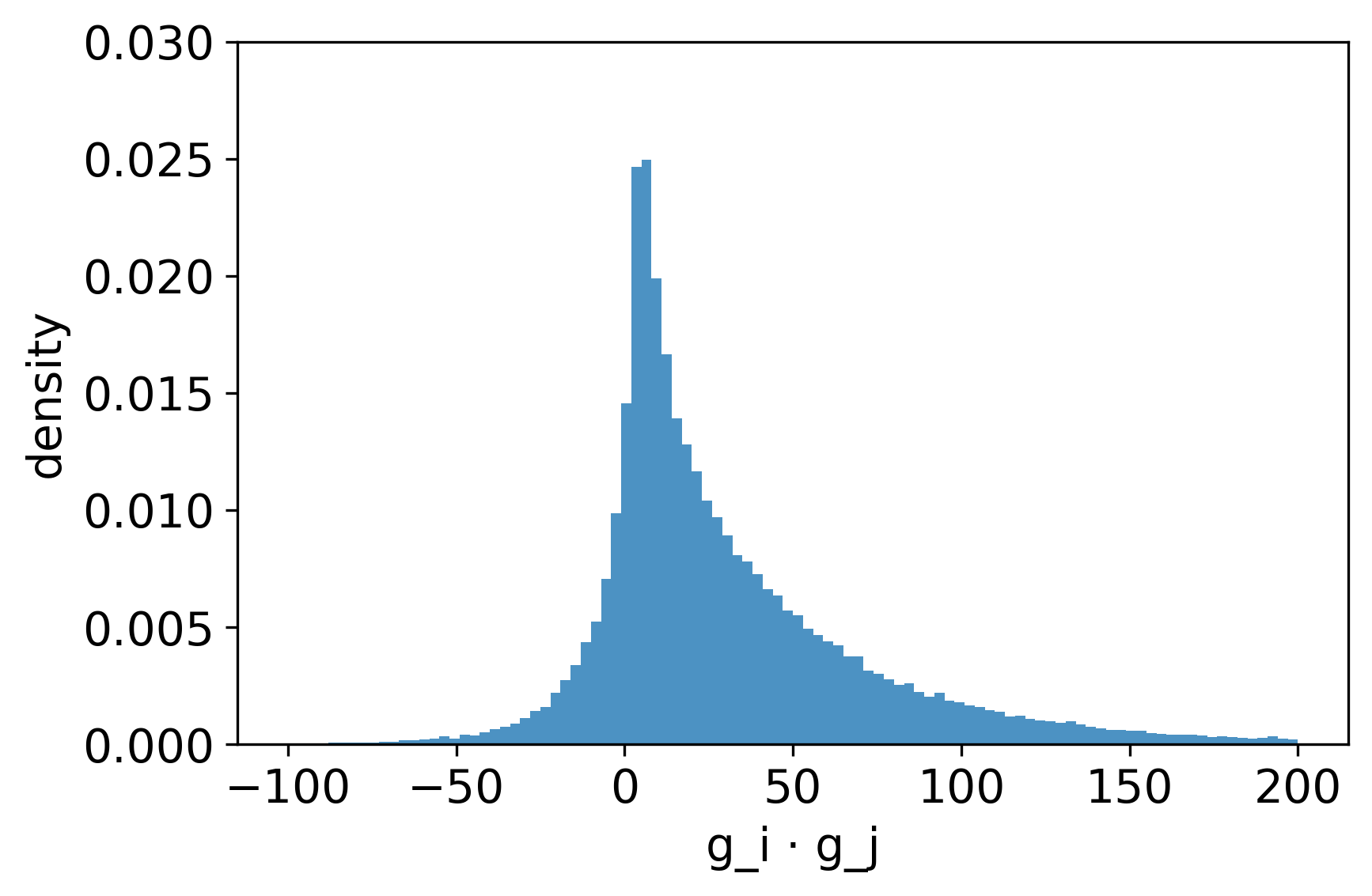}
  \caption{Dot product at \textsc{EoSS}}
\end{subfigure}\hfill
\begin{subfigure}[t]{0.32\textwidth}
  \centering \includegraphics[height=3.2cm]{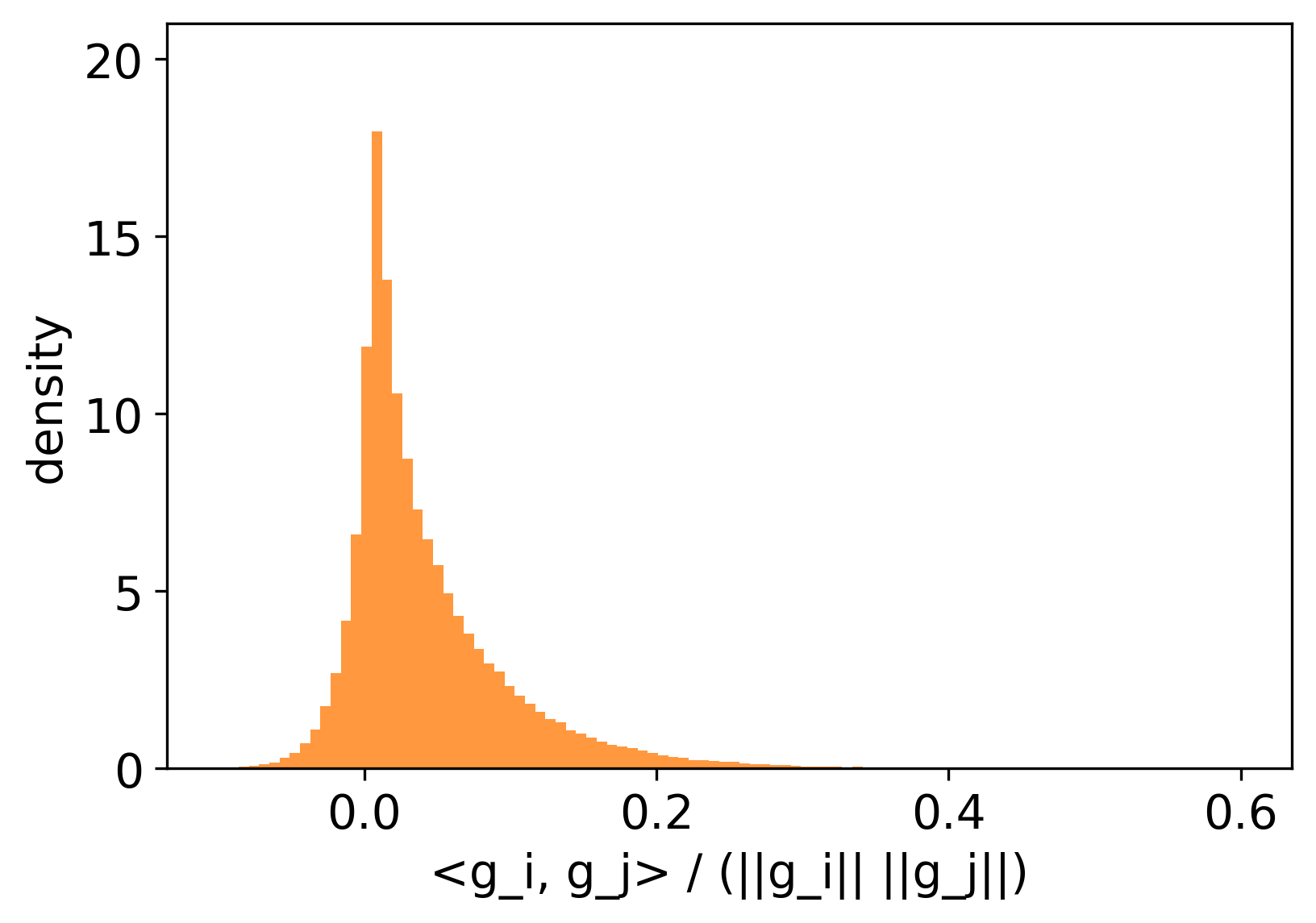}
  \caption{Cosine similarity at \textsc{EoSS}}
\end{subfigure}\hfill
\begin{subfigure}[t]{0.32\textwidth}
  \centering \includegraphics[height=3.2cm]{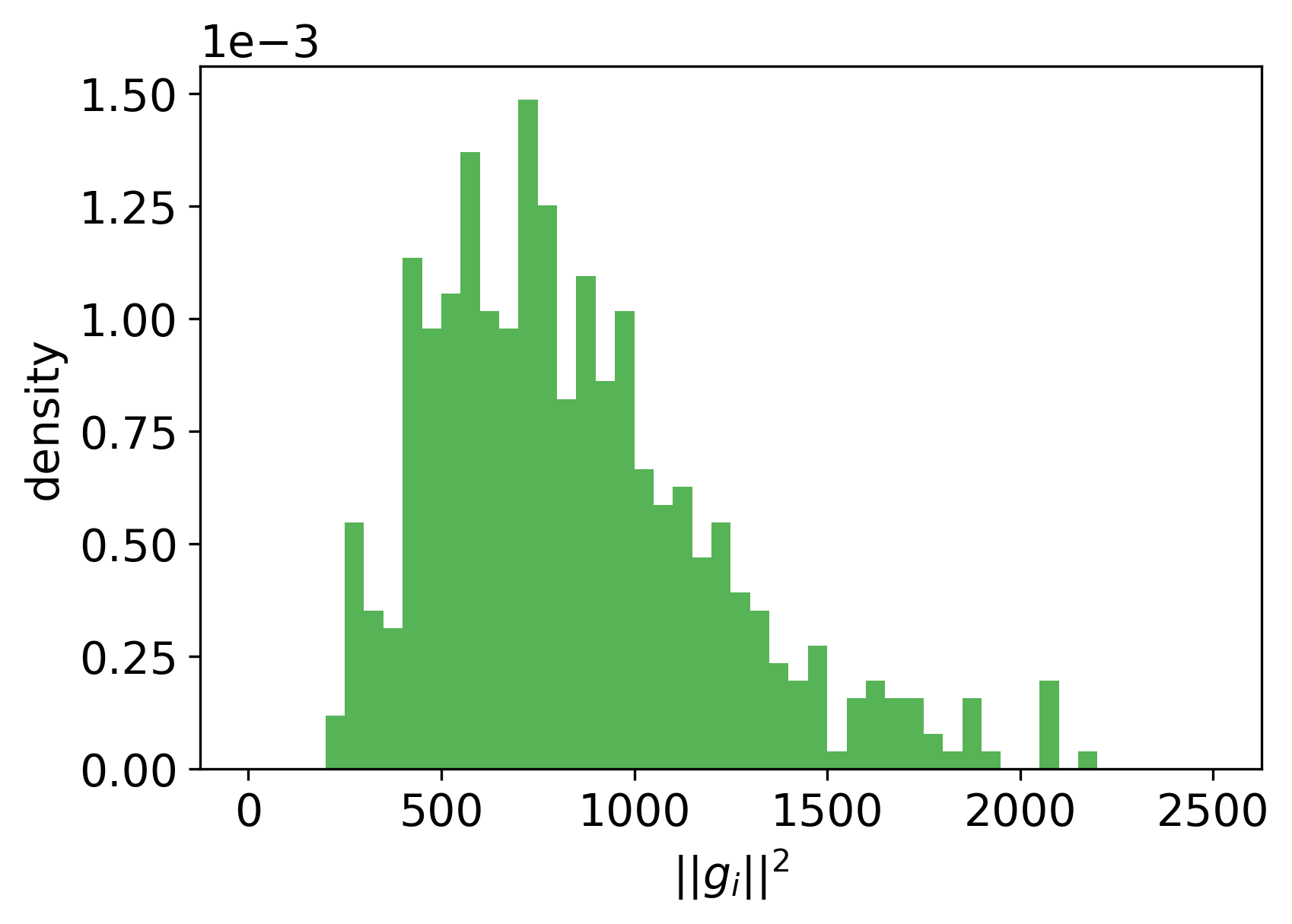}
  \caption{Norms at \textsc{EoSS}}
\end{subfigure}

\medskip

\begin{subfigure}[t]{0.32\textwidth}
  \centering \includegraphics[height=3.2cm]{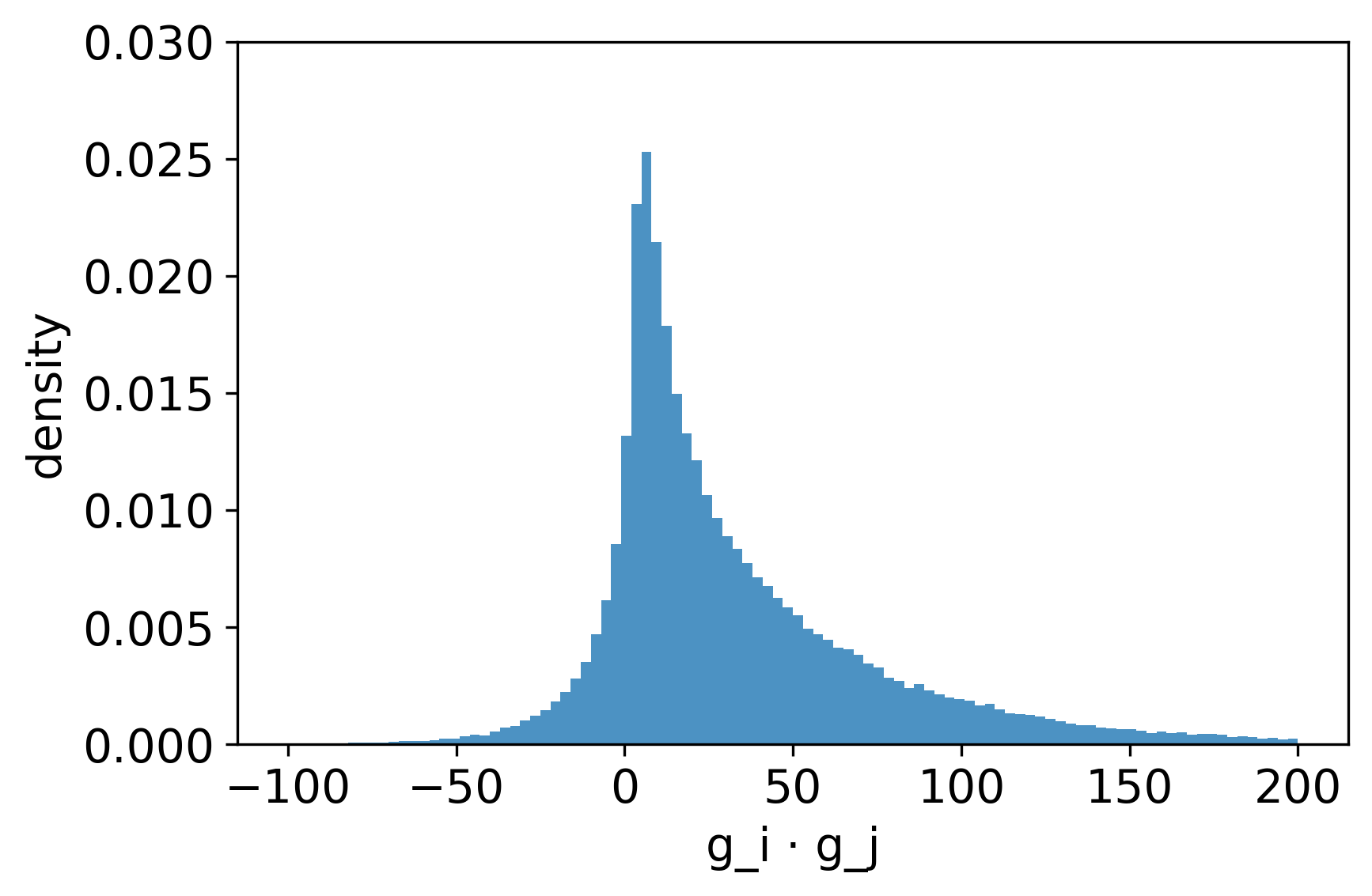}
  \caption{Dot product at convergence}
\end{subfigure}\hfill
\begin{subfigure}[t]{0.32\textwidth}
  \centering \includegraphics[height=3.2cm]{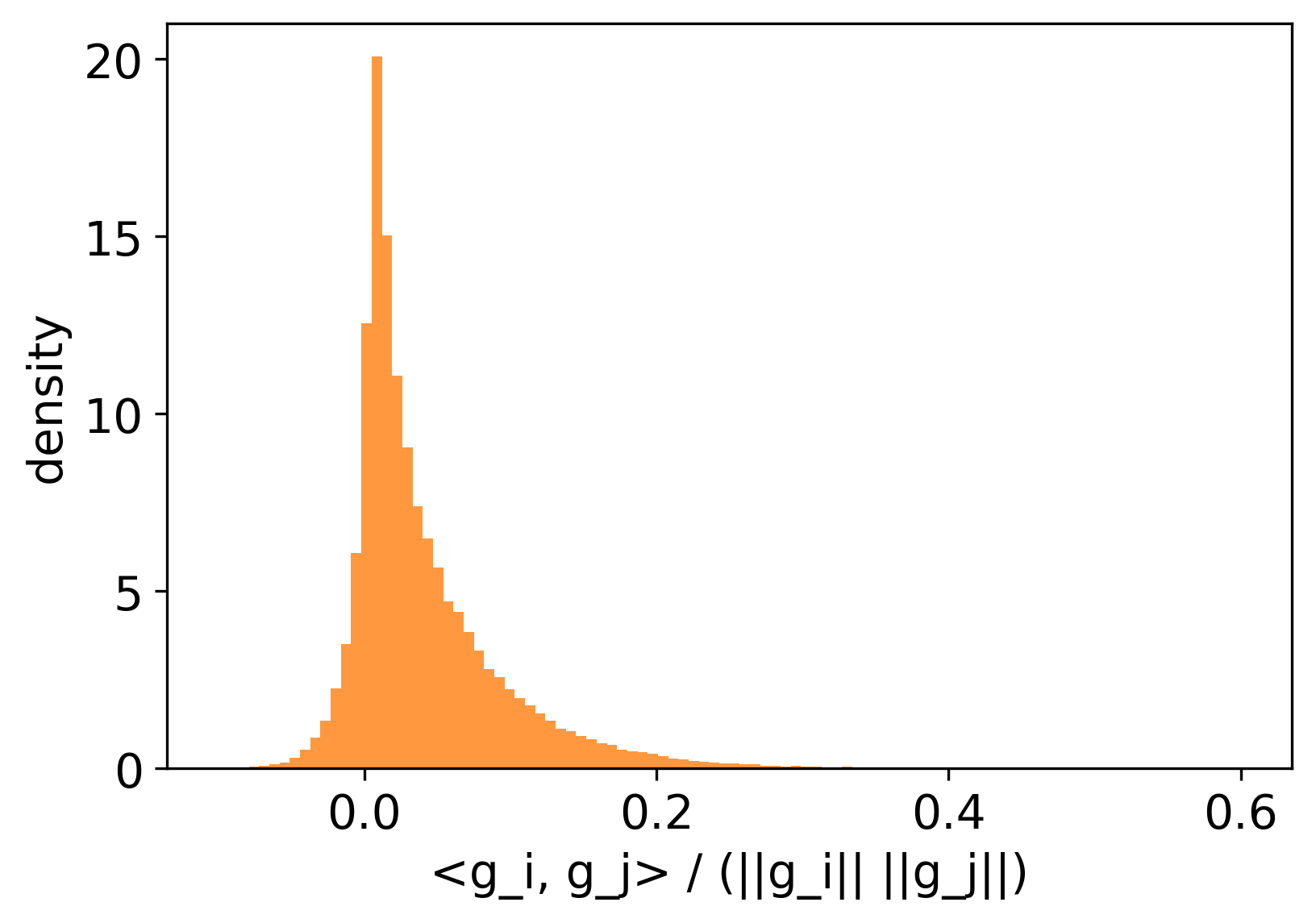}
  \caption{Cosine similarity at convergence}
\end{subfigure}\hfill
\begin{subfigure}[t]{0.32\textwidth}
  \centering \includegraphics[height=3.2cm]{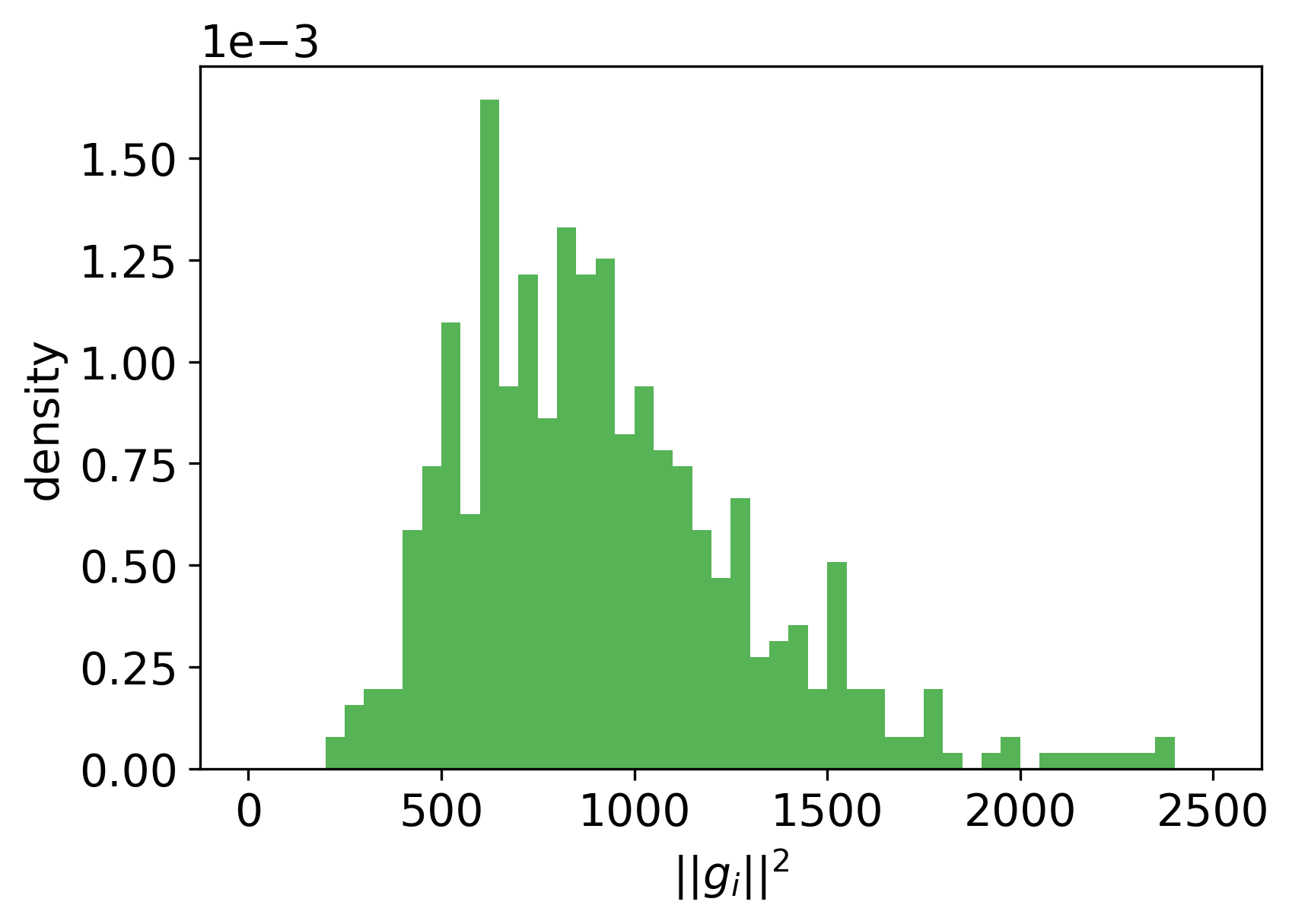}
  \caption{Norms at convergence}
\end{subfigure}

  \caption{\textbf{Model gradients alignment.} Pairwise alignment between model gradients forming the Hessian as the training progresses. We show dot products ($i\neq j$), cosine similarities and the squared norms of the model gradients. Each row corresponds to a stage of training---from initialization, to mid-training (during progressive sharpening), to the later stages (at \textsc{EoSS} and convergence). Notice the gradients become weakly aligned throughout the training (with the cosine similarities clustered around $0.1$), but not completely orthogonal, as it would have been with random vectors. MLP, CIFAR-8k (2 classes), $\eta=0.02$, batch size 32}
  \label{fig:alignmnet_mlp}
\end{figure}

\begin{figure}[htbp]
  \centering

\begin{subfigure}[t]{0.32\textwidth}
  \centering \includegraphics[height=3.2cm]{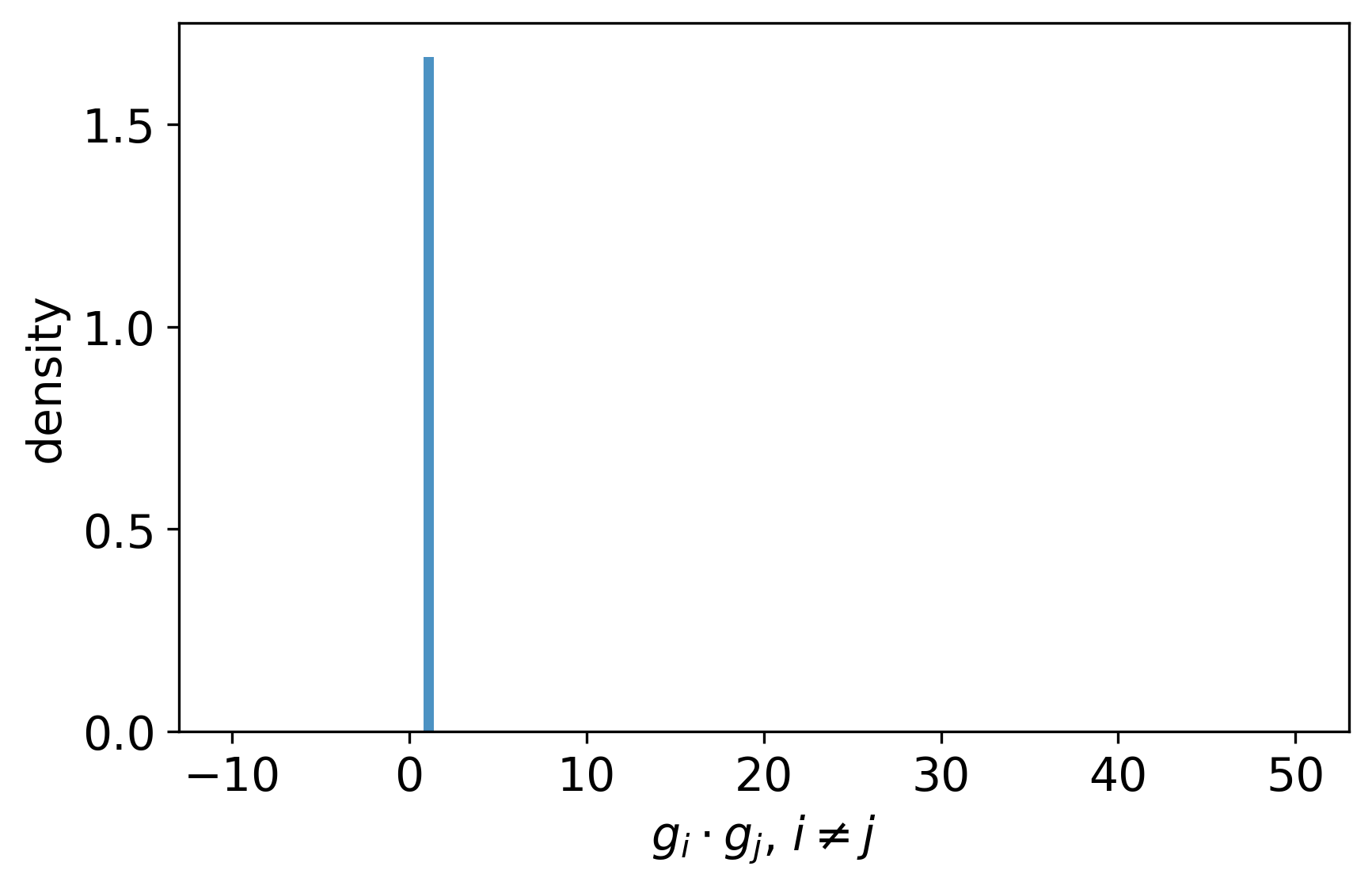}
  \caption{Dot product at init}
\end{subfigure}\hfill
\begin{subfigure}[t]{0.32\textwidth}
  \centering \includegraphics[height=3.2cm]{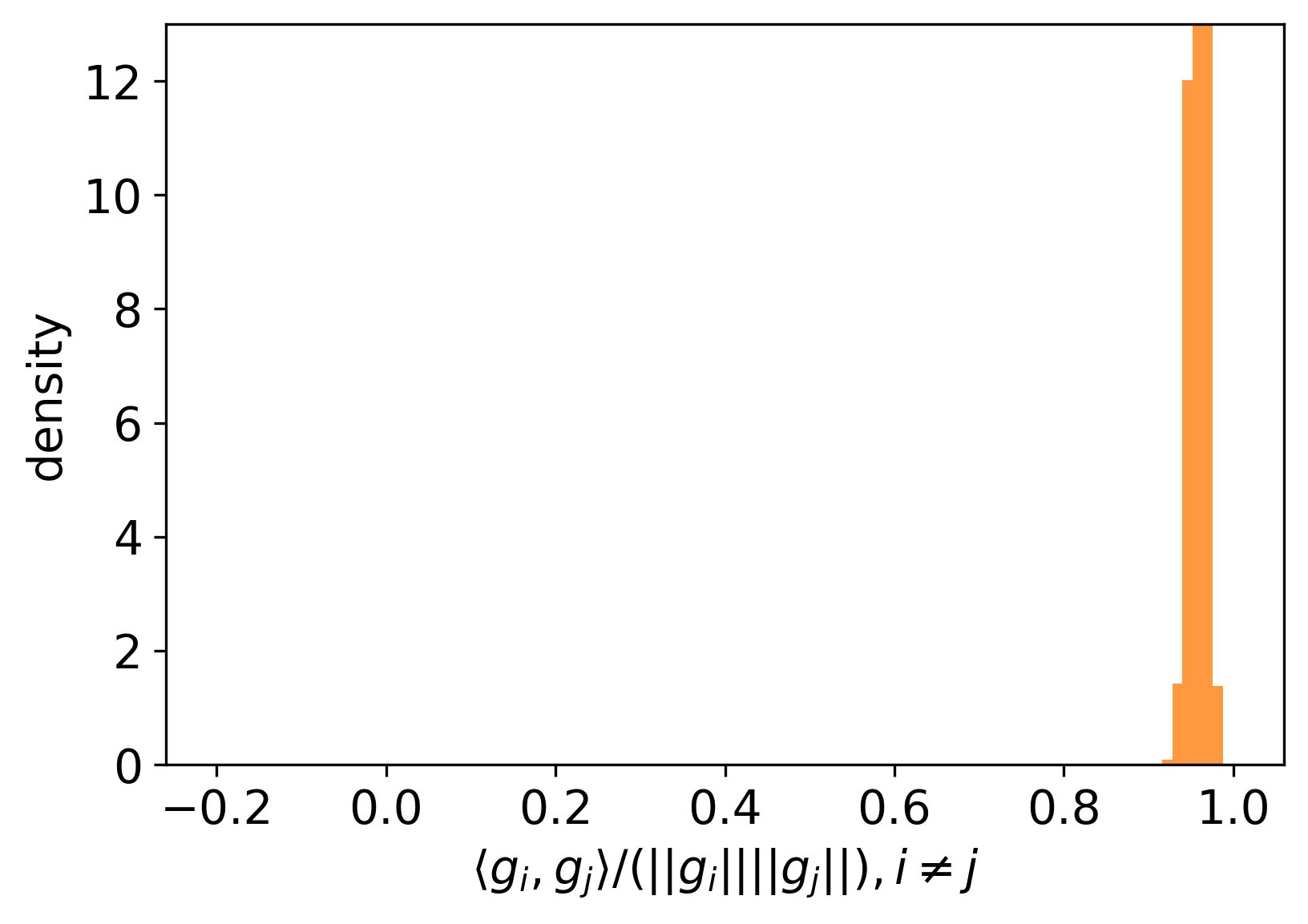}
  \caption{Cosine similarity at init}
\end{subfigure}\hfill
\begin{subfigure}[t]{0.32\textwidth}
  \centering \includegraphics[height=3.2cm]{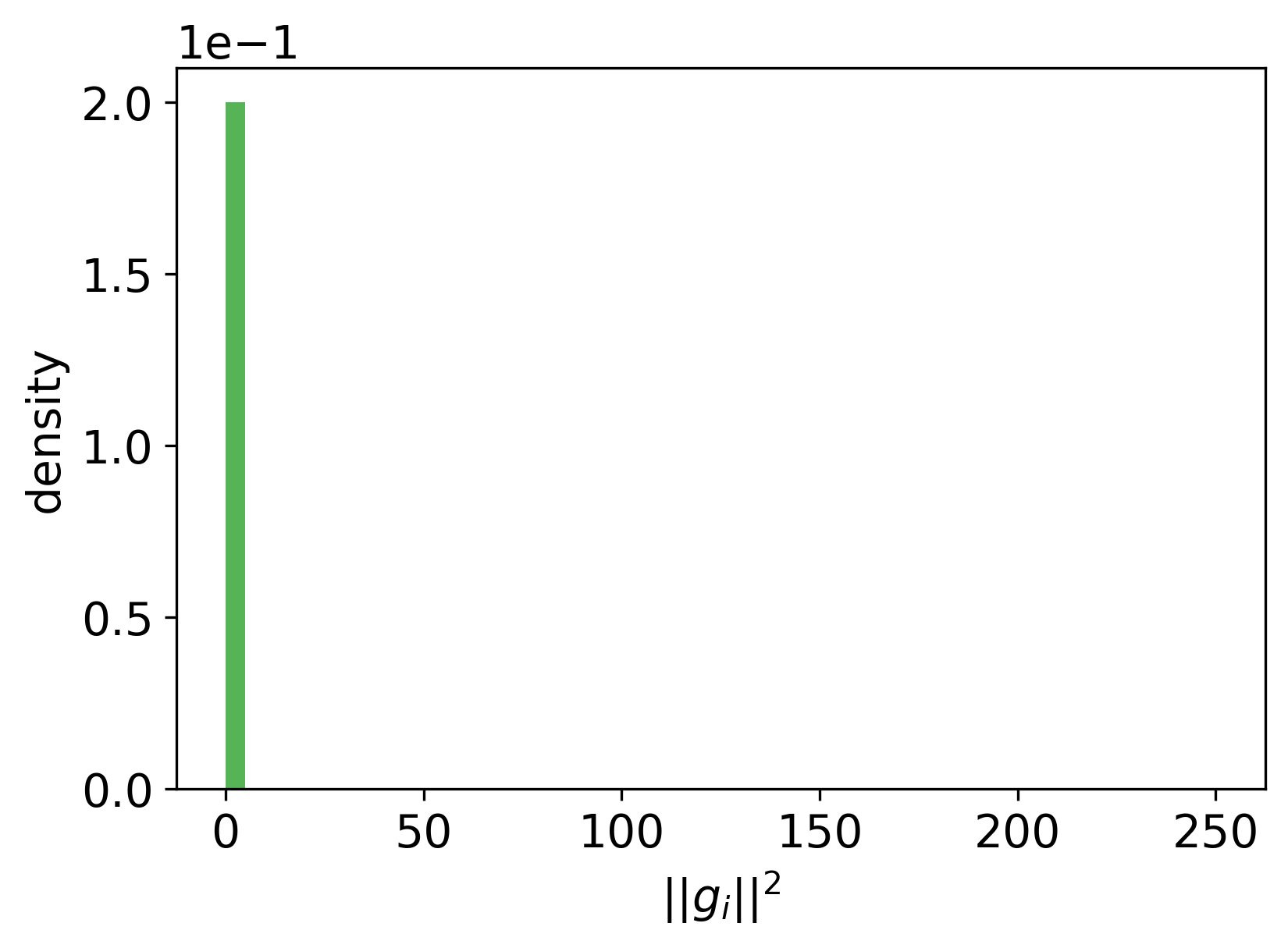}
  \caption{Norms at init}
\end{subfigure}

\medskip

\begin{subfigure}[t]{0.32\textwidth}
  \centering \includegraphics[height=3.2cm]{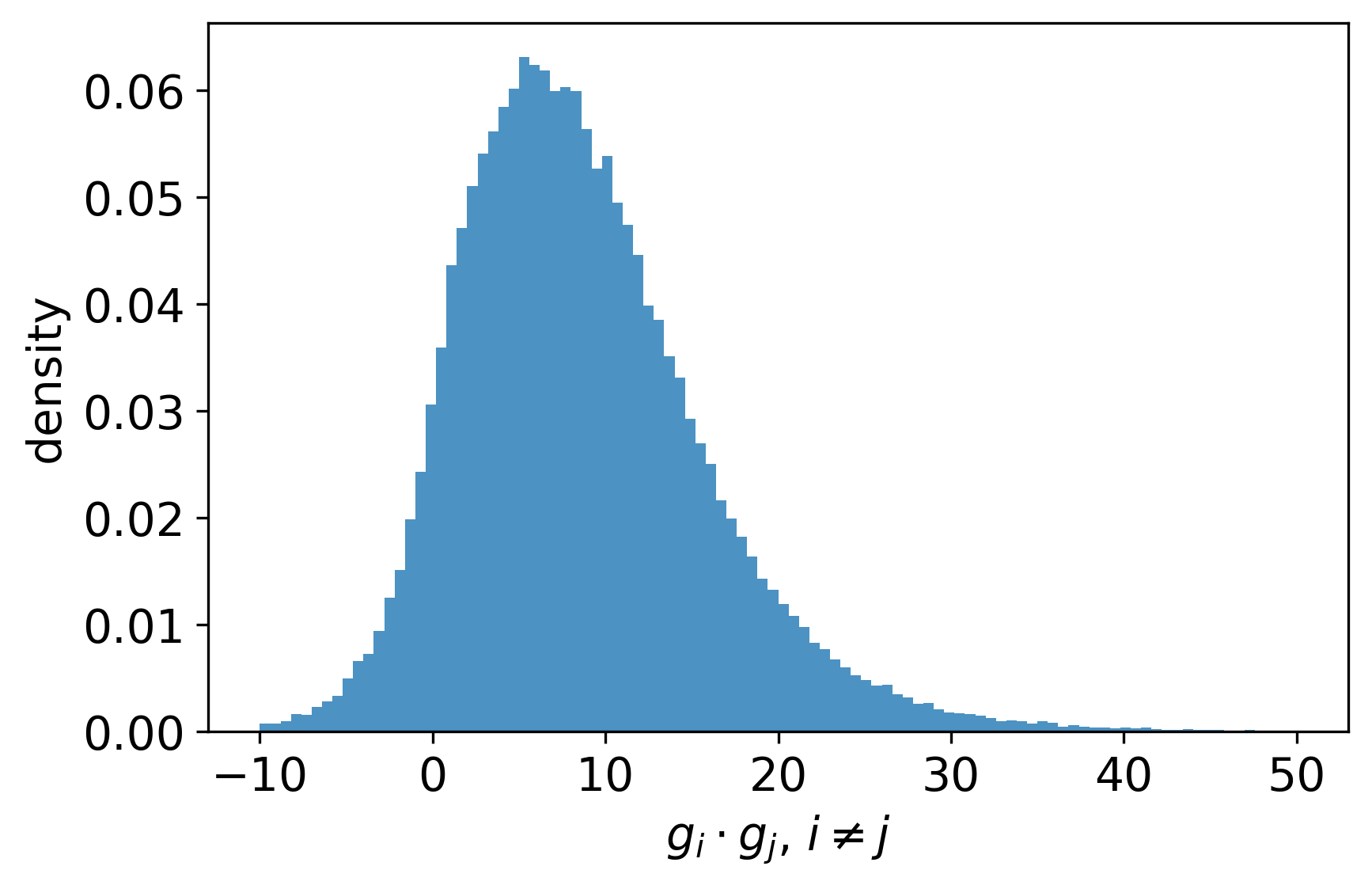}
  \caption{Dot product early-training}
\end{subfigure}\hfill
\begin{subfigure}[t]{0.32\textwidth}
  \centering \includegraphics[height=3.2cm]{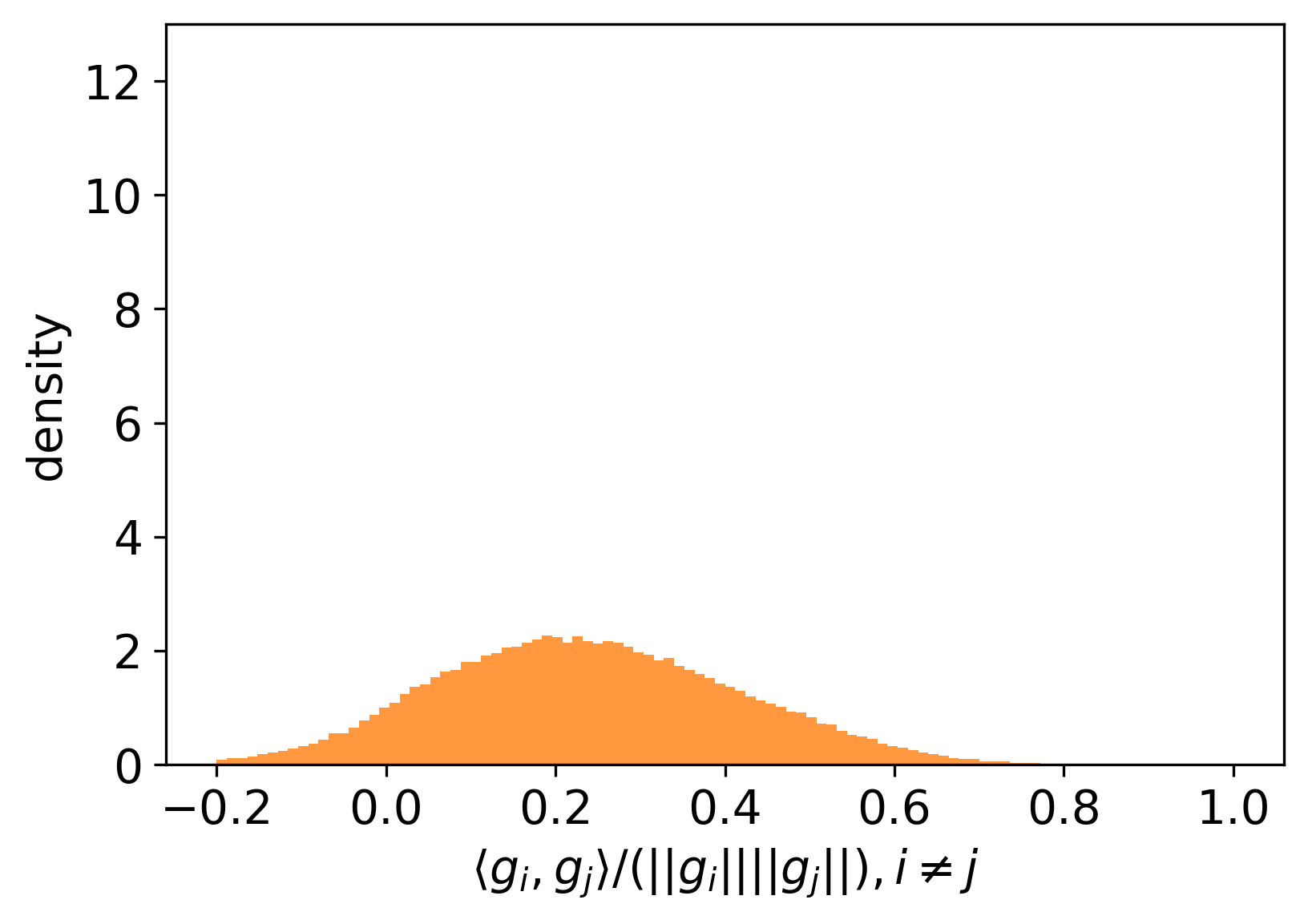}
  \caption{Cosine similarity early-training}
\end{subfigure}\hfill
\begin{subfigure}[t]{0.32\textwidth}
  \centering \includegraphics[height=3.2cm]{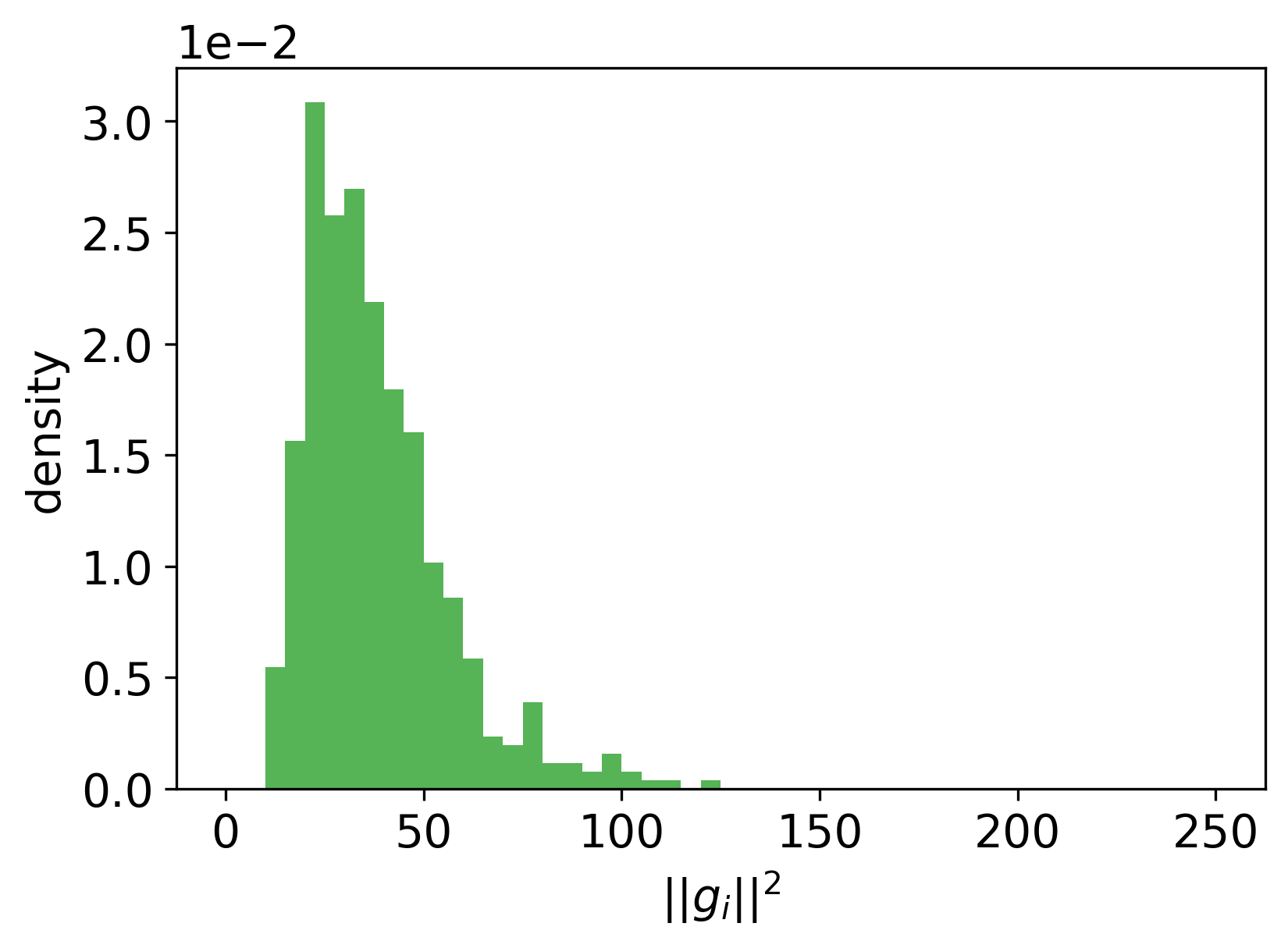}
  \caption{Norms early-training}
\end{subfigure}

\medskip

\begin{subfigure}[t]{0.32\textwidth}
  \centering \includegraphics[height=3.2cm]{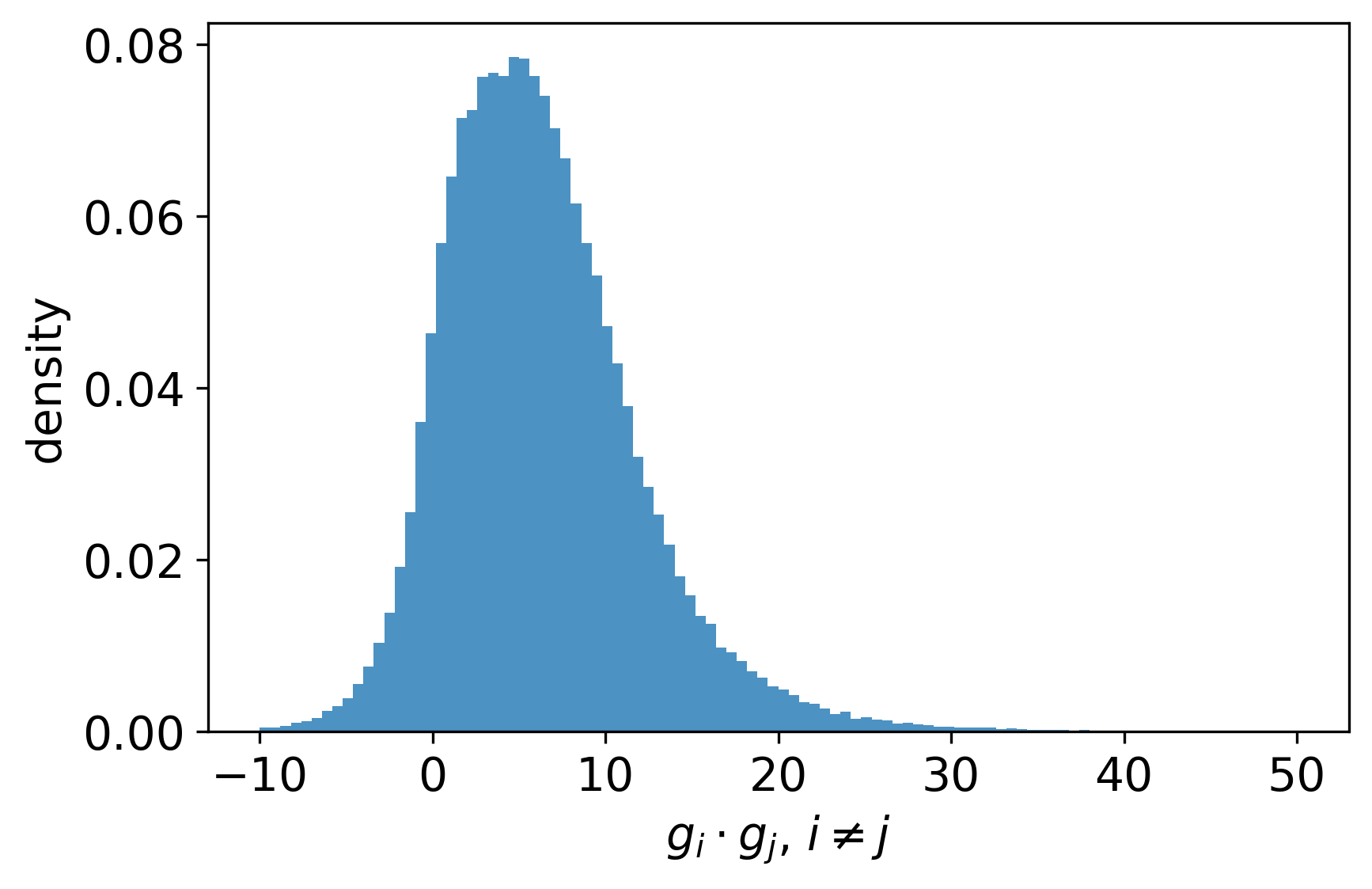}
  \caption{Dot product mid-training}
\end{subfigure}\hfill
\begin{subfigure}[t]{0.32\textwidth}
  \centering \includegraphics[height=3.2cm]{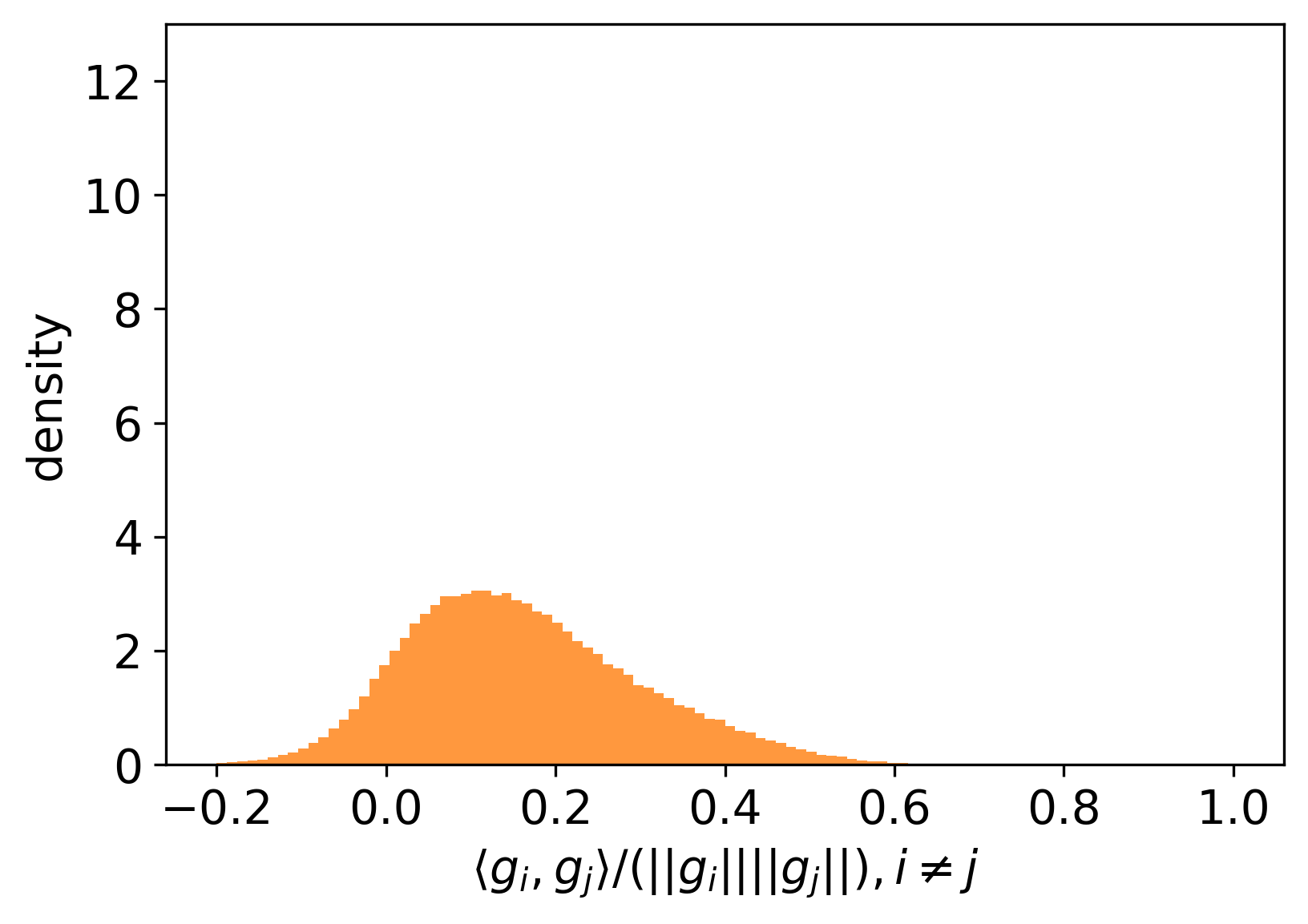}
  \caption{Cosine similarity mid-training}
\end{subfigure}\hfill
\begin{subfigure}[t]{0.32\textwidth}
  \centering \includegraphics[height=3.2cm]{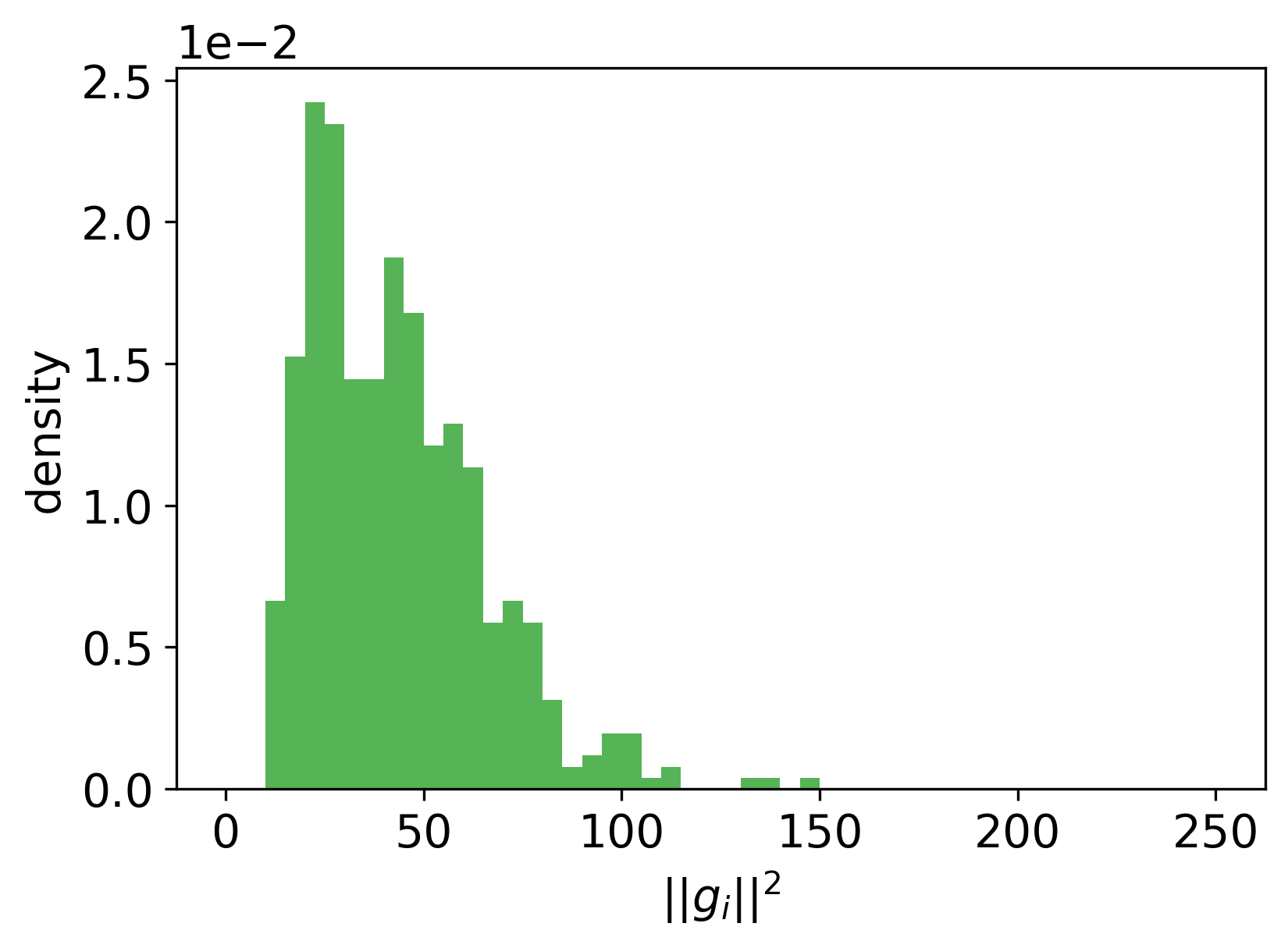}
  \caption{Norms mid-training}
\end{subfigure}

\medskip

\begin{subfigure}[t]{0.32\textwidth}
  \centering \includegraphics[height=3.2cm]{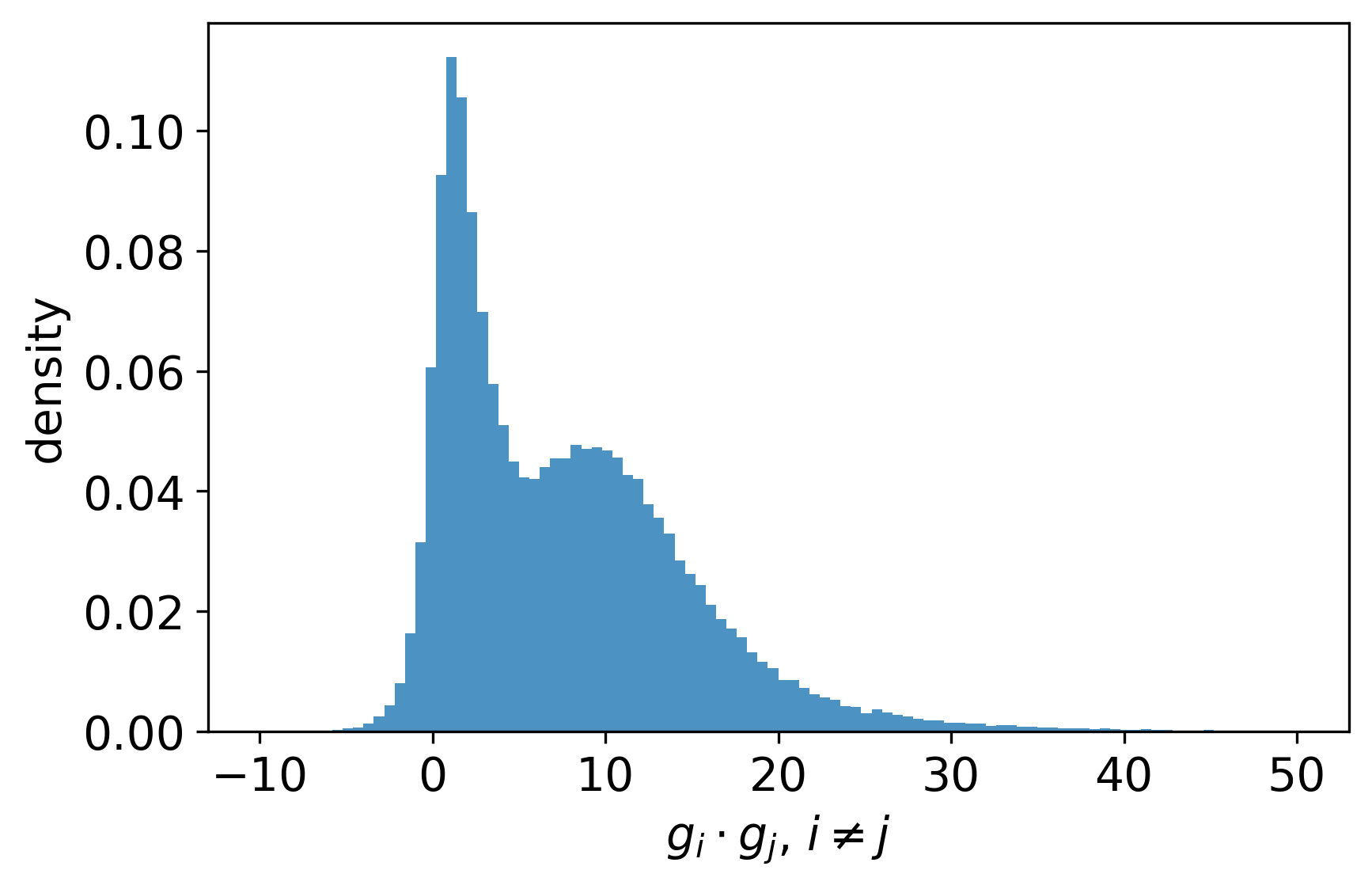}
  \caption{Dot product at \textsc{EoSS}}
\end{subfigure}\hfill
\begin{subfigure}[t]{0.32\textwidth}
  \centering \includegraphics[height=3.2cm]{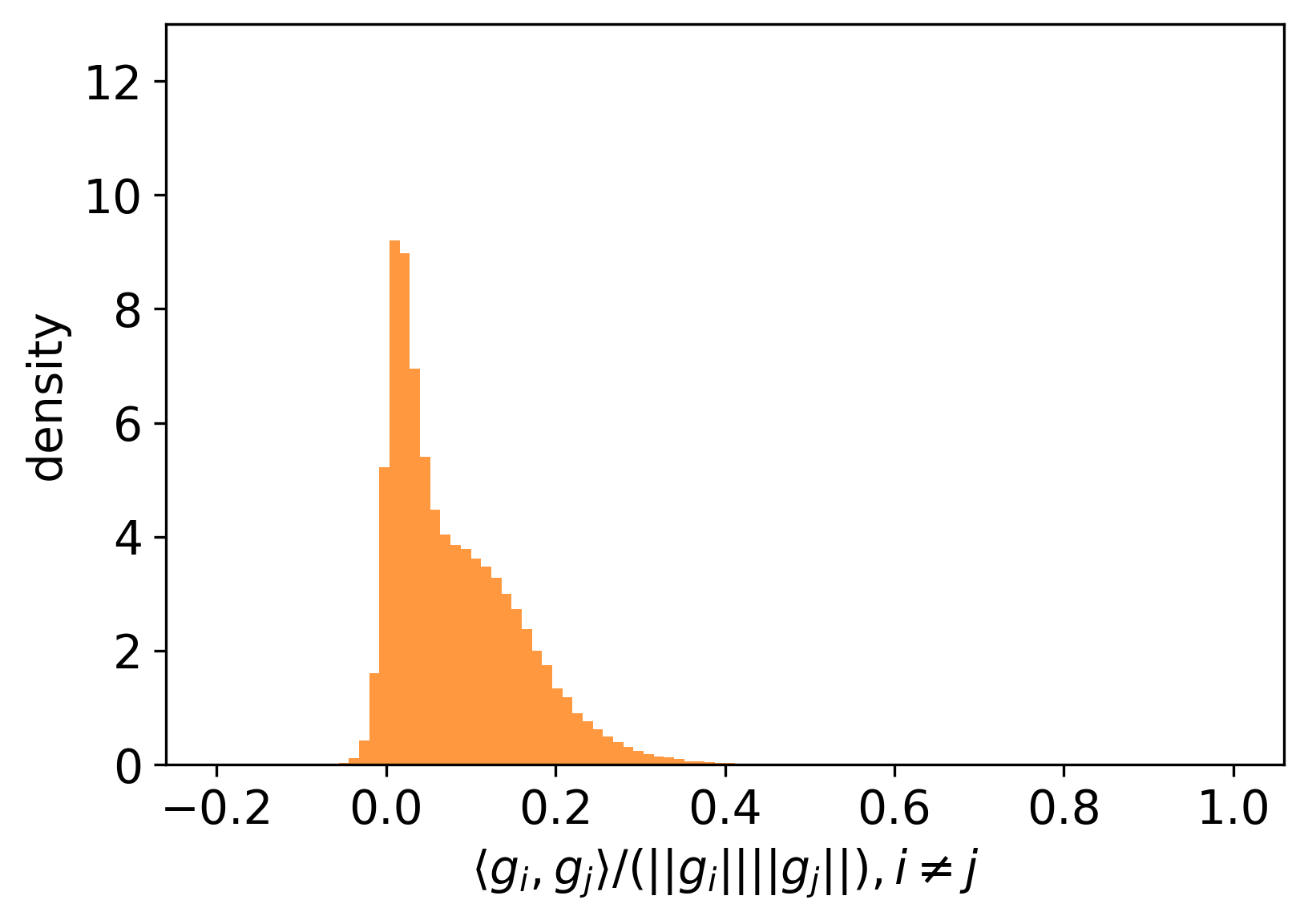}
  \caption{Cosine similarity at \textsc{EoSS}}
\end{subfigure}\hfill
\begin{subfigure}[t]{0.32\textwidth}
  \centering \includegraphics[height=3.2cm]{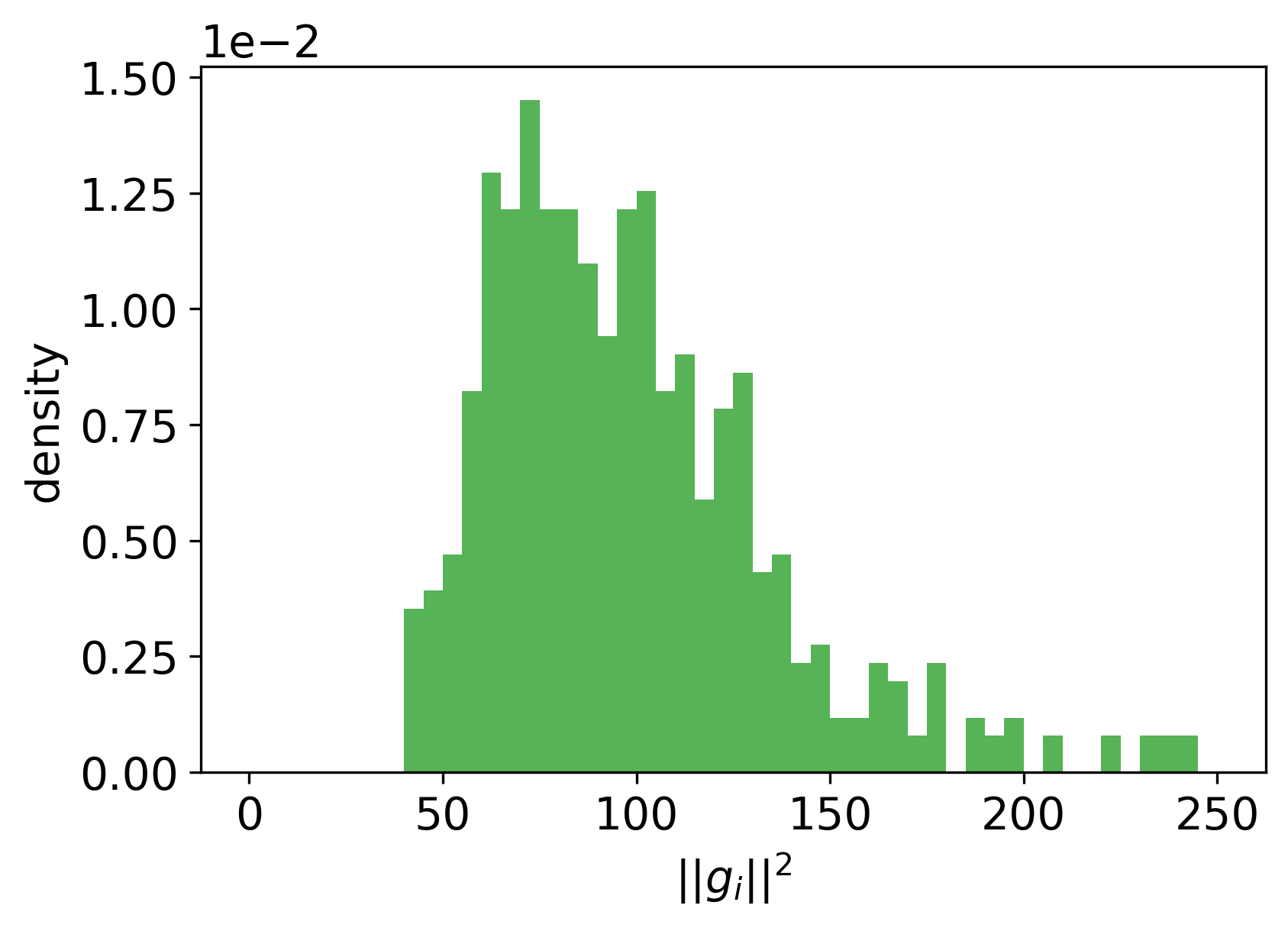}
  \caption{Norms at \textsc{EoSS}}
\end{subfigure}

\medskip

\begin{subfigure}[t]{0.32\textwidth}
  \centering \includegraphics[height=3.2cm]{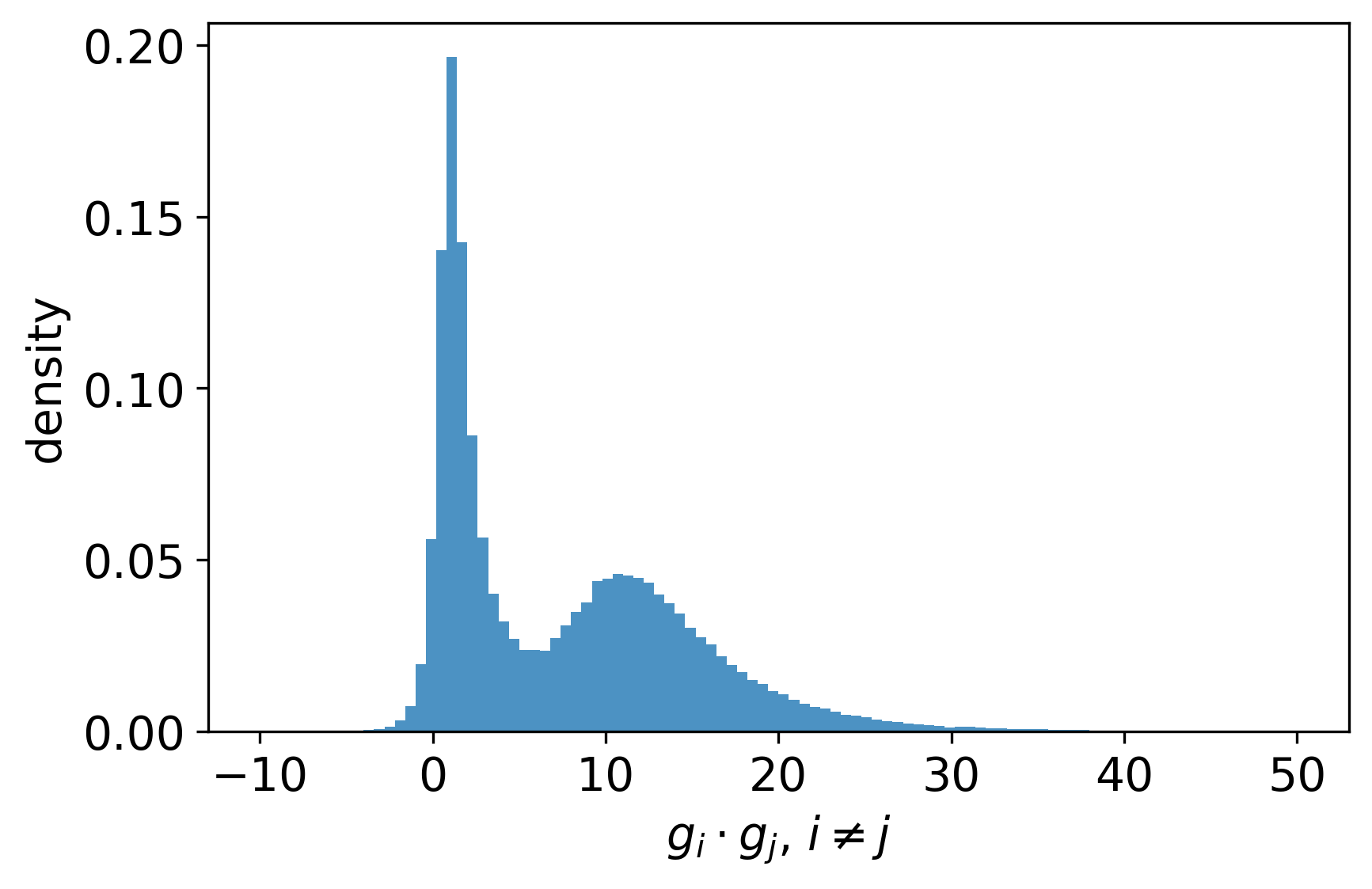}
  \caption{Dot product at convergence}
\end{subfigure}\hfill
\begin{subfigure}[t]{0.32\textwidth}
  \centering \includegraphics[height=3.2cm]{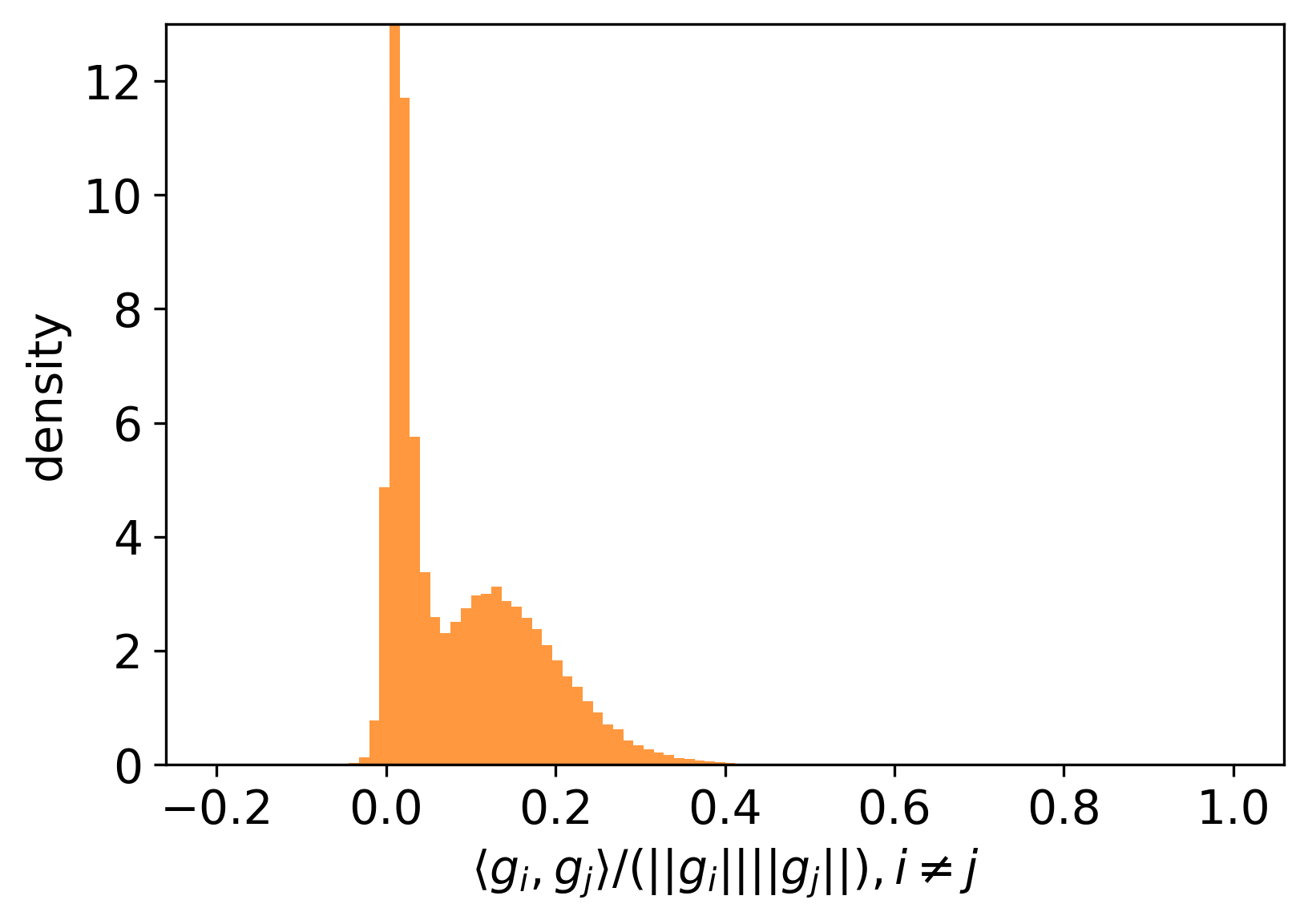}
  \caption{Cosine similarity at convergence}
\end{subfigure}\hfill
\begin{subfigure}[t]{0.32\textwidth}
  \centering \includegraphics[height=3.2cm]{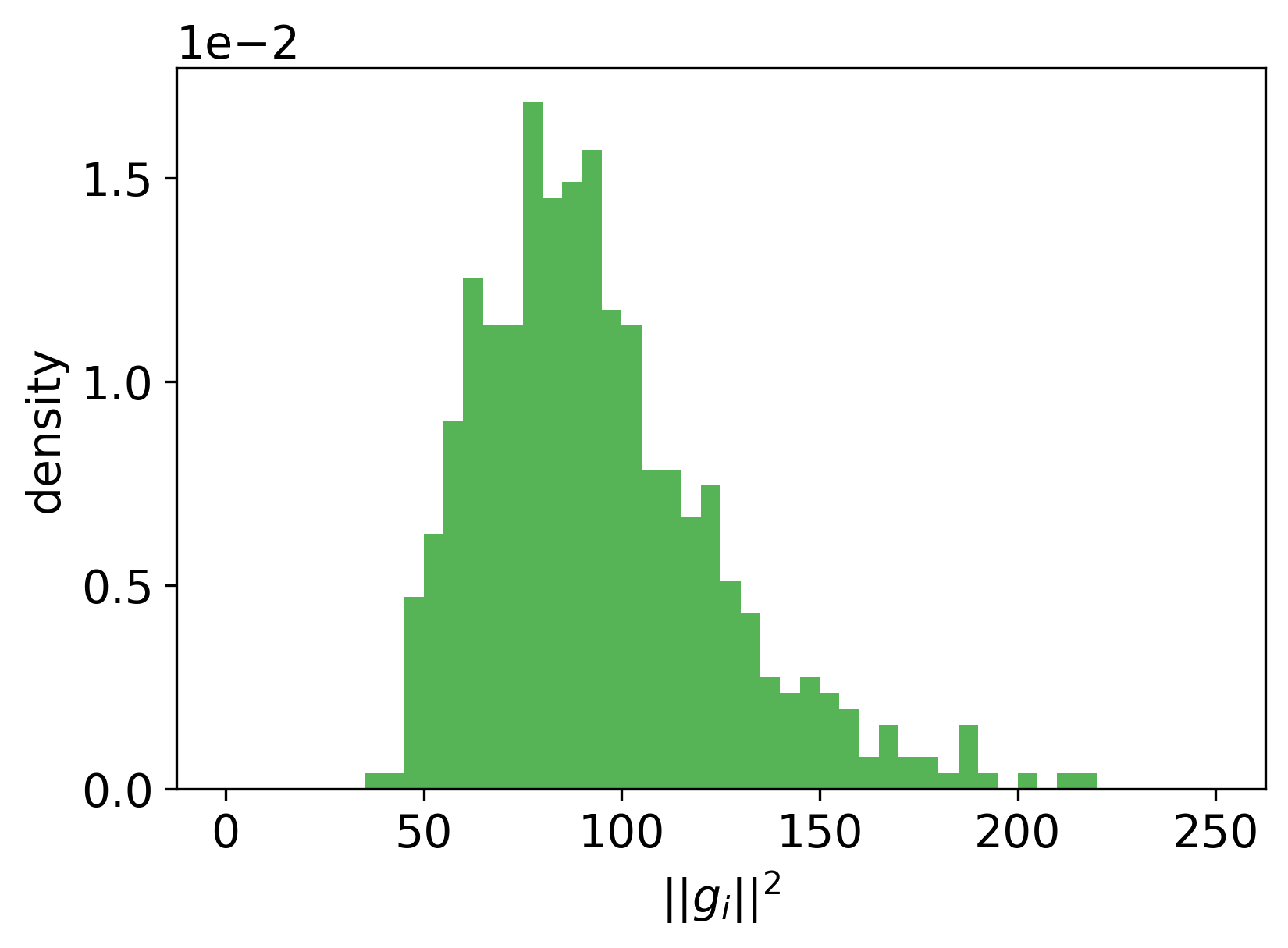}
  \caption{Norms at convergence}
\end{subfigure}

\caption{\textbf{Model gradients alignment.} Same setup as Figure \ref{fig:alignmnet_mlp}, but for CNN. CIFAR-8k (2 classes), $\eta=0.05$, batch size 8.}

  \label{fig:alignmnet_cnn}
\end{figure}

\clearpage


\section{Linear Stochastic Stability}
\label{appendix:linear_stochastic_stability}
This appendix extends the discussion from Section \ref{section:framework}, particularly regarding works addressing SGD stability from a linear stability viewpoint, including an analysis of the behavior of the bound introduced by \citet{wu_how_2018}.

\subsection{Notions of Linear Stability}
As the discussion in Section \ref{section:framework} highlights, we need two conditions to establish a regime of instability in mini-batch training: (i) a valid notion of stability as per Definition \ref{def:stab_not}---an inequality whose violation leads to divergence---and (ii) 
empirical saturation of this stability notion during SGD training, with continued training as the condition remains at saturation.

\citet{wu_how_2018} were the first to analyze linear stability of SGD, establishing a \textit{sufficient} (but, in general, necessary) condition for SGD stability. In particular, they prove that mini-batch gradient descent is stable when 
\begin{equation}
\label{eq:wu_et_al_appendix}
\lambda_{\max}\left( (I - \eta \mathcal{H})^2 + \frac{\eta^2}{b} \E[\mathcal{H}_i^2-\mathcal{H}^2] \right) 
\ = \
\lambda_{\max} \left( \E_B\big[ (I - \eta \mathcal{H}_B)^2 \big] \right) \leq 1.
\end{equation}
This criterion upper-bounds the spectral radius of the second-moment update operator.
Since this condition is only sufficient (and necessary solely when $d=1$), it does not strictly satisfy our criteria for a valid stability notion (Definition \ref{def:stab_not}). Importantly, while \citet{wu_how_2018} explicitly note this limitation, nothing a priori excludes this bound from being \textit{empirically} tight---after all, \textsc{EoSS} is fundamentally an empirical phenomenon. We show here that this criterion does not govern the \textsc{EoSS}---that is, that the "not necessary" part is not vacuous.


\paragraph{Necessary stability conditions.} Conditions derived from linear stochastic stability theory that are indeed valid stability notions often suffer from computational intractability. Indeed, \citet{ma_linear_2021} proved that mini-batch gradient descent is 2nd order linearly stable \textit{if and only if} the operator
\begin{equation}
    T_k:=\E_B \Big [(I-\eta H(L_B))^{\otimes 2} \Big ]
\end{equation}
is a contraction on the cone of PSD matrices. Importantly for us, this a necessary condition for stability, which would constitute a notion of stability.
While this condition is necessary and would represent a valid stability notion, it is computationally infeasible in high-dimensional neural networks, as it involves operations on $d^2 \times d^2$ tensors (making even tensor-vector product unfeasible).

\citet{mulayoff_exact_2024} showed that the PSD condition is inactive, which reduces the criterion to one on a spectral norm of this operator. Moreover, they express this notion in the form of (notion of curvature) $\leq$ (step-size dependent threshold). Although potentially useful to find the corresponding maximum stable learning rate, this reformulation did not solve the incomputability problem. \citep{mulayoff_exact_2024} also construct elegant lower bounds, which therefore also serve as a necessary condition for stability, and thus a valid notion of stability. However, as their empirical results show, these bounds never saturate, and thus do not effectively capture the empirical presence of an instability regime in mini-batch training.

\subsection{Empirical Behavior of \texorpdfstring{\citet{wu_how_2018}}{Wu et al.} Criterion}
\begin{figure}[t]
  \centering
  \includegraphics[width=0.99\linewidth]{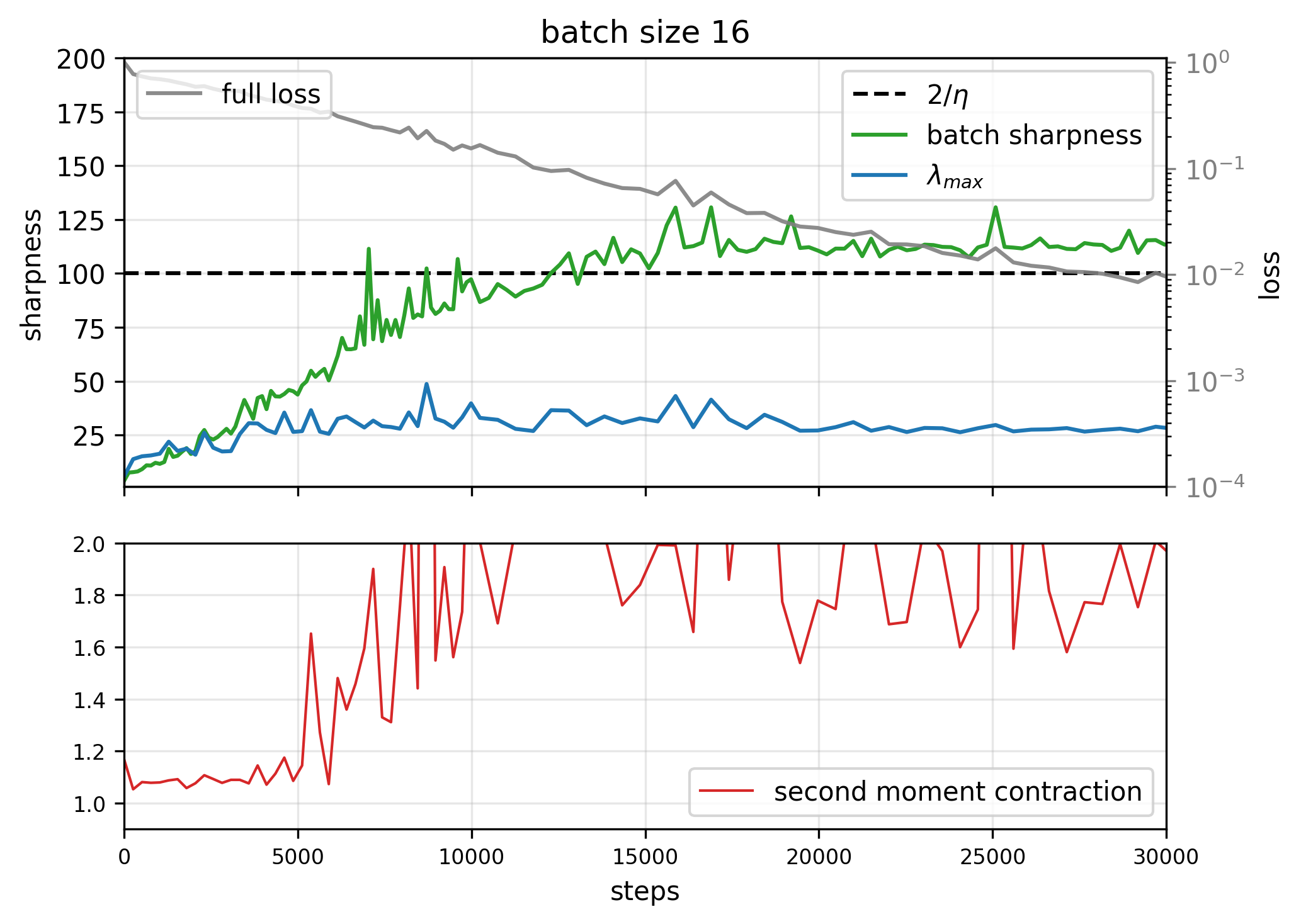}%
  \caption{\textit{Linear stochastic stability.} Tracking the condition of equation \eqref{eq:wu_et_al_appendix} (“second-moment contraction’’) during training. The value stays strictly above $1$.}
  \label{fig:wu1}
\end{figure}

\paragraph{Always above $1$.}
In neural networks negative Hessian eigenvalues are typically present, thus the quantity in Equation \eqref{eq:wu_et_al_appendix}, which we term \textit{second moment contraction}, is always bigger than $1$ (Figure \ref{fig:wu1}).
Ideally, when "the phase transition at \textsc{EoSS}" happens this quantity keeps being bigger than 1, but the highest singular value becomes the biggest eigenvalue instead of the smallest. In the deterministic full-batch algorithm case this can be seen cleanly:
\begin{multline*}
    \lambda_{\max} \left( \E_B\big[ (I - \eta \mathcal{H}_B)^2 \big] \right)
    =
    \lambda_{\max} \left( (I - \eta \mathcal{H})^2 \right) = \\
    \begin{cases}
        1 - \eta \lambda_{\min}(\mathcal{H}) & \sim 1 + \epsilon_1
        \qquad \text{when }\lambda_{\max}\leq 2/\eta \text{ and } \lambda_{\min} < 0 \\
        |\eta \lambda_{\max} - 1| 
        & \sim 1 + \epsilon_2
        \qquad \text{when }\lambda_{\max}\geq 2/\eta \text{ and } \lambda_{\min} < 2-\eta \lambda_{\max} \leq 0.\\
    \end{cases}
\end{multline*}
Thus we can think of plotting the quantity
\begin{equation}
\label{eq:wu_equivalent}
    \lambda_{\max} \left( - 2\eta \mathcal{H} + \eta^2 \E_B\big[ \mathcal{H}(L_B)^2 \big] \right) \lessgtr 0.
\end{equation}

\begin{figure}[t]
  \centering

  \begin{subfigure}[t]{0.48\linewidth}
    \centering
    \includegraphics[width=\linewidth]{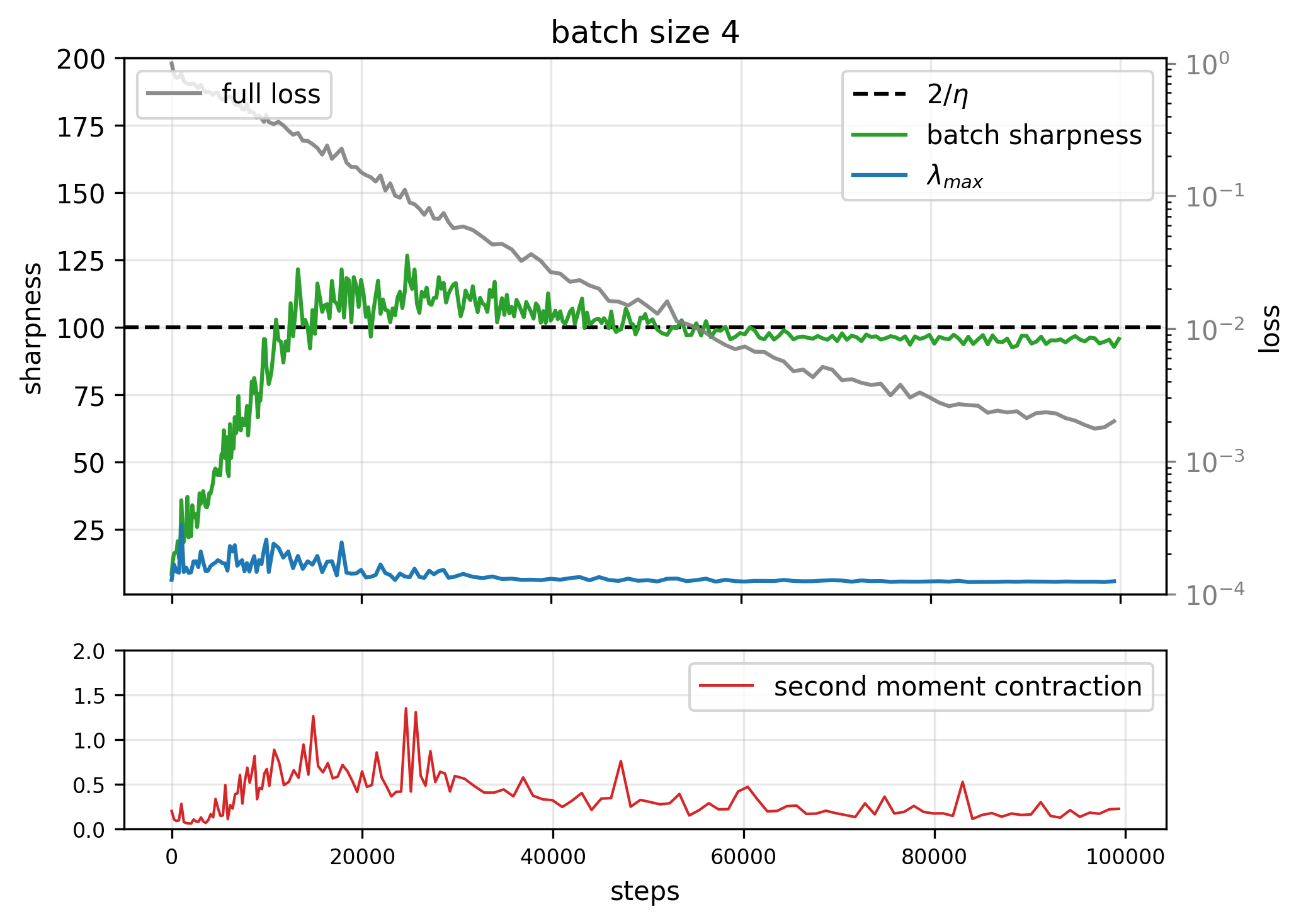}
    \label{fig:linstab-4}
  \end{subfigure}\hfill
  \begin{subfigure}[t]{0.48\linewidth}
    \centering
    \includegraphics[width=\linewidth]{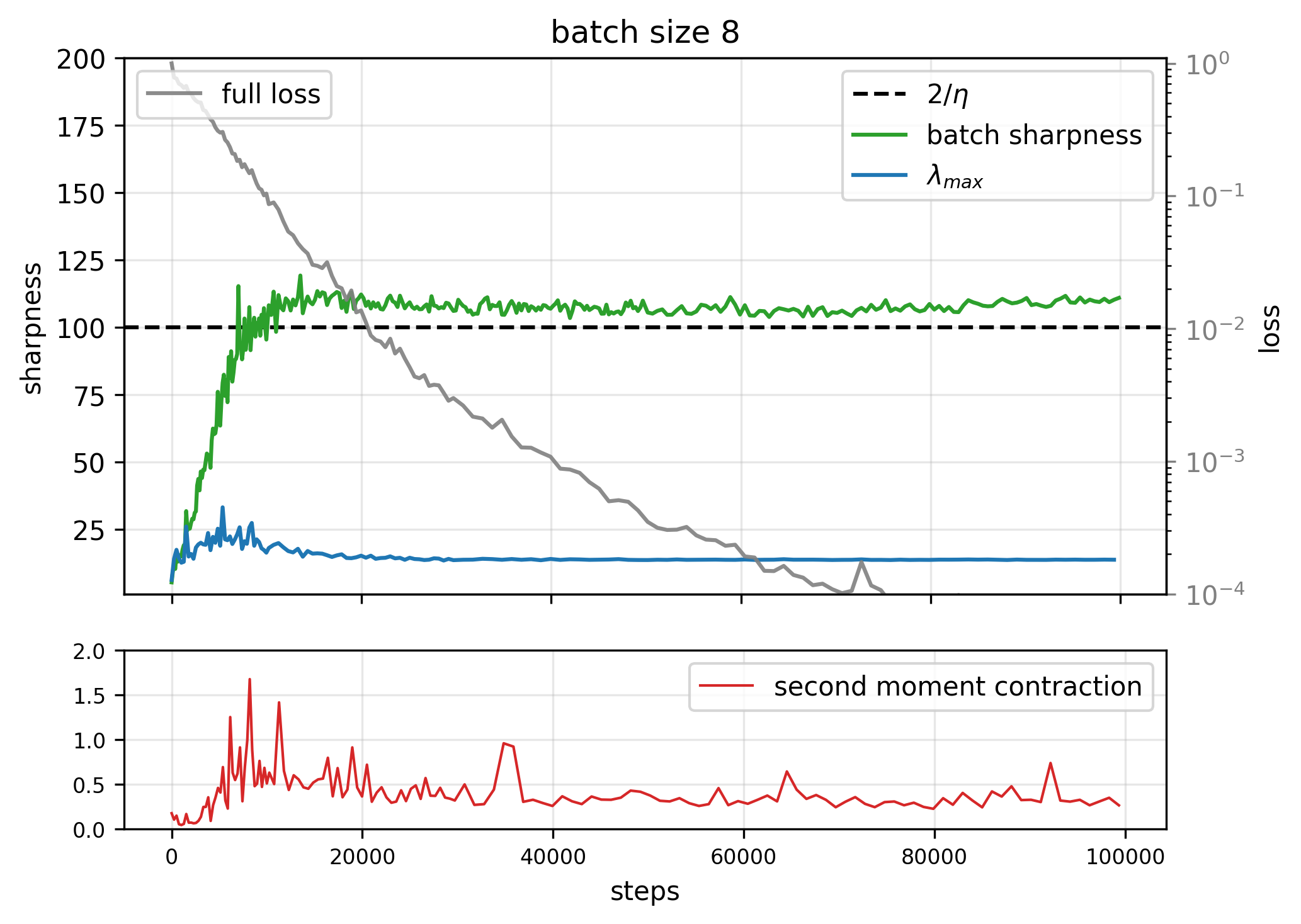}
    \label{fig:linstab-8}
  \end{subfigure}

  \vspace{0.6em}

  \begin{subfigure}[t]{0.48\linewidth}
    \centering
    \includegraphics[width=\linewidth]{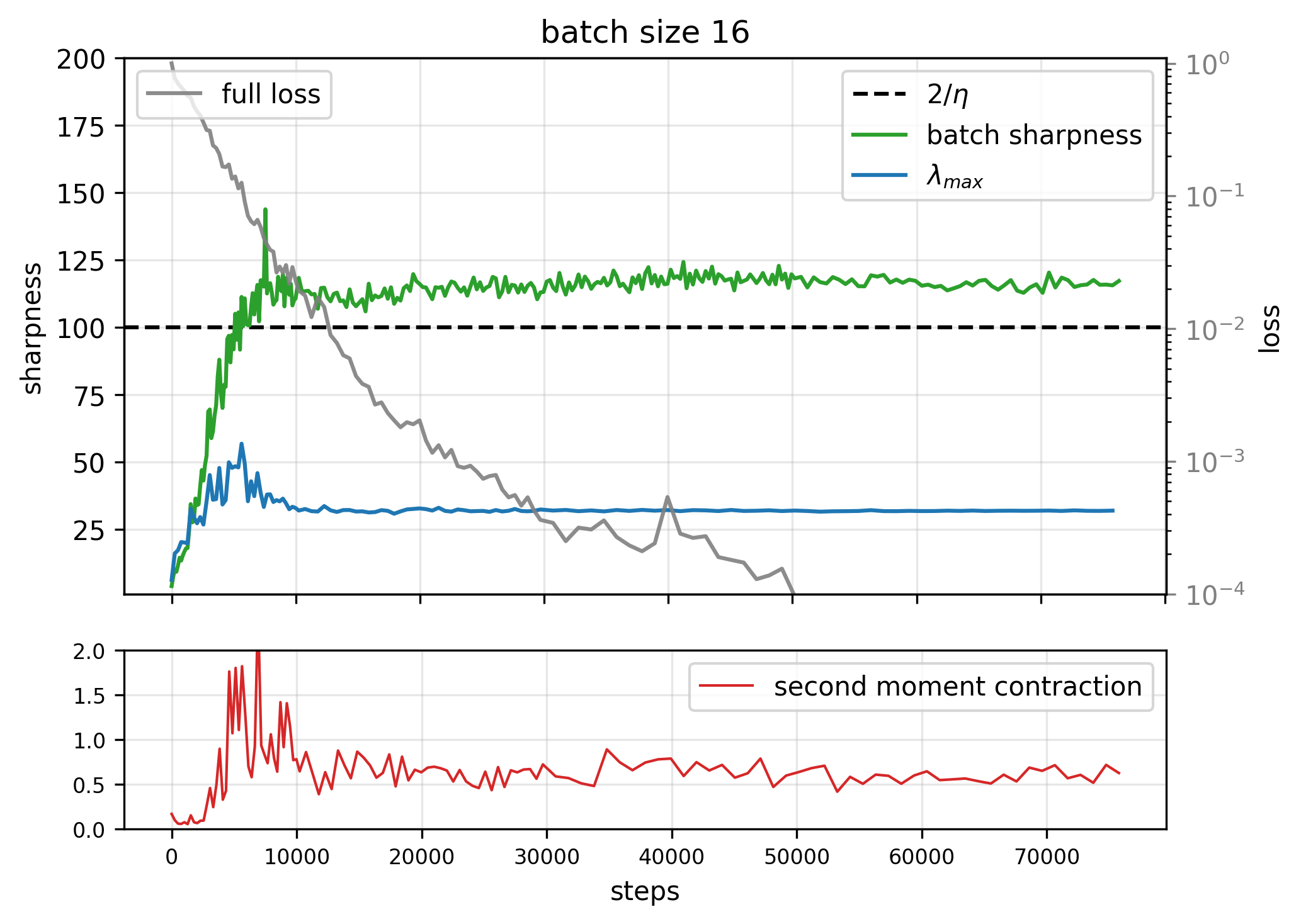}
    \label{fig:linstab-16}
  \end{subfigure}\hfill
  \begin{subfigure}[t]{0.48\linewidth}
    \centering
    \includegraphics[width=\linewidth]{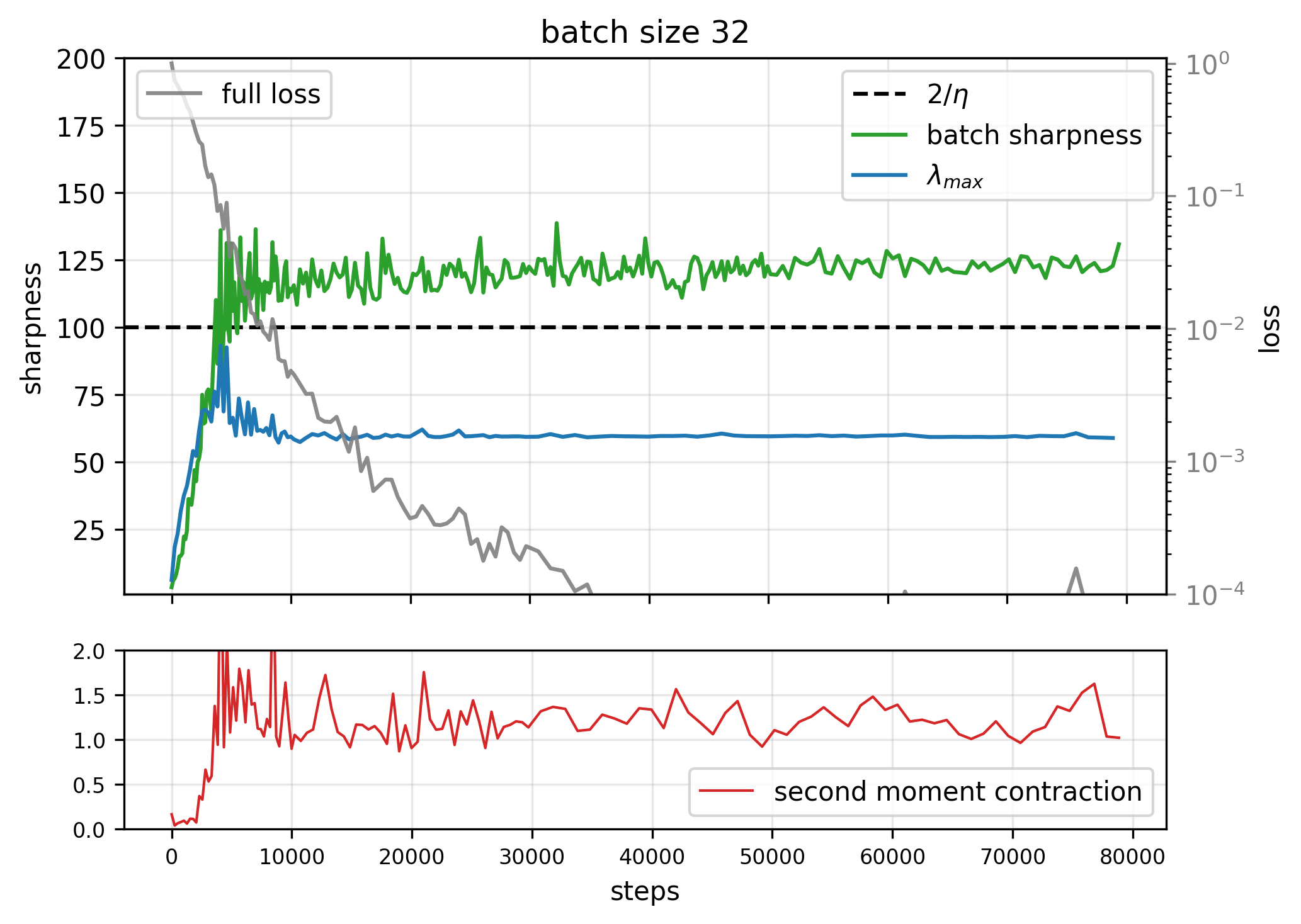}
    \label{fig:linstab-32}
  \end{subfigure}

  \vspace{0.6em}

  \hfill
  \begin{subfigure}[t]{0.48\linewidth}
    \centering
    \includegraphics[width=\linewidth]{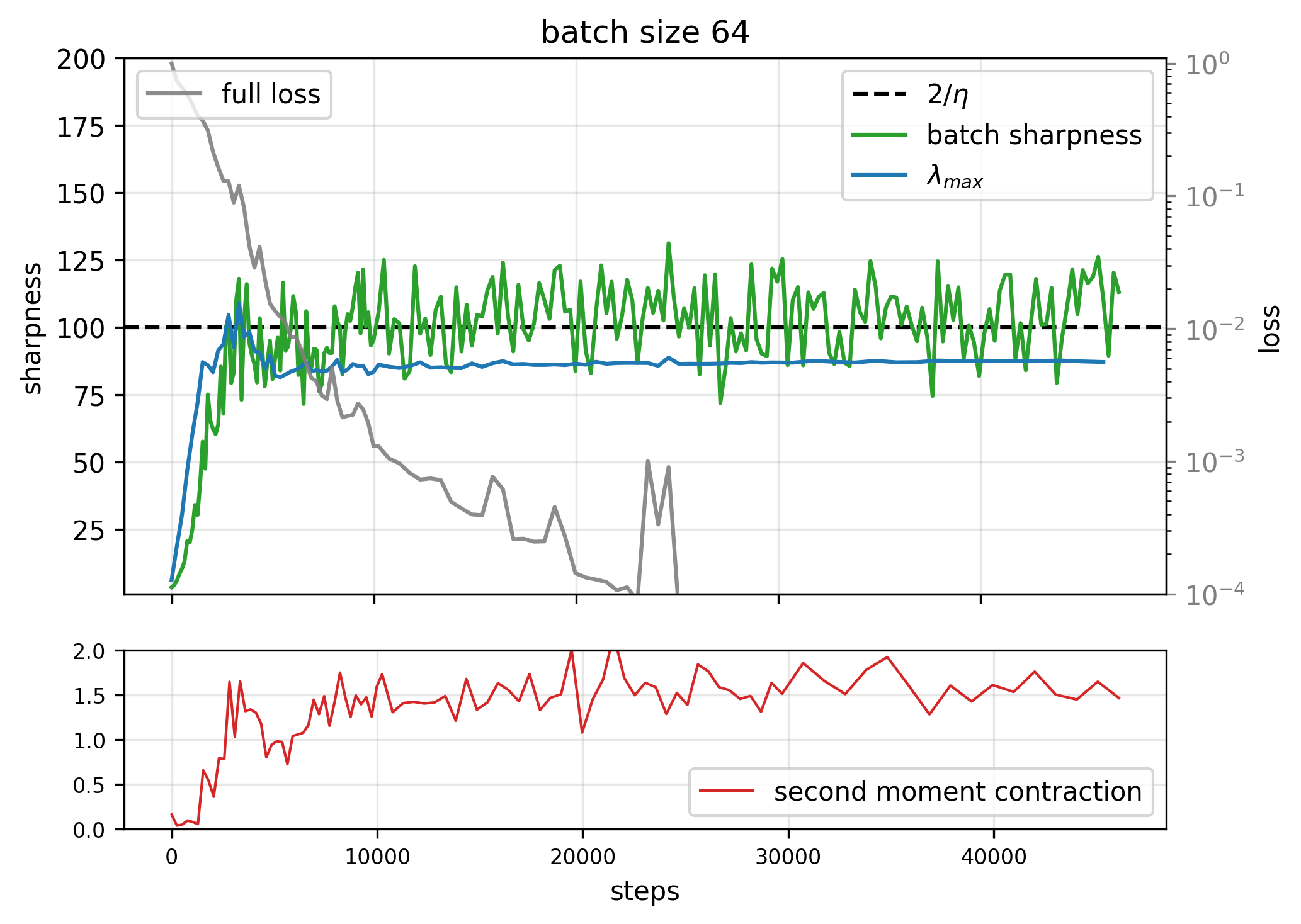}
    \label{fig:linstab-64}
  \end{subfigure}
  \hfill
  \hfill

  \caption{Tracking the condition of equation \eqref{eq:wu_equivalent} (the one without the identity), \textit{second moment contraction}, for different batch sizes. MLP, CIFAR-8k, $\eta=0.02$. Note how the stabilization levels is significantly above the threshold, and inconsistent across batch sizes.}
  \label{fig:wu2}
\end{figure}

As illustrated in Figure \ref{fig:wu2}, the condition \eqref{eq:wu_equivalent} is initially slightly above zero (due to negative eigenvalues). Then it undergoes the aforementioned phase transition and starts growing---although the exact moment of that transition is unclear. After that it seems to stabilize around the time that $\lmax$ stabilizes, and \textit{Batch Sharpness} stabilizes around $2/\eta$.

\paragraph{Not a stability measure empirically.}
Empirically, the criterion \eqref{eq:wu_equivalent} behaves in a way that precludes it from being a quantity that governs stability in a \textsc{EoS}-like fashion, as evidenced by the experiments in Figure \ref{fig:wu2}:
\begin{itemize}
    \item \textbf{The bound is not tight.} In particular, if \textit{empirically} the \textit{second moment contraction} was governing stability of SGD training, the condition of equation \eqref{eq:wu_et_al_appendix} would only be violated by a small margin, just like is the case with that condition in the full-batch GD---equivalently, as it is the case with $\lmax$ being just slightly above $2/\eta$ during GD training. Instead, \textit{second moment contraction} hovers at $2$ (equivalently, the quantity of Equation \eqref{eq:wu_equivalent} hovering around $1$).
    \item \textbf{No up and down oscillations.} The higher-order-term driven \textsc{EoS} stabilization mechanism of \citet{damian_self-stabilization_2023} and the dynamics in \citet{cohen_understanding_2024} prescribe a notion of stability to go up and down around the stability threshold. On top of the above point of being significantly above the threshold level, \textit{second moment contraction} does not oscillate in a way prescribed by a stabilization that's based on higher order terms.
    \item \textbf{Inconsistent level of stabilization.} Finally, a notion that would govern SGD dynamics would present a consistent level of stabilization, independent of hyperparameters. This does not happen in the case of \textit{second moment contraction}. 
\end{itemize}

\subsection{Implications}

\paragraph{Alignment matters.} The reason why the condition of \citet{wu_how_2018} (equation \eqref{eq:wu_et_al_appendix}) is only a sufficient one, while the condition of \citet{ma_linear_2021} (equation \eqref{eq:ma-ying}) is necessary and sufficient, is that the mini-batch Hessians are not commuting/not simultaneously diagonalizable. In particular, they would have been simultaneously diagonalizable if either the model gradients forming their Gauss-Newton approximation were either all the same or all orthogonal, which is hypothetically possible as we have $N \ll d$. Now, in Appendix \ref{appendix:alignment} it was show that neither is the case---they are misaligned, but not completely orthogonal. Now, the condition \eqref{eq:wu_et_al_appendix} not being a governing quantity of \textsc{EoSS}, and not being tight as an upper bound on instability condition, is evidence that this not-complete misalignment has non-trivial effects. That is, unlike in deterministic full-batch gradient descent settings, instability in SGD is dependent on the alignment between notions of curvatures. Therefore, a true notion of stability has to involve a notion of alignment, not only the magnitude of curvature---and, correspondingly, \textit{Batch Sharpness} does consider the alignment, as opposed to, for example, $\lmax$ or $\lmax^b$.

\paragraph{Importance of instability (not "stability").}
We define \textsc{EoSS} as being at the edge of instability. This quantity of equation \eqref{eq:wu_et_al_appendix} smaller than $1$ is \textit{only} a sufficient condition for stability. 
The fact that breaking Eq.\eqref{eq:wu_et_al_appendix} is not enough for assessing the behavior of SGD is a further proof that what matters is the instability, not the stability. What matters is an inequality that implies divergence if broken, not that implies convergence if satisfied.

\clearpage
\section{The Hessian and Gauss-Newton Approximation Overlap}
\label{appendix:gauss_newton_approx}

In the theoretical analysis of stability of SGD dynamics it is assumed that the loss Hessian can be well-approximated by its Gauss-Newton approximation---in particular, it is often assumed we are at the minima, where there is an equality between the two. 
Concretely, having $C$ classes, we have:
\[
L(\theta) \;=\; \frac{1}{N}\sum_{i=1}^{N} \ell\big(z_i(\theta),\, y_i\big),
\qquad
z_i(\theta) \;=\; f_\theta(x_i)\in\mathbb{R}^{C}.
\]
\[
J_i \;:=\; \frac{\partial z_i(\theta)}{\partial \theta}\in\mathbb{R}^{C\times d},\quad
g_{z,i} \;:=\; \nabla_{z}\,\ell\big(z_i(\theta),y_i\big)\in\mathbb{R}^{C},\quad
H_{z,i} \;:=\; \nabla^{2}_{z}\,\ell\big(z_i(\theta),y_i\big)\in\mathbb{R}^{C\times C},
\]
and for \(j=1,\dots,C\), let \(\nabla_\theta^{2} f_j(x_i)\in\mathbb{R}^{d\times d}\) denote the Hessian (w.r.t.\ \(\theta\)) of the \(j\)-th output component. With this notation, we have:
\[
\nabla_\theta L(\theta) \;=\; \frac{1}{N} \sum_{i=1}^{N} J_i^\top g_{z,i}.
\]
and
\begin{gather*}
\nabla_\theta^{2}{L}(\theta)
= \frac{1}{N} \sum_{i=1}^{N} \Big( J_i^\top H_{z,i}\, J_i \;+\; \sum_{j=1}^{C} \big[g_{z,i}\big]_j\, \nabla_\theta^{2} f_j(x_i) \Big) \\
= \underbrace{\frac{1}{N}\sum_{i=1}^{N} J_i^\top H_{z,i}\, J_i}_{\text{Gauss-Newton approx}} \;+\;
\underbrace{\frac{1}{N}\sum_{i=1}^{N}\sum_{j=1}^{C} \big[g_{z,i}\big]_j\, \nabla_\theta^{2} f_j(x_i)}_{\text{remainder / model-curvature term}}.
\end{gather*}
For MSE, this simplifies to:
\begin{align*}
\nabla_\theta^{2}\mathcal{L}(\theta)
&= \underbrace{\frac{1}{N}\sum_{i=1}^{N} J_i^\top J_i}_{\text{Gauss-Newton for MSE}}
\;+\; \underbrace{\frac{1}{N}\sum_{i=1}^{N}\sum_{j=1}^{C} r_{i,j}\, \nabla_\theta^{2} f_j(x_i)}_{\text{remainder / model-curvature term}}.
\end{align*}
in particular, if we are at minima, the residuals are 0, so the second terms completely disappears.

Yet, the dynamics enter the \textsc{EoSS} regime away from minima (which is showcased by the continued decrease of the loss), where the Gauss-Newton approximation might not hold. In this Appendix we illustrate empirically that the Gauss-Newton approximation is close to the actual loss Hessian---at least from the perspective of \textsc{EoSS} and SGD stability. In particular, we compare \textit{Batch Sharpness}, $\lmax$ and $\lambda^b_{\max}$ when computed on the actual loss Hessian and on its Gauss-Newton approximations. \cref{fig:gn_mlp,fig:gn_cnn,fig:gn-approx} illustrate that the computed quantities coincide throughout the \textit{whole} training.
Notice that due to the fact that the Gauss-Newton approximation and the NTK have the same spectrum, the $\lmax$ and $\lmax^b$ results also apply to the highest eigenvalues of the full-batch and mini-batch NTKs ($\frac{1}{B}J_B J_B^\top$). Note that this agrees with the findings in literature, see e.g. \citet{papyan_full_2019}, but it was not clear whether the Gauss-Newton approximation holds before convergence. 
Our experiments demonstrate that it does the throughout the whole training, in particular, during \textsc{EoSS}.


\begin{figure}[t]
    \centering
    \begin{subfigure}[b]{0.48\textwidth}
        \centering
        \includegraphics[width=\linewidth]{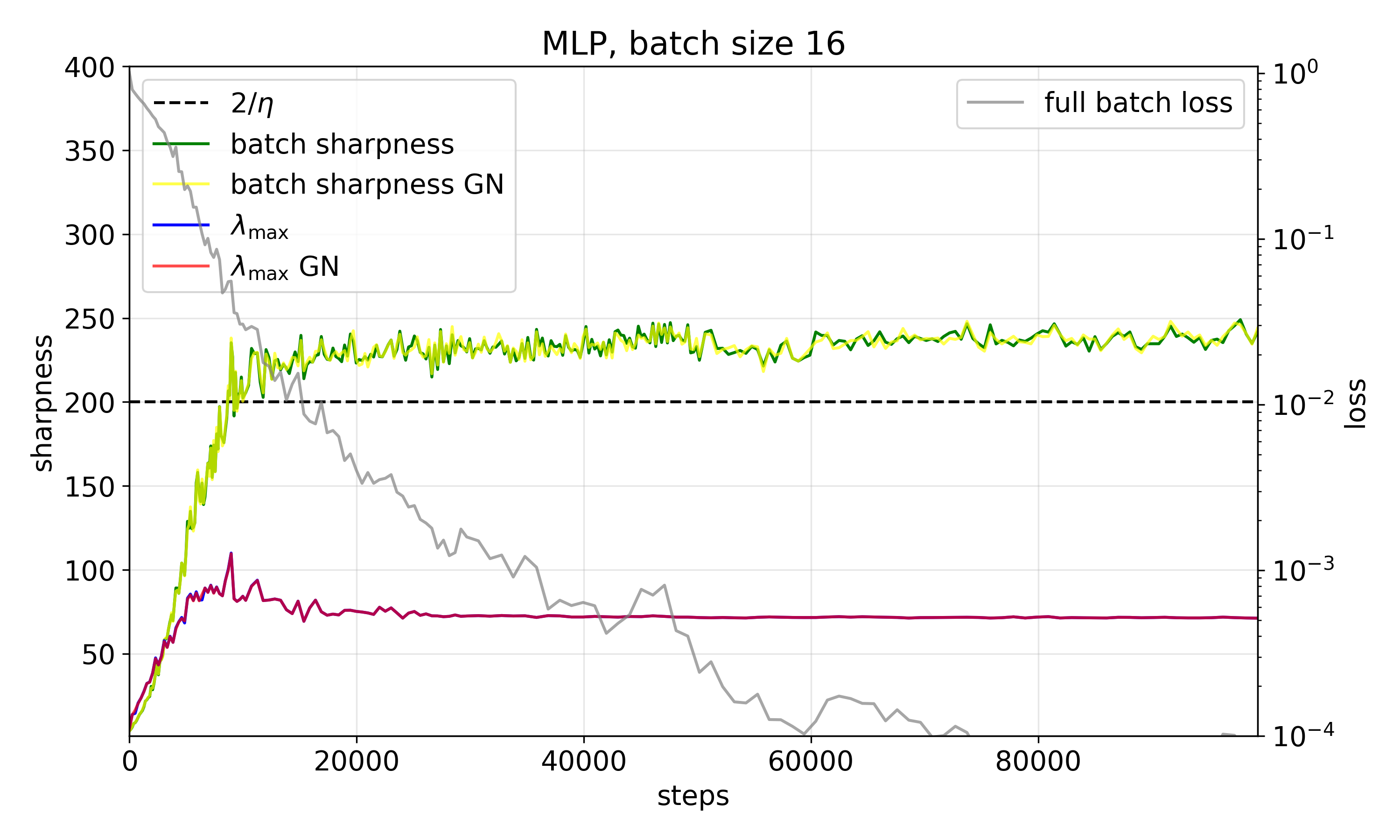}
        \label{fig:mlp_hessian}
    \end{subfigure}
    \hfill
    \begin{subfigure}[b]{0.48\textwidth}
        \centering
        \includegraphics[width=\linewidth]{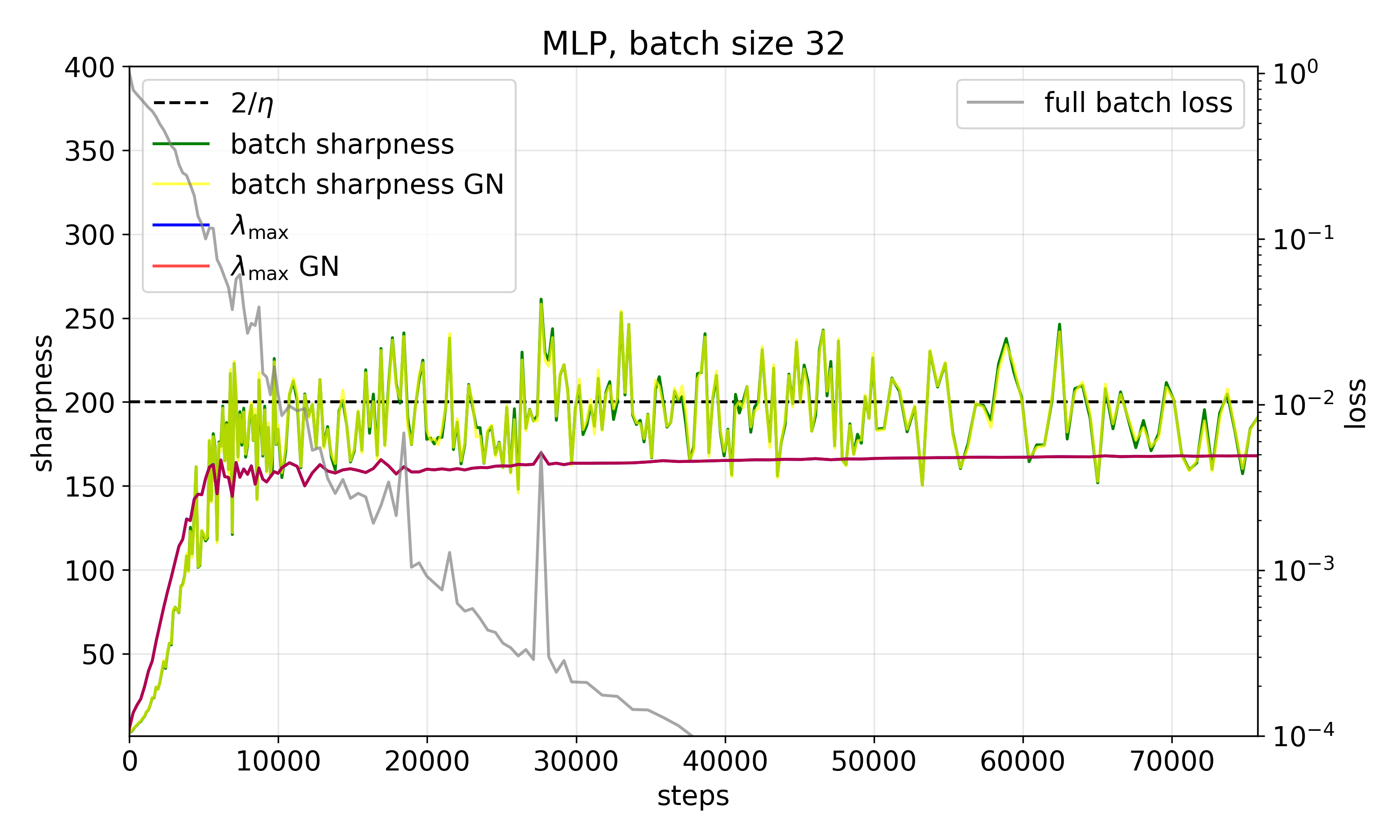}
        \label{fig:mlp_gn}
    \end{subfigure}
    \caption{
    \textbf{Gauss--Newton approximation (MLP).} 
    Comparison of \textit{Batch Sharpness} and $\lmax$ computed with the true loss Hessian and its Gauss--Newton approximation, showing the validity of approximation. Both of the lines overlap almost perfectly}
    \label{fig:gn_mlp}
\end{figure}

\begin{figure}[t]
    \centering
    \begin{subfigure}[b]{0.48\textwidth}
        \centering
        \includegraphics[width=\linewidth]{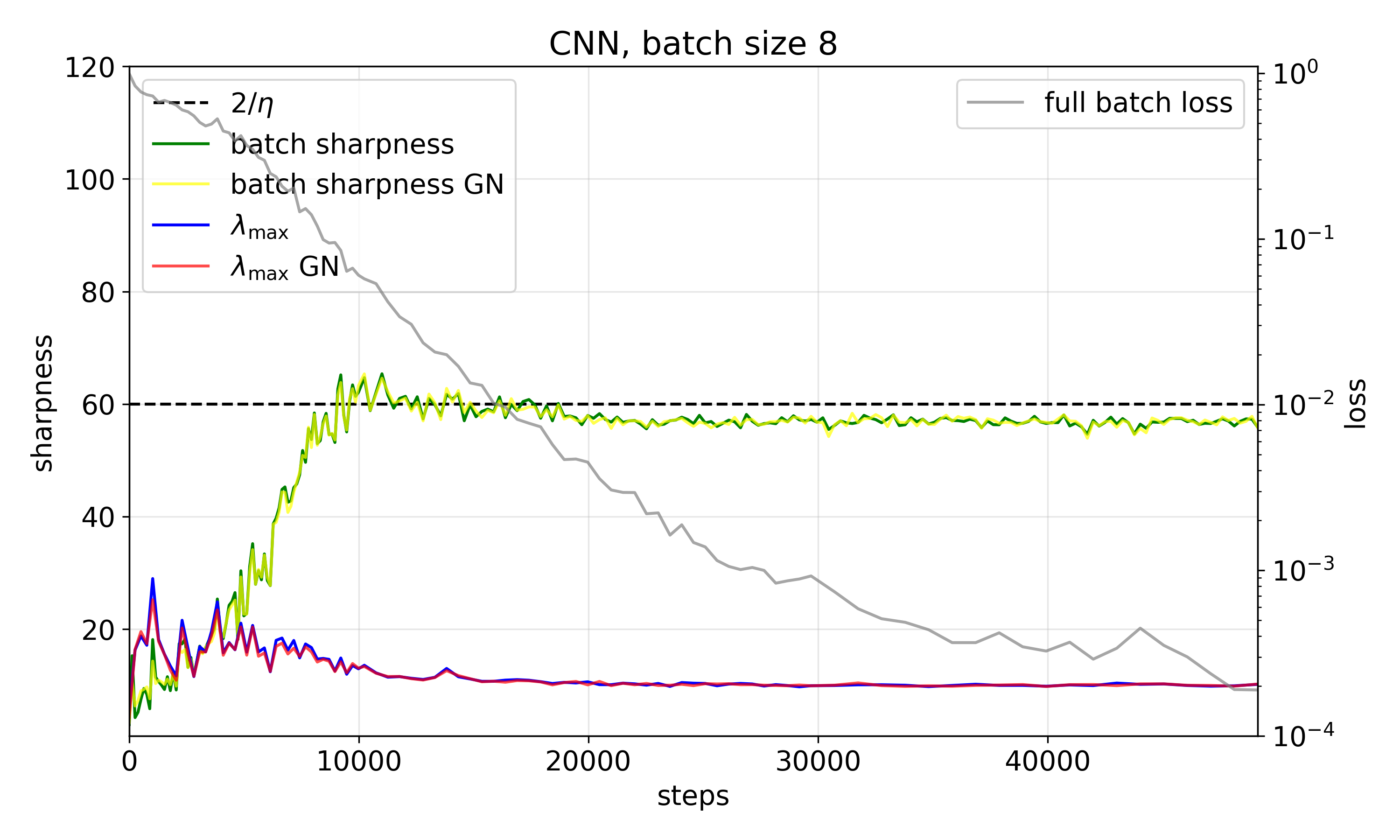}
    \end{subfigure}
    \hfill
    \begin{subfigure}[b]{0.48\textwidth}
        \centering
        \includegraphics[width=\linewidth]{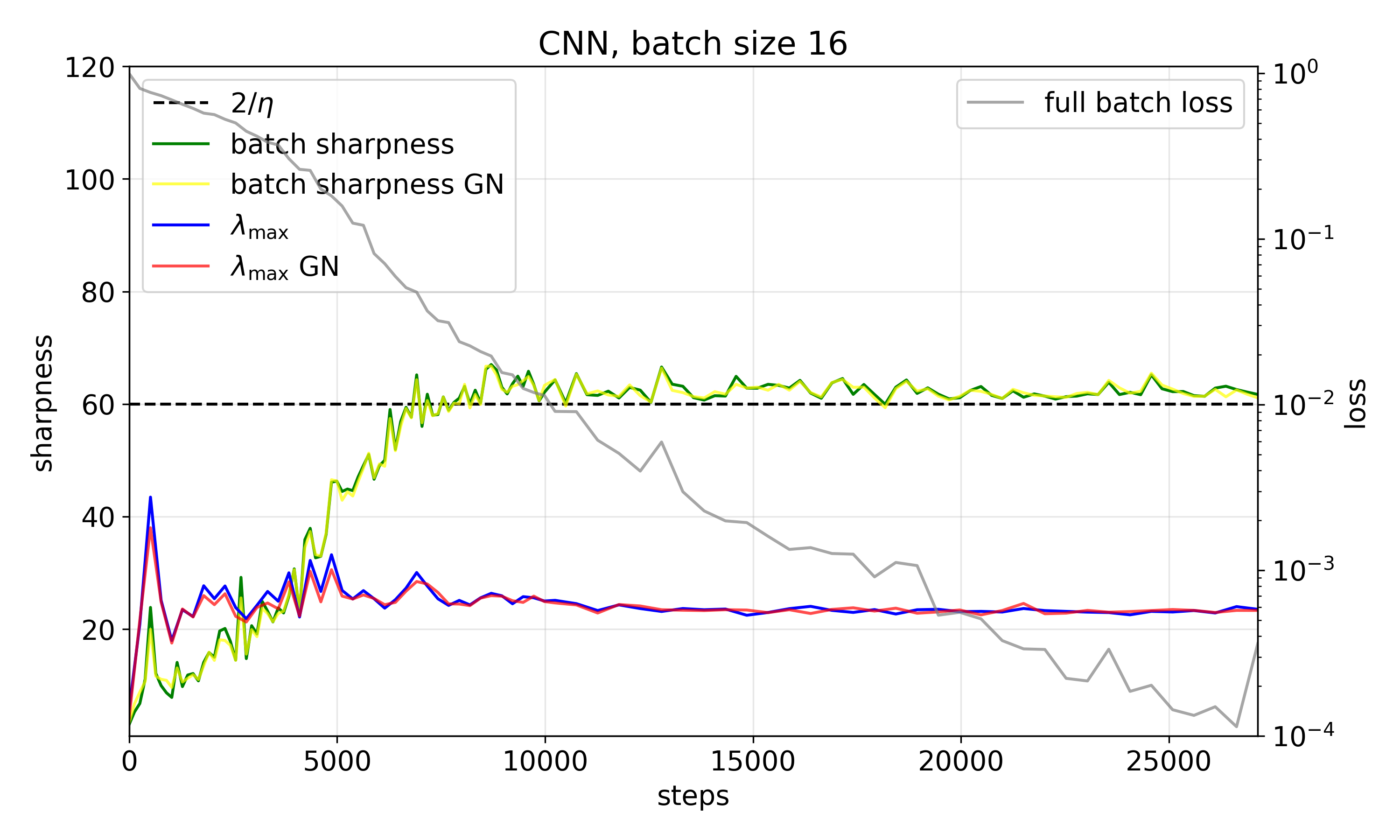}
    \end{subfigure}
    \caption{
    \textbf{Gauss--Newton approximation (CNN).} 
    Comparison of \textit{Batch Sharpness} and $\lmax$ computed with the true loss Hessian and its Gauss--Newton approximation, showing the validity of approximation. Both of the lines overlap almost perfectly}
    \label{fig:gn_cnn}
\end{figure}

\begin{figure}[htbp]
  \centering

  \begin{subfigure}[t]{0.48\textwidth}
    \centering
    \includegraphics[width=\linewidth]{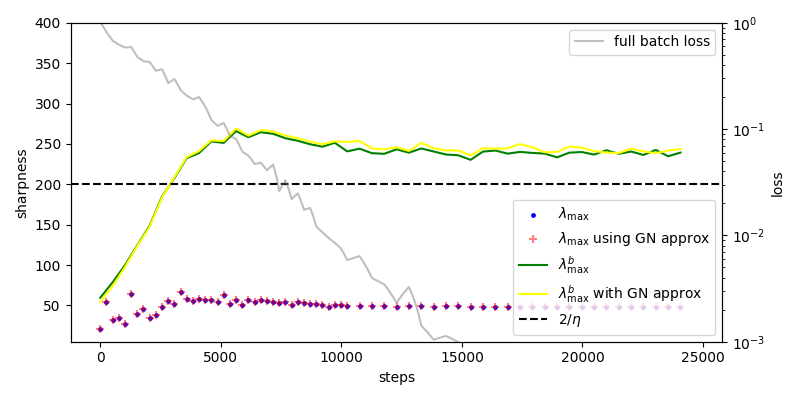}
    \subcaption{batch size 8}
    \label{fig:gn-8}
  \end{subfigure}\hfill
  \begin{subfigure}[t]{0.48\textwidth}
    \centering
    \includegraphics[width=\linewidth]{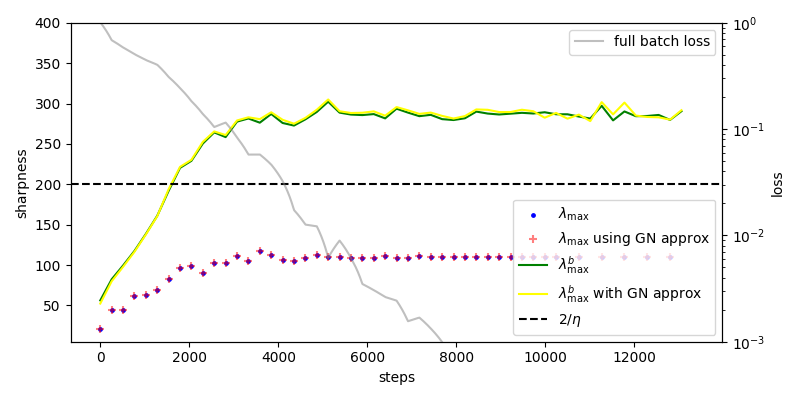}
    \subcaption{batch size 16}
    \label{fig:gn-16}
  \end{subfigure}

  \smallskip

  \begin{subfigure}[t]{0.48\textwidth}
    \centering
    \includegraphics[width=\linewidth]{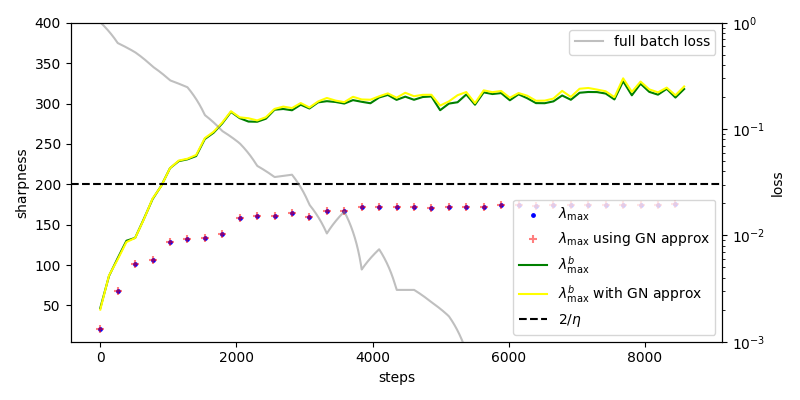}
    \subcaption{batch size 32}
    \label{fig:gn-32}
  \end{subfigure}\hfill
  \begin{subfigure}[t]{0.48\textwidth}
    \centering
    \includegraphics[width=\linewidth]{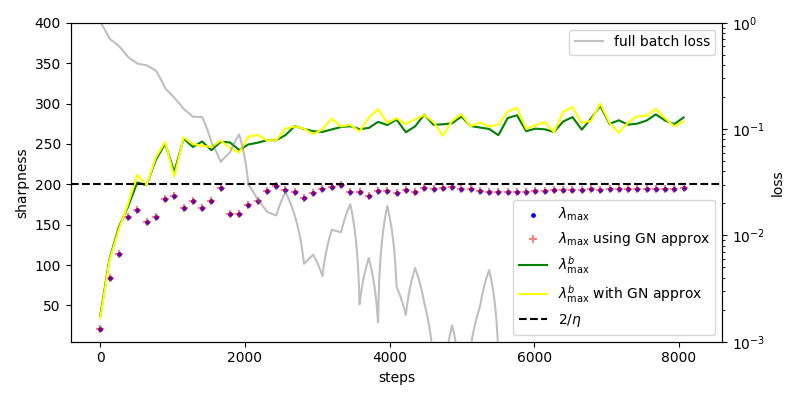}
    \subcaption{batch size 64}
    \label{fig:gn-64}
  \end{subfigure}

  \caption{\textbf{Gauss-Newton approximation.} Showcasing the correctness of using the Gauss-Newton approximation to the loss Hessian by comparing the $\lmax$ and $\lmax^b$ computed with the true loss Hessian and its GN-approximation across different batch sizes.}
  \label{fig:gn-approx}
\end{figure}


\clearpage


\section{Other Quantities of SGD Dynamics}
\label{appendix:other_quantities}

In this Appendix we explore other quantities that describe the SGD dynamics, and discuss their role from the point of view of governing stability. In particular, we are covering the following quantities: trace of full-batch loss Hessian, $\lmax^b$ (average max eigenvalues of mini-batch Hessians), and a modified version of \textit{Batch Sharpness}.

\subsection{Trace of the Loss Hessian}
\label{appendix:trace}
\begin{figure}[t]
  \centering
  \begin{subfigure}{0.48\linewidth}
    \includegraphics[width=\linewidth]{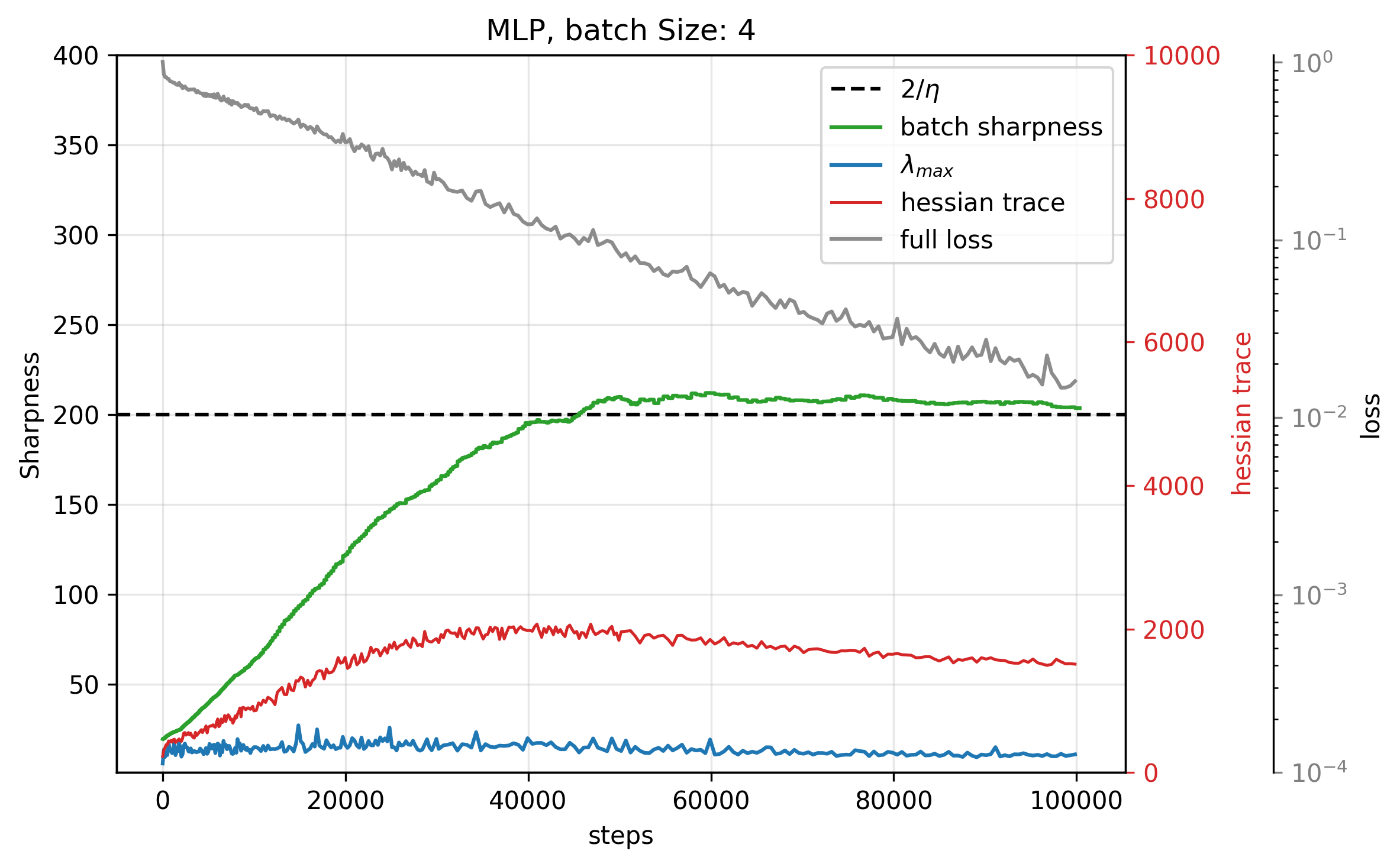}
    \caption{batch size 4}
  \end{subfigure}\hfill
  \begin{subfigure}{0.48\linewidth}
    \includegraphics[width=\linewidth]{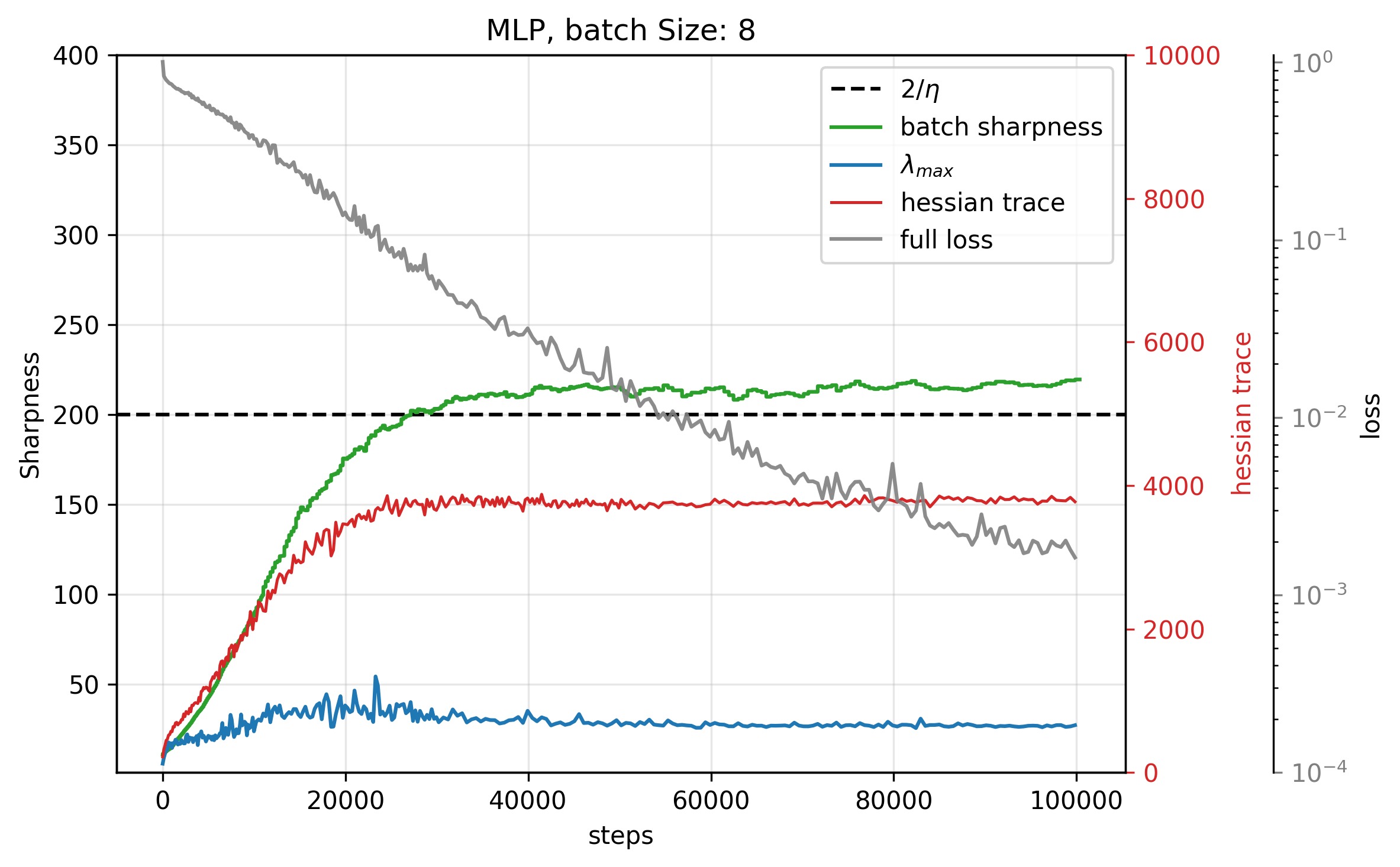}
    \caption{batch size 8}
  \end{subfigure}

  \vspace{0.6em}

  \begin{subfigure}{0.48\linewidth}
    \includegraphics[width=\linewidth]{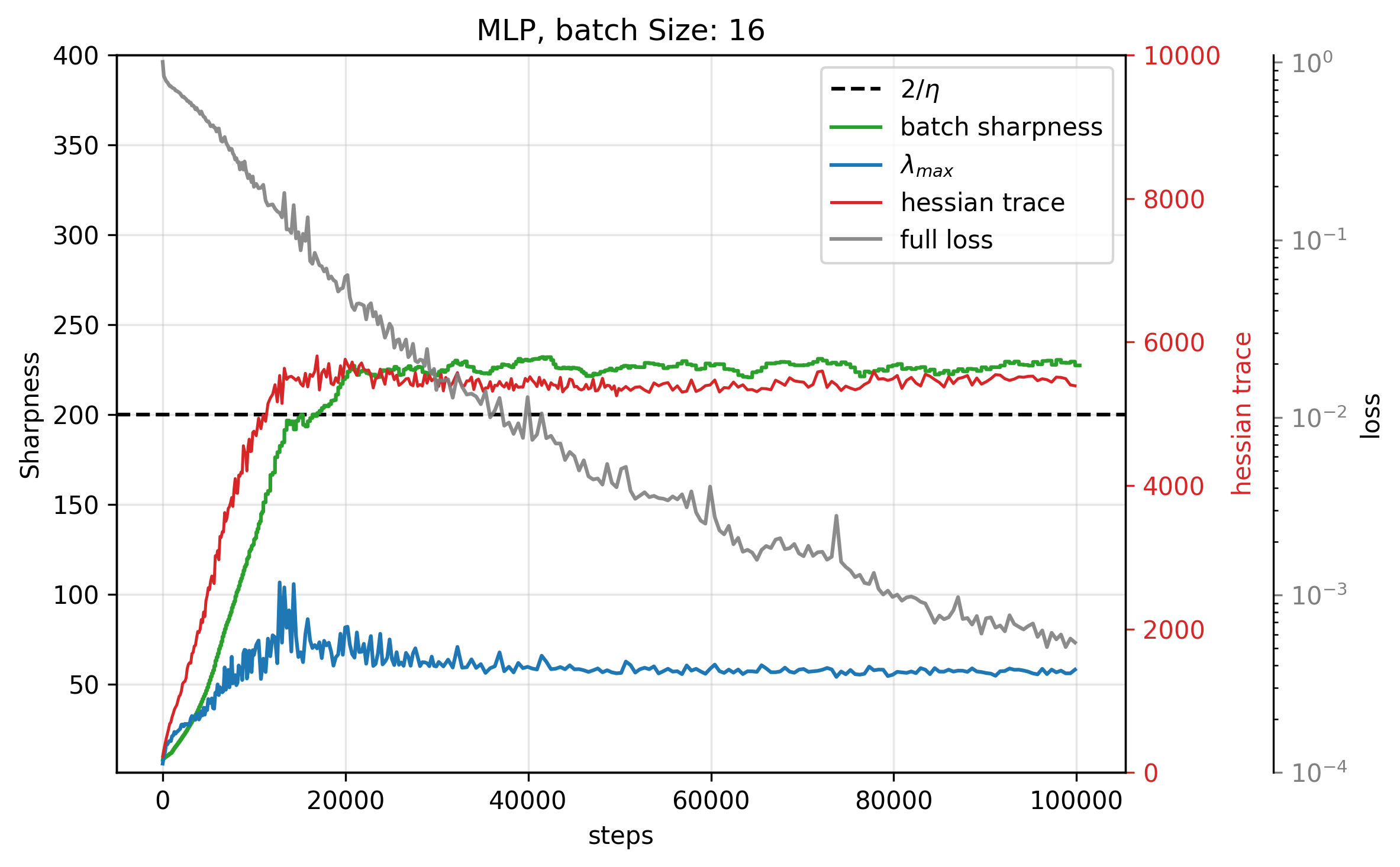}
    \caption{batch size 16}
  \end{subfigure}\hfill
  \begin{subfigure}{0.48\linewidth}
    \includegraphics[width=\linewidth]{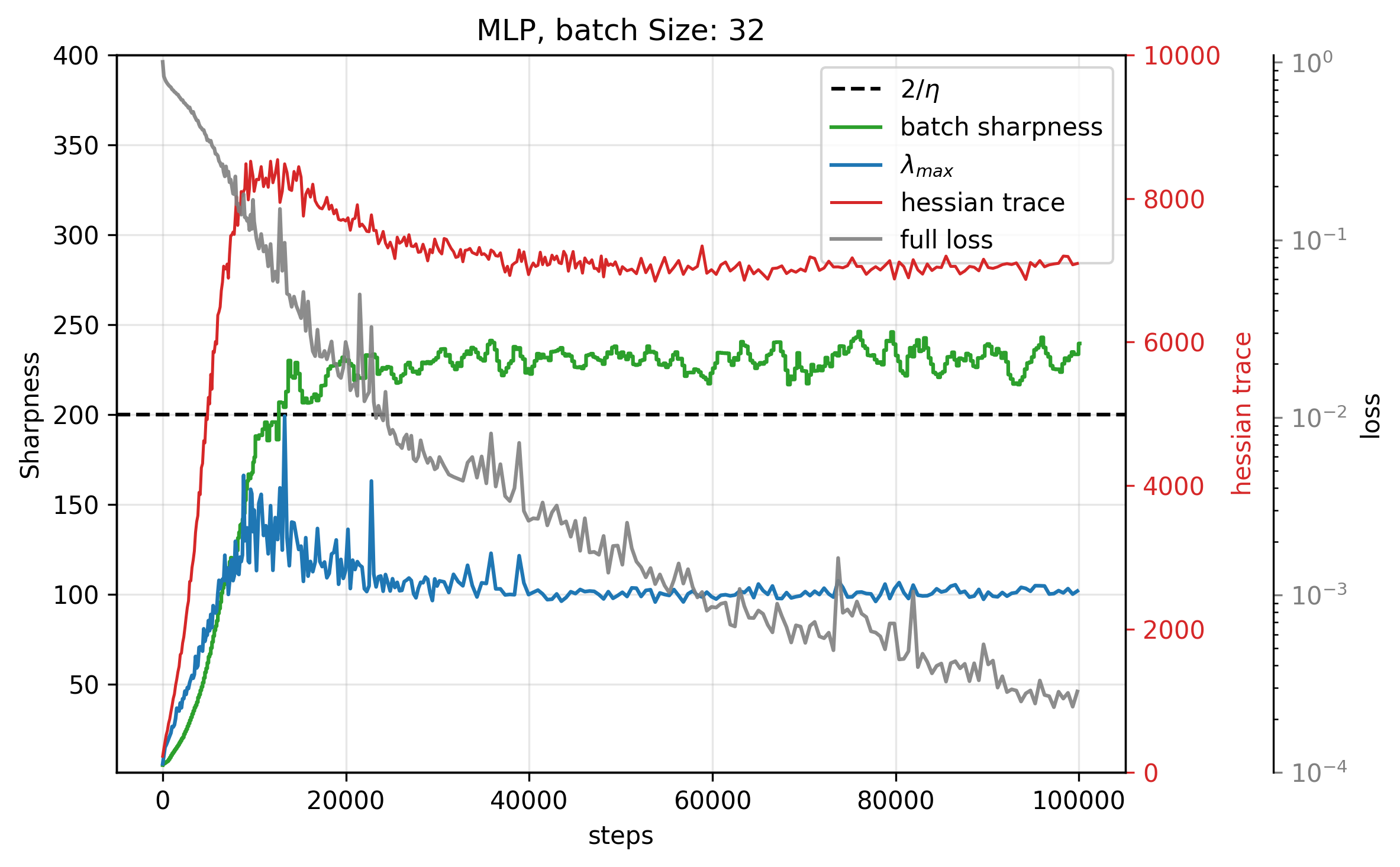}
    \caption{batch size 32}
  \end{subfigure}

  \caption{\textbf{Trace of the Hessian.} We plot the trace of the full-batch loss Hessian (red), together with the usual \emph{Batch Sharpness} (green) and $\lmax$ (blue). Notice that the scale of the trace of the Hessian is much bigger than the rest of the quantities, and it follows the axis on the right (in particular, has no particular relation to $2/\eta$. The plots showcase that trace behaves in a similar manner as $\lmax$---its level of stabilization is highly dependent on the batch size, it raises as long as \emph{Batch Sharpness} is rising, and it is stabilizes as \textit{Batch Sharpness} stabilizes. Here, we are doing experiments with MLP on CIFAR-10-8k and $\eta=0.01$}
  \label{fig:trace-mlp}
\end{figure}
A number of works \citet{ma_linear_2021,wu_implicit_2023,agarwala2024high} have linked the trace of the full-batch loss Hessian to implicit regularization by SGD. We plot in Figure \ref{fig:trace-mlp} and \ref{fig:trace-cnn}: $\lambda_{\max}$, \textit{Batch Sharpness}, and the trace of the Hessian along the training for a variety of models and batch sizes.
We observe here that trace of the Hessian behaves very similarly to the previously studied $\lmax$. In particular, it does not have a consistent stabilization level, and depends significantly on the batch size---with smaller batch sizes leading to lower stabilization level of the trace (aka flatter solutions). Also analogous to $\lmax$, it undergoes progressive sharpening, as long as \emph{Batch Sharpness} is under $2/\eta$. Analogously, the stabilization of \emph{Batch Sharpness} leads to stabilization of the trace. All of this showcases that trace of the Hessian is not the quantity that governs stability of the SGD dynamics. Yet, it might be a useful indicator of the end phase of progressive sharpening---in the potential situation when we have $\lmax$ stabilize, but other eigenvalues continue growing, as illustrated, for example, in \citet{cohen_understanding_2024}.

It is noteworthy that, in the context of MSE loss combined with piecewise-linear activation functions (e.g., ReLU), the trace of the full-batch loss Hessian coincides with the trace of its Gauss–Newton approximation. Furthermore, under MSE loss, the trace of the Gauss–Newton matrix is equal to the trace of the NTK. Consequently, in our setup (MLP with ReLU under MSE) evaluating the trace of the loss Hessian subsumes these cases.

\begin{figure}[t]
  \centering
  \begin{subfigure}{0.48\linewidth}
    \includegraphics[width=\linewidth]{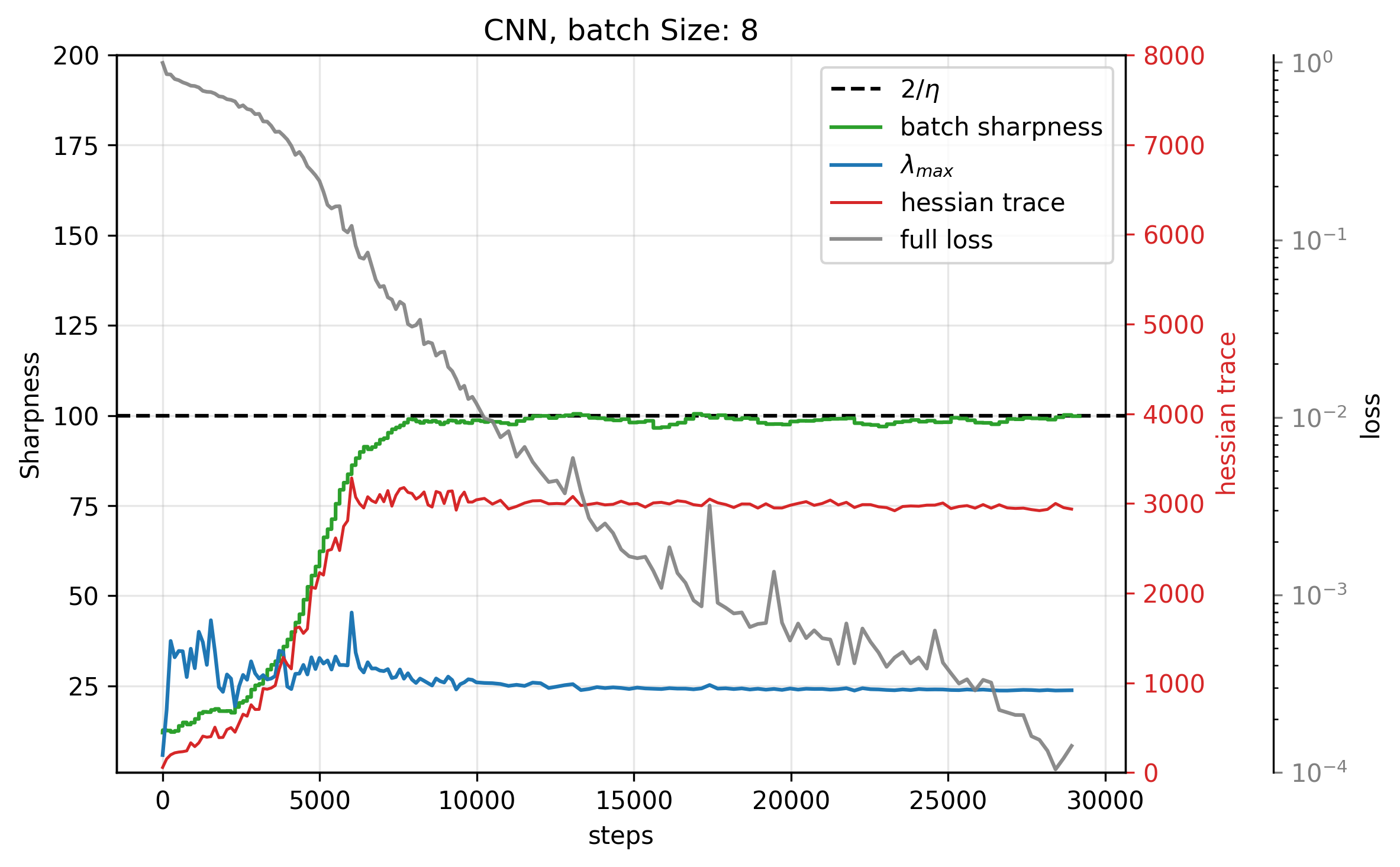}
    \caption{batch size 8}
  \end{subfigure}\hfill
  \begin{subfigure}{0.48\linewidth}
    \includegraphics[width=\linewidth]{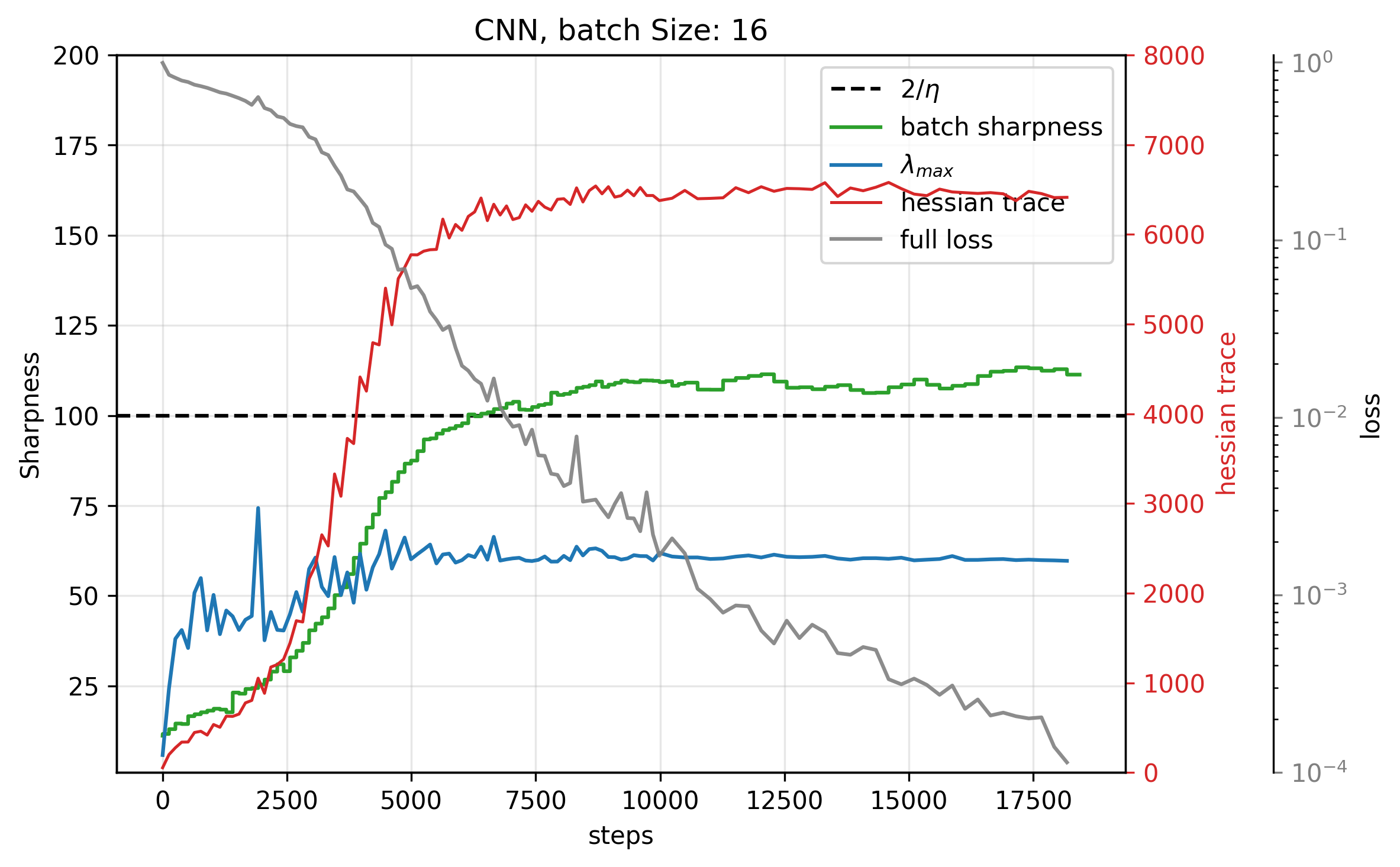}
    \caption{batch size 16}
  \end{subfigure}

  \caption{\textbf{Trace of the Hessian.} Similar to \ref{fig:trace-mlp}, but for CNN, and with $\eta=0.02$}
  \label{fig:trace-cnn}
\end{figure}

\subsection{$\lmax^b$: Expected Highest Eigenvalue of Mini-Batch Hessians}
\label{appendix:lambda_b_max}
In the early versions of this work we have looked at another promising quantity that we term $\lmax^b$:
\[
\lmax^b := \E_{B \sim \mathcal{P}_b}\Big[\lambda_{\max}(\mathcal{H}(L_\mathbf{B}))\Big].
\]
\begin{wrapfigure}{r}{0.5\columnwidth}  
    \centering
    \includegraphics[width=\linewidth]{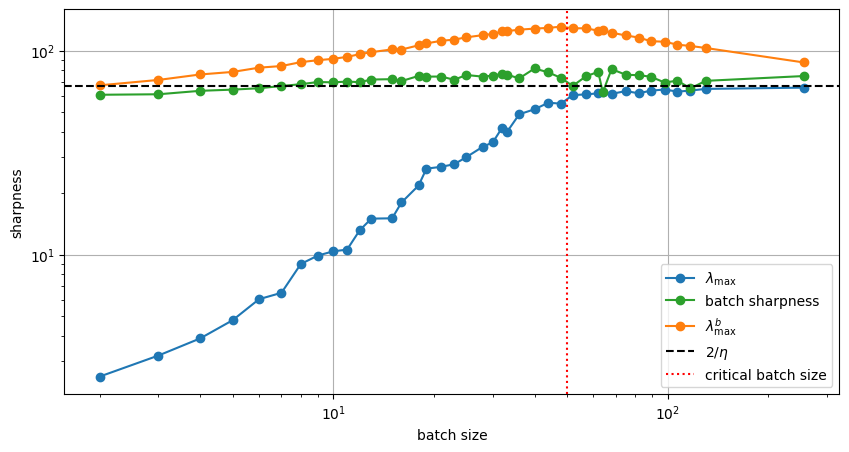}
    \vspace{-0.8cm}                             
    \caption{\textit{Final} stabilization levels of \emph{Batch Sharpness}, $\lmax$ and $\lmax^b$ \textit{vs} batch size. Only the stabilization level of \textit{Batch Sharpness} does not depend on batch size. Setting: MLP, CIFAR10-8k, $\eta=0.03$}
    \label{fig:bs_vs_lmax_vs_lmaxb}
\end{wrapfigure}
In particular, the significance of this quantity lies in its characterization of the worst-case sharpness of mini-batch loss landscapes. 
Yet, the reason why this quantity does not govern SGD dynamics arises from the very phenomenon distinguishing SGD dynamics from full-batch gradient descent---the misalignment of the mini-batch Hessians, see \ref{appendix:alignment}. 
Specifically, while individual mini-batch Hessians may exhibit considerable sharpness in their individual directions, these directions typically fail to align, preventing the emergence of a single dominant sharp direction. This scenario closely mirrors the behavior of the operator analyzed in Appendix \ref{appendix:linear_stochastic_stability}, illustrating why \textit{Batch Sharpness}, which dictates the stability of SGD dynamics, relies on both the size of the mini-batch Hessians together with their alignment with the step direction.


\begin{figure}[htbp]
  \centering
  \begin{subfigure}[t]{0.31\textwidth}
    \centering
    \includegraphics[width=\linewidth]{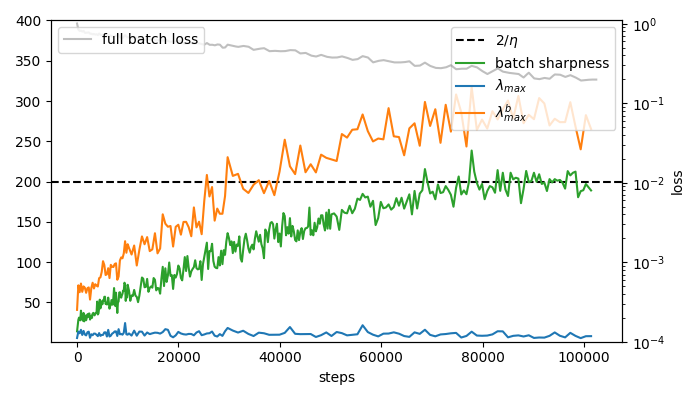}
    \subcaption{batch size 2}
    \label{fig:batch-2}
  \end{subfigure}\hfill
  \begin{subfigure}[t]{0.31\textwidth}
    \centering
    \includegraphics[width=\linewidth]{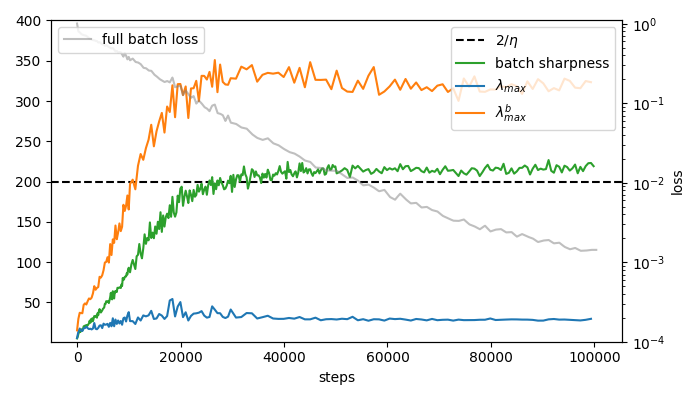}
    \subcaption{batch size 8}
    \label{fig:batch-8}
  \end{subfigure}\hfill
  \begin{subfigure}[t]{0.31\textwidth}
    \centering
    \includegraphics[width=\linewidth]{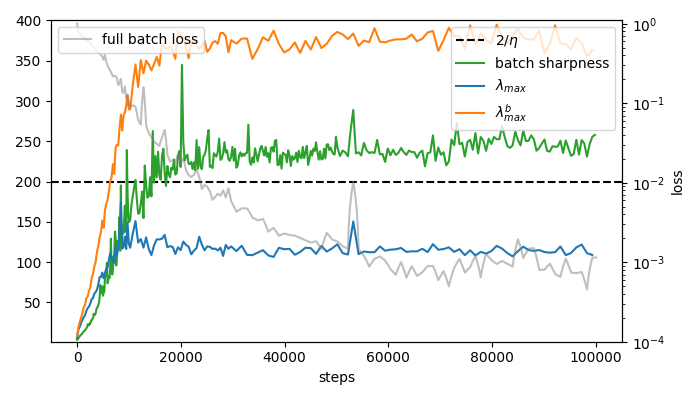}
    \subcaption{batch size 32}
    \label{fig:batch-32}
  \end{subfigure}
  \caption{\textbf{Behavior of $\lmax^b$:} $\lmax^b$ stabilizes higher than $2/\eta$, with its stabilization level dependent on the batch size. \textit{Batch Sharpness} and $\lmax$ also shown for comparison. Setting: MLP, $\eta=0.01$, CIFAR10-8k}
  \label{fig:lambda_b_max}
\end{figure}

Consequently, $\lmax^b$ stabilizes at a level higher than the threshold $2/\eta$ and \textit{Batch Sharpness}. Moreover, the precise stabilization level is sensitive to the chosen batch size, as showcased in Figure \ref{fig:lambda_b_max}, with the dependence of the level of stabilization on batch size shown in Figure \ref{fig:bs_vs_lmax_vs_lmaxb}. For additional experiments with illustration of behavior of $\lmax^b$ refer to Appendix \ref{appendix:further_lambda_b_max}. Now, this inconsistency of stabilization means that $\lmax^b$ does not govern the stability of SGD, and its stabilization is a by-product of \textsc{EoSS} and \emph{Batch Sharpness} stabilization.

In particular, we establish that:
\begin{enumerate}[label=(\roman*)]
    \item $\lambda_{\max}^b$ also stabilizes. 
    \item $\lambda_{\max}^b$ stabilizes at a level ranging between $2/\eta$ and $4/\eta$. The level is lower for very small and very large batch sizes, and higher for intermediate batch sizes.
    \item $\lambda_{\max}^b$ increases concurrently with \emph{Batch Sharpness} and stabilizes simultaneously, indicating insights into the nature of \emph{Batch Sharpness} growth and progressive sharpening.
    \item $\lambda_{\max}^b$ is by construction greater than both $\lambda_{\max}$ and \emph{Batch Sharpness}. The stabilization of \emph{Batch Sharpness} around $2/\eta$ for SGD and $\lambda_{\max}$ for GD ensures that $\lambda_{\max}^b$ stabilizes at or above $2/\eta$ in the \textsc{EoS/EoSS} regime.
\end{enumerate}
Concerning (iv), the inequality $\lambda_{\max}^b \geq$ \emph{Batch Sharpness} follows directly from the definition of \emph{Batch Sharpness} as an expectation of Rayleigh quotients. Furthermore, the inequality $\lambda_{\max}^b \geq \lambda_{\max}$ results from the following reasoning.
The largest singular value of the Hessian matrix derived from single data points is positive. This observation is crucial in establishing the following well-known property of matrix eigenvalues.
\begin{lemma}
\label{lemma:bigger}
    Let $m,b \in \N$ and consider $m$ matrices $M_1,M_2,\ldots,M_b \in \R^{m \times m}$ satisfying $\lambda_{\max} > |\lambda_{\min}|$. Then, the largest eigenvalue of their sum satisfies
\begin{equation}
    \lambda_{\max}\left( \sum_{i=1}^b M_i \right) 
    \, \ \leq \ \ 
    \sum_{i=1}^b \lambda_{\max}\left( M_i \right)
\end{equation}
    with equality only if all $M_i$ are identical. 
\end{lemma}
This lemma is a direct consequence of the convexity of the operator norm in matrices and the fact that the largest eigenvalue is positive in our setting. In our setting, it implies that with non-simultaneously-diagonalizeable matrices, the maximum eigenvalue of the sum is strictly less than the sum of the maximum eigenvalues of the individual matrices.
To illustrate, consider eigenvalue sequences for batch sizes that are powers of four, though the result generalizes to any $b_1 < b_2$:
\begin{equation}
    \lambda_{\max}^1 
    \ > \
    \lambda_{\max}^4 
    \ > \
    \lambda_{\max}^{16}
    \ > \
    \lambda_{\max}^{64} 
    \ > \
    \lambda_{\max}^{256} 
    \ > \
    \ldots
\end{equation}
Importantly, this ordering is the case only for "static" model -- i.e. when we take a model, and without changing the weights, evaluate $\lmax^b$ as we change the batch size. 

As noted in point (ii) above, this ordering does not hold in the \textit{trained} case, as different batch sizes affect also the progressive sharpening and the nature of the mini-batch Hessians. Specifically, since \emph{Batch Sharpness} stabilizes at $2/\eta$ at \textsc{EoSS}, the level of stabilization of $\lmax^b$ depends on its relation to \textit{Batch Sharpness}. Since the two are quite similar, with the difference that \textit{Batch Sharpness} also takes into account the alignment of the mini-batch landscapes sharpest directions with the mini-batch gradients---the gap between $\lambda_{\max}^b$ and \textit{Batch Sharpness} is governed precisely by this alignment.
As illustrated in Figure \ref{fig:bs_vs_lmax_vs_lmaxb}, the level of stabilization of $\lmax^b$ is similar to that of \textit{Batch Sharpness} for very small and very large batch sizes. 
For large batch sizes, this result is straightforward, as the dynamics approach full-batch GD, in which all relevant quantities equalize at \textsc{EoS}. 
Conversely, the small-batch case emerges because, for smaller batch sizes, the mini-batch Hessian (or its Gauss-Newton approximation) comprises averages of only a few per-sample \textit{model} gradient outer products, causing mini-batch gradients to align closely with the largest eigenvalues.

This alignment diminishes as the batch size increases, leading to a widening gap between \textit{Batch Sharpness} and $\lambda_{\max}^b$. Intriguingly, our experiments reveal that this gap only widens up to the aforementioned \textit{critical batch size}, which also serves as a switch between SGD and GD dynamics from the point of view of $\lmax$ stabilization. Beyond the \textit{critical batch size} the gap begins to narrow again, as depicted in Figure \ref{fig:bs_vs_lmax_vs_lmaxb}. Clarifying this phenomenon fully would be a key outcome of a comprehensive theory of progressive sharpening and SGD stability.

Another significant consequence of the stabilization of $\lambda_{\max}^b$ is that it provides insights into the mechanisms underlying the progressive sharpening of \textit{Batch Sharpness}.
Specifically, the growth in \textit{Batch Sharpness} could be attributed either to a general increase in the sharpness of the mini-batch landscapes or to an increase in alignment between mini-batch Hessians and gradients.
Notably, throughout the period of \textit{Batch Sharpness} increase, both $\lambda_{\max}^b$ and the trace of the loss Hessian consistently rise and stabilize simultaneously with \textit{Batch Sharpness}.
This suggests that 
 at least portion of the increase in \textit{Batch Sharpness} arises from the overall sharpening of the mini-batch landscapes, rather than solely from alignment of the mini-batch gradients and Hessians. Consequently, \textit{Batch Sharpness} appears closely linked with the progressive sharpening phenomenon itself, with its eventual stabilization marking the end of progressive sharpening.
 This points to the fact that \textit{Batch Sharpness} is closely connected to progressive sharpening.

\subsection{Modified \textit{Batch Sharpness}}

In the earlier versions of this work we also looked at a modified definition of \textit{Batch Sharpness}:
\begin{definition}[Modified \textit{Batch Sharpness}.]
\label{def:mod_BS}
    We call Modified \textit{Batch Sharpness} the quantity defined as
    \[
    \text{Modified \textit{Batch Sharpness}}(\theta)
    \quad := \quad
\frac{
\E_{B \sim \mathcal{P}_b}\Big[
\nabla L_B(\theta)^\top\,\mathcal{H}(L_B)\,\nabla L_B(\theta)\Big]
}{
\E_{B \sim \mathcal{P}_b}\Big[
\|\nabla L_B(\theta)\|^2
\Big ]
}.
\]
\end{definition}
The difference from the definition of \textit{Batch Sharpness} is that in this one the expectation over batches in taken inside the fraction. The intuition for this quantity comes from a notion of average stability on mini-batch landscapes. That is,
\[
\frac{
\E_{B \sim \mathcal{P}_b}\Big[
\nabla L_B(\theta)^\top\,\mathcal{H}(L_B)\,\nabla L_B(\theta)\Big]
}{
\E_{B \sim \mathcal{P}_b}\Big[
\|\nabla L_B(\theta)\|^2
\Big ]
} \leq 2/\eta \quad
\iff
\E\Big[L_B(\theta^B_{t+1}) - L_B(\theta_t) \Big] \leq 0
\]
where $\theta^B_{t+1}$ is the parameters that we are getting if we are stepping on the given mini-batch.
This means that \textit{Modified Batch Sharpness} $<2/\eta$ is equivalent to "on average, the mini-batch loss does not increase when stepping the corresponding landscape". This formulation is an attempt to extend the descent lemma to the mini-batch landscapes that govern the SGD dynamics insted of the descent lemma on the full-batch landscape that govern GD. Empirically, it turns out that \textit{Modified Batch Sharpness} also stabilizes, but its stabilization level is higher than that of the \textit{Batch Sharpness} and therefore $2/\eta$, as illustrated in Figure \ref{fig:modified-batch-sharpness}. Moreover, its stabilization level is dependent on the batch size.

\paragraph{Modified \textit{Batch Sharpness} and mini-batch gradients.}
Importantly, we show in Proposition \ref{Prop:mod_BS} that Modified \textit{Batch Sharpness} is a valid Instability Criterion and it governs the explosion of the expectation of the norm squared of the mini-batch gradients. 
\[
\text{Modified \textit{Batch Sharpness}}(\theta_{t}) \ > \ 2/\eta + c\eta
    \quad \implies \quad \E_{B \sim \mathcal{P}_b}[\| \nabla L_B (\theta_{t+1})\|^2_2] \ > \ \E_{B \sim \mathcal{P}_b}[\| \nabla L_B (\theta_{t})\|^2_2].
\]

\begin{figure}[htbp]
  \centering

  \begin{subfigure}[t]{0.48\textwidth}
    \centering
    \includegraphics[width=\linewidth]{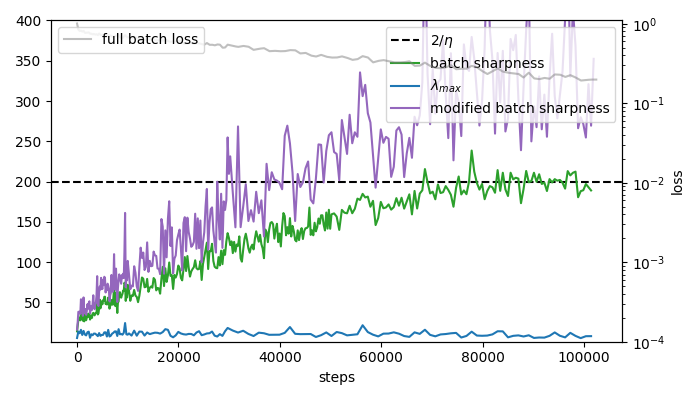}
    \subcaption{batch size 2}
    \label{fig:batch-2}
  \end{subfigure}\hfill
  \begin{subfigure}[t]{0.48\textwidth}
    \centering
    \includegraphics[width=\linewidth]{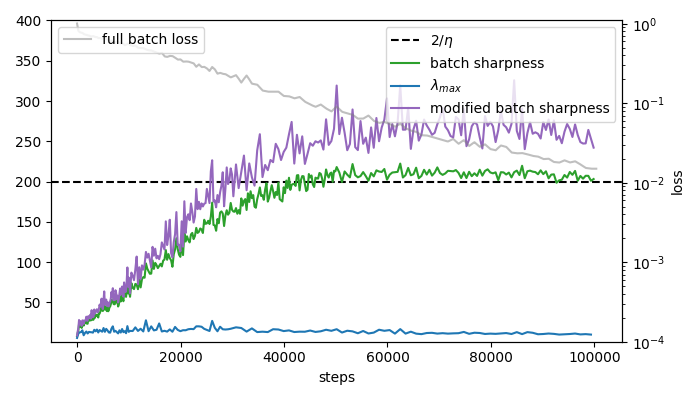}
    \subcaption{batch size 4}
    \label{fig:batch-4}
  \end{subfigure}

  \smallskip

  \begin{subfigure}[t]{0.48\textwidth}
    \centering
    \includegraphics[width=\linewidth]{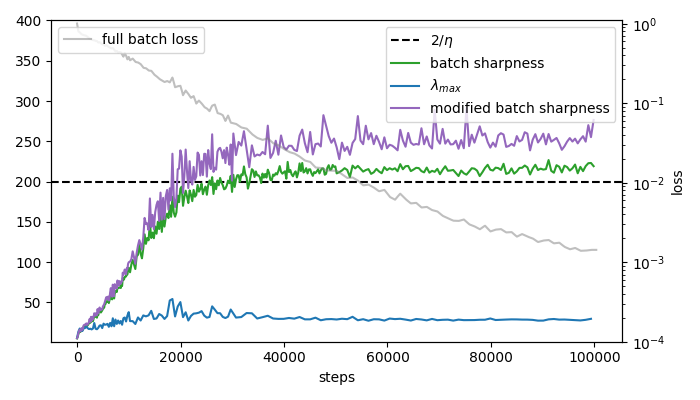}
    \subcaption{batch size 8}
    \label{fig:batch-8}
  \end{subfigure}\hfill
  \begin{subfigure}[t]{0.48\textwidth}
    \centering
    \includegraphics[width=\linewidth]{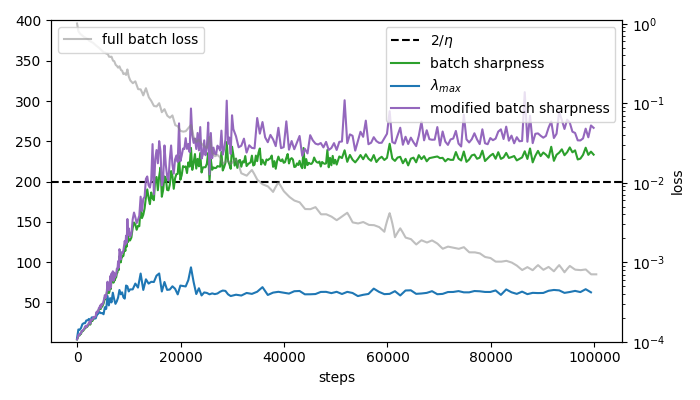}
    \subcaption{batch size 16}
    \label{fig:batch-16}
  \end{subfigure}

  \smallskip

  \begin{subfigure}[t]{0.48\textwidth}
    \centering
    \includegraphics[width=\linewidth]{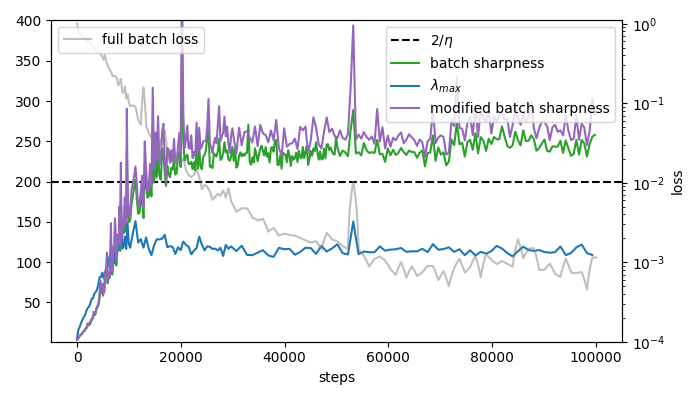}
    \subcaption{batch size 32}
    \label{fig:batch-24}
  \end{subfigure}\hfill
  \begin{subfigure}[t]{0.48\textwidth}
    \centering
    \includegraphics[width=\linewidth]{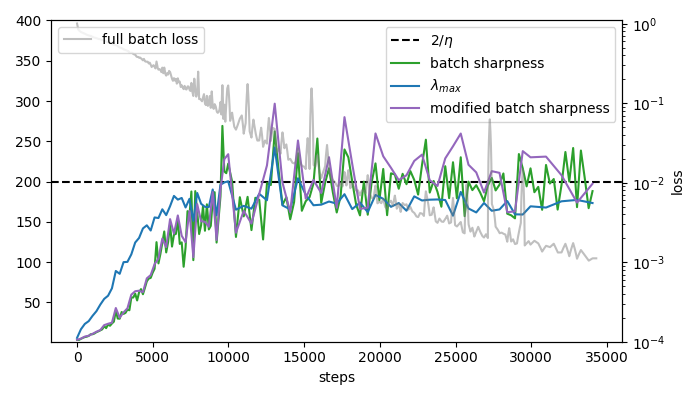}
    \subcaption{batch size 64}
    \label{fig:batch-32}
  \end{subfigure}

  \smallskip

  \caption{\textbf{Modified Batch Sharpness}: Behavior of \textit{Modified Batch Sharpness}, definition of which is similar to \textit{Batch Sharpness}, but with the expectation "inside" the fraction ($\E[\nabla L_B(\theta)^T H_B \nabla L_B(\theta)] / \E[\|\nabla L_B(\theta)\|^2]$). It stabilizes above $2/\eta$ with its stabilization level highly dependent on the batch size.}
  \label{fig:modified-batch-sharpness}
\end{figure}

\clearpage



\clearpage

\section{Modified \textit{Batch Sharpness} Mathematically}
\label{appendix:mod_insta}
We show here that Modified \textit{Batch Sharpness} (Definition \ref{def:mod_BS}) is a locally valid instability criterion (Definition \ref{def:stab_not}).
\begin{mdframed}
\textit{Importantly, while it is a valid instability criterion, it does not stabilize at $2/\eta$ in practice, thus it is not the quantity that self-stabilization tames, but its stabilization is a byproduct of \textit{Batch Sharpness} stabilizing.}
\end{mdframed}
We now compute what the update of the norm of the gradients $\E_i [\mednorm{Y_i}_2^2]$ after one step is. Precisely, with the notations of Appendix \ref{appendix:proof_theo}, we are computing here the value of $\E_{t} \E_i [\mednorm{Y_i^{t+1}}_2^2]$ so the average over the iterations of the update to the quantity $\mathcal{C}$ above.
Precisely we here prove the following Proposition.
\begin{proposition}
\label{prop:div}
    There exists an absolute constant $c>0$ such that when Modified \textit{Batch Sharpness} $> 2/\eta + c\eta$, then $\E\mednorm{\nabla L_B}_2^2$ increases in size exponentially and the trajectory diverges (is quadratically unstable, see Definition \ref{def:insta}).
\end{proposition}

\begin{proof}
In the proof we use the notations  of Appendix \ref{appendix:order}.

\paragraph{Step 1: One step on the gradient's second moment}.
Remind that the SGD iterate satisfy
\[
  \theta_{t+1}
     \;=\;
     \theta_t - \eta\,Y_{j_t}(\theta_t),
     \quad i_t \stackrel{\text{i.i.d.}}{\sim} \mathcal D,
\]
and define a \emph{fresh}, independent index \(j\) used only for the
outer expectation in \(\mathcal{C}^{t+1}\).  Because \(j\perp i_t\) we may write
\[
     Y_i(\theta_{t+1})
       \;=\;
       H_i\!\bigl(\theta_{t+1}-x_i\bigr)
       = Y_i(\theta_t)\;-\;\eta\,H_iY_{j_t}(\theta_t).
\]
Squaring, expanding, and averaging over \(j\) gives
\begin{equation}
\begin{split}
   \mathcal{C}^{t+1}
     &= \E_i\bigl\|Y_i(\theta_t)-\eta H_iY_{j_t}(\theta_t)\bigr\|^{2}\\
     &= \mathcal{C}^t
        -2\eta\,\underbrace{
             \E_{i,j_t}\!\bigl[
               Y_i(\theta_t)^{\!\top}H_iY_{j_t}(\theta_t)
             \bigr]}_{\text{cross term}}
        +\eta^{2}\,
          \underbrace{
            \E_{i,j_t}\!\bigl[
              Y_{j_t}(\theta_t)^{\!\top}H_i^{2}Y_{j_t}(\theta_t)
            \bigr]}_{\text{variance term}}.
\end{split}
\end{equation}

\paragraph{Step 2: Decoupling the indices.}
Note that \eqref{eq:Delta_A_B} in the proof of Lemma \ref{lem:cross-new} establishes that 
\begin{equation}
\label{eq:cross_prod}
    2\E_{i,j_t}\!\bigl[
    Y_i(\theta_t)^{\!\top}H_iY_{j_t}(\theta_t)
    \bigr]
    \; = \;
    \mathcal{A} - \mathcal{B} - \widetilde \Delta.
\end{equation}
This implies that we can rewrite
\begin{equation}
\begin{split}
   \mathcal{C}^{t+1}
     &= \mathcal{C}^t
        -\eta\,(\mathcal{A} - \mathcal{B} - \widetilde \Delta)
        +\eta^{2}\,
          (\text{variance term}).
\end{split}
\end{equation}
Next note that if we are at the \textsc{EoSS}, then $\mathcal{A} \approx \tfrac{2}{\eta}(1+\delta) \mathcal{C}$ for some $\delta \in \R$. This implies that we can rewrite the term above as
\begin{equation}
\begin{split}
   \mathcal{C}^{t+1} 
     \;&\approx \; -(1+2\delta)\,\mathcal{C}^t 
     \; + \;
     \underbrace{\eta \mathcal{B} + \eta \widetilde \Delta
     + \eta^{2}\,
          (\text{variance term})}_{\text{rest}}.
\end{split}
\end{equation}
Let us know understand the size of the rest, the trajectory diverges if and only if:
\begin{equation}
    \eta \mathcal{B} + \eta \widetilde \Delta + \eta^2 
    \E_{i,j_t}\!\bigl[
              Y_{j_t}(\theta_t)^{\!\top}H_i^{2}Y_{j_t}(\theta_t)
            \bigr]
    \; > \; 2(1+\delta)\,\mathcal{C}^t.
\end{equation}
Next note that by applying Jensen inequality to the term multiplied by $\eta^2$ we obtain that 
\begin{equation}
    \sqrt{\underbrace{\E_{i,j_t}\!\bigl[
    Y_{j_t}(\theta_t)^{\!\top}H_i^{2}Y_{j_t}(\theta_t)
    \bigr]}_{\text{variance term}}
    \cdot \underbrace{\E_{i}\!\bigl[
    Y_{i}(\theta_t)^{\!\top} Y_i(\theta_t)
    \bigr]}_{\mathcal{C}}} \; \geq \; \underbrace{\E_{i,j_t}\!\bigl[
    Y_{j_t}(\theta_t)^{\!\top} H_i \cdot Y_i(\theta_t)
    \bigr]}_{\mathcal{D}}.
\end{equation}
\paragraph{Step 3: Final algebra.}
Plugging this above, we obtain that the trajectory diverges when
\begin{equation}
    \eta \mathcal{B} + \eta \widetilde \Delta + \eta^2 
    \frac{\mathcal{D}^2}{\mathcal{C}}
    \; > \; 2(1+\delta)\,\mathcal{C}.
\end{equation}
Again applying \eqref{eq:cross_prod} we obtain that this is equivalent to
\begin{equation}
    \eta \mathcal{B} + \eta \widetilde \Delta + \eta^2 
    \frac{\big( \mathcal{A} - \mathcal{B} - \widetilde \Delta \big)^2}{4\mathcal{C}}
    \; > \; 2(1+\delta)\,\mathcal{C}.
\end{equation}
Since $\eta \mathcal{A} = 2(1+\delta) \mathcal{C}$, then $\eta^2 \mathcal{A}^2 = 4(1+\delta)^2 \mathcal{C}^2$ to asking
\begin{equation}
    \eta \mathcal{B} + \eta \widetilde \Delta + \eta^2 
    \frac{\mathcal{B}^2 + \widetilde \Delta^2 -2 \mathcal{A} \widetilde \Delta - 2 \mathcal{A}\mathcal{B} + 2 \mathcal{B} \widetilde \Delta }{4\mathcal{C}}
    \; > \; 2(1+\delta)\,\mathcal{C} - \frac{4(1+\delta)^2 \mathcal{C}^2}{4 \mathcal{C}} .
\end{equation}
Furthermore, equivalent to asking
\begin{equation}
    \eta \mathcal{B} + \eta \widetilde \Delta - \frac{2(1+\delta)}{2}\eta \widetilde \Delta - \frac{2(1+\delta)}{2}\eta\mathcal{B} + \eta^2 
    \frac{\mathcal{B}^2 + \widetilde \Delta^2 + 2 \mathcal{B} \widetilde \Delta }{4\mathcal{C}}
    \; > \; (1-\delta+\delta^2) \,\mathcal{C}
\end{equation}
or, even further simplified
\begin{equation}
    \eta \delta (\mathcal{B} + \widetilde \Delta) + \eta^2 
    \frac{(\mathcal{B} + \widetilde \Delta)^2}{4\mathcal{C}}
    \; > \; (1-\delta+\delta^2) \,\mathcal{C}.
\end{equation}
Here we plug in \eqref{eq:Delta_A_B} again and we can rewrite this as
\begin{equation}
    \eta \delta (\mathcal{A} - 2\mathcal{D}) + \eta^2 
    \frac{(\mathcal{A} - 2\mathcal{D})^2}{4\mathcal{C}}
    \; > \; (1-\delta+\delta^2) \,\mathcal{C}.
\end{equation}
By plugging, as before, $\eta \mathcal{A} = 2(1+\delta) \mathcal{C}$ we obtain
\begin{equation}
    2\delta (1 + \delta) \mathcal{C} - 2\eta \delta 2\mathcal{D} - 2\eta (1+\delta) \mathcal{D} + \eta^2 
    \frac{\mathcal{D}^2}{\mathcal{C}}
    \; > \; \big(1-\delta+\delta^2-(1+\delta)^2\big) \,\mathcal{C}
\end{equation}
which simplifies as 
\begin{equation}
    \underbrace{2\eta (1+2\delta) \mathcal{D} }_{\mathcal{O}(\eta^2)}
    - \underbrace{\eta^2 \frac{\mathcal{D}^2}{\mathcal{C}}}_{\mathcal{O}(\eta^4)}
    \; < \; \delta (5+2\delta) \,\underbrace{\mathcal{C}}_{\mathcal{O}_\eta(1)}.
\end{equation}
Thus there exists a constant $c > 0$, such that if $\delta > c \eta^2$ the trajectory diverges exponentially, if $\delta < c \eta^2$ the trajectory is stable.

\end{proof}

\clearpage

\section{Hardware \& Compute Requirements}
All experiments were executed on a single NVIDIA A100 GPU (80 GB) with 256 GB of host RAM.
The software stack comprises Python 3.12 and PyTorch 2.5.1 (built with the default CUDA tool-chain supplied by the wheel).
The biggest constraint is the VRAM when computing $\lmax$, as it requires double-backprop of the computational graph of the NN. It is possible to run the MLP experiments on MIGs with 10 GB of memory, other architectures require GPUs with higher VRAM.

\paragraph{Baseline MLP (2M parameters, Section \ref{appendix:further_bs})}
Training for 100k steps on the 8 k-image CIFAR-10 subset finishes in $\approx$ 5 min wall-clock while computing step sharpness every 8 steps, \textit{Batch Sharpness} every 128 steps and $\lambda_{\max}$ every 256 steps.
Peak device memory is less than 1GB during ordinary training and $\approx$ 2 GB while estimating $\lambda_{\max}$ on a 8k subset, comfortably fitting a 10 GB VRAM GPU.

\paragraph{Algorithmic caveats.}
We rely on power iteration for $\lambda_{\max}$; while Lanczos would reduce the number of Hessian–vector products, the official PyTorch implementation remains CPU-only. In the later versions we switched to using LOBPCG, adopted from the JAX implementation. For some speed-up, we cache the forward and the first backward pass through the network to reuse when calculating the Hessian-vector products. We cache the eigenvectors from previous $\lmax$ calculations to do either of the above algorithms with warm restarts, which provides significant speed-ups (e.g. reducing the number of LOBPCG iterations from $\sim 20$ to $\sim 1$). This provides another reason for choice of power iteration or LOBPCG over Lanczos, despite the one mentioned above---as most of the Lanczos implementations do not support warm restarts.

We also tried additional techniques to speed-up computations or reduce memory consumption. Batching vectors in the Hessian-vector products (for use in LOBPCG) did not provide any noticeable speed-up. Using forward-mode AD did not provide any noticeable reduction in memory consumption.


\section{Experiments: $\lambda_{\max}^b$}
\label{appendix:further_lambda_b_max}

In this appendix, we provide additional empirical evidence for both emergence of \textsc{EoSS} and to Appendix \ref{appendix:lambda_b_max}, varying across models, step sizes, and batch sizes. Consistent with our primary findings, we observe that $\lambda_{\max}^b$ consistently stabilizes within the interval \(\bigl(2/\eta,\; 2\times 2/\eta\bigr]\), in particular always higher than \textit{Batch Sharpness} and $\lambda_{\max}$. We are conducting experiments on MLP, CNN and ResNet-20 in Figures \ref{fig:appendix_blmax_mlp}, \ref{fig:appendix_blmax_cnn}, \ref{fig:appendix_blmax_resnet} respectively. 

Note that the fact that $\lambda_{\max}^b$ consistently stabilizes above $2/\eta$ implies that the supremum of Lipchitz constants of the gradient of individual mini-batch losses also stabilizes above $2/\eta$, thus clearly indicating that the usual assumptions in theory works on SGD about step size break. Same applies for supremum of Lip constants of gradients of per-sample losses. The former fact trivially follows from the inequality between sup and mean, and the second one from the same plus \cref{lemma:bigger}.

\begin{figure}[t]
    \centering
    \begin{tabular}{cc}
        \includegraphics[width=0.45\textwidth]{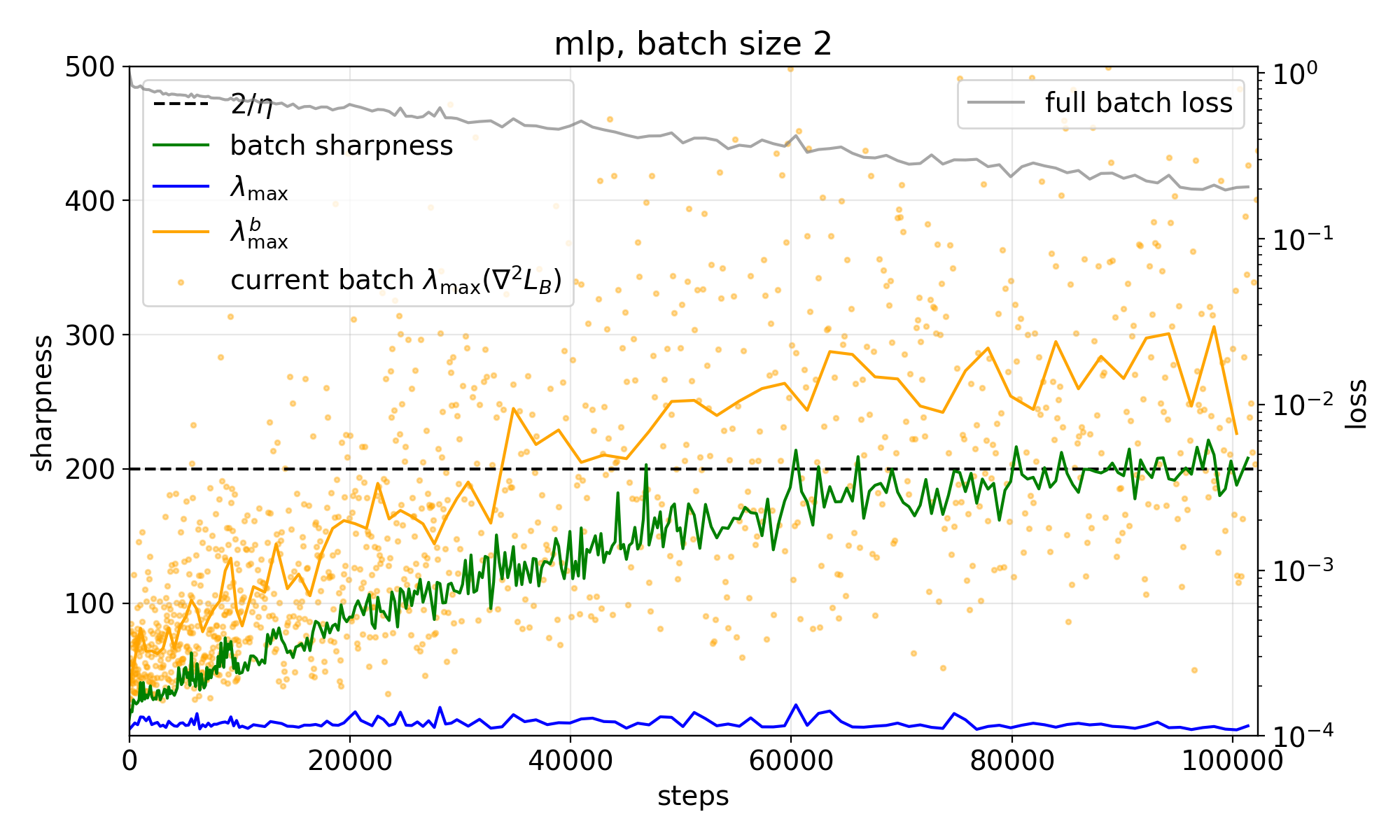} &
        \includegraphics[width=0.45\textwidth]{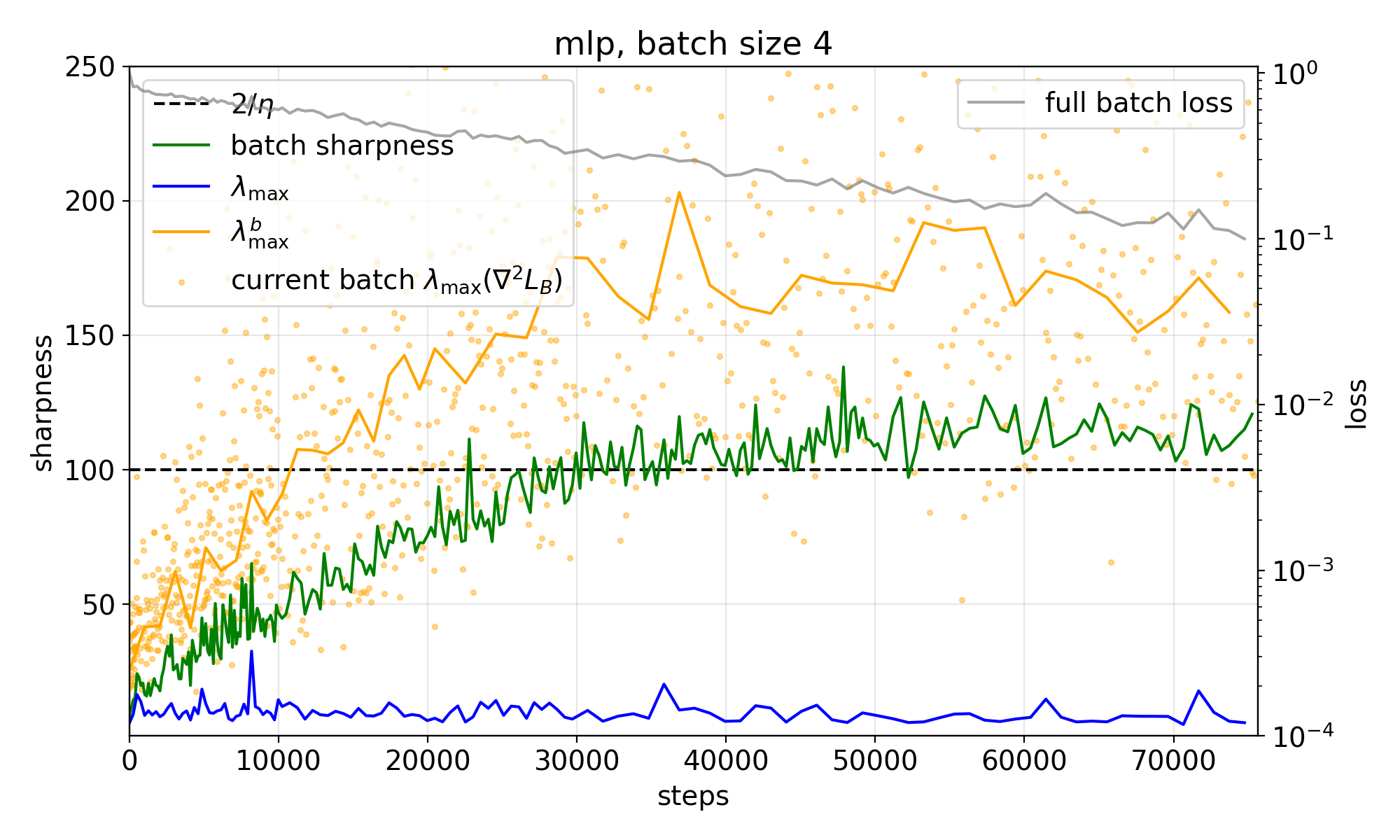} \\
        \includegraphics[width=0.45\textwidth]{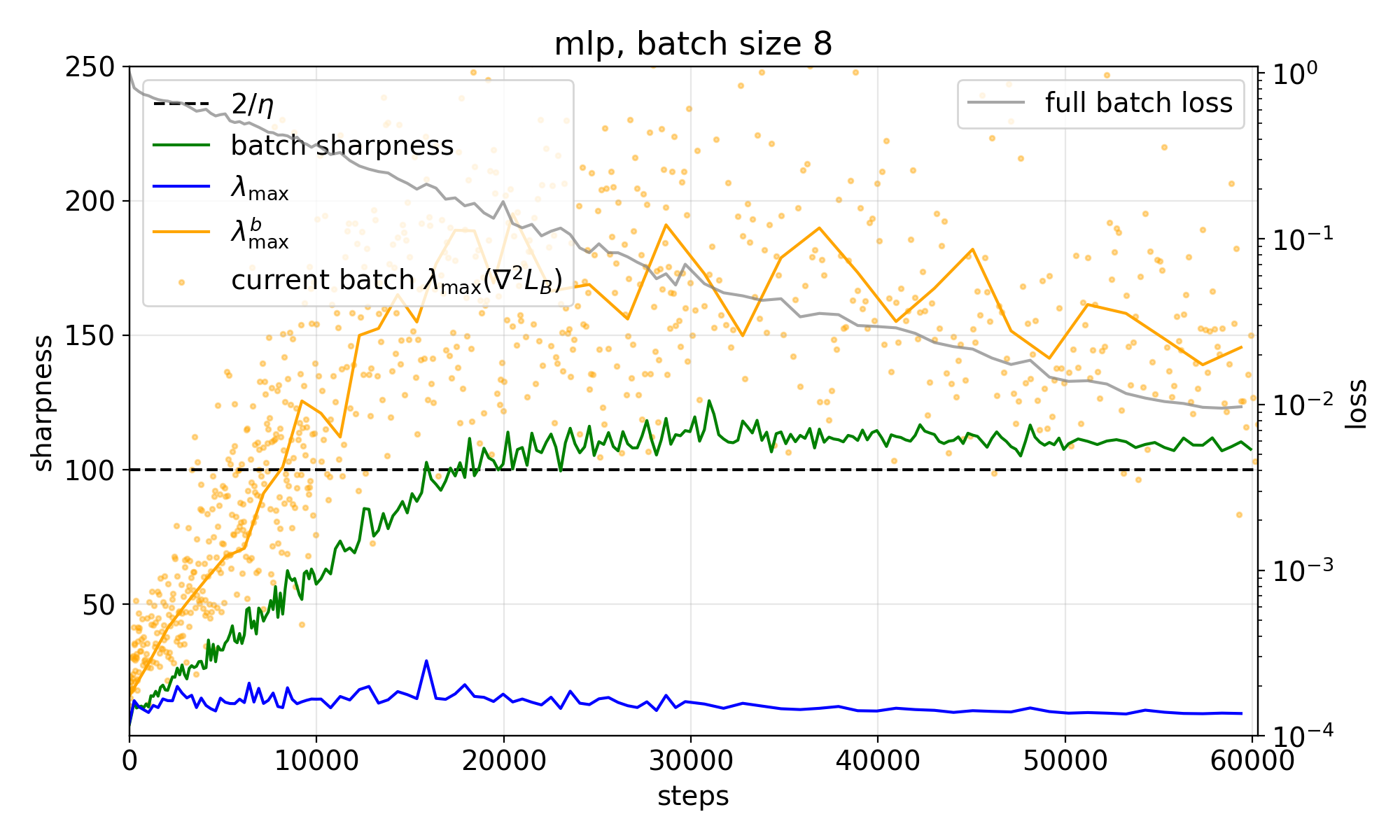} &
        \includegraphics[width=0.45\textwidth]{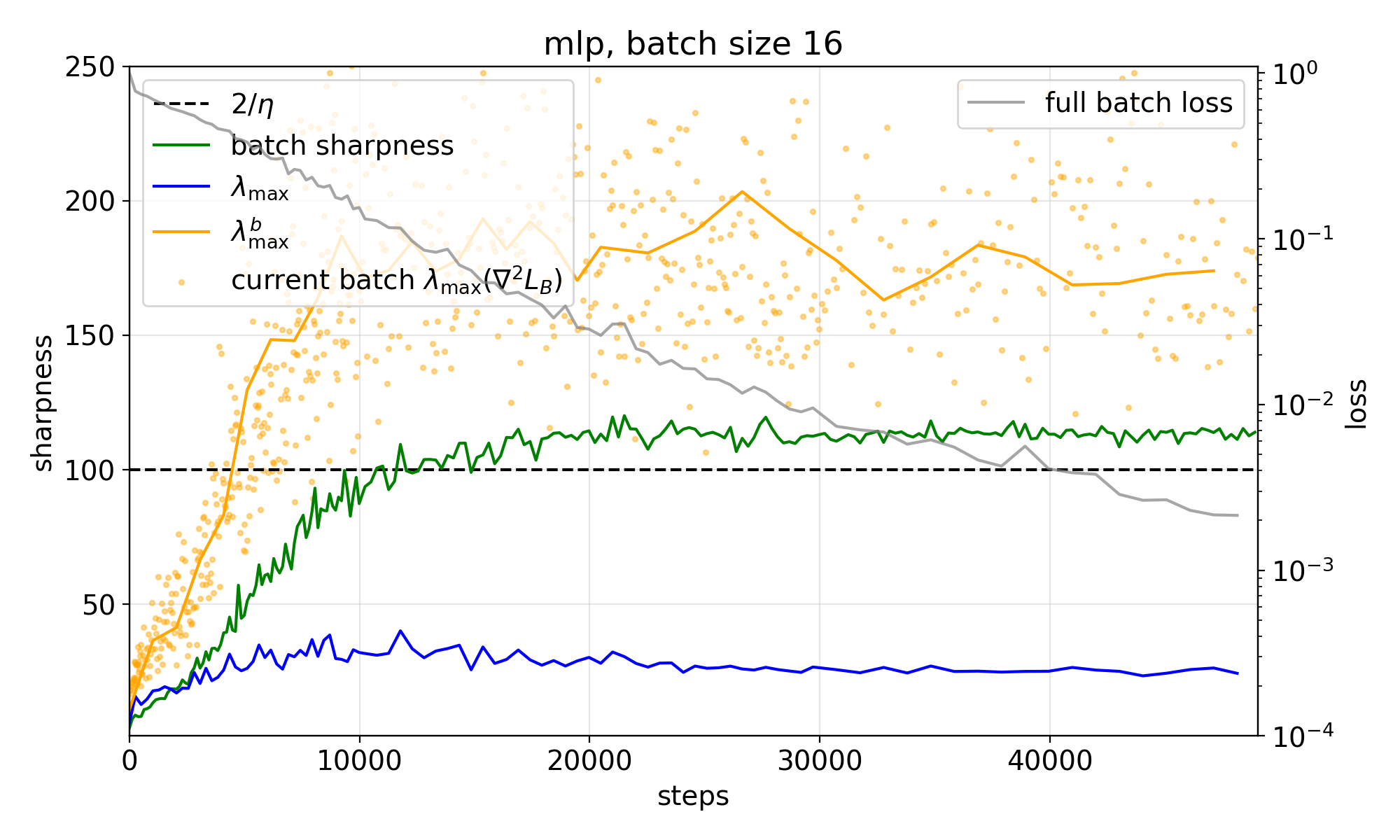} \\
        \includegraphics[width=0.45\textwidth]{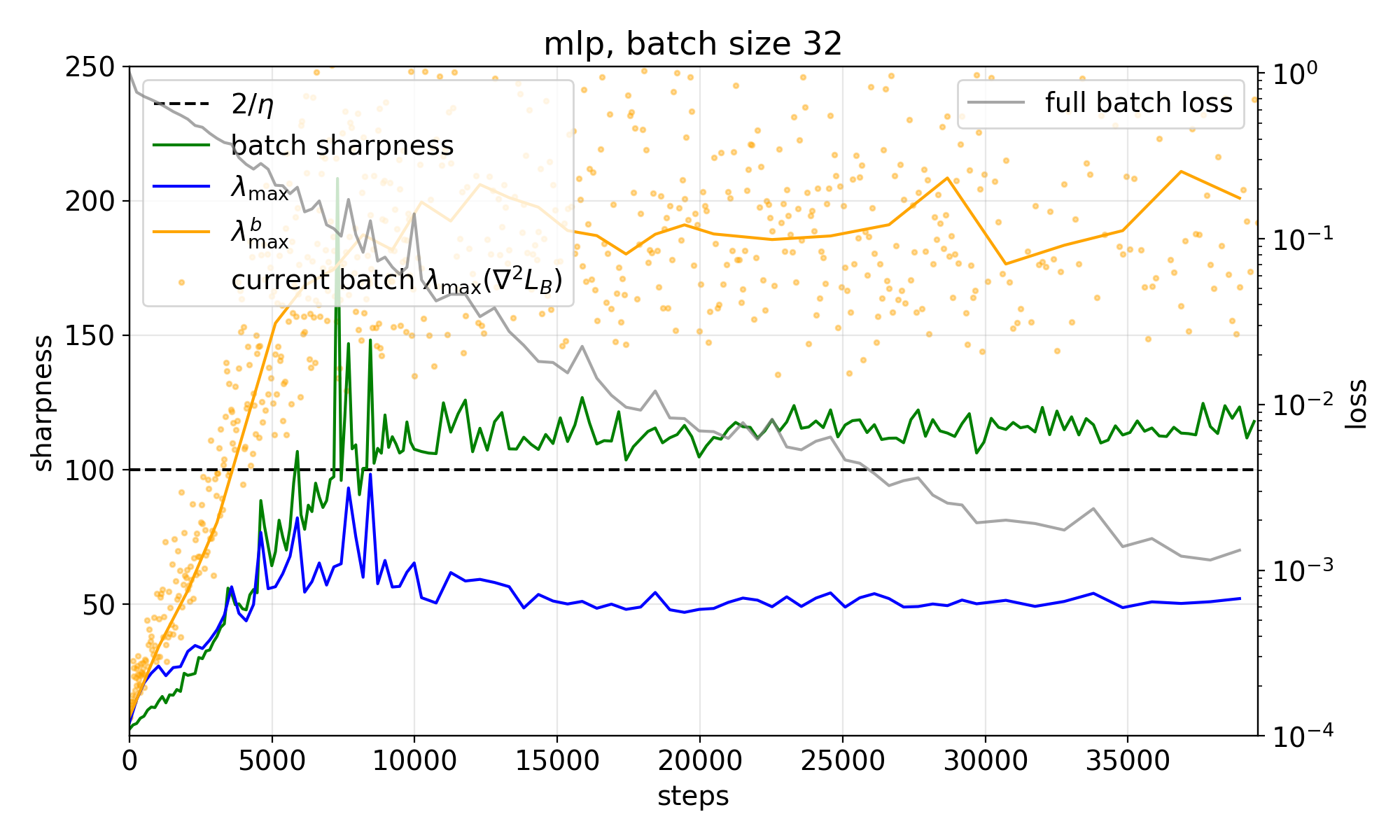} &
        \includegraphics[width=0.45\textwidth]{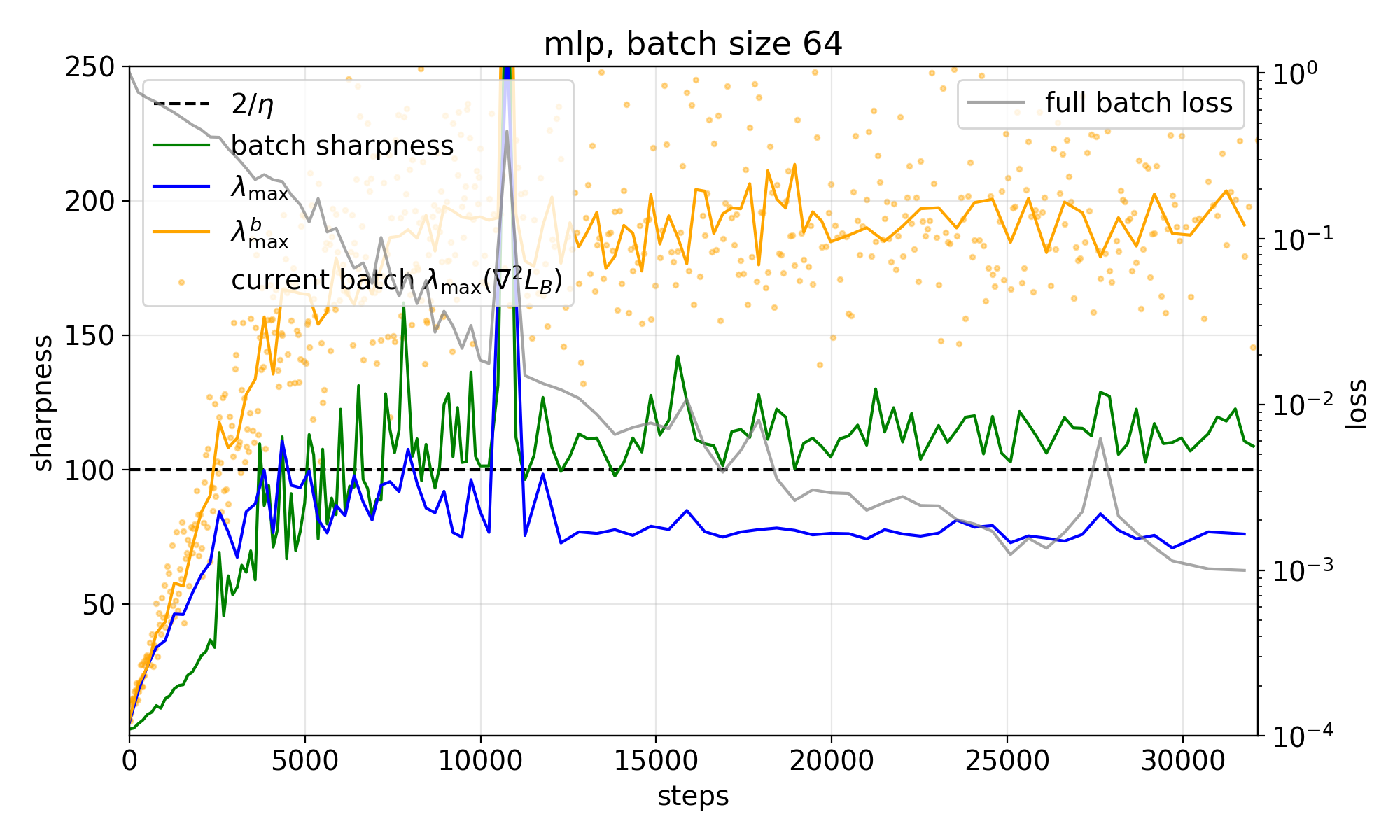} \\
        \includegraphics[width=0.45\textwidth]{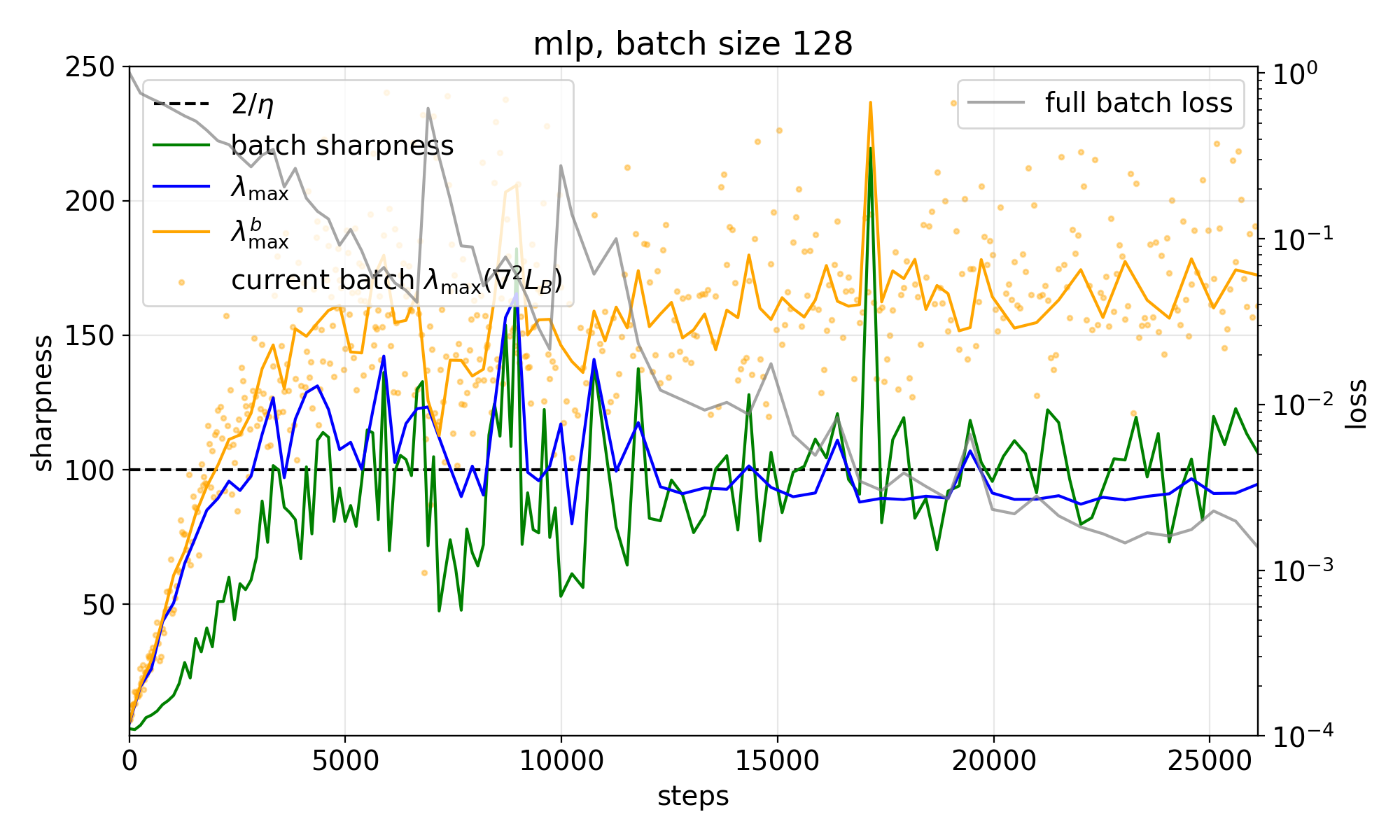} &
        \includegraphics[width=0.45\textwidth]{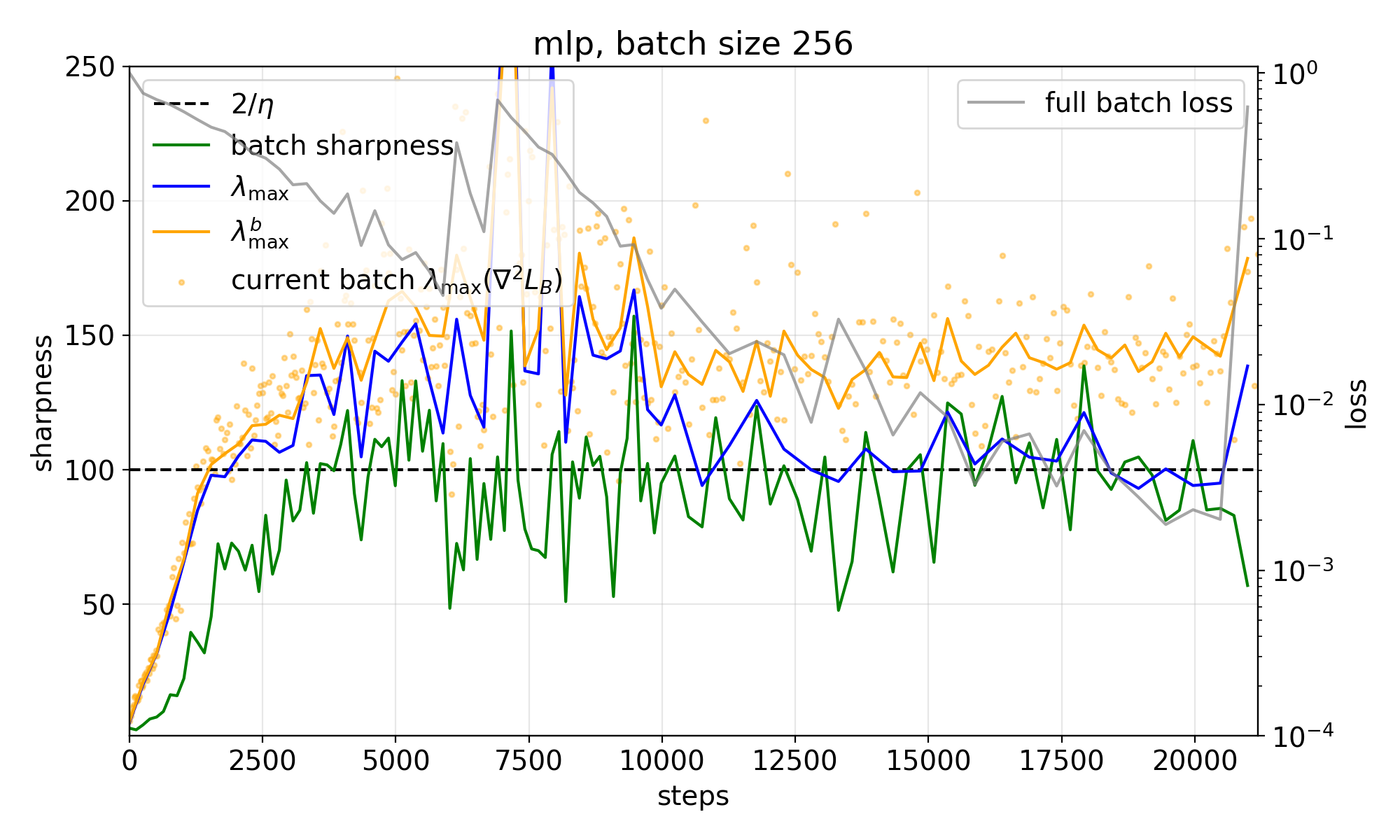} \\
    \end{tabular}
    \caption{\textbf{Tracking $\boldsymbol{\lambda^b_{\max}}$, MLP.}
    MLP with 2 hidden layers of width 512, step size $0.02$, trained on an $8$k subset of CIFAR-10.
    Comparison between the highest eigenvalue of the Hessian of the current mini-batch loss
    (orange dots, time-smoothed $\approx \lambda^b_{\max}$), \textit{Batch Sharpness} (green line), $\lambda_{\max}$ (blue line), and
    $\lambda^b_{\max}$ (orange line). Note that \textit{Batch Sharpness} stabilizes as $2/\eta$, while $\lambda^b_{\max}$ is above it, and $\lambda_{\max}$ is below for small batch sizes. Note that for batch size 2 we were using a lower step size of $0.01$, as otherwise the network wasn't converging.}
    \label{fig:appendix_blmax_mlp}
\end{figure}

\begin{figure}[t]
    \centering
    \begin{tabular}{cc}
        \includegraphics[width=0.45\textwidth]{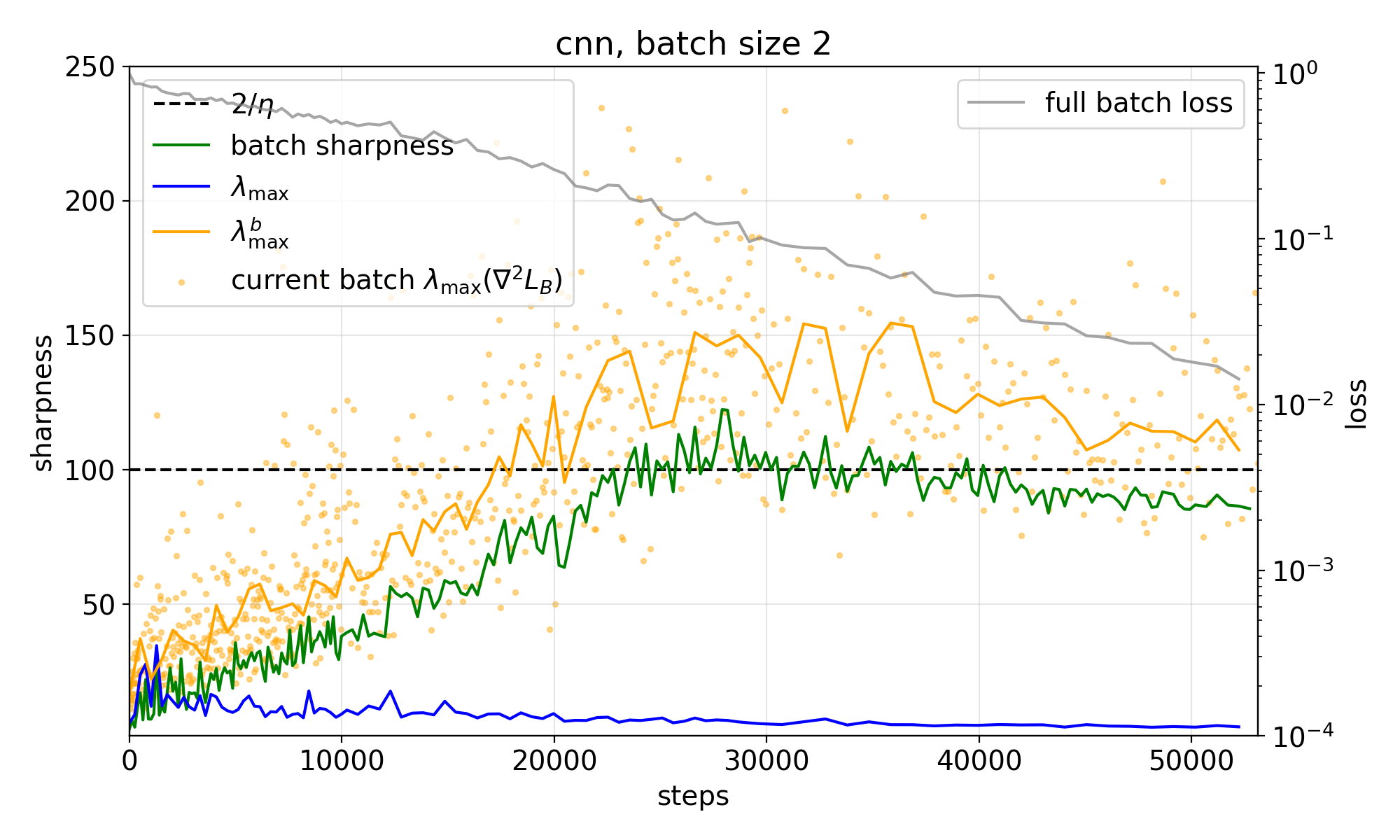} &
        \includegraphics[width=0.45\textwidth]{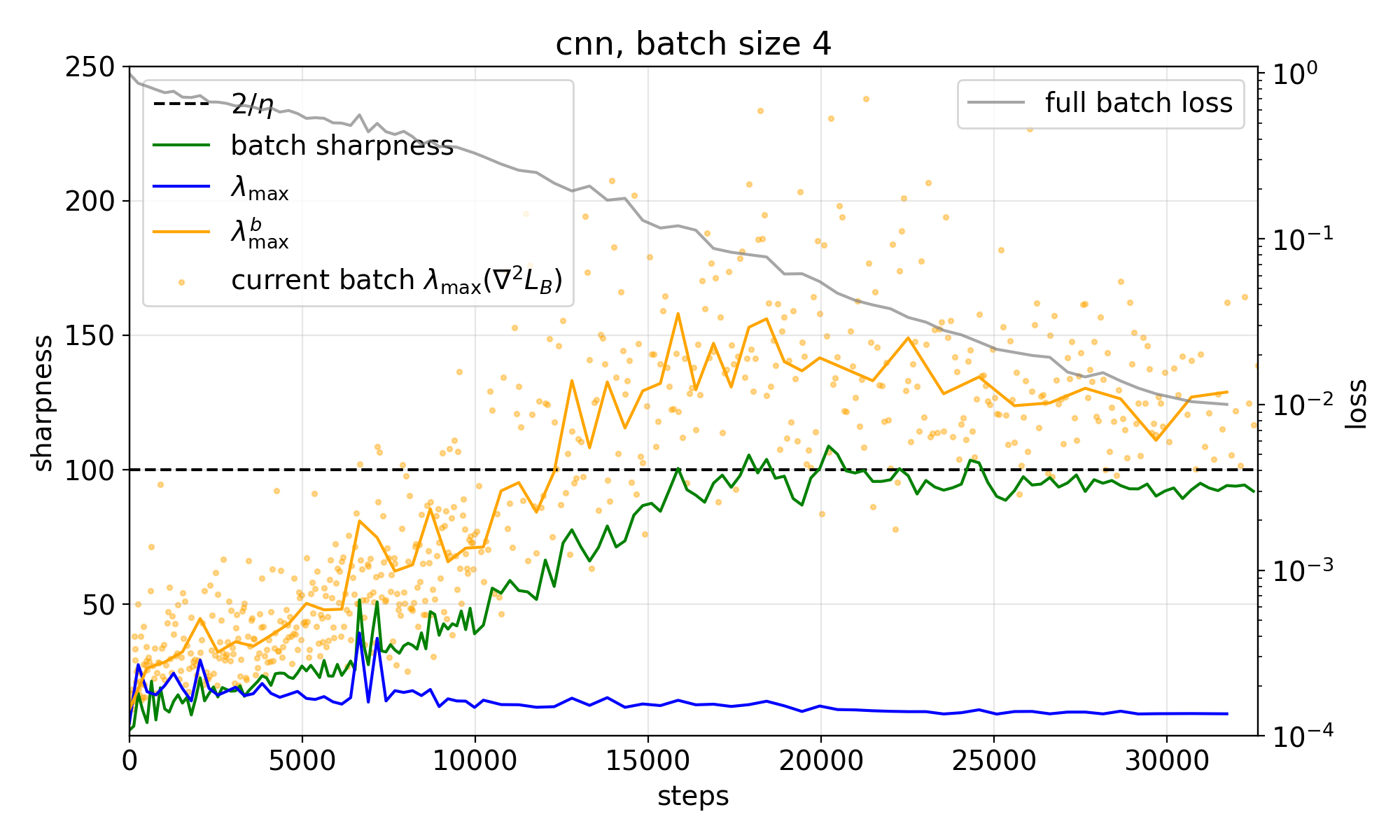} \\
        \includegraphics[width=0.45\textwidth]{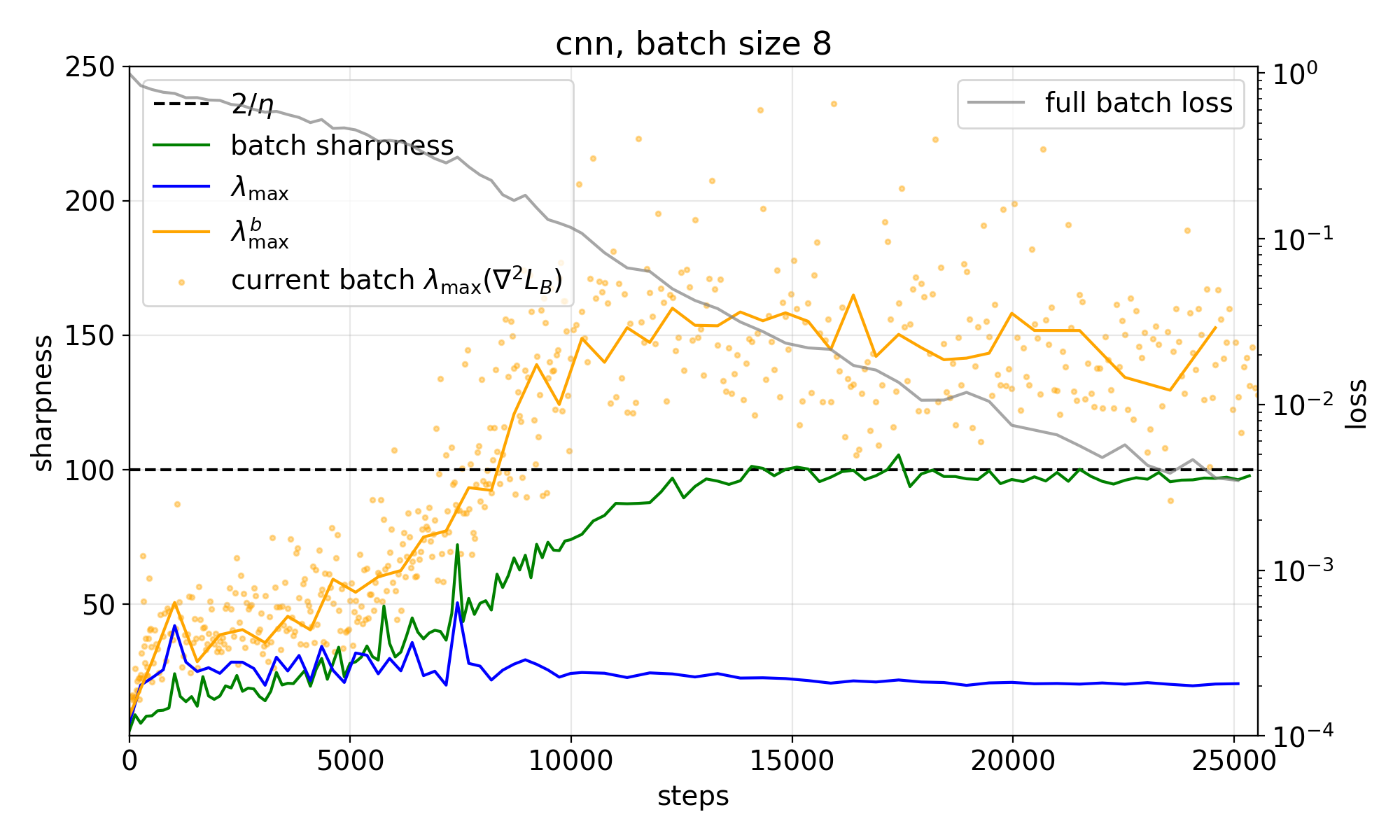} &
        \includegraphics[width=0.45\textwidth]{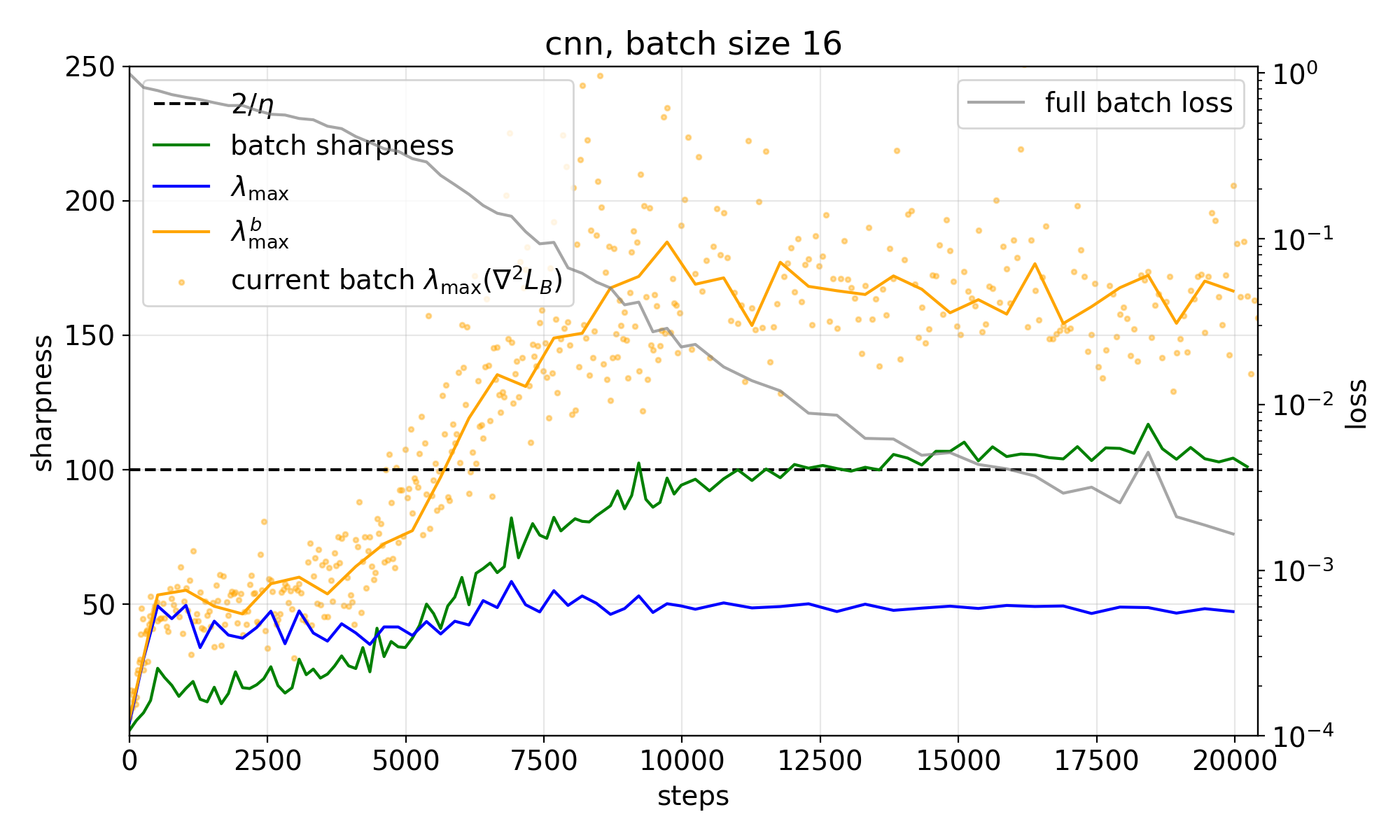} \\
        \includegraphics[width=0.45\textwidth]{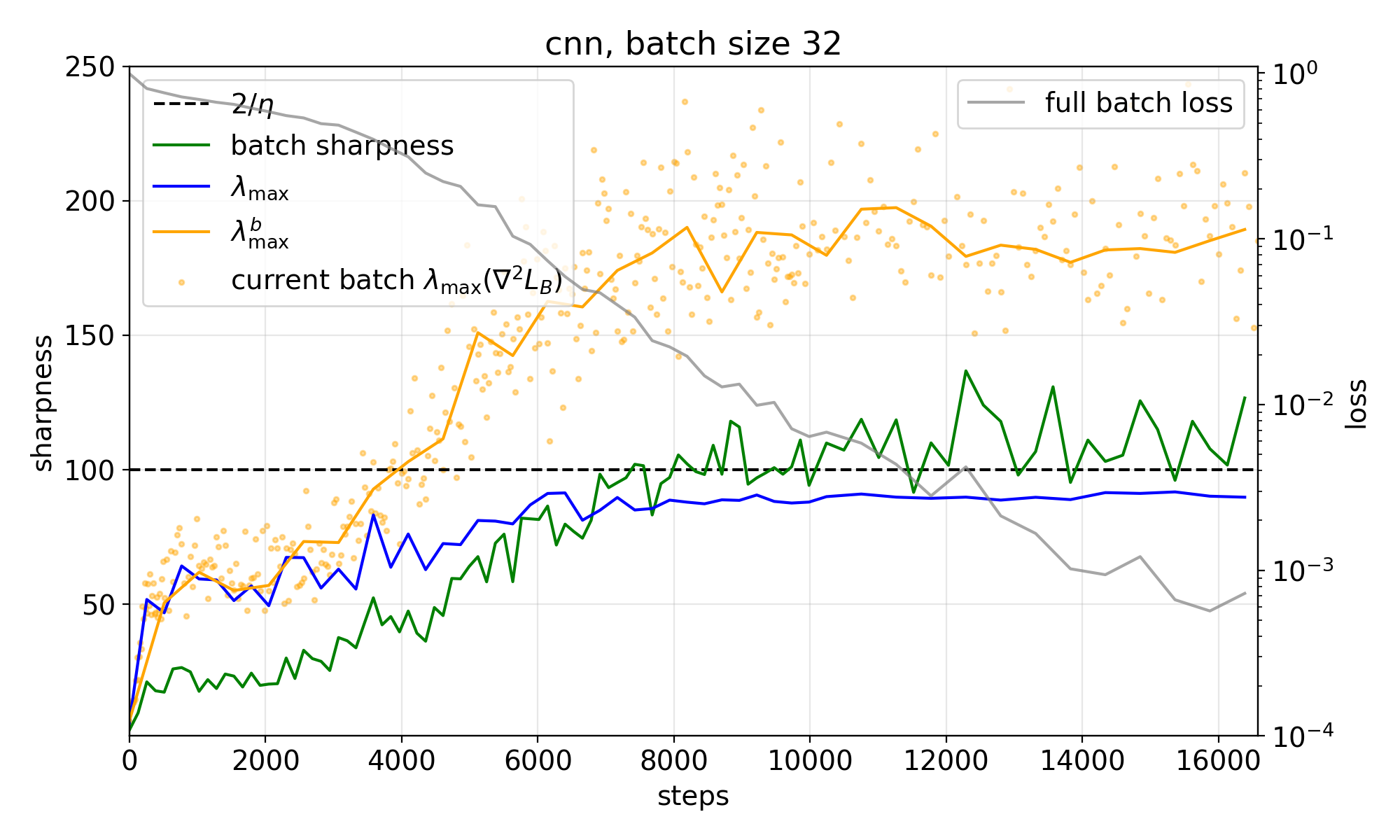} &
        \includegraphics[width=0.45\textwidth]{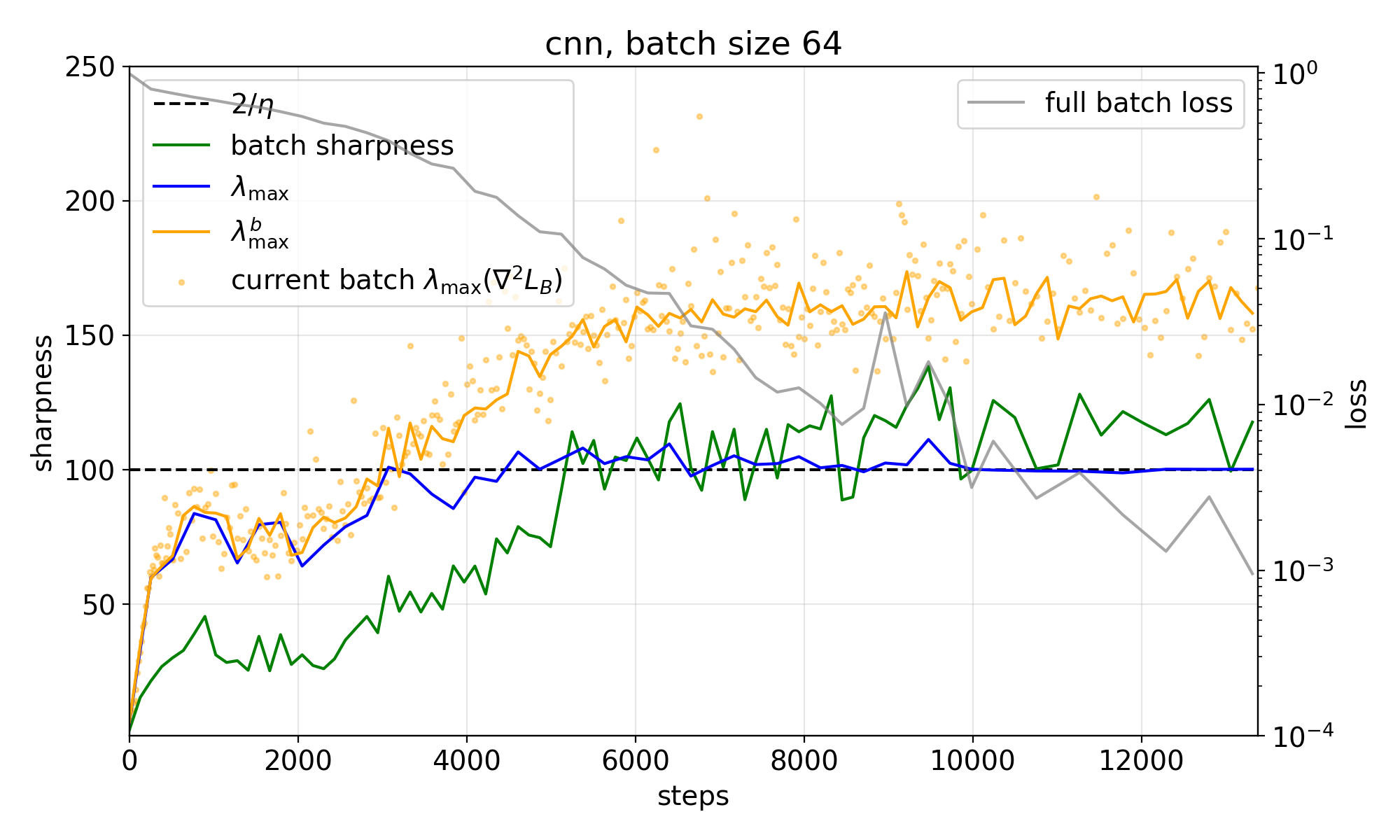} \\
        \includegraphics[width=0.45\textwidth]{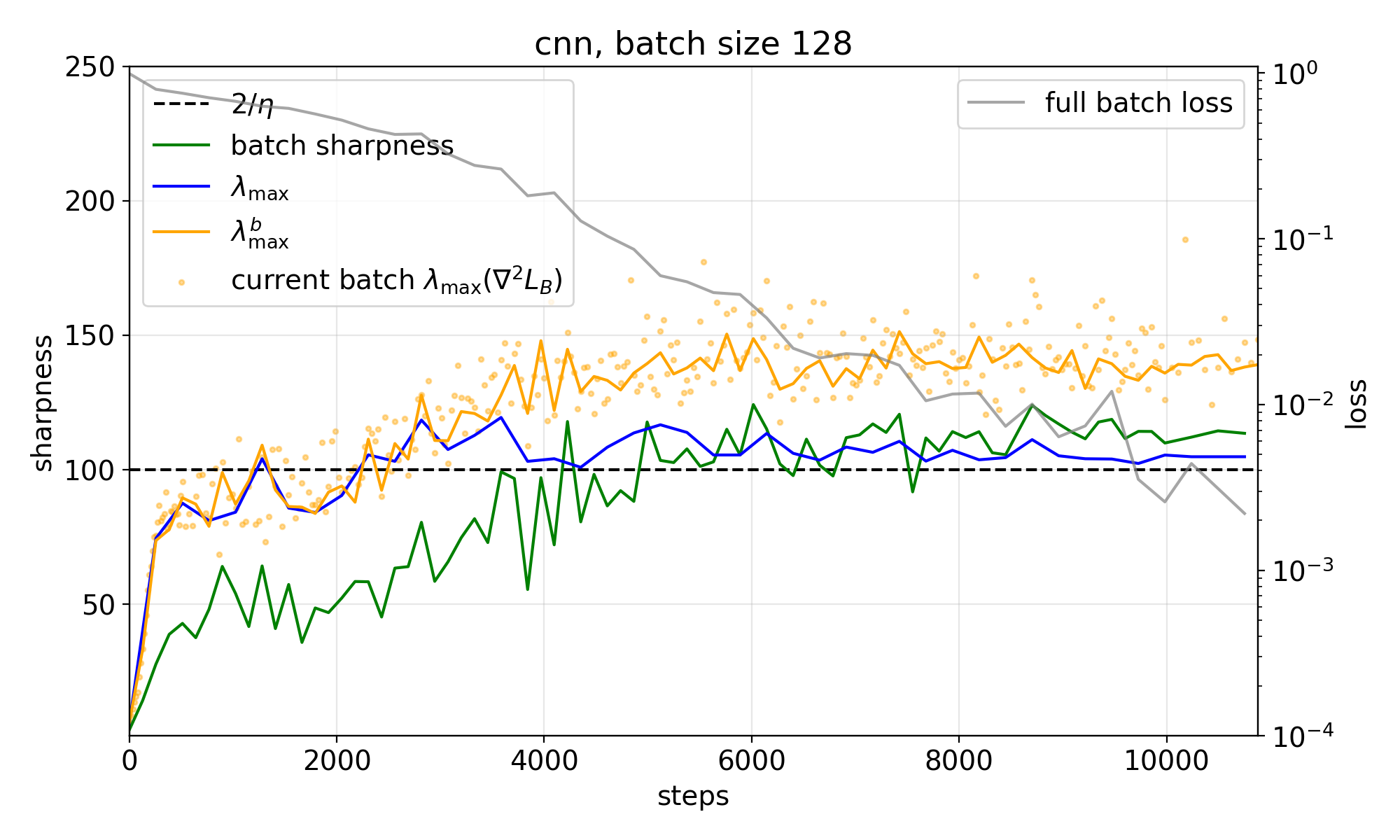} &
        \includegraphics[width=0.45\textwidth]{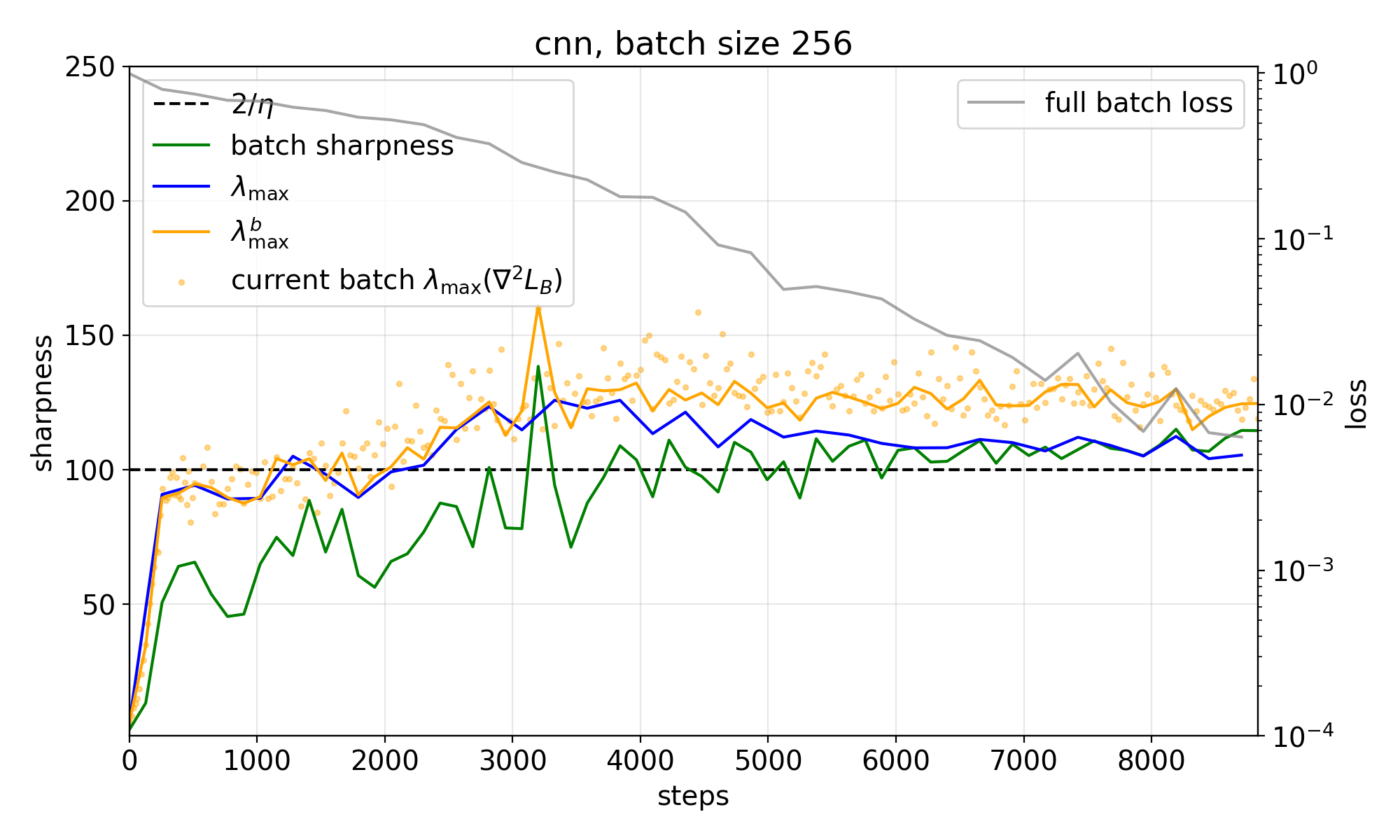} \\
    \end{tabular}
    \caption{\textbf{Tracking $\boldsymbol{\lambda^b_{\max}}$, CNN.}
    CNN with 5 layers (3 convolutional, 2 fully connected), step size $0.03$, trained on an $8$k
    subset of CIFAR-10. Comparison between the highest eigenvalue of the Hessian of the current mini-batch loss
    (orange dots, time-smoothed $\approx \lambda^b_{\max}$), \textit{Batch Sharpness} (green line), $\lambda_{\max}$ (blue line), and
    $\lambda^b_{\max}$ (orange line). Note that \textit{Batch Sharpness} stabilizes as $2/\eta$, while $\lambda^b_{\max}$ is above it, and $\lambda_{\max}$ is below for small batch sizes.}
    \label{fig:appendix_blmax_cnn}
    
\end{figure}

\begin{figure}[t]
    \centering
    \begin{tabular}{cc}
        \includegraphics[width=0.45\textwidth]{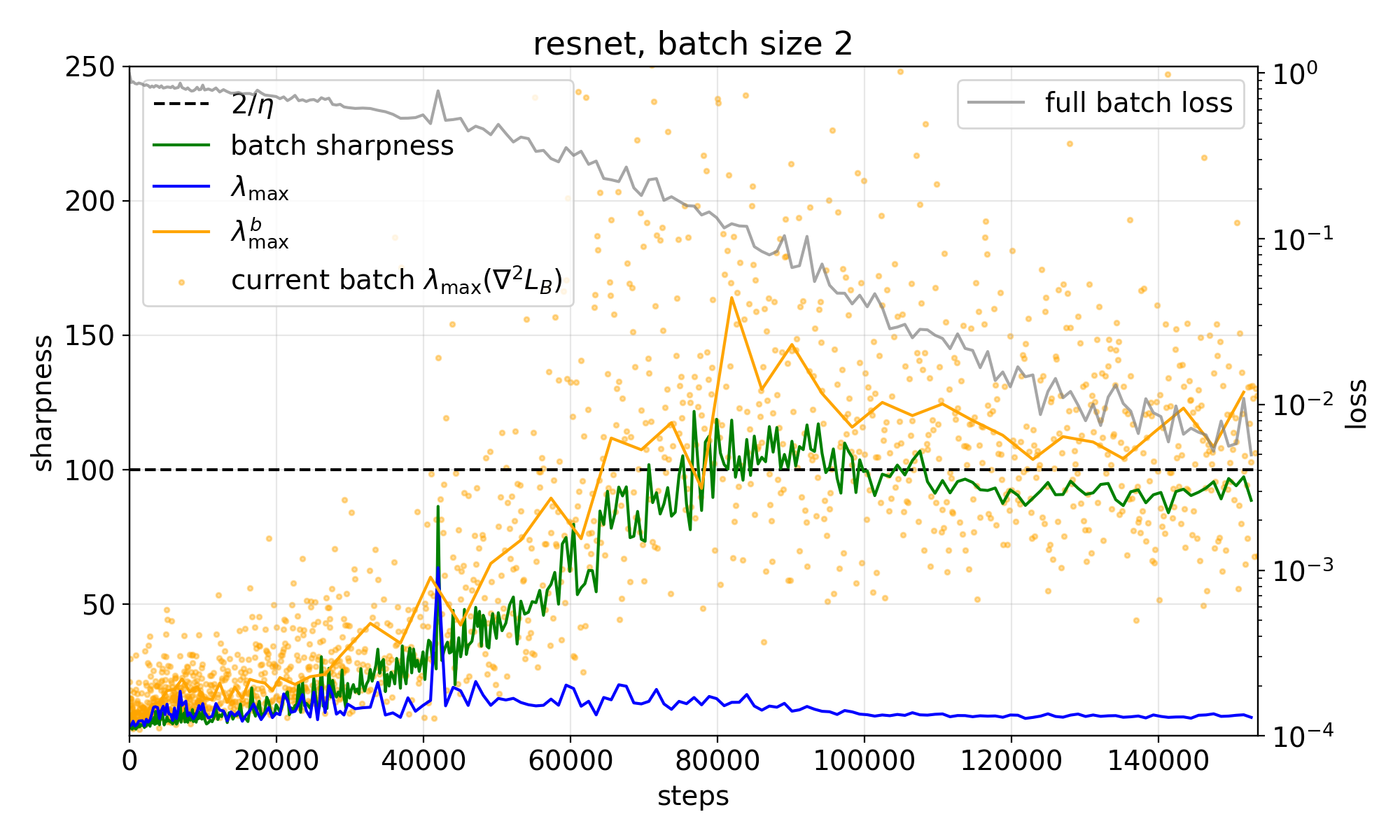} &
        \includegraphics[width=0.45\textwidth]{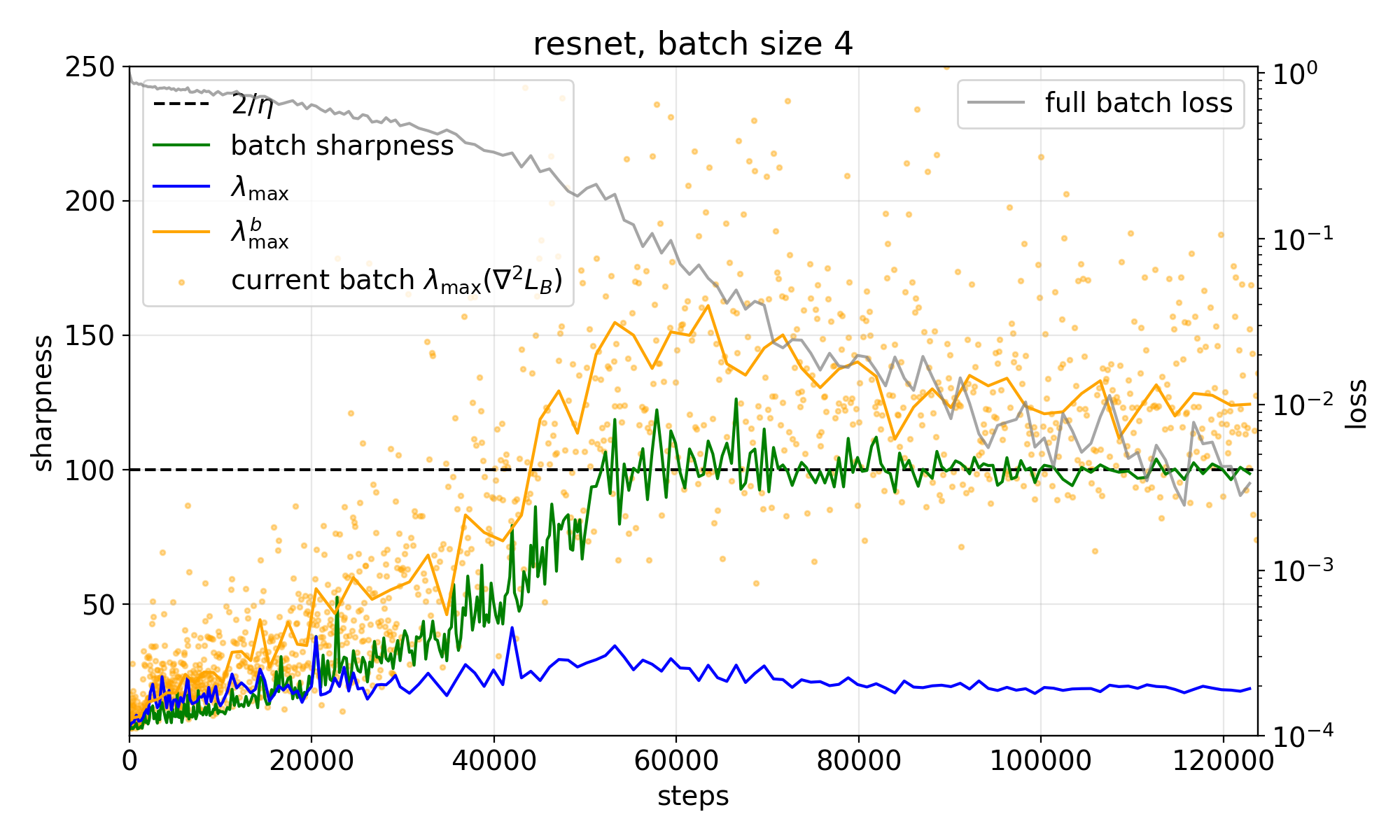} \\
        \includegraphics[width=0.45\textwidth]{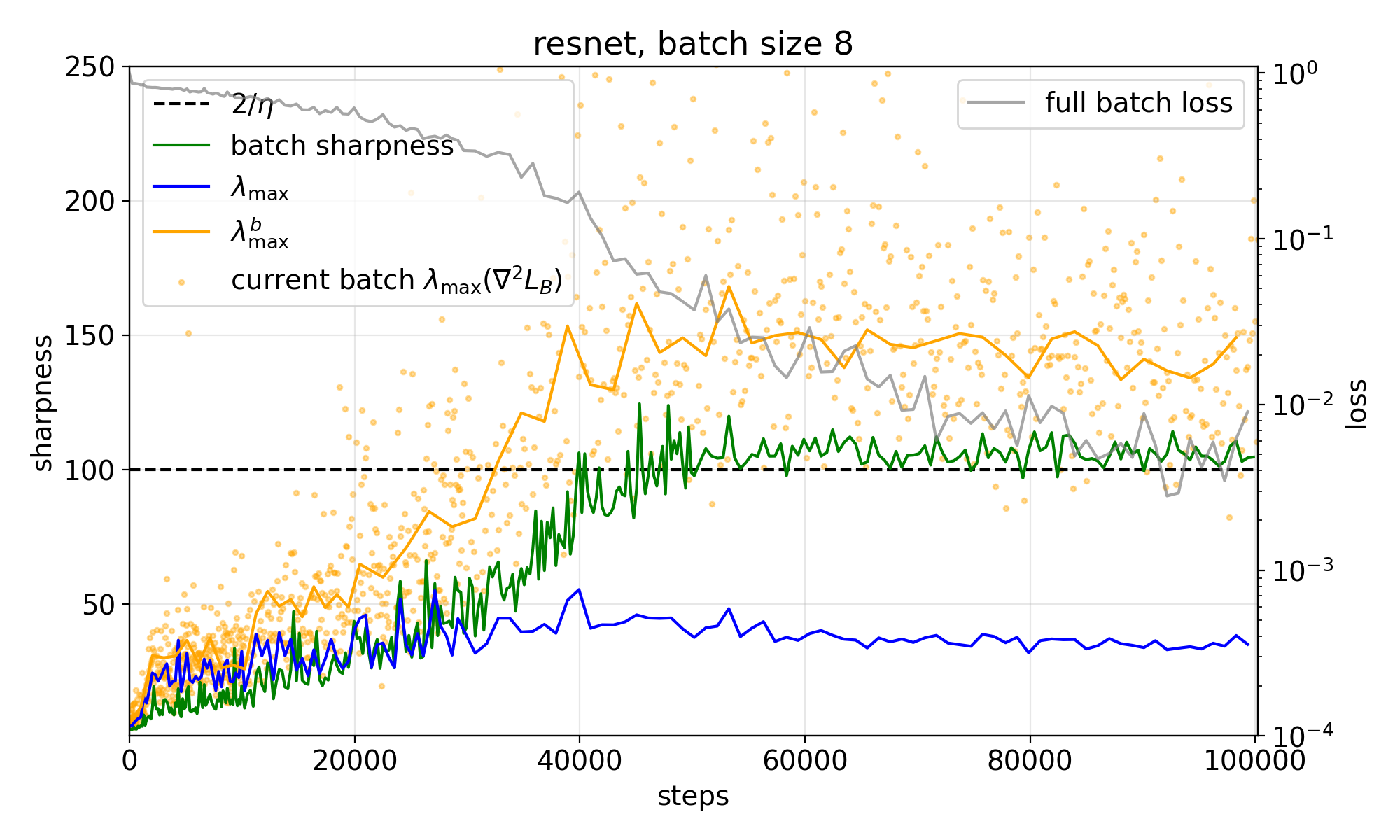} &
        \includegraphics[width=0.45\textwidth]{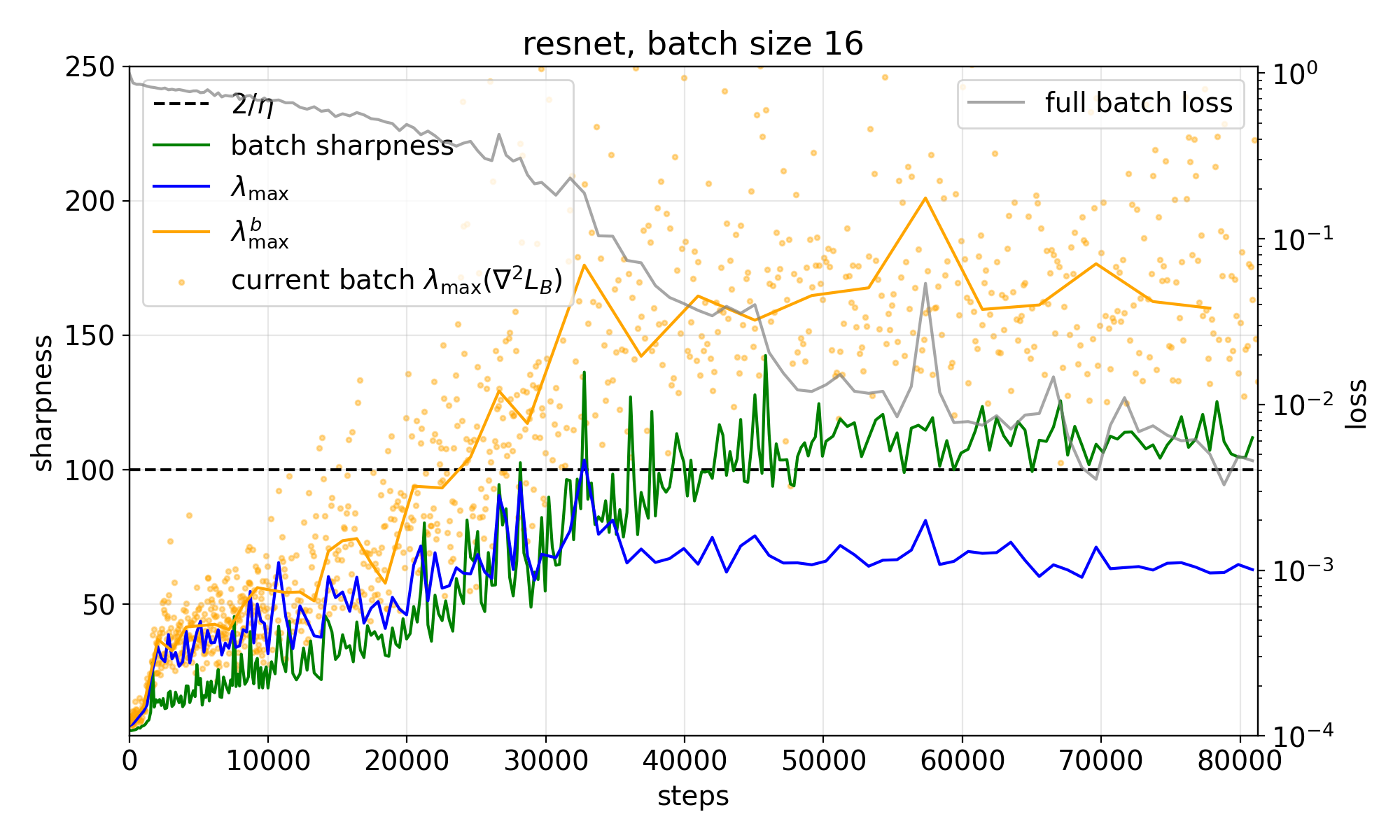} \\
        \includegraphics[width=0.45\textwidth]{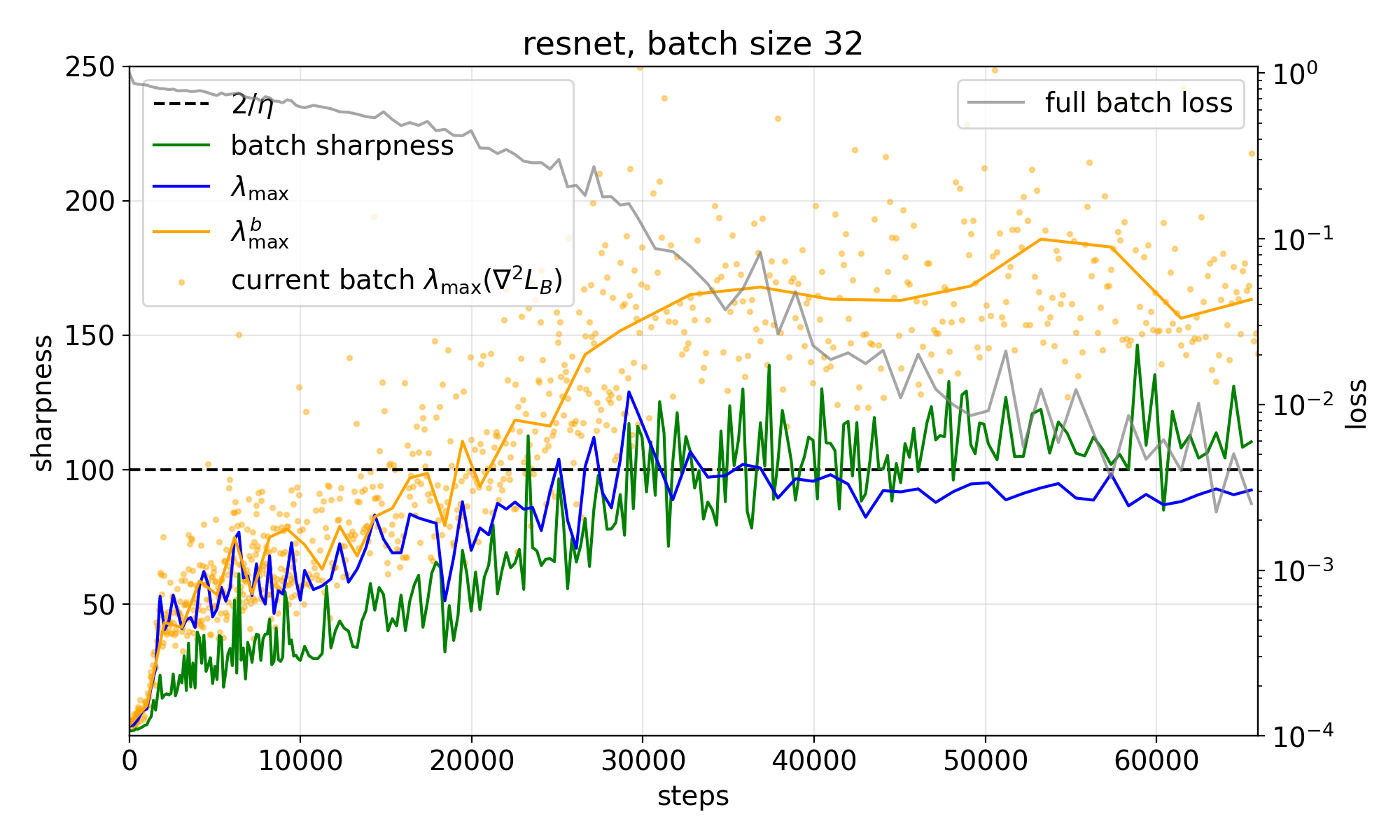} &
        \includegraphics[width=0.45\textwidth]{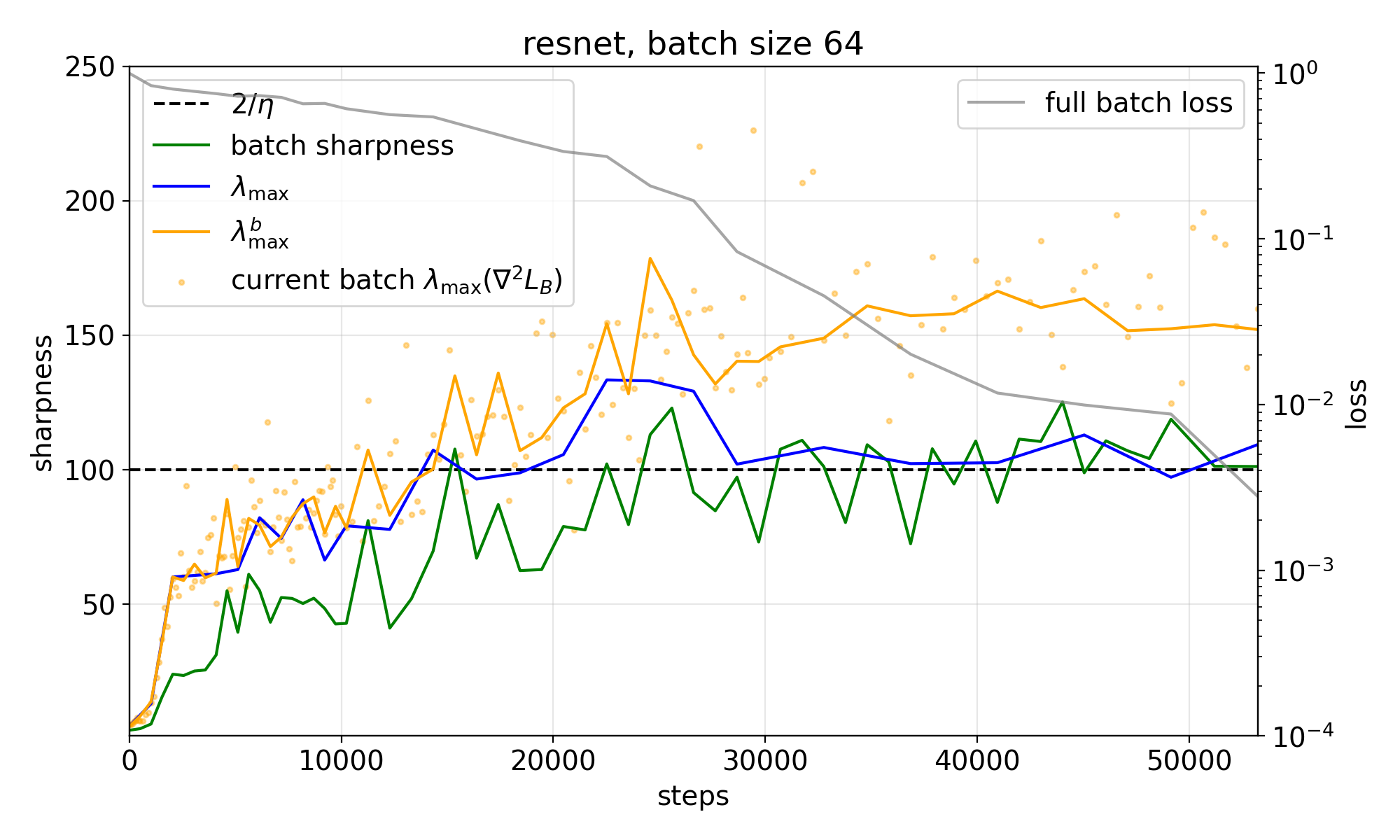} \\
        \includegraphics[width=0.45\textwidth]{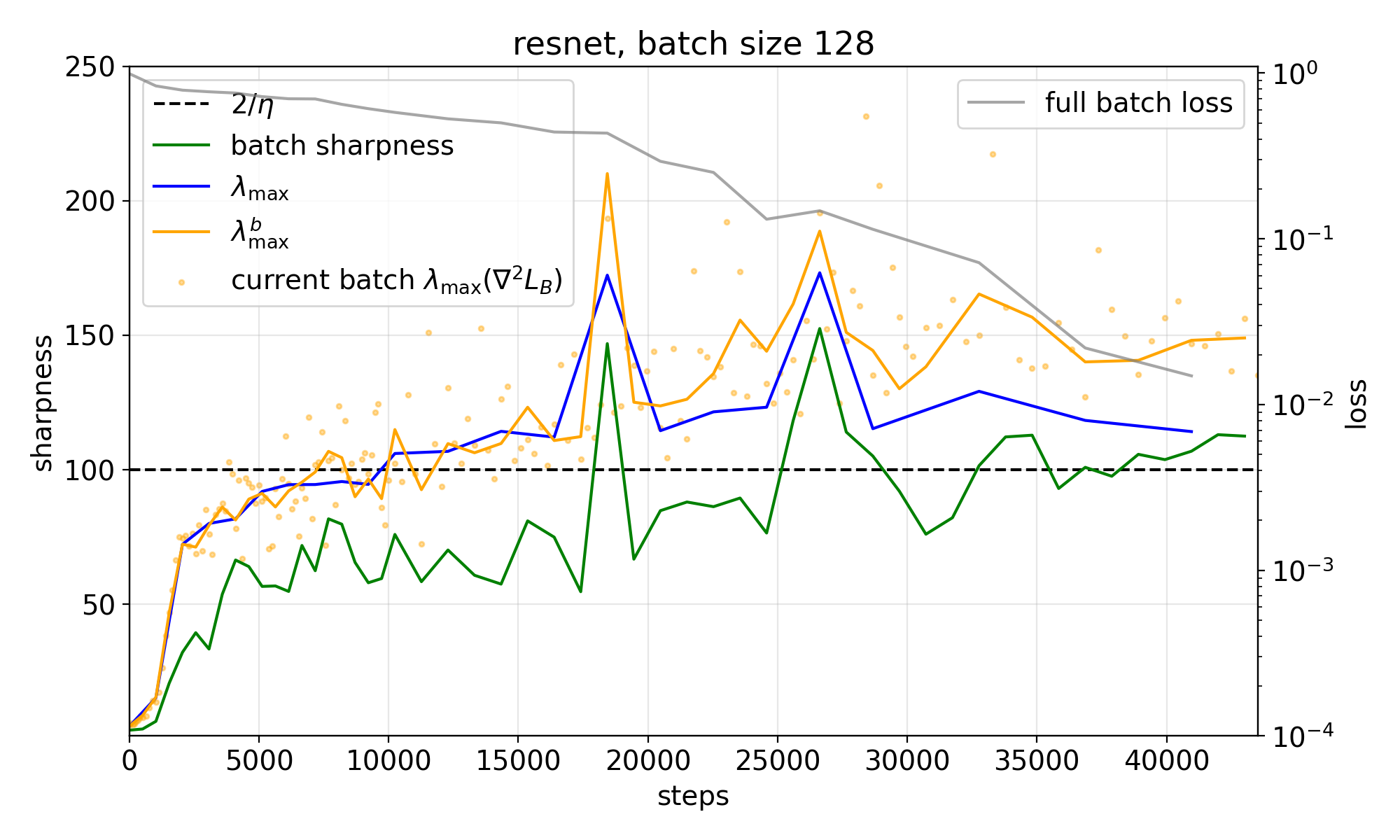} &
        \includegraphics[width=0.45\textwidth]{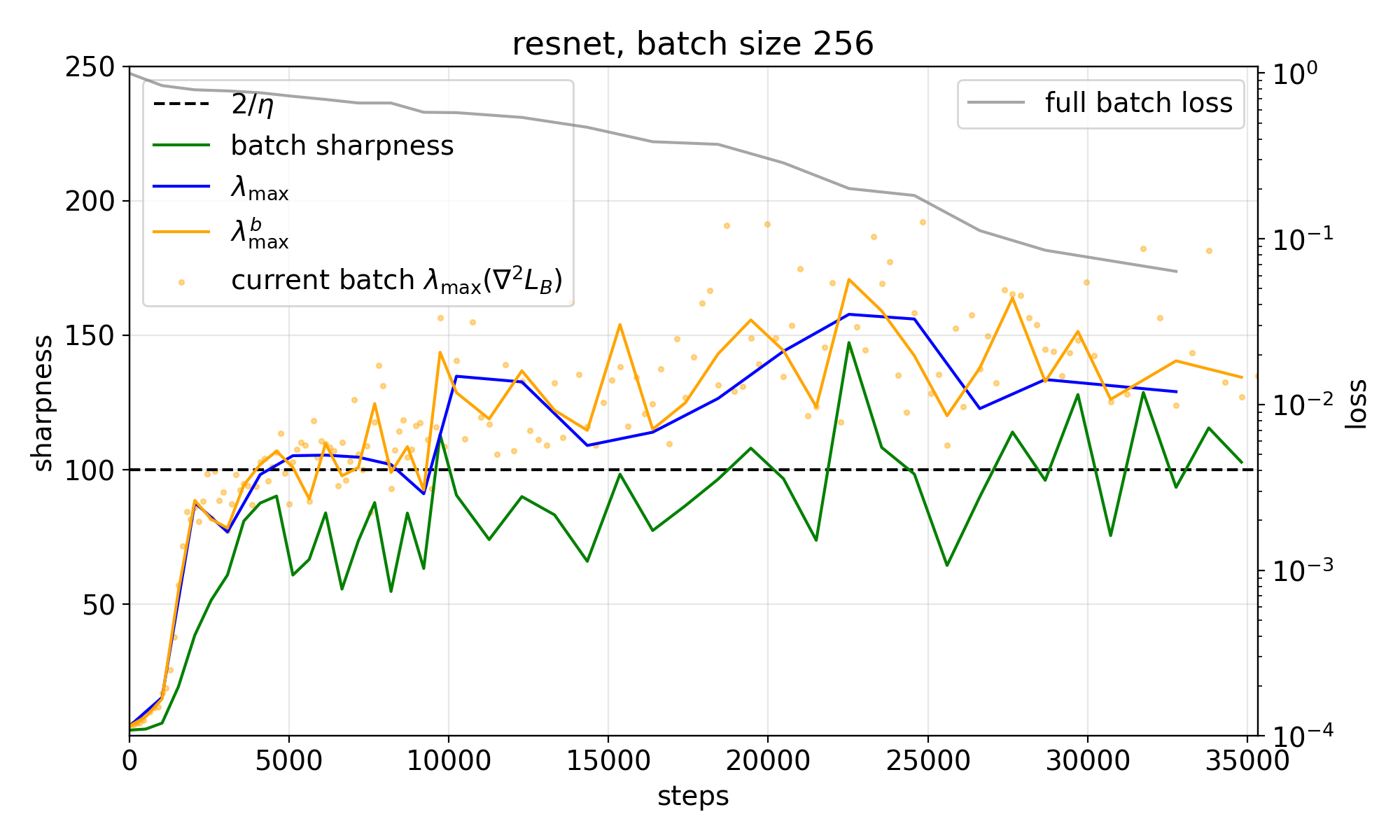} \\
    \end{tabular}
    \caption{\textbf{Tracking $\boldsymbol{\lambda^b_{\max}}$, ResNet-20.}
    ResNet-20 on CIFAR-10 with step size $0.02$, trained on an $8$k subset of the dataset.
    Comparison between the highest eigenvalue of the Hessian of the current mini-batch loss
    (orange dots, time-smoothed $\approx \lambda^b_{\max}$), \textit{Batch Sharpness} (green line), $\lambda_{\max}$ (blue line), and
    $\lambda^b_{\max}$ (orange line). Note that \textit{Batch Sharpness} stabilizes as $2/\eta$, while $\lambda^b_{\max}$ is above it, and $\lambda_{\max}$ is below for small batch sizes.}
    \label{fig:appendix_blmax_resnet}
\end{figure}

\clearpage

\section{Experiments: Varying Architecture}
\label{appendix:further_bs}

In this appendix, we provide further empirical evidence that \textsc{EoSS} arises robustly across a variety of architectures, step sizes, and batch sizes. For each experiment, we plot three quantities: $\lmax$, \textit{Batch Sharpness}, and \textit{step sharpness} as a point cloud, which constitutes \textit{Batch Sharpness} without expectation, and measured only on the \textit{current} batch. Notice that time-averaging \textit{step sharpness} is approximately the same as taking expectation over batches (albeit with slowly changing parameters), so it is approximately equal to \textit{Batch Sharpness}, which takes this expectation at a point. Consistent with our main observations, we find that \textit{Batch Sharpness} invariably stabilizes around \(2/\eta\). Also refer to \cref{appendix:further_lambda_b_max} for additional experiments illustrating \textsc{EoSS}

\paragraph{MLP (2-Layer) Baseline.}
Figure~\ref{fig:eoss_bs_mlp} illustrates \textsc{EoSS} for our baseline network, an \textsc{MLP} with two hidden layers of dimension 512, trained on an $8192$-sample subset of \textsc{CIFAR}-10 with step size \(\eta = 0.01\).  As the training proceeds, \textit{Batch Sharpness} stabilizes around \(2/\eta\), whereas $\lambda_{\max}$ plateaus strictly below \textit{Batch Sharpness}.  

\paragraph{5-Layer CNN.}
We further confirm the \textsc{EoSS} regime in a five-layer \textsc{CNN}.  As depicted in Figures~\ref{fig:eoss_bs_cnn}, \textit{Batch Sharpness} continues to plateau near the instability threshold for two distinct step sizes, while $\lambda_{\max}$ once again settles at a lower level. Notably, as we vary the batch size, the gap between \textit{Batch Sharpness} and $\lambda_{\max}$ increases for smaller batches, mirroring the patterns described in Section~\ref{section:lambda_max}. 

\paragraph{ResNet-14.}
Finally, we demonstrate that the \textsc{EoSS} regime also emerges for a deeper, residual architectures. In our case we are using \textsc{ResNet}-14 without BatchNor. Figure~\ref{fig:eoss_bs_resnet} highlights the same qualitative behavior, with \textit{Batch Sharpness} stabilizing at $2/\eta$.

Overall, these experiments provide further confirmation that \textsc{EoSS} is a robust phenomenon across different architectures, step sizes, and batch sizes. 

\paragraph{CNN with Full CIFAR-10.}
We also demonstrate in Figure \ref{fig:full-cifar} the emergence of \textsc{EoSS} when training on the full CIFAR-10 dataset. Consistent with the rest of the experiments, \textit{Batch Sharpness} consistently stabilizes at $2/\eta$. Notably, in these experiments we also include a plot of the accuracy on the training set, to illustrate that \textsc{EoSS} happens away from the manifold of minima, and thus cannot be attributed solely to the structure around the manifold of minima.


\begin{figure}[htbp]
  \centering

  \begin{subfigure}[t]{0.48\textwidth}
    \centering
    \includegraphics[width=\linewidth]{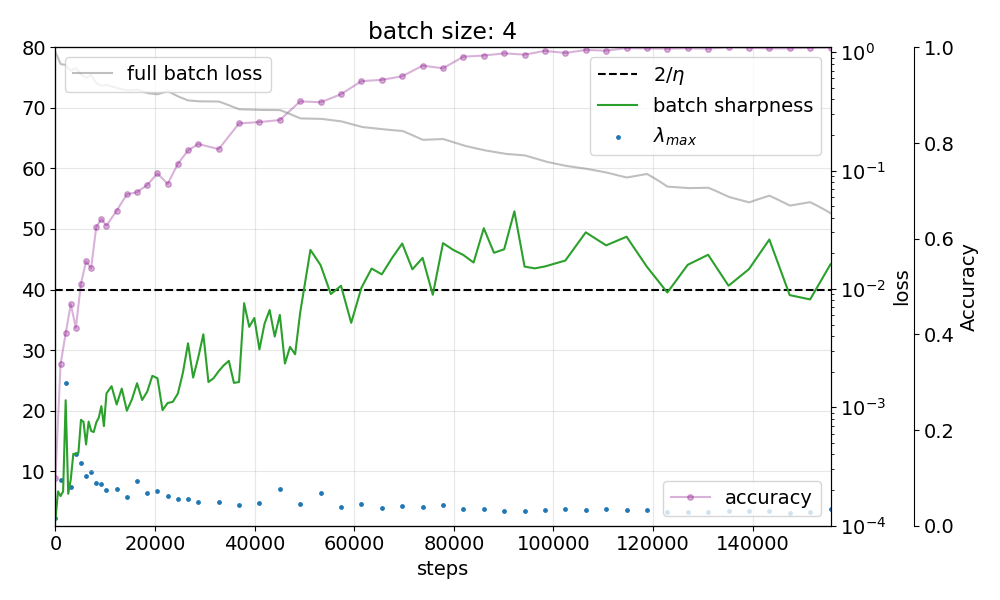}
    \subcaption{batch size 4}
    \label{fig:cifar-4}
  \end{subfigure}\hfill
  \begin{subfigure}[t]{0.48\textwidth}
    \centering
    \includegraphics[width=\linewidth]{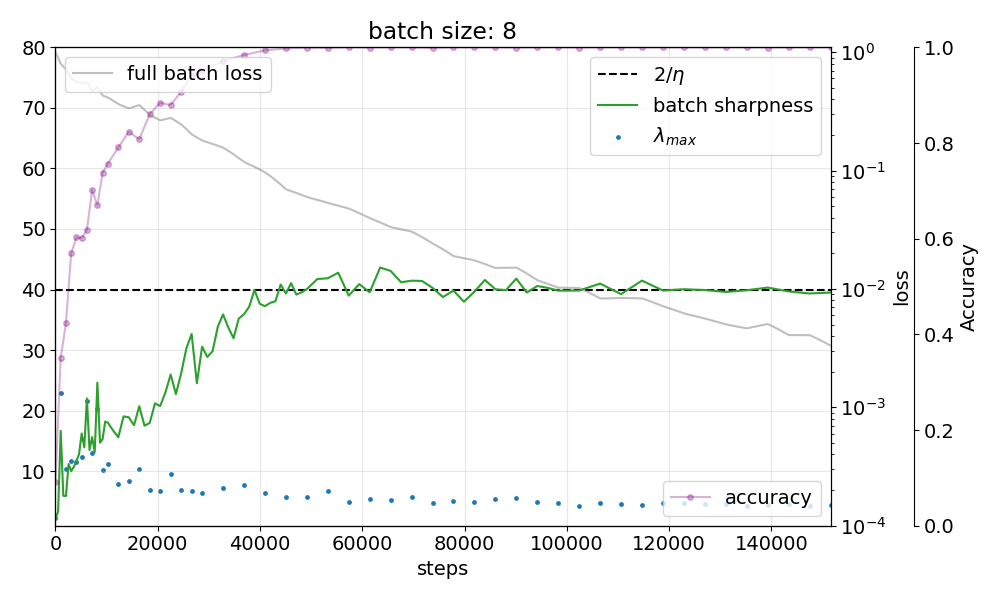}
    \subcaption{batch size 8}
    \label{fig:cifar-8}
  \end{subfigure}

  \smallskip

  \begin{subfigure}[t]{0.48\textwidth}
    \centering
    \includegraphics[width=\linewidth]{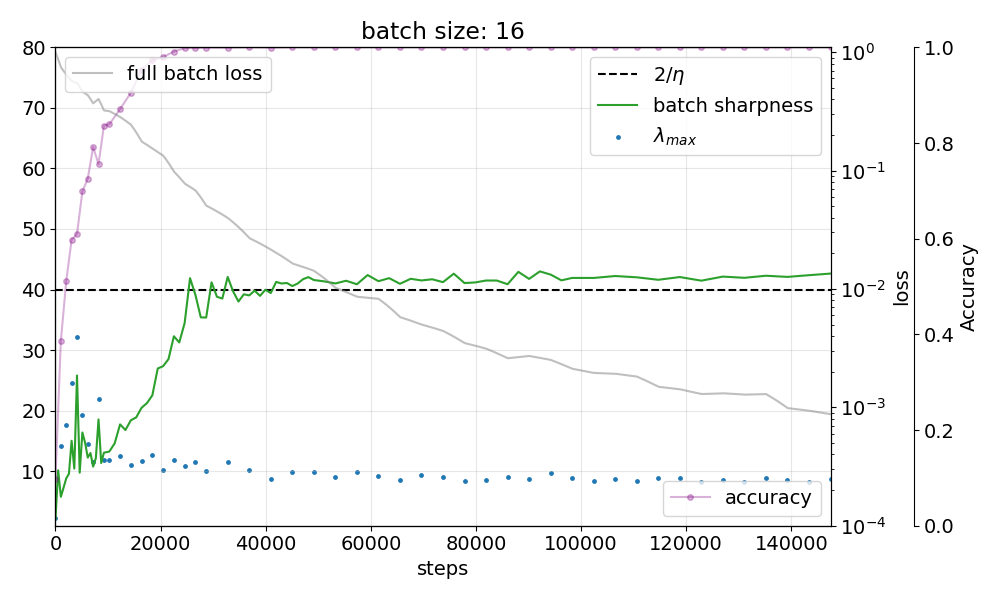}
    \subcaption{batch size 16}
    \label{fig:cifar-16}
  \end{subfigure}\hfill
  \begin{subfigure}[t]{0.48\textwidth}
    \centering
    \includegraphics[width=\linewidth]{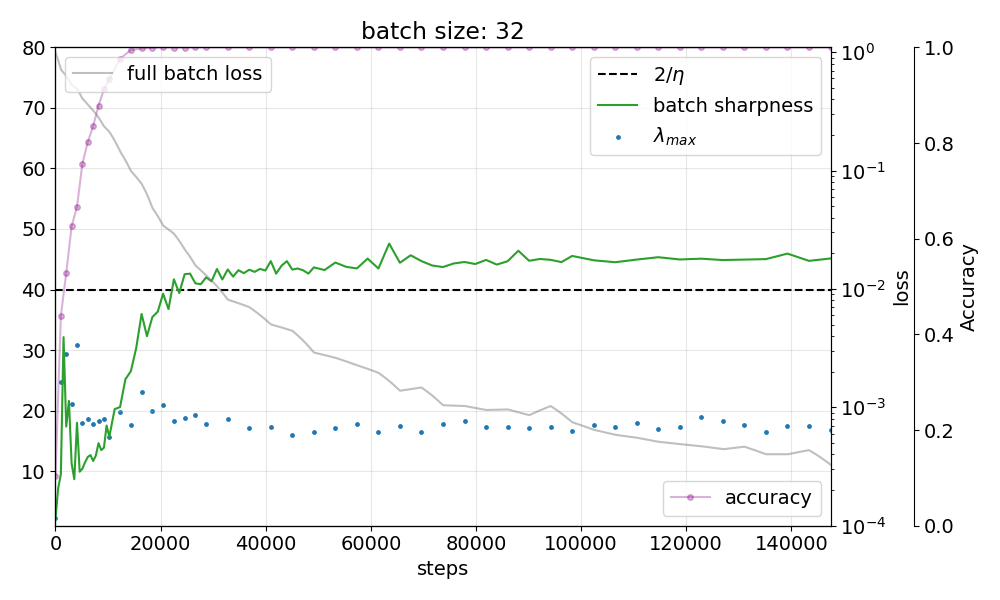}
    \subcaption{batch size 32}
    \label{fig:cifar-32}
  \end{subfigure}

  \caption{\textbf{CNN on Full CIFAR-10:} Same architecture as in Figure \ref{fig:eoss_bs_cnn}, but trained at a larger $\eta=0.05$ on the full CIFAR-10 dataset, illustrating the emergence of \textsc{EoSS} away from manifold of minima.}
  \label{fig:full-cifar}
\end{figure}

\begin{figure}[ht!]
    \centering
    \begin{tabular}{cc}
      \includegraphics[width=0.45\textwidth]{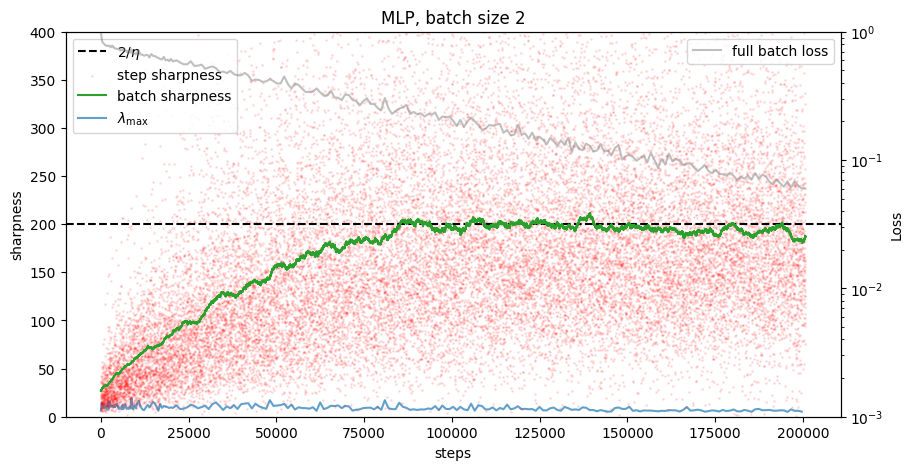} & \includegraphics[width=0.45\textwidth]{img/main_new/cifar/mlp/4.png} \\
      \includegraphics[width=0.45\textwidth]{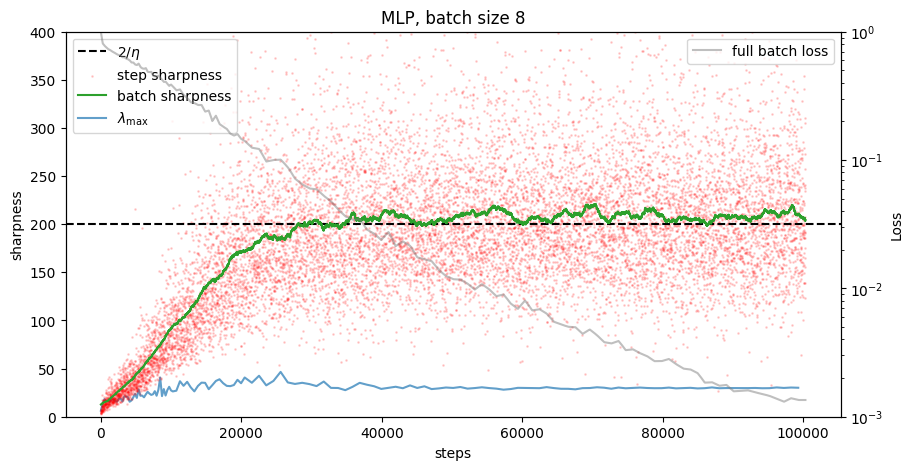} & \includegraphics[width=0.45\textwidth]{img/main_new/cifar/mlp/16.png} \\
      \includegraphics[width=0.45\textwidth]{img/main_new/cifar/mlp/32.png} & \includegraphics[width=0.45\textwidth]{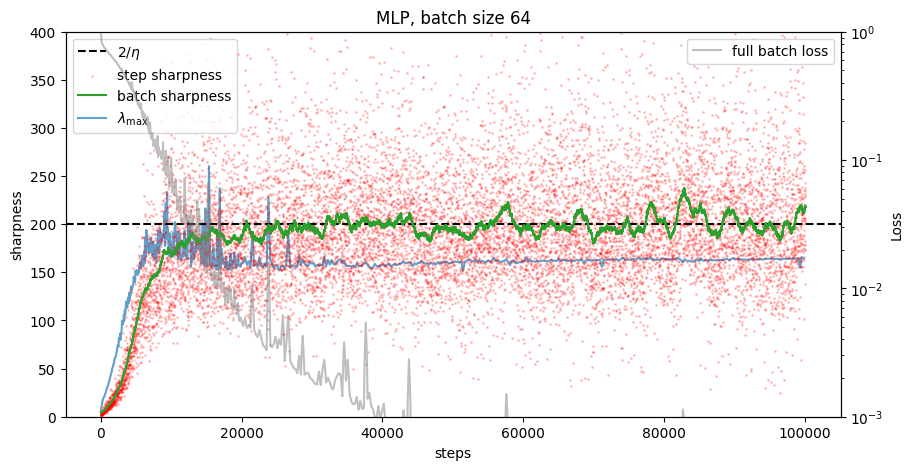} \\
      \includegraphics[width=0.45\textwidth]{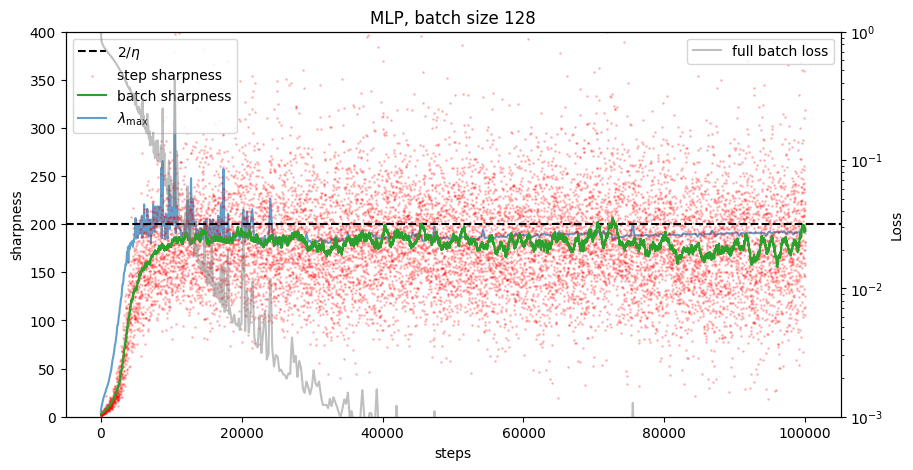} & \includegraphics[width=0.45\textwidth]{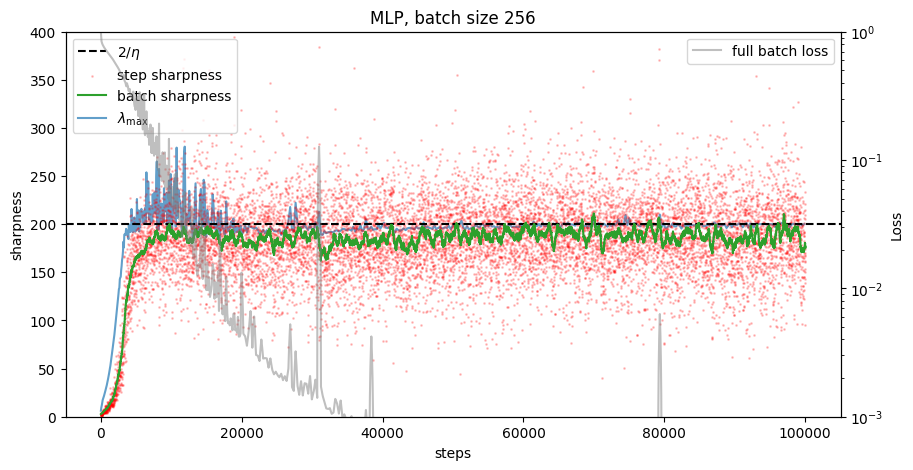} \\
    \end{tabular}
    \caption{\textbf{MLP}: 2 hidden layers, hidden dimension 512; \textbf{step size 0.01}, 8k subset of CIFAR-10. Comparison between: step sharpness, aka \textit{Batch Sharpness} without expectation over batches and measured on the current batch (red dots, time-smoothing would be $\approx$ \textit{Batch Sharpness}), the empirical \textit{Batch Sharpness} (green line), the $\lambda_{\max}$ (blue line).}
    \label{fig:eoss_bs_mlp}
\end{figure}

\newpage

\begin{figure}[ht!]
    \centering
    \begin{tabular}{cc}
      \includegraphics[width=0.45\textwidth]{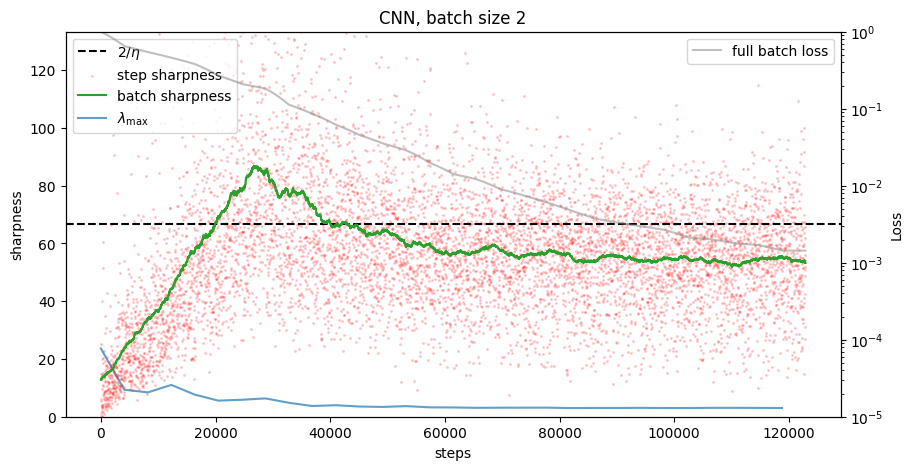} & \includegraphics[width=0.45\textwidth]{img/main_new/cifar/cnn/4.png} \\
      \includegraphics[width=0.45\textwidth]{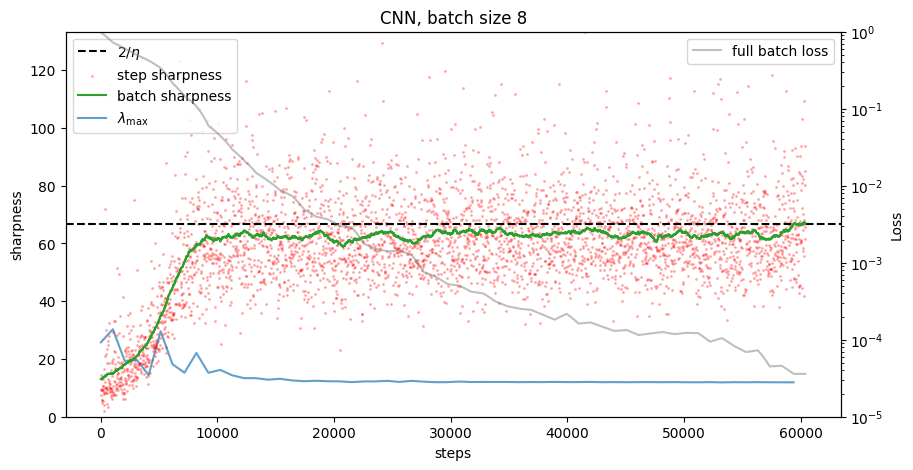} & \includegraphics[width=0.45\textwidth]{img/main_new/cifar/cnn/16.png} \\
      \includegraphics[width=0.45\textwidth]{img/main_new/cifar/cnn/32.png} & \includegraphics[width=0.45\textwidth]{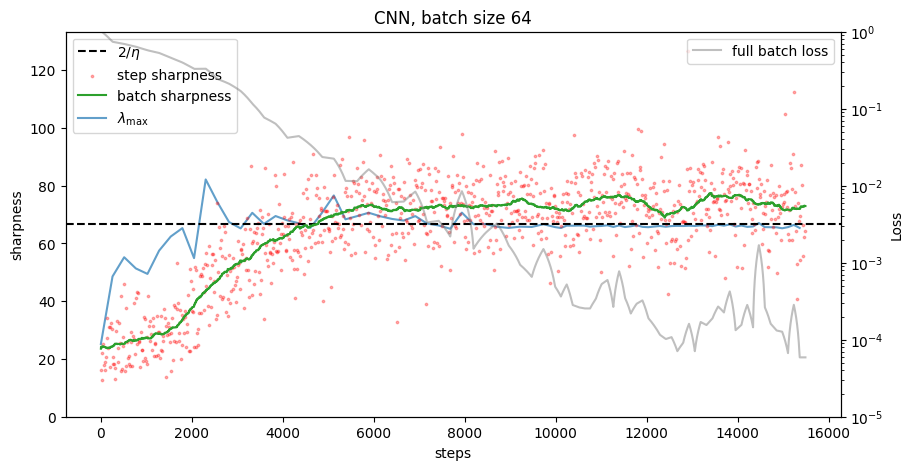} \\
      \includegraphics[width=0.45\textwidth]{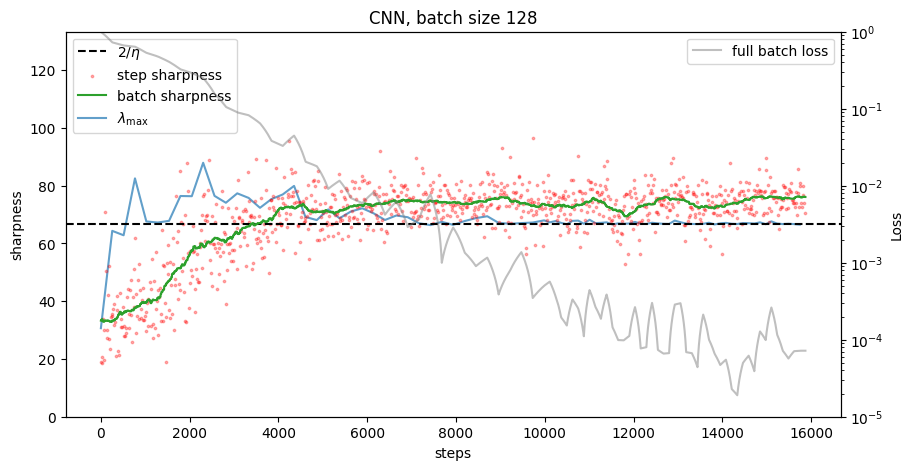} & \includegraphics[width=0.45\textwidth]{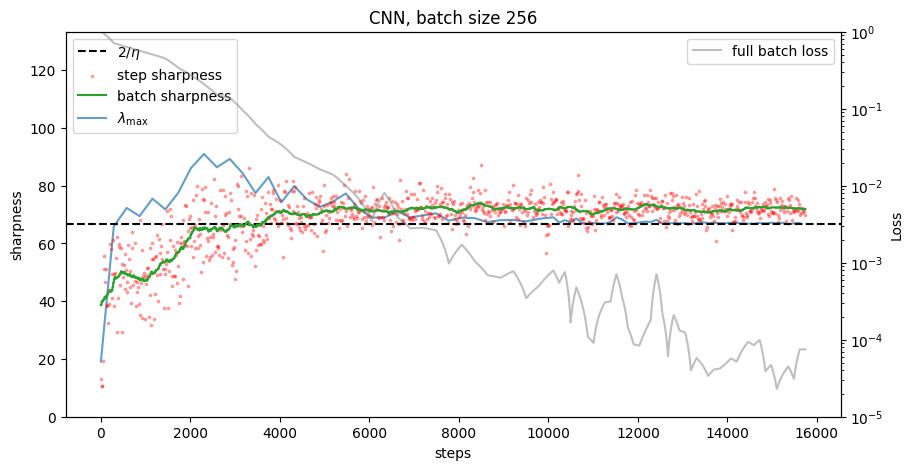} \\
    \end{tabular}
    \caption{\textbf{CNN}: 5 layers (3 convolutional, 2 fully-connected), \textbf{step size 0.03}, 8k subset of CIFAR-10. Comparison between: step sharpness, aka \textit{Batch Sharpness} without expectation over batches and measured on the current batch (red dots, time-smoothing would be $\approx$ \textit{Batch Sharpness}), the empirical \textit{Batch Sharpness} (green line), the $\lambda_{\max}$ (blue line).}
    \label{fig:eoss_bs_cnn}
\end{figure}

\newpage

\begin{figure}[ht!]
    \centering
    \begin{tabular}{cc}
      \includegraphics[width=0.45\textwidth]{img/main_new/cifar/resnet/4.png} &
      \includegraphics[width=0.45\textwidth]{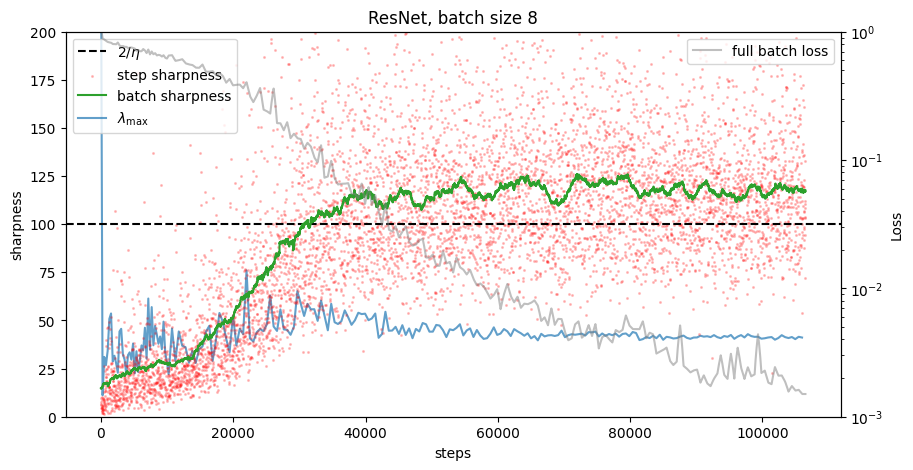} \\ \includegraphics[width=0.45\textwidth]{img/main_new/cifar/resnet/16.png} &
      \includegraphics[width=0.45\textwidth]{img/main_new/cifar/resnet/32.png} \\ \includegraphics[width=0.45\textwidth]{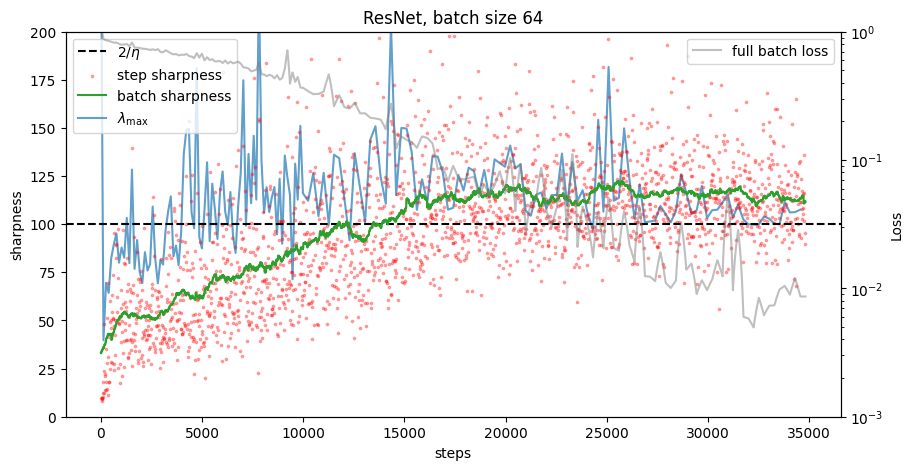} &
      \includegraphics[width=0.45\textwidth]{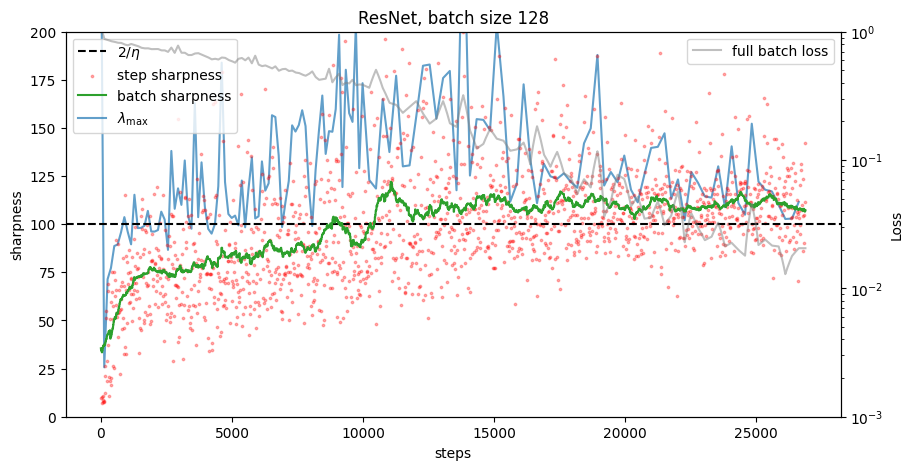} \\ \includegraphics[width=0.45\textwidth]{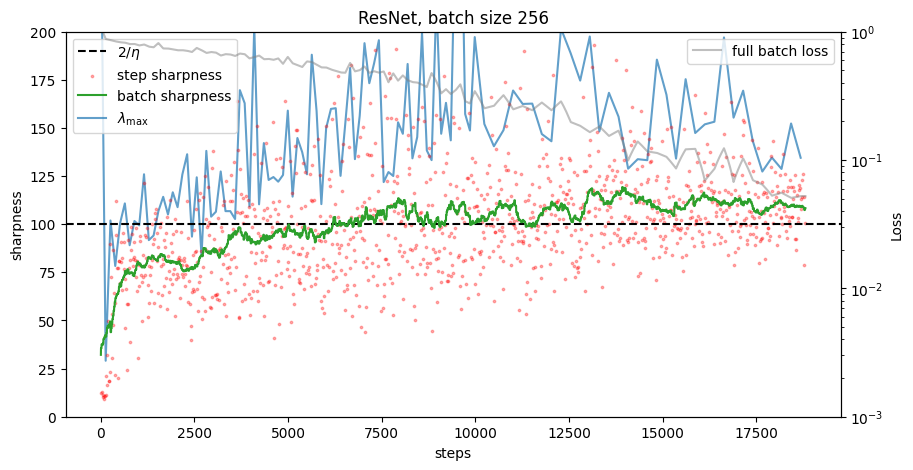}
    \end{tabular}
    \caption{\textbf{ResNet-10}, step size 0.005, 8k subset of CIFAR-10. Comparison between: step sharpness, aka \textit{Batch Sharpness} without expectation over batches and measured on the current batch (red dots, time-smoothing would be $\approx$ \textit{Batch Sharpness}), the empirical \textit{Batch Sharpness} (green line), the $\lambda_{\max}$ (blue line).}
    \label{fig:eoss_bs_resnet}
\end{figure}

\clearpage

\section{Experiments: Varying Datasets}
\label{appendix:further_bs_svhn}

This appendix complements Appendix~\ref{appendix:further_bs} by verifying that the \textsc{EoSS} phenomena are not specific to CIFAR-10 but persist under a change of dataset. We repeat the experiments of Appendix~\ref{appendix:further_bs}—sweeping architectures (MLP, CNN, ResNet) and batch sizes—on an 8k subset of the SVHN dataset, and track step sharpness, Batch Sharpness, and the full-batch $\lambda_{\max}$ along the training trajectory. Across all settings we again observe progressive sharpening followed by stabilization of \textit{Batch Sharpness} at $2/\eta$, catapult-like spikes, and suppression of $\lambda_{\max}$ below $2/\eta$, mirroring the behavior seen on CIFAR-10 and supporting the claim that \textsc{EoSS} is a robust feature of mini-batch SGD on standard vision benchmarks.

\begin{figure}[ht!]
    \centering
    \begin{tabular}{cc}
      \includegraphics[width=0.45\textwidth]{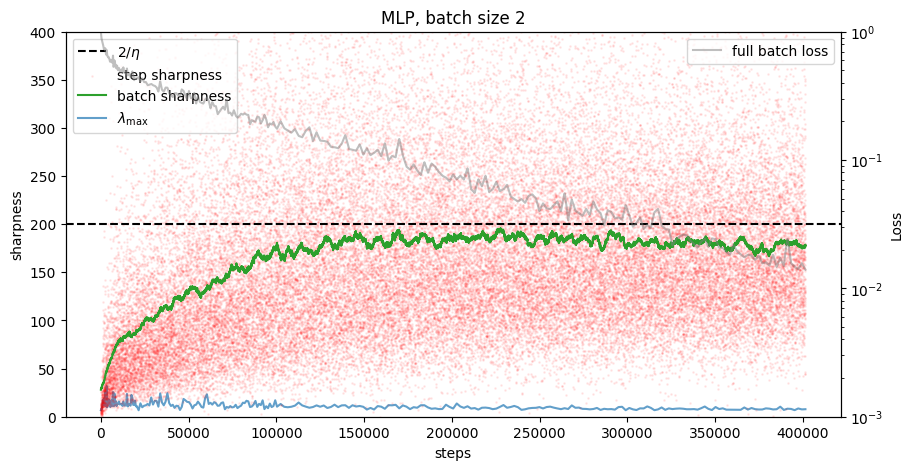} & \includegraphics[width=0.45\textwidth]{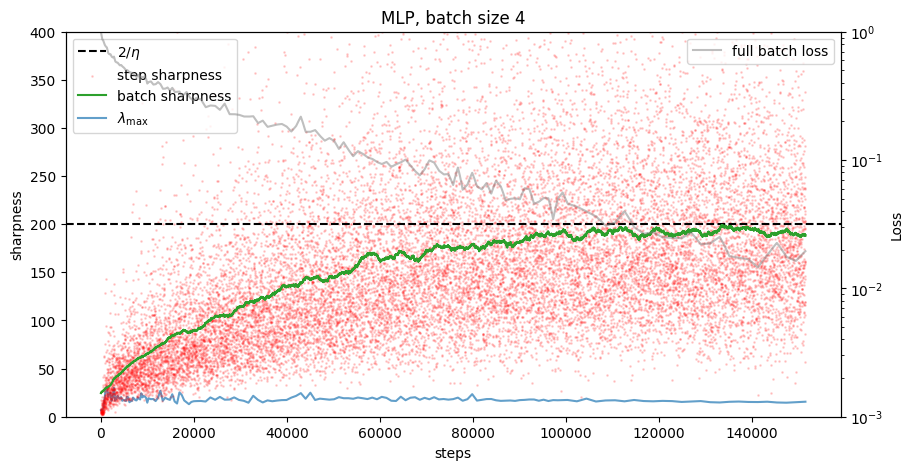} \\
      \includegraphics[width=0.45\textwidth]{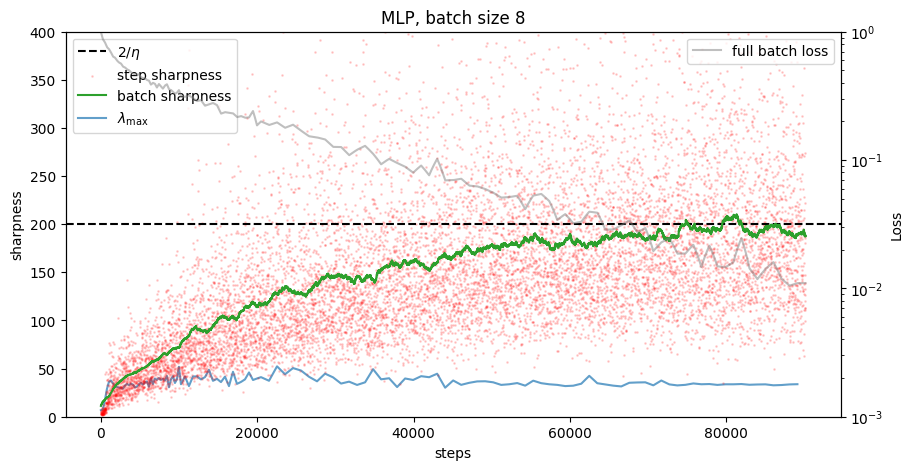} & \includegraphics[width=0.45\textwidth]{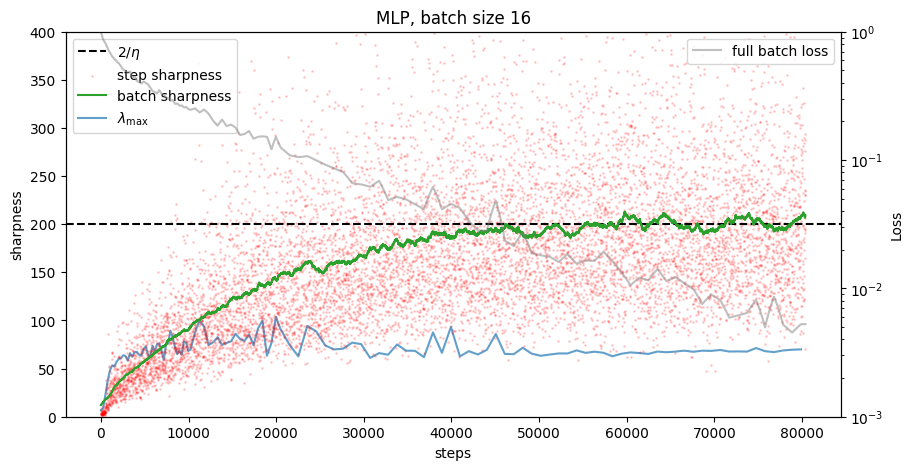} \\
      \includegraphics[width=0.45\textwidth]{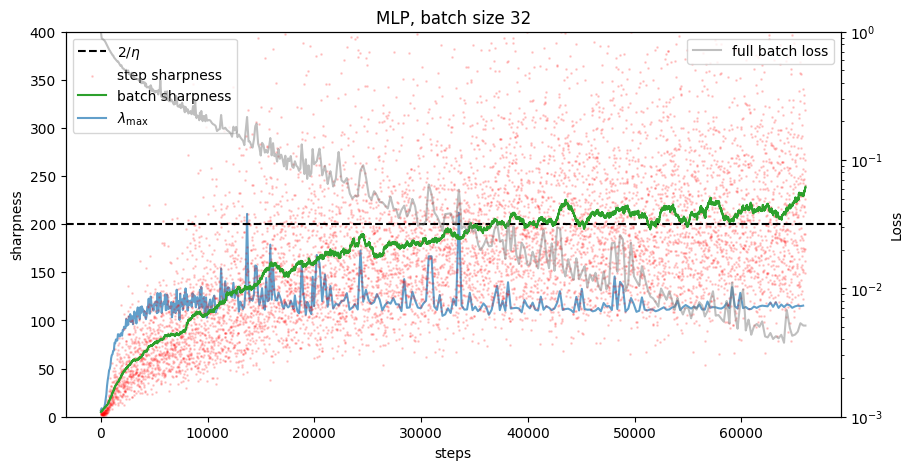} & \includegraphics[width=0.45\textwidth]{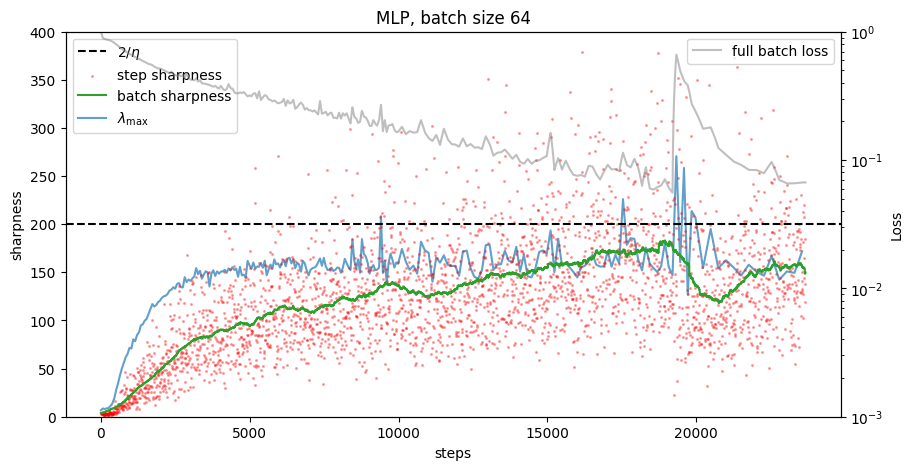} \\
      \includegraphics[width=0.45\textwidth]{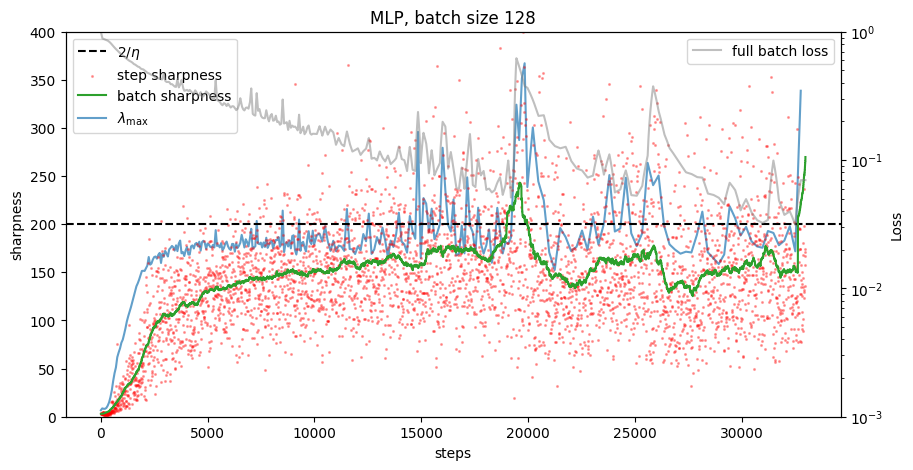} & \includegraphics[width=0.45\textwidth]{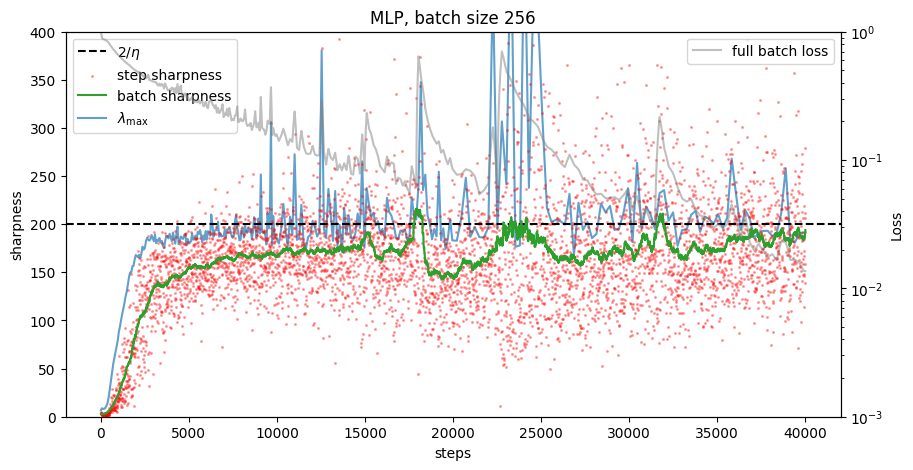} \\
    \end{tabular}
    \caption{\textbf{MLP}: 2 hidden layers, hidden dimension 512; \textbf{step size 0.01}, 8k subset of \textbf{SVHN}. Comparison between: step sharpness, aka \textit{Batch Sharpness} without expectation over batches and measured on the current batch (red dots, time-smoothing would be $\approx$ \textit{Batch Sharpness}), the empirical \textit{Batch Sharpness} (green line), the $\lambda_{\max}$ (blue line).}
    \label{fig:eoss_svhn_bs_mlp}
\end{figure}

\begin{figure}[ht!]
  \centering
  \begin{tabular}{cc}
    \includegraphics[width=0.45\textwidth]{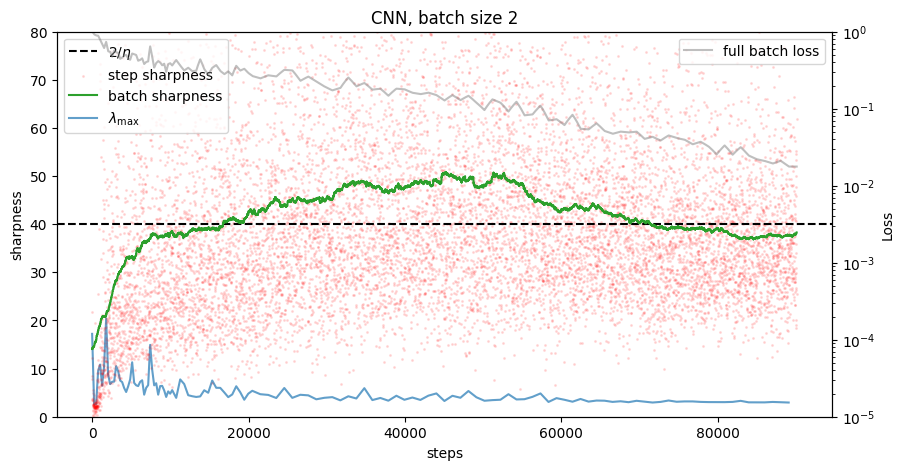} & \includegraphics[width=0.45\textwidth]{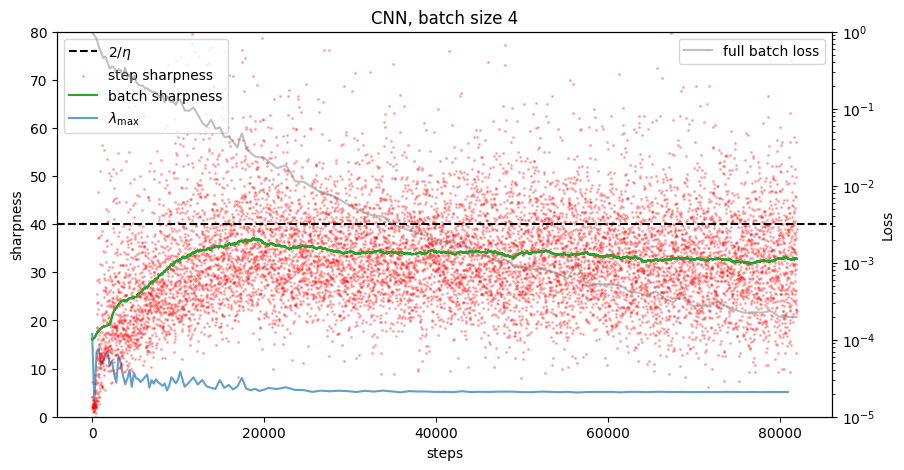} \\
    \includegraphics[width=0.45\textwidth]{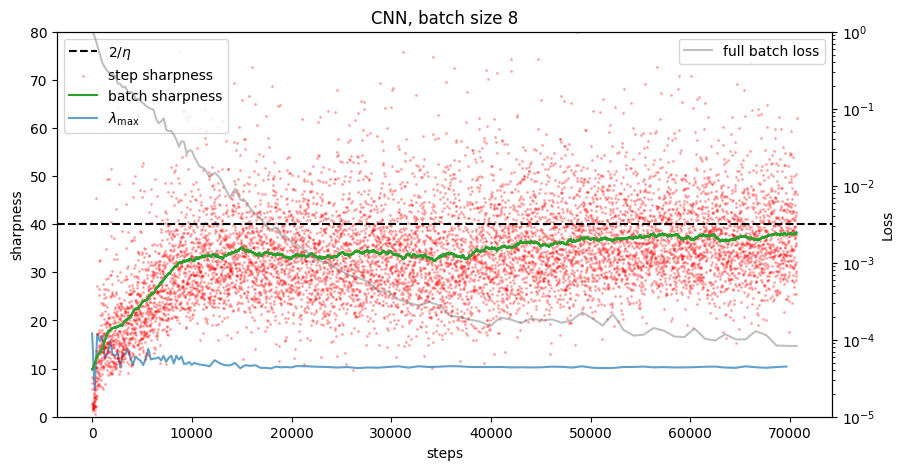} & \includegraphics[width=0.45\textwidth]{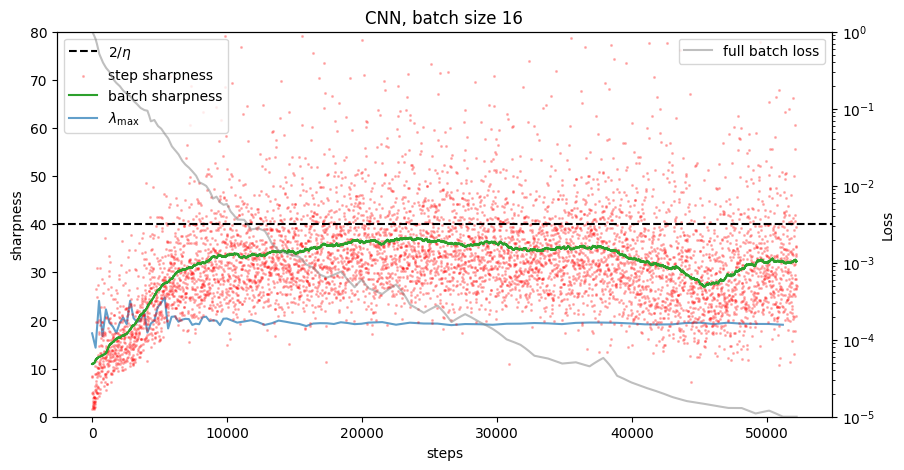} \\
    \includegraphics[width=0.45\textwidth]{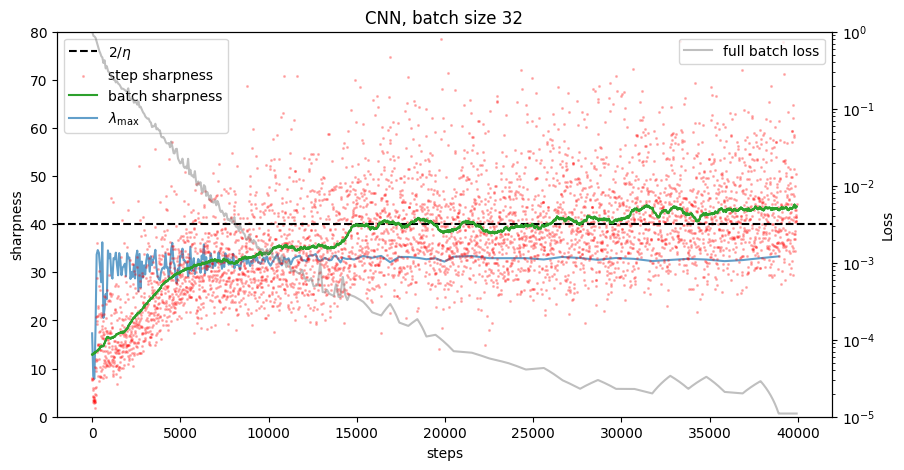} & \includegraphics[width=0.45\textwidth]{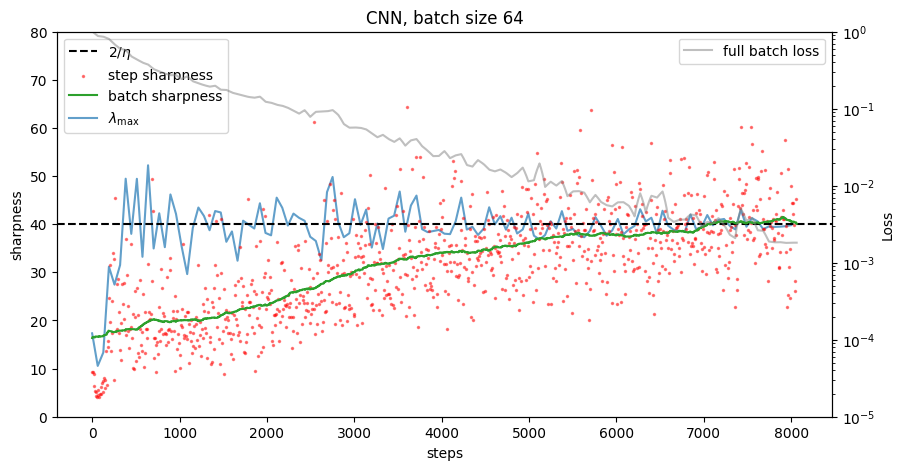} \\
    \includegraphics[width=0.45\textwidth]{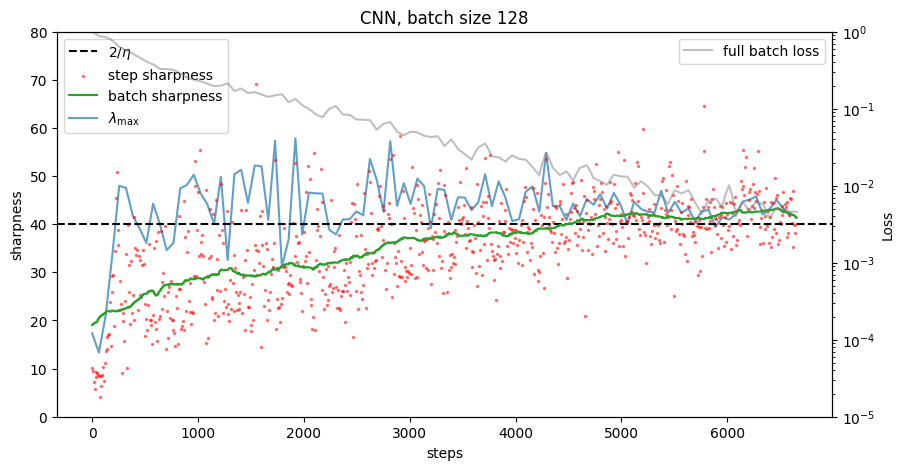} & \includegraphics[width=0.45\textwidth]{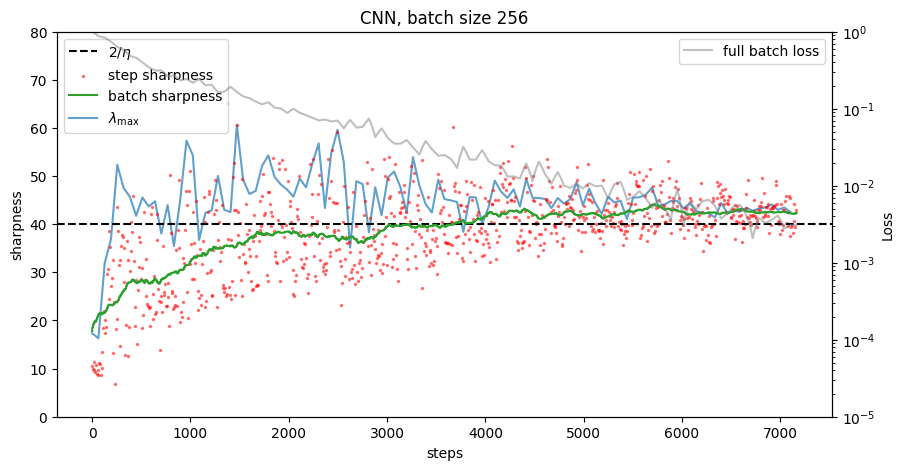} \\
  \end{tabular}
  \caption{\textbf{CNN}: 5 layers (3 convolutional, 2 fully-connected), \textbf{step size 0.05}, 8k subset of \textbf{SVHN}. Comparison between: step sharpness, aka \textit{Batch Sharpness} without expectation over batches and measured on the current batch (red dots, time-smoothing would be $\approx$ \textit{Batch Sharpness}), the empirical \textit{Batch Sharpness} (green line), the $\lambda_{\max}$ (blue line).}
  \label{fig:eoss_svhn_bs_cnn}
\end{figure}

\newpage

\begin{figure}[ht!]
  \centering
  \begin{tabular}{cc}
    \includegraphics[width=0.45\textwidth]{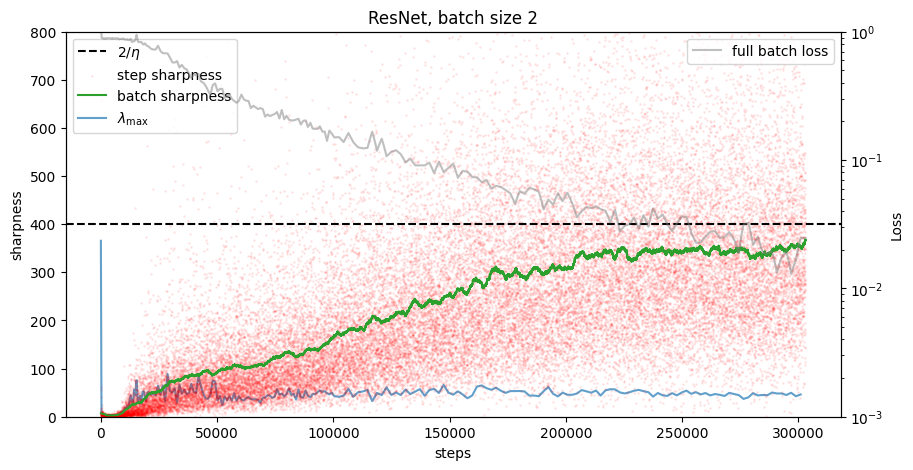} & \includegraphics[width=0.45\textwidth]{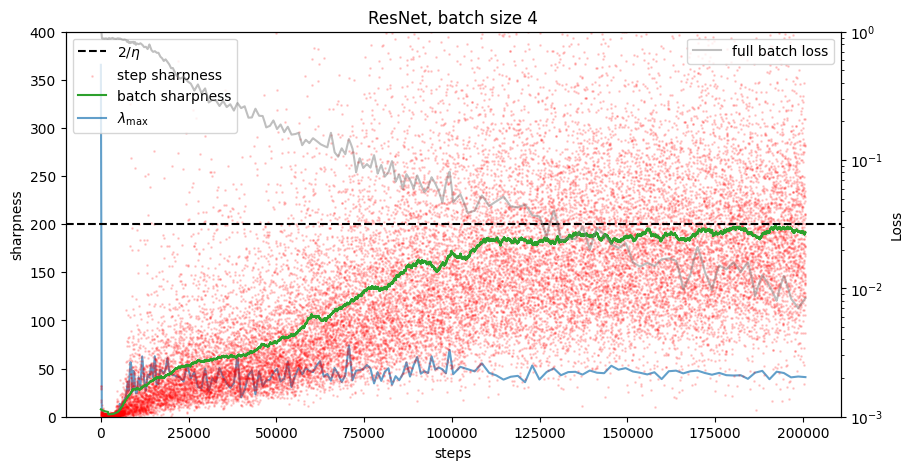} \\
    \includegraphics[width=0.45\textwidth]{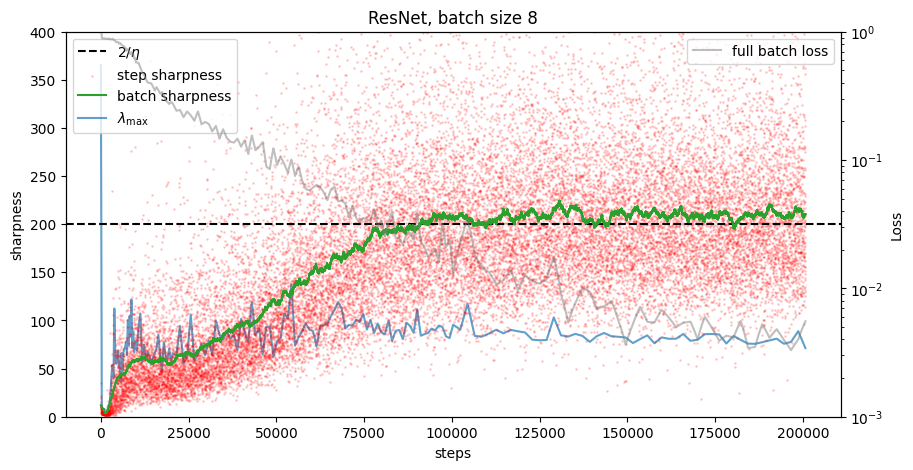} & \includegraphics[width=0.45\textwidth]{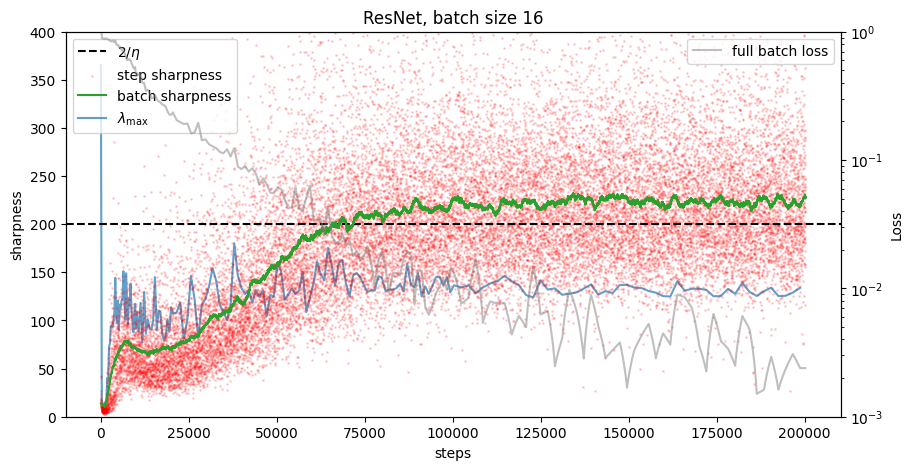} \\
    \includegraphics[width=0.45\textwidth]{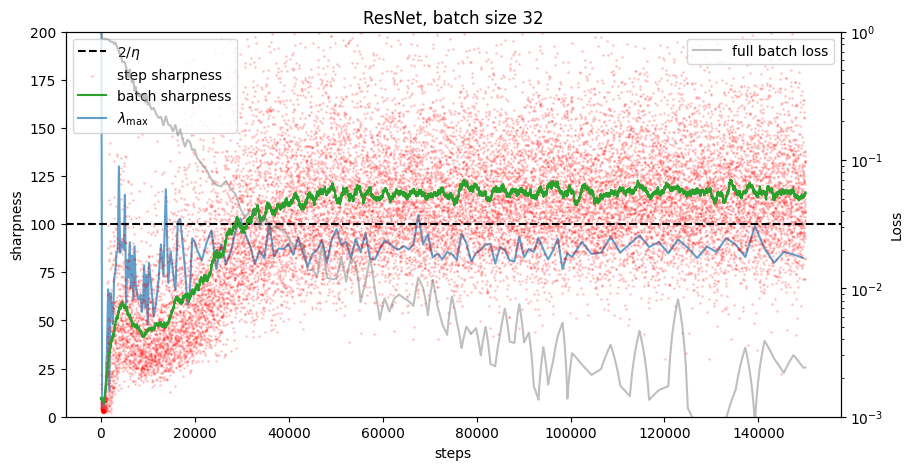} & \includegraphics[width=0.45\textwidth]{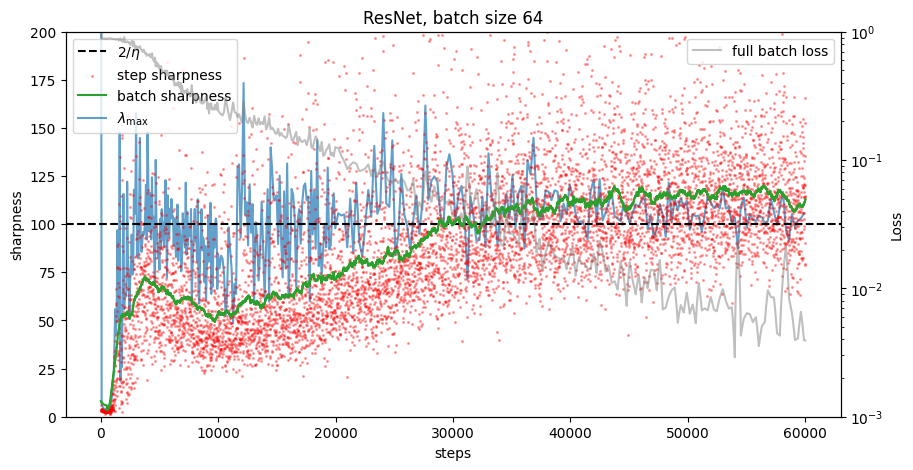} \\
    \includegraphics[width=0.45\textwidth]{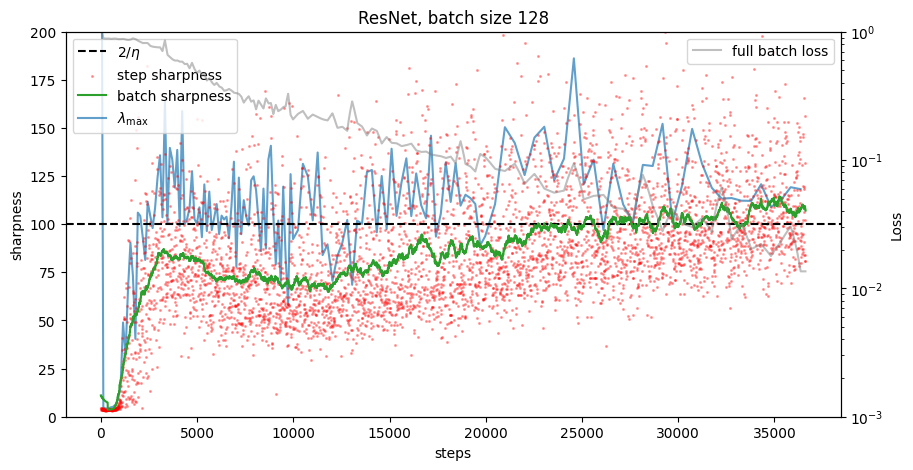} & \includegraphics[width=0.45\textwidth]{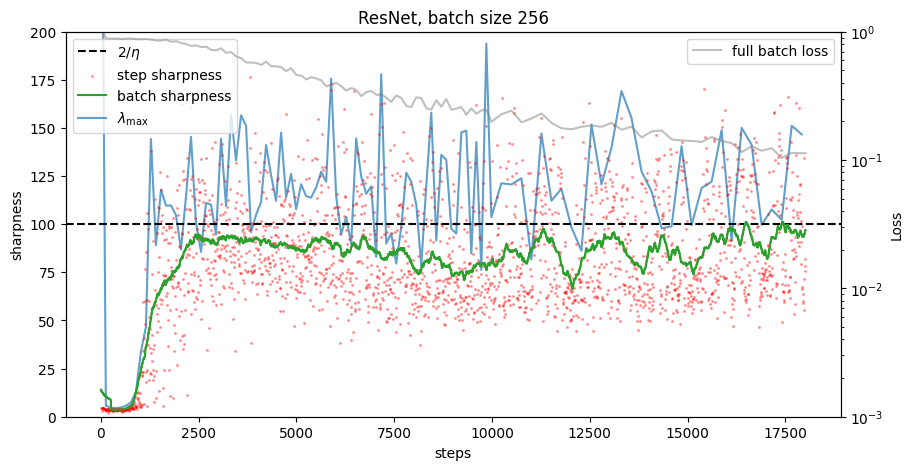} \\
  \end{tabular}
  \caption{\textbf{ResNet-10}, step size 0.005, 8k subset of \textbf{SVHN}. Comparison between: step sharpness, aka \textit{Batch Sharpness} without expectation over batches and measured on the current batch (red dots, time-smoothing would be $\approx$ \textit{Batch Sharpness}), the empirical \textit{Batch Sharpness} (green line), the $\lambda_{\max}$ (blue line).}
  \label{fig:eoss_bs_resnet}
\end{figure}


\clearpage

\end{document}